\documentclass{article}

\PassOptionsToPackage{numbers, compress}{natbib}

\usepackage[final]{neurips_2022}

\usepackage[utf8]{inputenc} 
\usepackage[T1]{fontenc}    
\usepackage{hyperref}       
\usepackage{url}            
\usepackage{booktabs}       
\usepackage{amsfonts}       
\usepackage{nicefrac}       
\usepackage{microtype}      
\usepackage{xcolor, color, colortbl}         

\usepackage[leqno]{amsmath}
\usepackage{amssymb}
\usepackage{mathtools}
\usepackage{amsthm}
\usepackage{bm}
\usepackage{pifont}
\usepackage{xcolor}
\usepackage{calc}
\usepackage{enumitem}
\usepackage{mathtools}

\usepackage{subcaption}
\usepackage{graphics}
\usepackage{graphicx}
\usepackage{caption}
\usepackage{wrapfig}
\usepackage{minitoc}
\usepackage{adjustbox}
\usepackage{algorithm}
\usepackage{algpseudocode}
\usepackage{multirow}
\usepackage{threeparttable}

\newtheorem{theorem}{Theorem}
\newtheorem{proposition}{Proposition}
\newtheorem{lemma}{Lemma}

\theoremstyle{definition}
\newtheorem{definition}{Definition}

\theoremstyle{remark}
\newtheorem{remark}{Remark}

\DeclareMathOperator*{\argmin}{arg\,min}

\newcommand\numberthis{\addtocounter{equation}{1}\tag{\theequation}}
\newcommand*\diff{\mathop{}\!\mathrm{d}}
\newsavebox{\leftbox}
\newsavebox{\rightbox}
\newcommand{\cmark}{\ding{51}}%
\newcommand{\xmark}{\ding{55}}%
\definecolor{Gray}{gray}{0.9}

\newcommand{\cc}[1]{\cellcolor{gray!#1}}
\hypersetup{colorlinks}
\allowdisplaybreaks
\makeatletter
\newcommand{\printfnsymbol}[1]{%
	\textsuperscript{\@fnsymbol{#1}}%
}
\newcommand{\leqnomode}{\tagsleft@true}
\newcommand{\reqnomode}{\tagsleft@false}

\title{Maximum Likelihood Training of\\Implicit Nonlinear Diffusion Models}

\author{%
	Dongjun Kim$^{*}$\\
	KAIST \\
	\texttt{dongjoun57@kaist.ac.kr} \\
	\And
	Byeonghu Na$^{*}$\\
	KAIST \\
	\texttt{wp03052@kaist.ac.kr} \\
	\And
	Se Jung Kwon \\
	NAVER CLOVA \\
	\And
	Dongsoo Lee \\
	NAVER CLOVA \\
	\And
	Wanmo Kang \\
	KAIST \\
	\And
	Il-Chul Moon \\
	KAIST / Summary.AI \\
}

\begin{document}

	\maketitle
	\doparttoc
	\parttoc
	\def\thefootnote{*}\footnotetext{Equal contribution}
	\def\thefootnote{\arabic{footnote}}
	
	\begin{abstract}
		Whereas diverse variations of diffusion models exist, extending the linear diffusion into a nonlinear diffusion process is investigated by very few works. The nonlinearity effect has been hardly understood, but intuitively, there would be promising diffusion patterns to efficiently train the generative distribution towards the data distribution. This paper introduces a data-adaptive nonlinear diffusion process for score-based diffusion models. The proposed Implicit Nonlinear Diffusion Model (INDM) learns by combining a normalizing flow and a diffusion process. Specifically, INDM implicitly constructs a nonlinear diffusion on the \textit{data space} by leveraging a linear diffusion on the \textit{latent space} through a flow network. This flow network is key to forming a nonlinear diffusion, as the nonlinearity depends on the flow network. This flexible nonlinearity improves the learning curve of INDM to nearly Maximum Likelihood Estimation (MLE) against the non-MLE curve of DDPM++, which turns out to be an inflexible version of INDM with the flow fixed as an identity mapping. Also, the discretization of INDM shows the sampling robustness. In experiments, INDM achieves the state-of-the-art FID of 1.75 on CelebA. We release our code at \url{https://github.com/byeonghu-na/INDM}.
	\end{abstract}
	
	\reqnomode
	
	\section{Introduction}
	\label{sec:Introduction}
	
	Diffusion models have recently achieved success on the task of sample generation, and some works \cite{song2020score, dhariwal2021diffusion} claim state-of-the-art performance over Generative Adversarial Networks (GAN) \citep{karras2019style}. This success is highlighted particularly in likelihood-based models, including normalizing flows \citep{grcic2021densely}, autoregressive models \citep{parmar2018image}, and Variational Auto-Encoders (VAE) \citep{vahdat2020nvae}. Moreover, this success is noteworthy because it is achieved merely using linear diffusion processes, such as Variance Exploding (VE) Stochastic Differential Equation (SDE) \citep{song2020improved}, and Variance Preserving (VP) SDE \citep{ho2020denoising}.
	
	This paper extends linear diffusions of VE/VP SDEs to a data-adaptive trainable nonlinear diffusion. To motivate the extension, though there are structural similarities between diffusion models and VAEs, the inference part of a linear diffusion process has not been trained while its counterpart of VAE (the encoder) has been trained. We introduce Implicit Nonlinear Diffusion Models (INDM) to train its \textit{forward} SDE, the inference part in diffusion models. INDM constructs the nonlinearity of the data diffusion by transforming a linear \textit{latent} diffusion back to the data space.
	
	We implement the transformation between the data and latent spaces with a normalizing flow. The invertibility of the flow mapping is key to learning a nonlinear inference part. Invertibility is necessary for constructing the nonlinearity, and we clarify this by comparing INDM with LSGM \cite{vahdat2021score}, a latent diffusion model with VAE. Altogether, INDM provides the following advantages over the existing models.
	\vspace{-2mm}
	\begin{itemize}\setlength\itemsep{0.2em}
	\item INDM achieves fast and tractable optimization with \textit{implicit} modeling.
	\item INDM learns not only drift but \textit{volatility} coefficients of the forward SDE.
	\item INDM trains its network with \textit{Maximum Likelihood Estimation} (MLE).
	\item INDM is \textit{robust} on the sampling discretization.
	\end{itemize}
	\vspace{-2mm}
	
	\section{Preliminary}\label{preliminary}
	
	A diffusion model is constructed with bidirectional \textit{forward} and \textit{reverse} stochastic processes. 
	
	\textbf{Forward and Reverse Diffusions} A forward diffusion process diffuses an input data variable, $\mathbf{x}_{0}\sim p_{r}$, to a noise variable, and the corresponding reverse diffusion process \cite{anderson1982reverse} of this forward diffusion denoises a noise variable to regenerate the input variable. The forward diffusion is fully described by an SDE of $\diff\mathbf{x}_{t}=\mathbf{f}(\mathbf{x}_{t},t)\diff t+\mathbf{G}(\mathbf{x}_{t},t)\diff\mathbf{w}_{t}$, and the corresponding reverse SDE becomes $\diff\mathbf{x}_{t}=\big[\mathbf{f}(\mathbf{x}_{t},t)-\text{div}(\mathbf{G}\mathbf{G}^{T})(\mathbf{x}_{t},t)-(\mathbf{G}\mathbf{G}^{T})(\mathbf{x}_{t},t)\nabla_{\mathbf{x}_{t}}\log{p_{t}(\mathbf{x}_{t})}\big]\diff \bar{t}+\mathbf{G}(\mathbf{x}_{t},t)\diff\mathbf{\bar{w}}_{t}$. Here, $\mathbf{w}_{t}\in\mathbb{R}^{d}$ is an abstraction of a random walk process with independent increments, where $d$ is the data dimension, and $\diff\mathbf{\bar{w}}_{t}$ is the standard Wiener processes with backwards in time. 
	
	\textbf{Generative Diffusion} Having that the drift ($\mathbf{f}\in\mathbb{R}^{d}$) and the volatility ($\mathbf{G}\in\mathbb{R}^{d\times d}$) terms are given a-priori, diffusion models \citep{song2020score} estimate the data score, $\nabla_{\mathbf{x}_{t}}\log{p_{t}(\mathbf{x}_{t})}$, with the score network, $\mathbf{s}_{\bm{\theta}}(\mathbf{x}_{t},t)$. By plugging the score network in the data score, we obtain another diffusion process, called the \textit{generative} SDE, described by $\diff\mathbf{x}_{t}^{\bm{\theta}}=\big[\mathbf{f}(\mathbf{x}_{t}^{\bm{\theta}},t)-\text{div}(\mathbf{G}\mathbf{G}^{T})(\mathbf{x}_{t}^{\bm{\theta}},t)-(\mathbf{G}\mathbf{G}^{T})(\mathbf{x}_{t}^{\bm{\theta}},t)\mathbf{s}_{\bm{\theta}}(\mathbf{x}_{t}^{\bm{\theta}},t)\big]\diff \bar{t} + \mathbf{G}(\mathbf{x}_{t}^{\bm{\theta}},t)\diff\mathbf{\bar{w}}_{t}$. Starting from a prior distribution of $\mathbf{x}_{T}^{\bm{\theta}}\sim\pi$ and solving the SDE time backwards, \citet{song2020score} construct the generative stochastic process of $\{\mathbf{x}_{t}^{\bm{\theta}}\}_{t=0}^{T}$ that perfectly reconstructs the reverse process of $\{\mathbf{x}_{t}\}_{t=0}^{T}$ under two conditions: 1) $\mathbf{s}_{\bm{\theta}}(\mathbf{x}_{t},t)=\nabla_{\mathbf{x}_{t}}\log{p_{t}(\mathbf{x}_{t})}$ and 2) $\mathbf{x}_{T}\sim\pi$. We define a generative distribution, $p_{\bm{\theta}}$, as the distribution of $\mathbf{x}_{0}^{\bm{\theta}}$.
	
	\begin{figure*}[t]
		\centering
		\includegraphics[width=\linewidth]{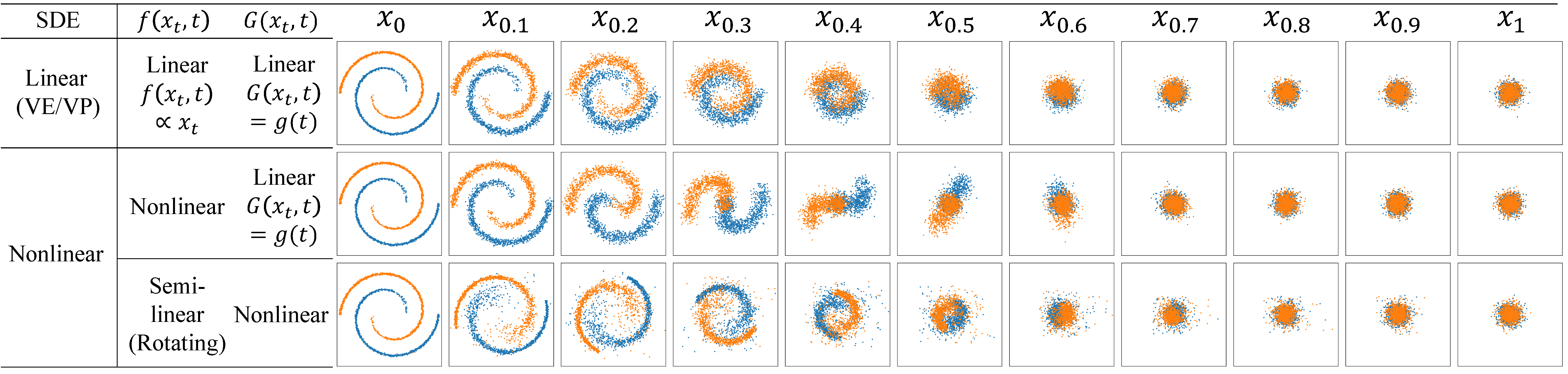}
		\caption{Examples of linear (top row) and nonlinear (middle/bottom rows) diffusion processes.}
		\label{fig:thumbnail_icml}
	\end{figure*}
	
	\textbf{Score Estimation} The diffusion model estimates the data score with the score network by minimizing the denoising score loss \citep{song2020score}, given by $\mathcal{L}(\{\mathbf{x}_{t}\}_{t=0}^{T},\lambda;\bm{\theta})=\int_{0}^{T}\lambda(t)\mathbb{E}_{\mathbf{x}_{0}, \mathbf{x}_{t}}[\Vert\mathbf{s}_{\bm{\theta}}(\mathbf{x}_{t},t)-\nabla_{\mathbf{x}_{t}}\log{p_{0t}(\mathbf{x}_{t}\vert\mathbf{x}_{0})}\Vert_{2}^{2}]\diff t$, where $p_{0t}(\mathbf{x}_{t}\vert\mathbf{x}_{0})$ is a transition probability of $\mathbf{x}_{t}$ given $\mathbf{x}_{0}$; and $\lambda$ is the weighting function that determines the level of contribution for each diffusion time. When $\mathbf{G}(\mathbf{x}_{t},t)=g(t)$, \citet{song2021maximum, huang2021variational} proved that this loss with the likelihood weighting ($\lambda=g^{2}$) turns out to be a variational bound of the negative log-likelihood: $\mathbb{E}_{\mathbf{x}_{0}}[-\log{p_{\bm{\theta}}(\mathbf{x}_{0})}]\le \mathcal{L}(\{\mathbf{x}_{t}\}_{t=0}^{T},g^{2};\bm{\theta})-\mathbb{E}_{\mathbf{x}_{T}}[\log{\pi(\mathbf{x}_{T})}]$, up to a constant, see Appendix \ref{appendix:derivation_of_variational_bound} for a detailed discussion.
	
	\textbf{Choice of Drift ($\mathbf{f}$) and Volatility ($\mathbf{G}$) Terms} The original diffusion model strictly limits the scope of diffusion process to be a family of linear diffusions that $\mathbf{f}$ is a linear function of $\mathbf{x}_{t}$ and $\mathbf{G}$ is an identity matrix multiplied by a $t$-function. For instance, VESDE \citep{song2020score, song2020improved} satisfies $\mathbf{f}\equiv 0$ with $\mathbf{G}=\sqrt{\diff\sigma^{2}(t)/\diff t}\mathbf{I}$ and VPSDE \citep{song2020score, ho2020denoising} satisfies $\mathbf{f}=-\frac{1}{2}\beta(t)\mathbf{x}_{t}\propto\mathbf{x}_{t}$ with $\mathbf{G}=\sqrt{\beta(t)}\mathbf{I}$. Few concurrent works have extended linear diffusions to nonlinear diffusions by 1) applyng a latent diffusion using VAE in LSGM \cite{vahdat2021score}, 2) applying a flow network to nonlinearize the drift term in DiffFlow \cite{zhang2021diffusion}, and 3) reformulating the diffusion model into a Schrodinger Bridge Problem (SBP) \cite{vargas2021solving,de2021diffusion,chen2021likelihood}. We further analyze these approaches in Section \ref{sec:related_work}. 
	
	\section{Motivation of Nonlinear Diffusion Process}\label{sec:motivation}
	
	Figure \ref{fig:thumbnail_icml} illustrates various diffusion processes on a spiral toy dataset. In the top row, the diffusion path of VPSDE keeps its overall structure of the initial data manifold during the data deformation procedure to $\mathcal{N}(0,\mathbf{I})$. The drift vector field illustrated in Figure \ref{fig:spiral}-(a) as black arrows presents that VPSDE \textit{linearly} deforms its data distribution.

	\begin{wrapfigure}{r}{0.5\textwidth}
		\vskip -0.18in
		\centering
		\begin{subfigure}{0.32\linewidth}
			\includegraphics[width=\linewidth]{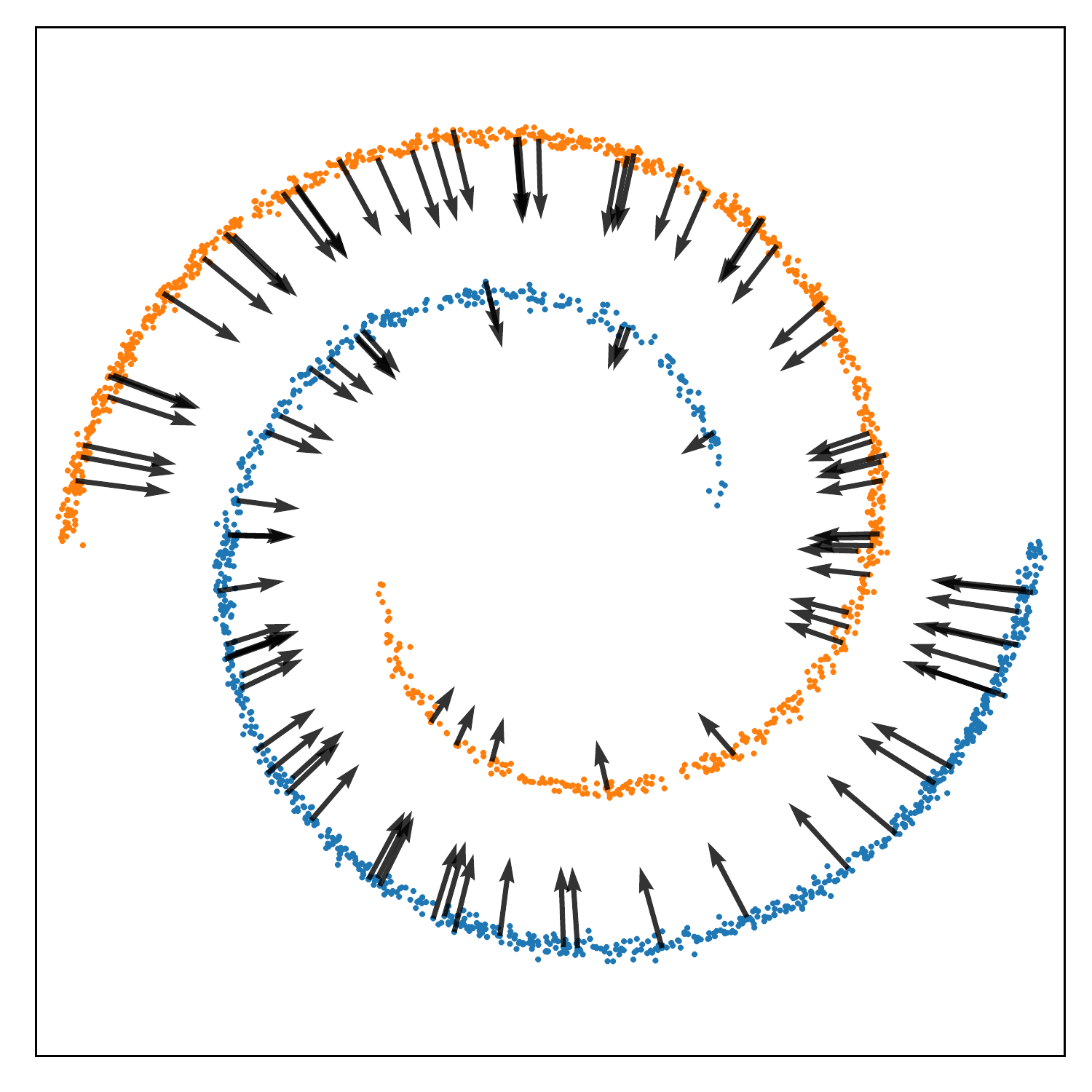}
			\subcaption{Linear $\mathbf{f}\propto\mathbf{x}_{t}$}
		\end{subfigure}
		\hfill
		\begin{subfigure}{0.32\linewidth}
			\includegraphics[width=\linewidth]{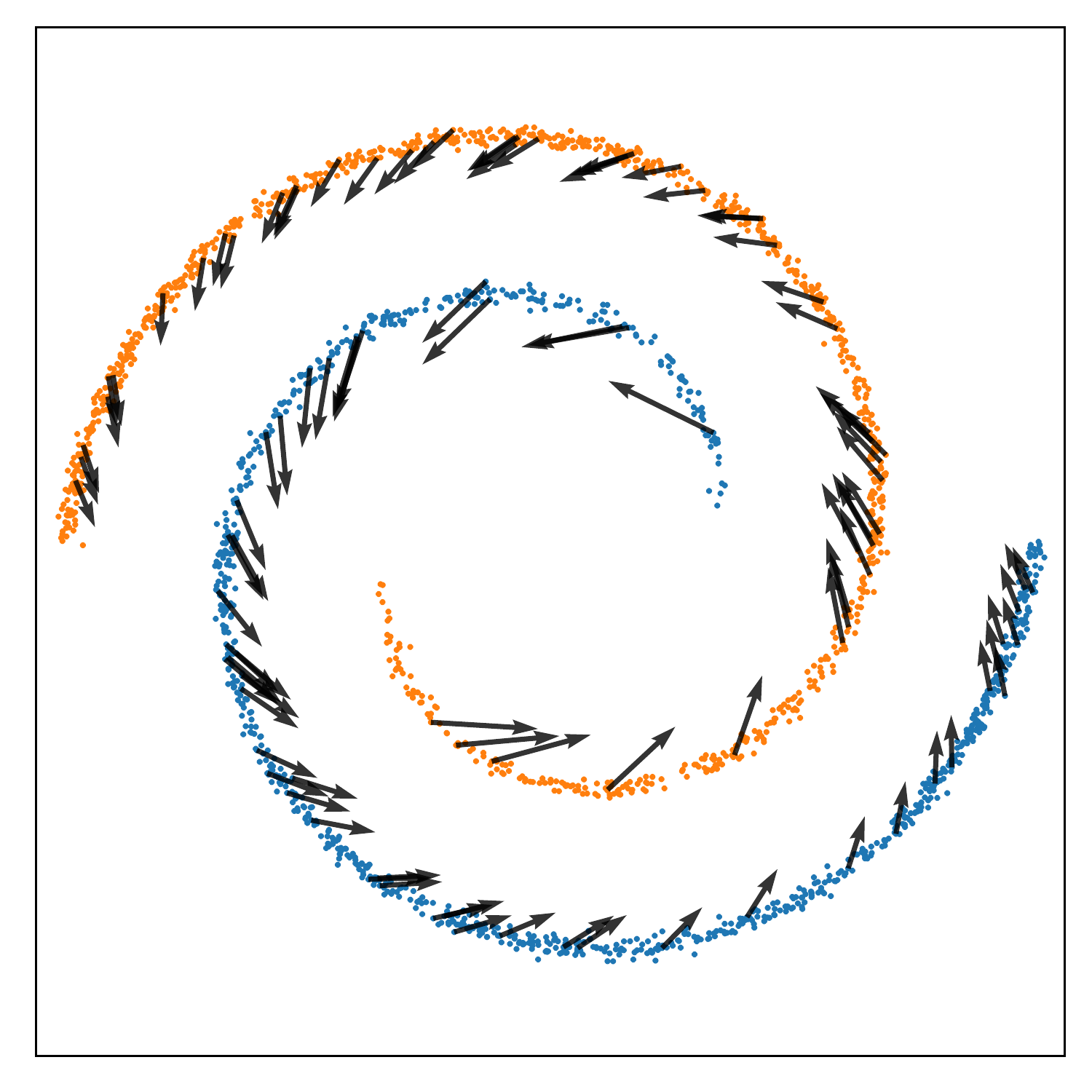}
			\subcaption{Nonlinear $\mathbf{f}$}
		\end{subfigure}
		\hfill
		\begin{subfigure}{0.32\linewidth}
			\includegraphics[width=\linewidth]{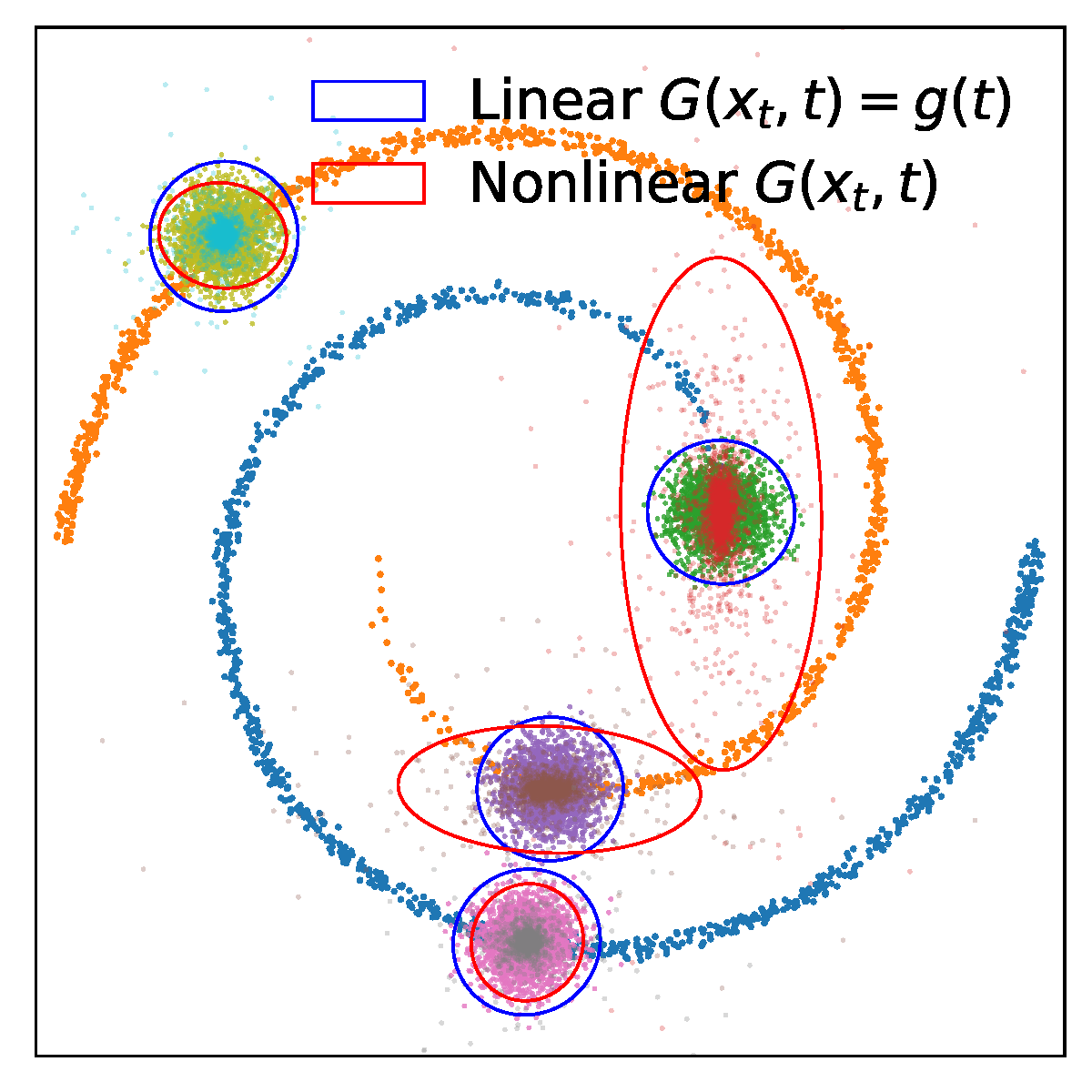}
			\subcaption{Nonlinear $\mathbf{G}$}
		\end{subfigure}
		\vskip -0.05in
		\caption{Vector fields on various SDEs at $t=0$.}
		\label{fig:spiral}
		\vskip -0.15in
	\end{wrapfigure}
	Unlike the linear diffusion, the middle row of Figure \ref{fig:thumbnail_icml} with a nonlinear drift shows that the data is not linearly deformed to $\mathcal{N}(0,\mathbf{I})$. Figure \ref{fig:spiral}-(b) illustrates the corresponding vector field, in which two distinctive components (orange/blue) are forced to separate each other. The nonlinearity of the drift term represented as rotating black arrows is the source of such nonlinear deformation at the intermediate steps, $\mathbf{x}_{0.2}\sim\mathbf{x}_{0.6}$. When it comes to the volatility term, the last row of Figure \ref{fig:thumbnail_icml} presents the process with nonlinear $\mathbf{G}$. Figure \ref{fig:spiral}-(c) illustrates the covariance matrices of the perturbation distribution at $t=0$ with linear and nonlinear volatility terms, where the perturbation distribution induced by the volatility term is $\mathcal{N}(0,\mathbf{G}(\mathbf{x}_{t},t)\mathbf{G}^{T}(\mathbf{x}_{t},t))$\footnote{The covariance is $\frac{\diff}{\diff t}\mathbb{E}_{\mathbf{x}_{t+\diff t}\vert\mathbf{x}_{t}}[(\mathbf{x}_{t+\diff t}-\mathbf{x}_{t}-\mathbf{f}\diff t)(\mathbf{x}_{t+\diff t}-\mathbf{x}_{t}-\mathbf{f}\diff t)^{T}]=\mathbf{G}(\mathbf{x}_{t},t)\mathbf{G}^{T}(\mathbf{x}_{t},t)$.}. It shows the non-diagonal and data-dependent covariances of $\mathbf{G}\mathbf{G}^{T}$ in red ellipses of a nonlinear volatility term, and the isotropic blue circles of linear diffusions.
	
	\begin{figure*}[t]
		\vskip -0.05in
		\centering
		\includegraphics[width=\linewidth]{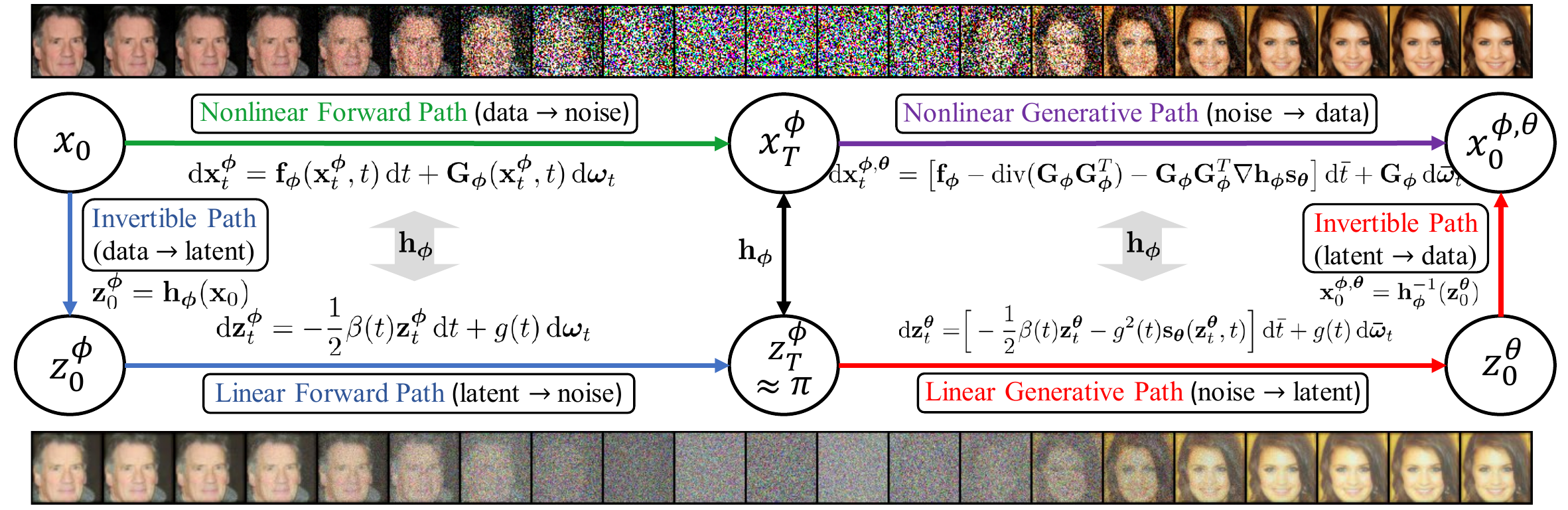}
		\caption{INDM attains a ladder structure between the data space and the latent space. The latent vector is visualized by normalizing the latent value, see Appendix \ref{appendix:visualization} for further visualization.}
		\label{fig:pdm_gm}
		\vskip -0.1in
	\end{figure*}
	
	\leqnomode
	
	\section{Implicit Nonlinear Diffusion Model}\label{sec:methodology}
	
	There are two ways to nonlinearize the drift and volatility coefficients in SDE: explicit and implicit parametrizations. While explicit is a straightforward way to model the nonlinearity, it becomes impractical particularly in the training procedure. Concretely, in each of the training iteration, the denoising loss $\mathcal{L}(\{\mathbf{x}_{t}\}_{t=0}^{T},\lambda;\bm{\theta})$ requires 1) the perturbed samples $\mathbf{x}_{t}$ from $p_{0t}(\mathbf{x}_{t}\vert\mathbf{x}_{0})$ and 2) the calculation of $\nabla\log{p_{0t}(\mathbf{x}_{t}\vert\mathbf{x}_{0})}$, and these two steps require long execution time because the transition probability $p_{0t}(\mathbf{x}_{t}\vert\mathbf{x}_{0})$ is intractable for nonlinear diffusions in general. Therefore, we parametrize $\mathbf{f}_{\bm{\phi}}$ and $\mathbf{G}_{\bm{\phi}}$ \textit{implicitly} for fast and tractable optimization. As visualized in Figure \ref{fig:pdm_gm}, we impose a linear diffusion model on the latent space, and connect this latent variable with the data variable through a normalizing flow. The nonlinear diffusion on the data space, then, is induced from the latent diffusion leveraged to the data space.
	
	\subsection{Data and Latent Diffusion Processes}
	
	\textbf{Latent Diffusion Processes} Let us define $\mathbf{z}_{0}^{\bm{\phi}}$ to be a transformed latent variable $\mathbf{z}_{0}^{\bm{\phi}}=\mathbf{h}_{\bm{\phi}}(\mathbf{x}_{0})$, where $\mathbf{h}_{\bm{\phi}}$ is a transformation of the normalizing flow. Then, a forward linear diffusion
	\begin{align*}
	\diff\mathbf{z}_{t}^{\bm{\phi}}=-\frac{1}{2}\beta(t)\mathbf{z}_{t}^{\bm{\phi}}\diff t+g(t)\diff \mathbf{w}_{t},\tag{\text{Latent Forward SDE}}
	\end{align*}
	starting at $\mathbf{z}_{0}^{\bm{\phi}}=\mathbf{h}_{\bm{\phi}}(\mathbf{x}_{0})$ with $\mathbf{x}_{0}\sim p_{r}$, describes the forward diffusion process on the latent space (blue diffusion path in Figure \ref{fig:pdm_gm}). The corresponding reverse latent diffusion is given by $\diff\mathbf{z}_{t}^{\bm{\phi}}=[-\frac{1}{2}\beta(t)\mathbf{z}_{t}^{\bm{\phi}}-g^{2}(t)\nabla_{\mathbf{z}_{t}^{\bm{\phi}}}\log{p_{t}^{\bm{\phi}}(\mathbf{z}_{t}^{\bm{\phi}})}]\diff\bar{t}+g(t)\diff\mathbf{\bar{w}}_{t}$, where $p_{t}^{\bm{\phi}}$ is the probability of $\mathbf{z}_{t}^{\bm{\phi}}$.
	
	\textbf{Forward Data Diffusion} We have not defined the data diffusion process yet. We build the data diffusion from the latent diffusion and the normalizing flow. From the invertibility, we define random variables on the data space by transforming the latent linear diffusion back to the data space: $\mathbf{x}_{t}^{\bm{\phi}}:=\mathbf{h}_{\bm{\phi}}^{-1}(\mathbf{z}_{t}^{\bm{\phi}})$ for any $t\in [0,T]$. Then, from the Ito's formula \citep{oksendal2013stochastic}, the process $\{\mathbf{x}_{t}^{\bm{\phi}}\}_{t=0}^{T}$ follows
	\begin{align*}
	\diff\mathbf{x}_{t}^{\bm{\phi}}=\mathbf{f}_{\bm{\phi}}(\mathbf{x}_{t}^{\bm{\phi}},t)\diff t+\mathbf{G}_{\bm{\phi}}(\mathbf{x}_{t}^{\bm{\phi}},t)\diff\mathbf{w}_{t},\tag{\text{Data Forward SDE}}
	\end{align*}
starting with $\mathbf{x}_{0}^{\bm{\phi}}=\mathbf{h}_{\bm{\phi}}^{-1}(\mathbf{z}_{0}^{\bm{\phi}})$. From $\mathbf{x}_{0}^{\bm{\phi}}=\mathbf{h}_{\bm{\phi}}^{-1}(\mathbf{h}_{\bm{\phi}}(\mathbf{x}_{0}))=\mathbf{x}_{0}\sim p_{r}$, we call this process by \textit{induced diffusion} that permeates the data variable on the data space. We emphasize that this induced diffusion collapses to a linear diffusion if $\mathbf{h}_{\bm{\phi}_{id}}=id$. See Appendix \ref{appendix:derivation_of_diffusion} for details on drift and volatility terms.
	
	\textbf{Generative Data Diffusion} A diffusion model estimates the forward latent score $\mathbf{s}_{\bm{\phi}}(\mathbf{z},t)=\nabla\log{p_{t}^{\bm{\phi}}(\mathbf{z})}$ with the score network, $\mathbf{s}_{\bm{\theta}}(\mathbf{z},t)$, to mimic the forward linear diffusion on the latent space. Then, the generative SDE on the latent space becomes 
	\begin{align*}
	\diff\mathbf{z}_{t}^{\bm{\theta}}=\bigg[-\frac{1}{2}\beta(t)\mathbf{z}_{t}^{\bm{\theta}}-g^{2}(t)\mathbf{s}_{\bm{\theta}}(\mathbf{z}_{t}^{\bm{\theta}},t)\bigg]\diff \bar{t}+g(t)\diff\bar{\mathbf{w}}_{t} \tag{\text{Latent Gen. SDE}}
	\end{align*}
	with a starting variable $\mathbf{z}_{T}^{\bm{\theta}}\sim\pi$. Thus, the process $\{\mathbf{x}_{t}^{\bm{\phi},\bm{\theta}}\}_{t=0}^{T}$ of $\mathbf{x}_{t}^{\bm{\phi},\bm{\theta}}:=\mathbf{h}_{\bm{\phi}}^{-1}(\mathbf{z}_{t}^{\bm{\theta}})$ becomes a generative data diffusion (purple path in Figure \ref{fig:pdm_gm}) with SDE of
	\begin{align*}
	\diff\mathbf{x}_{t}^{\bm{\phi},\bm{\theta}}=\big[\mathbf{f}_{\bm{\phi}}-\text{div}(\mathbf{G}_{\bm{\phi}}\mathbf{G}_{\bm{\phi}}^{T})-(\mathbf{G}_{\bm{\phi}}\mathbf{G}_{\bm{\phi}}^{T}\nabla\mathbf{h}_{\bm{\phi}})\mathbf{s}_{\bm{\theta}}\big(\mathbf{h}_{\bm{\phi}}(\mathbf{x}_{t}^{\bm{\phi},\bm{\theta}}),t\big)\big]\diff \bar{t} + \mathbf{G}_{\bm{\phi}}\diff\mathbf{\bar{w}}_{t}.\tag{\text{Data Gen. SDE}}
	\end{align*}

	\subsection{Model Training and Sampling}
	
	\textbf{Likelihood Training} Theorem \ref{thm:1} estimates Negative Evidence Lower Bound (NELBO) of Negaitve Log-Likelihood (NLL). For the notational simplicity, we define the targetted score function by
	\begin{align*}
	\mathbf{s}_{\bm{\phi}}(\mathbf{z}_{t}^{\bm{\phi}},t):=\nabla\log{p_{t}^{\bm{\phi}}(\mathbf{z}_{t}^{\bm{\phi}})}.\tag{\text{Target of Score Estimation}}
	\end{align*}
	Also, suppose $\mathcal{L}\big(\{\mathbf{z}_{t}^{\bm{\phi}}\}_{t=0}^{T},g^{2};\bm{\theta}\big)=\frac{1}{2}\int_{0}^{T}g^{2}(t)\mathbb{E}_{\mathbf{z}_{0}^{\bm{\phi}},\mathbf{z}_{t}^{\bm{\phi}}}\big[\Vert\mathbf{s}_{\bm{\theta}}(\mathbf{z}_{t}^{\bm{\phi}},t)-\nabla\log{p_{0t}(\mathbf{z}_{t}^{\bm{\phi}}\vert\mathbf{z}_{0}^{\bm{\phi}})}\Vert_{2}^{2}\big]\diff t$, where $p_{0t}(\mathbf{z}_{t}^{\bm{\phi}}\vert\mathbf{z}_{0}^{\bm{\phi}})$ is the transition probability of the latent forward diffusion. In Theorem \ref{thm:1}, we drop the constant terms that do not hurt the essence of the theorem to keep the simplicity. See full details and the proof in Appendix \ref{appendix:proofs}.
	\reqnomode
	\begin{theorem}\label{thm:1}
		Suppose that $p_{\bm{\phi},\bm{\theta}}$ is the likelihood of a generative random variable $\mathbf{x}_{0}^{\bm{\phi},\bm{\theta}}$. Then, the negative log-likelihood is upper bounded by
		\begin{align*}
		\mathbb{E}_{\mathbf{x}_{0}}\big[-\log{p_{\bm{\phi},\bm{\theta}}(\mathbf{x}_{0})}\big]\le\mathcal{L}\big(\{\mathbf{x}_{t}\}_{t=0}^{T},g^{2};\{\bm{\phi},\bm{\theta}\}\big),
		\end{align*}
		where
		\begin{align}
		\mathcal{L}\big(\{\mathbf{x}_{t}\}_{t=0}^{T}&,g^{2};\{\bm{\phi},\bm{\theta}\}\big)=\frac{1}{2}\int_{0}^{T}g^{2}(t)\mathbb{E}_{\mathbf{z}_{t}^{\bm{\phi}}}\big[\Vert\mathbf{s}_{\bm{\theta}}(\mathbf{z}_{t}^{\bm{\phi}},t)-\mathbf{s}_{\bm{\phi}}(\mathbf{z}_{t}^{\bm{\phi}},t)\Vert_{2}^{2}\big]\diff t+D_{KL}(p_{T}^{\bm{\phi}}\Vert \pi)\label{main_eq:nelbo}\\[1ex]
		&=-\mathbb{E}_{\mathbf{x}_{0}}\big[\log{\big\vert\det\big(\nabla\mathbf{h}_{\bm{\phi}}(\mathbf{x}_{0})\big)\big\vert}\big]+\mathcal{L}\big(\{\mathbf{z}_{t}^{\bm{\phi}}\}_{t=0}^{T},g^{2};\bm{\theta}\big)-\mathbb{E}_{\mathbf{z}_{T}^{\bm{\phi}}}\big[\log{\pi(\mathbf{z}_{T}^{\bm{\phi}})}\big].\label{main_eq:training_loss}
		\end{align}
	\end{theorem}
	Eq. \eqref{main_eq:nelbo} is the KL divergence $D_{KL}(\bm{\mu}_{\bm{\phi}}\Vert\bm{\nu}_{\bm{\phi},\bm{\theta}})$, where $\bm{\mu}_{\bm{\phi}}$ and $\bm{\nu}_{\bm{\phi},\bm{\theta}}$ are the path measures of the forward and generative diffusions on the data space. Eq. \eqref{main_eq:nelbo} explains the reasoning of why $\mathbf{s}_{\bm{\phi}}$ is the target of the score estimation. However, the KL divergence is intractable, and Theorem \ref{thm:1} provides an equivalent tractable loss by Eq. \eqref{main_eq:training_loss}, the summation of the flow loss with the denoising loss.
	
	Algorithm \ref{alg:INDM} describes the line-by-line algorithm of INDM training. We obtain the flow loss by taking a flow evaluation. Afterward, we compute the denoising loss. We train the flow with Eq. \eqref{main_eq:training_loss}. However, we train the score with $\mathcal{L}\big(\{\mathbf{x}_{t}\}_{t=0}^{T},\lambda;\{\bm{\phi},\bm{\theta}\}\big)$ with various $\lambda$ settings for a better Fr\'echet Inception Distance (FID) \cite{heusel2017gans}.
	
	\textbf{Latent Sampling} While either of red or purple path in Figure \ref{fig:pdm_gm} could synthesize the samples, we choose the red path for the fast sampling (because the red path feed-forwards the flow only once). Starting from a pure noise $\mathbf{z}_{T}^{\bm{\theta}}\sim\pi$, we denoise $\mathbf{z}_{T}^{\bm{\theta}}$ to $\mathbf{z}_{0}^{\bm{\theta}}$ by solving the generative process backward on the latent space. Then, we transform the fully denoised latent $\mathbf{z}_{0}^{\bm{\theta}}$ to the data space $\mathbf{x}_{0}^{\bm{\phi},\bm{\theta}}=\mathbf{h}_{\bm{\phi}}^{-1}(\mathbf{z}_{0}^{\bm{\theta}})$.
	
	\section{Related Work}\label{sec:related_work}
	
	\begin{wrapfigure}{R}{0.59\textwidth}
		\vskip -0.63in
		\begin{minipage}{0.59\textwidth}
			\begin{algorithm}[H]
				\centering
				\caption{Implicit Nonlinear Diffusion Model}\label{alg:INDM}
				\begin{algorithmic}[1]
					\Repeat
					\State Get latent with flow by $\mathbf{z}_{0}^{\bm{\phi}}=\mathbf{h}_{\bm{\phi}}(\mathbf{x}_{0})$ for $\mathbf{x}_{0}\sim p_{r}$ 
					\State Compute $-\mathbb{E}_{\mathbf{x}_{0}}\big[\log{\big\vert\det\big(\nabla\mathbf{h}_{\bm{\phi}}(\mathbf{x}_{0})\big)\big\vert}\big]$
					\State Get diffused latents $\{\mathbf{z}_{t}^{\bm{\phi}}\}_{t=0}^{T}$ with a linear SDE
					\State Compute $\mathcal{L}\big(\{\mathbf{z}_{t}^{\bm{\phi}}\}_{t=0}^{T},g^{2};\bm{\theta}\big)-\mathbb{E}_{\mathbf{z}_{T}^{\bm{\phi}}}\big[\log{\pi(\mathbf{z}_{T}^{\bm{\phi}})}\big]$
					\State Compute flow loss $\mathcal{L}_{f}=\mathcal{L}\big(\{\mathbf{x}_{t}\}_{t=0}^{T},g^{2};\{\bm{\phi},\bm{\theta}\}\big)$
					\State Update $\bm{\phi}\leftarrow \bm{\phi}-\eta\frac{\partial\mathcal{L}_{f}}{\partial\bm{\phi}}$
					\State Compute $\mathcal{L}\big(\{\mathbf{z}_{t}^{\bm{\phi}}\}_{t=0}^{T},\lambda;\bm{\theta}\big)-\mathbb{E}_{\mathbf{z}_{T}^{\bm{\phi}}}\big[\log{\pi(\mathbf{z}_{T}^{\bm{\phi}})}\big]$
					\State Compute score loss $\mathcal{L}_{s}=\mathcal{L}\big(\{\mathbf{x}_{t}\}_{t=0}^{T},\lambda;\{\bm{\phi},\bm{\theta}\}\big)$
					\State Update $\bm{\theta}\leftarrow \bm{\theta}-\eta\frac{\partial\mathcal{L}_{s}}{\partial\bm{\theta}}$
					\Until {converged}
				\end{algorithmic}
			\end{algorithm}
		\end{minipage}
		\vskip -0.2in
	\end{wrapfigure}
	In this section, we compare INDM with previous works, and summarize our arguments in Table \ref{tab:previous_research}.
	
	\textbf{LSGM} \citet{vahdat2021score} put a linear diffusion on the latent space like INDM but uses an auto-encoder structure. From this modeling choice, LSGM cannot be categorized as a nonlinear diffusion model in a strict sense. Concretely, recall that a diffusion process is (mathematically) defined as a sequence of random variables connected via a Markov chain. From this definition, one needs to satisfy two requirements to call it a diffusion process: 1) there must be multiple (possibly infinite) numbers of random variables; 2) the random variables should be connected via a Markov chain. Unlike INDM, LSGM cannot build forward data variables from the forward latent variables because there is no exact inverse function of the encoder map, as long as the data dimension differs to the latent dimension (Lemma \ref{lemma:3} of Appendix \ref{appendix:LSGM}). This leads that LSGM has no forward data diffusion process. From this point, analyzing the data nonlinearity becomes infeasible in LSGM.
	
	\begin{table*}[t]
		\vskip -0.05in
		\caption{Comparison of INDM with previous works. $N$ is the number of random variables.}
		\label{tab:previous_research}
		\vskip -0.05in
		\scriptsize
		\centering
		\begin{adjustbox}{max width=\textwidth}
			\begin{tabular}{lcccccccc}
				\toprule
				\multirow{2}{*}{Model} & \multirow{2}{*}{\shortstack{Nonlinear\\Diffusion}} & \multirow{2}{*}{\shortstack{Implemented\\Data Diffusion}} & \multirow{2}{*}{\shortstack{Latent\\Diffusion}} & \multirow{2}{*}{\shortstack{Nonlinear\\$\mathbf{f}$-Modeling}} & \multirow{2}{*}{\shortstack{Nonlinear\\$\mathbf{G}$-Modeling}} & \multirow{2}{*}{\shortstack{Explicit $\mathbf{f}\&\mathbf{G}$\\Derived}} & \multirow{2}{*}{\shortstack{Training\\Complexity}} & \multirow{2}{*}{\shortstack{Sampling\\Cost}} \\
				&&&&&&&&\\\midrule
				DDPM++ & \xmark & Continuous & \xmark & \xmark & \xmark & \cmark & $O(1)$ & $\downarrow$ \\
				LSGM & \xmark & \xmark & Continuous & \xmark & \xmark & \xmark & $O(1)$ & $\downarrow$ \\
				SBP & $\bm{\triangle}$ & Discrete & \xmark & Explicit & \xmark & \cmark & $O(N)$ & $\downarrow$ \\
				DiffFlow & \cmark & Discrete & \xmark & Explicit & \xmark & \cmark & $O(N)$ & $\uparrow$ \\
				\cc{15}INDM & \cc{15}\cmark & \cc{15}Continuous & \cc{15}Continuous & \cc{15}Implicit & \cc{15}Implicit & \cc{15}\cmark & \cc{15}$O(1)$ & \cc{15}$\downarrow$ \\
				\bottomrule
			\end{tabular}
		\end{adjustbox}
		\vskip -0.1in
	\end{table*}
	
	Moreover, LSGM has a pair of key differences in its training. First, the latent dimension of LSGM is 40,080, which is \textit{15$\times$ higher} dimension than the data dimension (3,072) on CIFAR-10 \cite{krizhevsky2009learning}. In contrast, INDM always keeps its latent dimension by the data dimension. See Table \ref{tab:dimension} to compare the latent dimensions of INDM with LSGM on benchmark datasets. Furthermore, LSGM is repeatedly reported \cite{vahdat2021score, dockhorn2021score} for its training instability on the best FID setting of $\lambda=\sigma^{2}$ (i.e., $L_{simple}$ \cite{ho2020denoising}). Meanwhile, INDM is consistently stable for any of training configurations, see Table \ref{tab:lsgm_comparison}.
	
	\textbf{DiffFlow} \citet{zhang2021diffusion} explicitly model the drift term $\mathbf{f}_{\bm{\phi}}$ as a flow network, so the forward diffusion becomes $\diff\mathbf{x}_{t}=\mathbf{f}_{\bm{\phi}}\diff t+g\diff\mathbf{w}_{t}$. However, there are differences between DiffFlow and INDM: 1) DiffFlow does not nonlinearize the volatility; 2) DiffFlow is too slow for its explicit parametrization (Table \ref{tab:elapsed_time}); 3) the flexibility of $\mathbf{f}_{\bm{\phi}}$ is too restricted; 4) DiffFlow has a larger loss variance (Table \ref{tab:variance}). See Appendix \ref{appendix:DiffFlow} for the full details of our arguments. Focusing on the slow training, observe that the denoising loss $\mathbb{E}_{\mathbf{x}_{0},\mathbf{x}_{t}^{\bm{\phi}}}[\Vert\mathbf{s}_{\bm{\theta}}(\mathbf{x}_{t}^{\bm{\phi}},t)-\nabla\log{p_{0t}(\mathbf{x}_{t}^{\bm{\phi}}\vert\mathbf{x}_{0})}\Vert_{2}^{2}]$ requires a pair of heavy computations: (A) sampling from $\mathbf{x}_{t}^{\bm{\phi}}$, and (B) computation of $\nabla\log{p_{0t}(\mathbf{x}_{t}^{\bm{\phi}}\vert\mathbf{x}_{0})}$. Intractable transition probability $p_{0t}(\mathbf{x}_{t}^{\bm{\phi}}\vert\mathbf{x}_{0})$ is the major bottleneck of the slow training. 
	
	To overcome the bottleneck, DiffFlow discretizes the continuous diffusion with $N$ variables of a discrete diffusion and uses the DDPM-style loss \cite{ho2020denoising}, which does not need to calculate the transition probability. Under the discretization, however, the forward sampling of $\mathbf{x}_{t}^{\bm{\phi}}$ takes $O(N)$ flow evaluations for every network update. This sampling issue is an inevitable fundamental problem when we parametrize the coefficients explicitly. Having that the flow evaluation is generally more expensive than score evaluation given the same number of parameters, a fast sampling is achievable only if we reduce $N$. However, it hurts the flexibility of a diffusion process, so DiffFlow suffers from the trade-off between training speed and model flexibility. On the other hand, the training of INDM is invariant of $N$, and INDM is free from such a trade-off. Analogously, DiffFlow generates a sample with the purple path in Figure \ref{fig:pdm_gm}, so it takes $O(N)$ flow evaluations, contrastive to INDM with a single flow evaluation in its sampling with the red path.
	
	\textbf{SBP} \citet{de2021diffusion} learn the diffusion process with a problem of $\min_{\bm{\rho}_{\bm{\theta}}\in\mathcal{P}(p_{r},\pi)}D_{KL}(\bm{\rho}_{\bm{\theta}}\Vert\bm{\mu})$, where $\mathcal{P}(p_{r},\pi)$ is the collection of path measure with $p_{r}$ and $\pi$ as its marginal distributions at $t=0$ and $t=T$, respectively. It is a bi-constrained optimization problem as any path measure on the search space that should satisfy boundary conditions at both $t=0$ and $t=T$. $\bm{\mu}$ is the reference measure of a linear diffusion $\diff\mathbf{x}_{t}=\mathbf{f}(\mathbf{x}_{t},t)\diff t+g(t)\mathbf{w}_{t}$; and the forward and reverse SDEs of $\bm{\rho}_{\bm{\theta}}$ are $\diff\mathbf{x}_{t}=[\mathbf{f}(\mathbf{x}_{t},t)+g^{2}(t)\nabla\log{\Psi(\mathbf{x}_{t},t)}]\diff t+g(t)\diff\mathbf{w}_{t}$ and $\diff\mathbf{x}_{t}=[\mathbf{f}(\mathbf{x}_{t},t)-g^{2}(t)\nabla\log{\hat{\Psi}(\mathbf{x}_{t},t)}]\diff \bar{t}+g(t)\diff\mathbf{\bar{w}}_{t}$, respectively, where $(\Psi,\hat{\Psi})$ is the solution of a coupled PDE, called Hopf-Cole transform \cite{leger2021hopf}. Solving this coupled PDE is intractable, so the estimation target of SBP is $\nabla\log{\Psi}$ and $\nabla\log{\hat{\Psi}}$. As $\mathbf{f}$ and $g$ are assumed to be linear functions, the nonlinearity of SBP is fully determined by $(\Psi,\hat{\Psi})$. 
		
	\begin{wrapfigure}{r}{0.3\textwidth}
		\vskip -0.2in
		\centering
		\includegraphics[width=\linewidth]{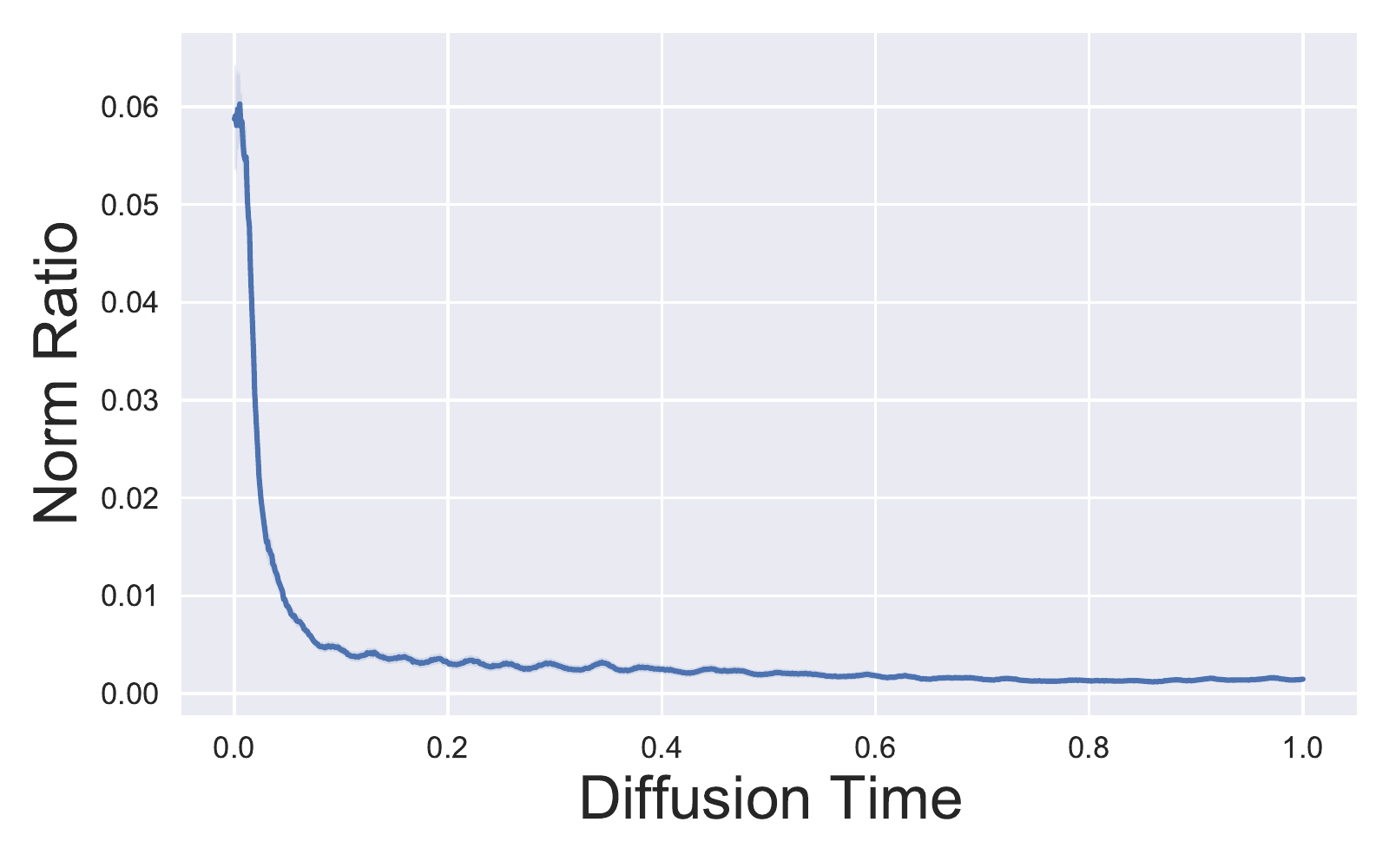}
		\vskip -0.1in
		\caption{Norm Ratio of SBP.}
		\label{fig:norm_ratio}
		\vskip -0.2in
	\end{wrapfigure}
	Analogous to DiffFlow, sampling from $\mathbf{x}_{t}$ in SBP needs a long time. Few works \cite{de2021diffusion,chen2021likelihood} detour this training issue using the experience replay memory. Aside from the training time, the KL minimization puts the global optimal nonlinear diffusion $\bm{\rho}_{\bm{\theta}^{*}}$ near a neighborhood of the linear diffusion $\bm{\mu}$. In other words, the optimal $\bm{\rho}_{\bm{\theta}^{*}}$ is the closest path measure on $\mathcal{P}(p_{r},\pi)$ to $\bm{\mu}$, so the inferred nonlinear diffusion would be the \textit{most} linear diffusion on the space of $\mathcal{P}(p_{r},\pi)$. For the demonstration, we illustrate $\Vert g^{2}(t)\nabla\log{\Psi(\mathbf{x}_{t},t)}\Delta t\Vert_{2}/\Vert g(t)\Delta\mathbf{w}_{t}\Vert_{2}$ in Figure \ref{fig:norm_ratio}. We used the released checkpoint of SB-FBSDE \cite{chen2021likelihood}, an algorithm for solving SBP, trained with VESDE on CIFAR-10. As $\mathbf{f}\equiv 0$ in VESDE, this norm ratio measures how much nonlinearity is counted on a diffusion trajectory compared to the linear effect. Figure \ref{fig:norm_ratio} shows that the ratio approaches zero except at the small range around $t\approx 0$, meaning that the nonlinear effect is virtually ignorable than the linear effect. Therefore, Figure \ref{fig:norm_ratio} implies that the diffusion process is nearly linear in most of the diffusion time. We give a detailed discussion of SBP in Appendix \ref{appendix:SBP}.
	
	\section{Discussion}\label{sec:discussion}
	
	This section investigates characteristics of INDM. We show that INDM training is faster and nearly MLE in Section \ref{sec:better_optimization}, and INDM sampling is robust on discretization step sizes in Section \ref{sec:sampling_robustness}.
	
	\subsection{Benefit of INDM in Training}\label{sec:better_optimization}
		
	Having that DDPM++ is a collapsed INDM with a fixed identity transformation $\mathbf{h}_{\bm{\phi}_{id}}=id$, the difference lies in whether to train $\bm{\phi}$ or not. This trainable nonlinearity provides the better optimization of INDM, as evidenced in Figure \ref{fig:discussion}-(a), experimented on CIFAR-10 using VPSDE. It shows a pair of critical characteristics of INDM training: 1) it is faster than DDPM++ training, and 2) it is asymptotically an MLE training. For the training speed, recall that the regression target of the score estimation is $\mathbf{s}_{\bm{\phi}}$, and this target is fixed in DDPM++ while keep moving in INDM. The target is constantly updated through the direction of $\mathbf{s}_{\bm{\theta}}$ in Eq. \eqref{main_eq:nelbo} by optimizing $\Vert\mathbf{s}_{\bm{\phi}}-\mathbf{s}_{\bm{\theta}}\Vert_{2}^{2}$. This \textit{bidirectional} attraction between $\mathbf{s}_{\bm{\theta}}$ and $\mathbf{s}_{\bm{\phi}}$ is what flow learning does in the optimization.
	
	For the MLE training, as the flow training is intricately entangled with the score training, we analyze INDM training for a specific class of score networks. First, we define $\mathbf{S}_{sol}$ (Definition \ref{def:1} in Appendix \ref{appendix:variational_gap}) to be the class of forward score functions of a linear diffusion with some initial distribution. Then, it turns out that it is the whole class of zero variational gap (=NLL-NELBO).
	\begin{theorem}\label{cor:2}
		$\textup{Gap}(\bm{\mu}_{\bm{\phi}},\bm{\nu}_{\bm{\phi},\bm{\theta}}):=D_{KL}(\bm{\mu}_{\bm{\phi}}\Vert\bm{\nu}_{\bm{\phi},\bm{\theta}})-D_{KL}(p_{r}\Vert p_{\bm{\phi},\bm{\theta}})=0$ if and only if $\mathbf{s}_{\bm{\theta}}\in\mathbf{S}_{sol}$. 
	\end{theorem}
	\citet{song2021maximum} partially reveal the connection between the gap with $\mathbf{S}_{sol}$, by proving the \textit{if} part of Theorem \ref{cor:2}, in Theorem 2 of \citet{song2021maximum} (see Lemma \ref{lemma:2} in Appendix \ref{appendix:variational_gap}). We completely characterize this connection by proving the \textit{only-if} part in Theorem \ref{cor:2}. Surprisingly, the variational gap is irrelevant to the flow parameters, and the MLE training of INDM implies that the score network is nearby $\mathbf{S}_{sol}$ throughout the training. Combining Theorem \ref{cor:2} with the global optimality analysis raises a qualitative discrepancy in the optimization of DDPM++ and INDM by Theorem \ref{thm:3}.
	\begin{theorem}\label{thm:3}
		For any fixed $\mathbf{s}_{\bm{\theta}}\in\mathbf{S}_{sol}$, if $\bm{\phi}^{*}\in\argmin_{\bm{\phi}}{D_{KL}(\bm{\mu}_{\bm{\phi}}\Vert \bm{\nu}_{\bm{\phi},\bm{\theta}})}$, then $\mathbf{s}_{\bm{\phi}^{*}}(\mathbf{z},t)=\nabla\log{p_{t}^{\bm{\phi}^{*}}(\mathbf{z})}=\mathbf{s}_{\bm{\theta}}(\mathbf{z},t)$, and $D_{KL}(\bm{\mu}_{\bm{\phi}^{*}}\Vert \bm{\nu}_{\bm{\phi}^{*},\bm{\theta}})=D_{KL}(p_{r}\Vert p_{\bm{\phi}^{*},\bm{\theta}})=\text{Gap}(\bm{\mu}_{\bm{\phi}^{*}},\bm{\nu}_{\bm{\phi}^{*},\bm{\theta}})=0$.
	\end{theorem}
	Theorem \ref{thm:3} implies that there exists an optimal flow that the forward and generative SDEs on the latent space coincide, for any score network in $\mathbf{S}_{sol}$, if the flow is flexible enough. Therefore, INDM attains infinitely many ($=\vert\mathbf{S}_{sol}\vert$) global optimums in its optimization space. On the other hand, DDPM++ has only a unique optimal score network, i.e., $\mathbf{s}_{\bm{\theta}^{*}}=\mathbf{s}_{\bm{\phi}_{id}}$. Thus, Theorem \ref{thm:3} potentially explains the faster convergence of INDM. We give a detailed analysis in Appendix \ref{appendix:variational_gap}.
	
	\begin{figure*}[t]
		\centering
		\vskip -0.05in
		\begin{subfigure}{0.32\linewidth}
			\centering
			\includegraphics[width=\linewidth]{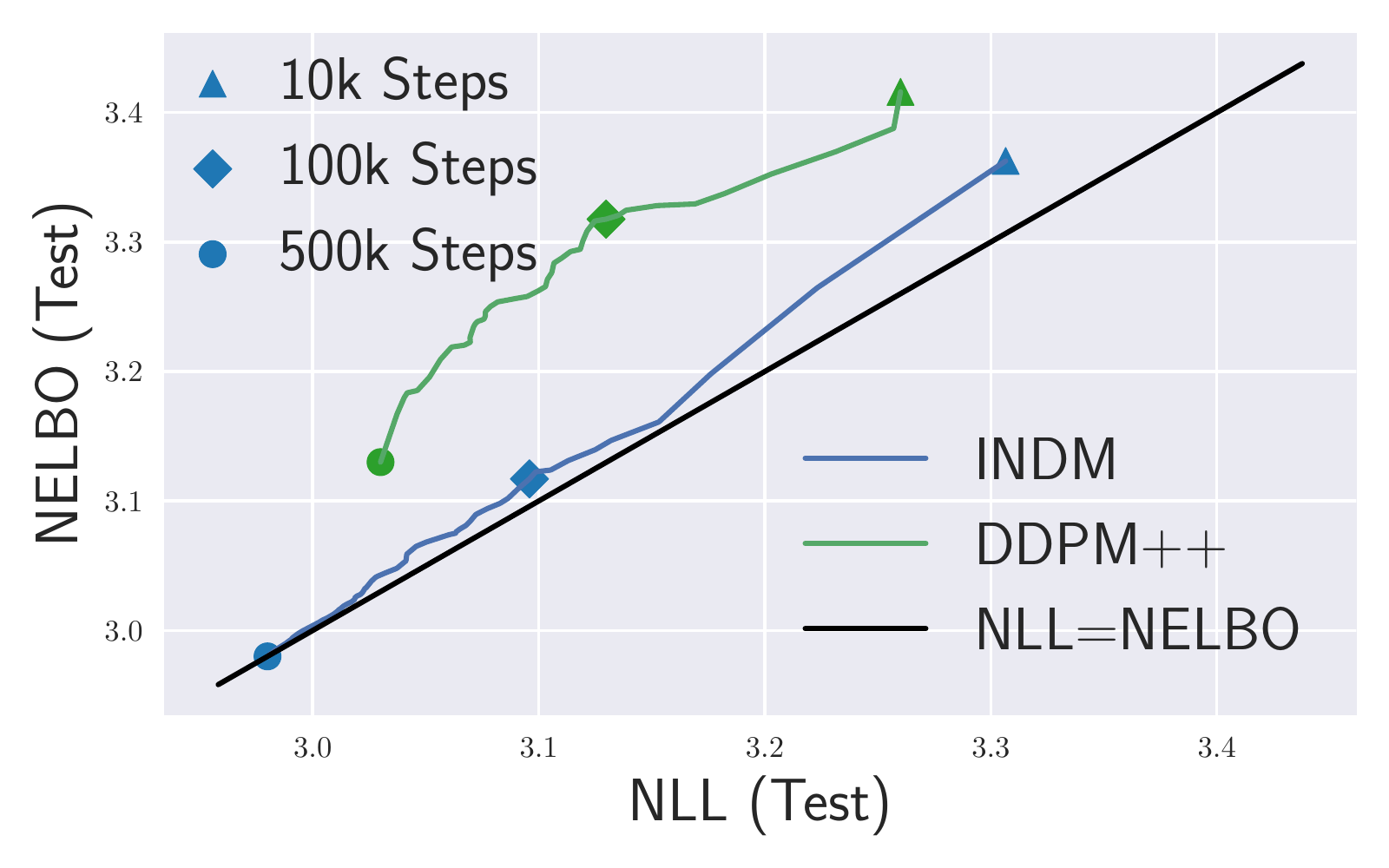}
			\vskip -0.05in
			\subcaption{Training Curve}
		\end{subfigure}
		\hfill
		\begin{subfigure}{0.32\linewidth}
			\centering
			\includegraphics[width=\linewidth]{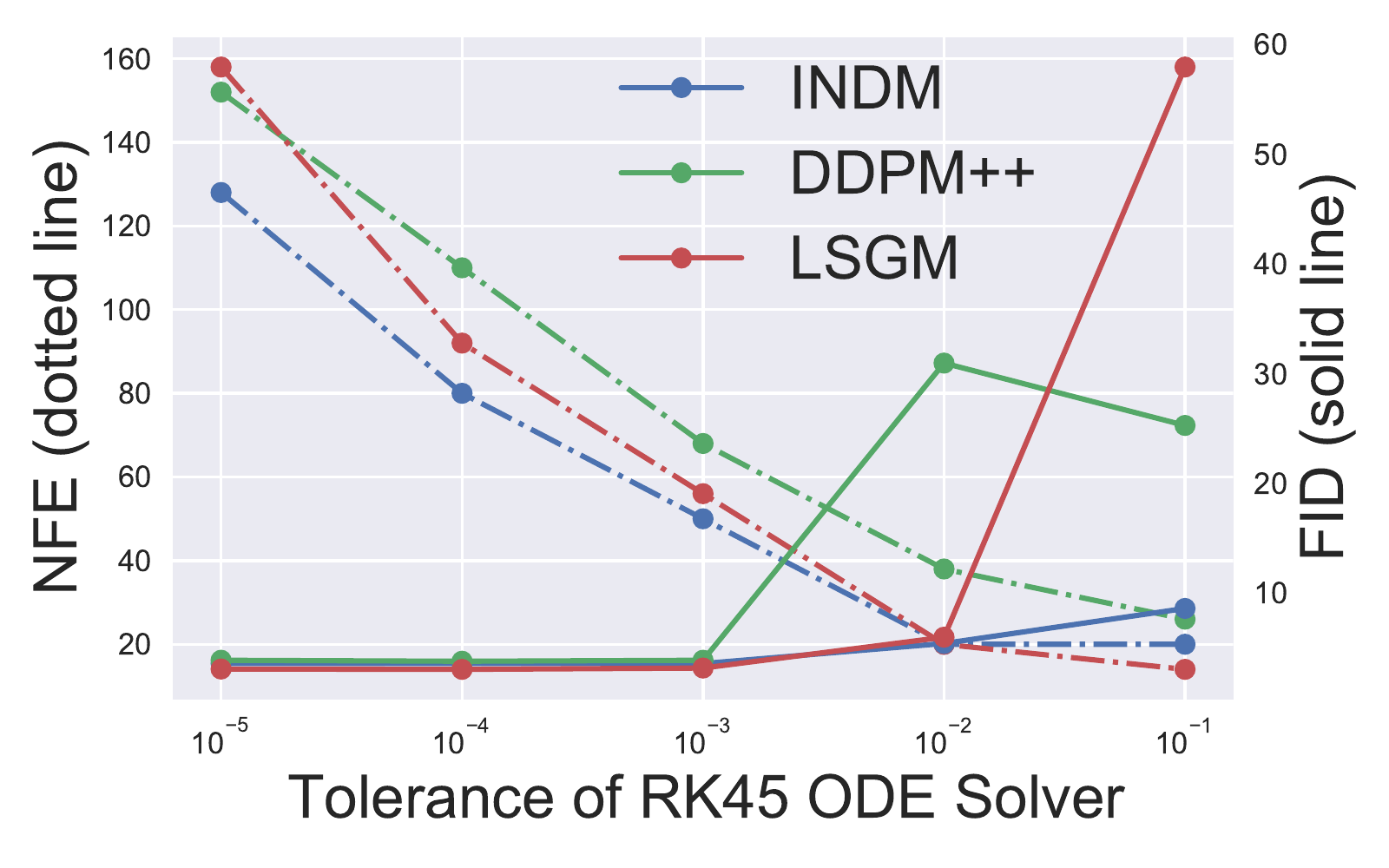}
			\vskip -0.05in
			\subcaption{FID by Tolerance}
		\end{subfigure}
		\hfill
		\begin{subfigure}{0.32\linewidth}
			\centering
			\includegraphics[width=\linewidth]{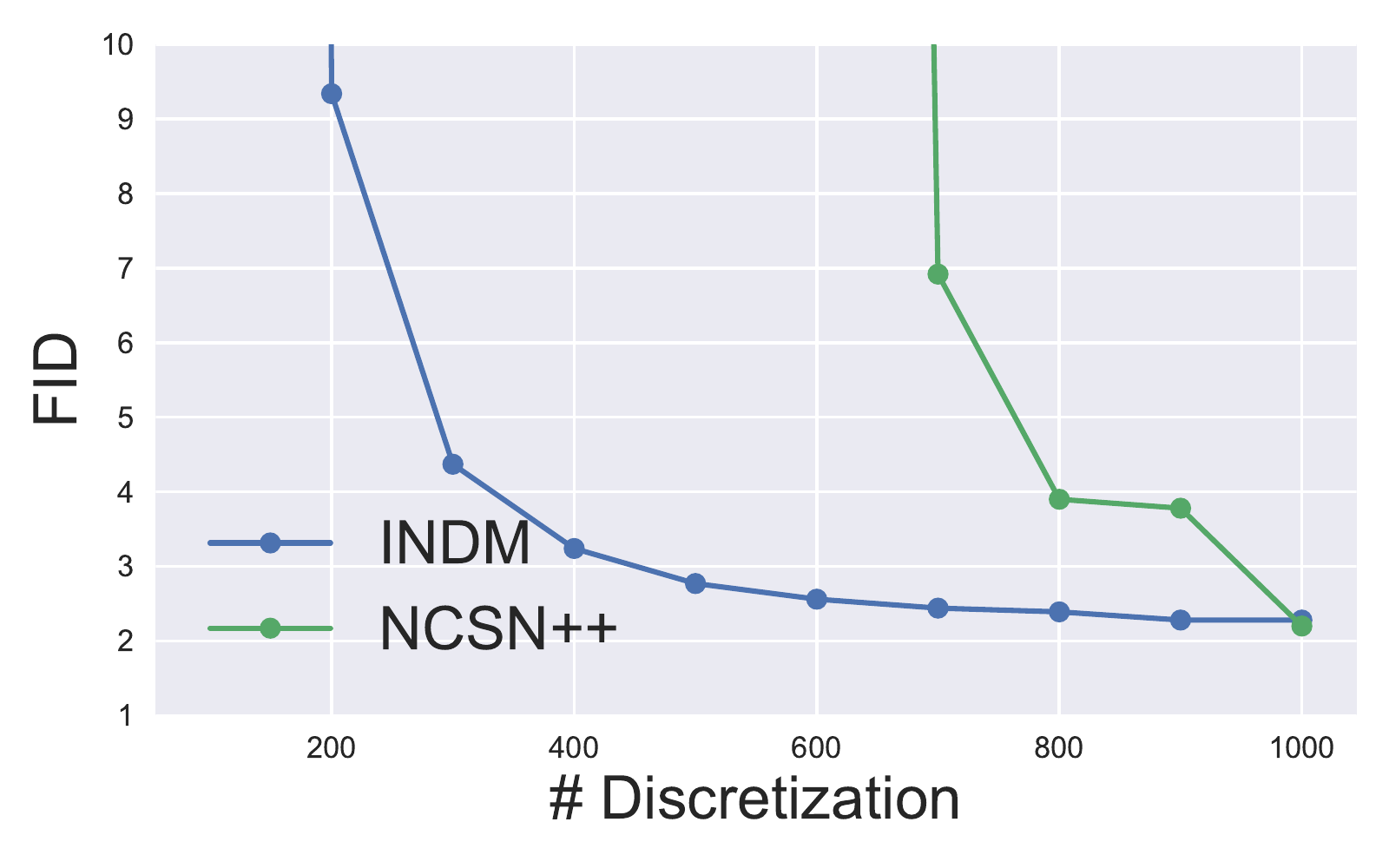}
			\vskip -0.05in
			\subcaption{FID by Discretization}
		\end{subfigure}
		\vskip -0.05in
		\caption{Comparison of INDM with baseline models, experimented on CIFAR-10.}
		\label{fig:discussion}
		\vskip -0.1in
	\end{figure*}
	
	\subsection{Benefit of INDM in Sampling}\label{sec:sampling_robustness}
	
	Figure \ref{fig:discussion}-(b,c) equally illustrate that INDM is more robust on the discretization step sizes in FID than DDPM++/NCSN++. To analyze the sample quality with respect to discretizations, recall that the Euler-Maruyama discretization of the generative SDE (or called the reverse diffusion sampler, or simply the predictor \cite{song2020score}) iteratively updates the sample $\mathbf{\tilde{z}}_{k}$ with
	\begin{align*}
	\mathbf{\tilde{z}}_{t_{k-1}}=\mathbf{\tilde{z}}_{t_{k}}+\gamma_{k}\bigg(\frac{1}{2}\beta(t_{k})\mathbf{\tilde{z}}_{t_{k}}+g^{2}(t_{k})\mathbf{s}_{\bm{\theta}}(\mathbf{\tilde{z}}_{t_{k}},t_{k})\bigg)+g(t_{k})\sqrt{\gamma_{k}}\bm{\epsilon},
	\end{align*}
	until time reaches to zero, where $\gamma_{k}=t_{k}-t_{k-1}$ is the step size of the discretized sampler and $\bm{\epsilon}\sim \mathcal{N}(0,\mathbf{I})$. The sampling error is the distributional discrepancy between the sample distribution of $\mathbf{h}_{\bm{\phi}}^{-1}(\mathbf{\tilde{z}}_{0})$ and the data distribution. Theorem \ref{thm:dsb} decomposes the sampling error with three factors: 1) the prior error $E_{pri}$, 2) the discretization error $E_{dis}$, and 3) the score error $E_{est}$. Note that Theorem \ref{thm:dsb} is a straightforward application of the analysis done by \citet{de2021diffusion} and \citet{guth2022wavelet}. We omit regularity conditions to avoid unnecessary complications; see Appendix \ref{appendix:dsb}.
	\begin{theorem}[\citet{de2021diffusion} and \citet{guth2022wavelet}]\label{thm:dsb}
		Assume that 1) $\sup_{\mathbf{z},t}\Vert\mathbf{s}_{\bm{\theta}^{*}}(\mathbf{z},t)-\nabla\log{p_{t}^{\bm{\phi}}(\mathbf{z})}\Vert\le M$, 2) $\sup_{\mathbf{z},t}\Vert\nabla^{2}\log{p_{t}^{\bm{\phi}}(\mathbf{z})}\Vert\le K$, and 3) $\sup_{\mathbf{z},t}\Vert\partial_{t}\nabla\log{p_{t}^{\bm{\phi}}(\mathbf{z})}\Vert/\Vert\mathbf{z}\Vert\le L e^{-\alpha t}$, for some $K,L,M,\alpha>0$. Then
		\begin{align*}
		\Vert p_{r}-(\mathbf{h}_{\bm{\phi}}^{-1})_{\#}\circ p_{0,N}^{\bm{\theta}}\Vert_{TV}\le E_{pri}(\bm{\phi})+E_{dis}(\bm{\phi})+E_{est}(\bm{\phi},\bm{\theta})+o(\sqrt{\delta}+e^{-T}),
		\end{align*}
		where $E_{pri}(\bm{\phi})=\sqrt{2}e^{-T}D_{KL}(p_{T}^{\bm{\phi}}\Vert\pi)^{1/2}$, $E_{dis}(\bm{\phi})=6\sqrt{\delta}(1+\mathbb{E}[\Vert\mathbf{z}\Vert^{4}]^{1/4})(1+K+L(1+1/\sqrt{2\alpha}))$, and $E_{est}(\bm{\phi},\bm{\theta})=2TM^{2}$ with $\delta=\max{\gamma_{k}}^{2}/\min{\gamma_{k}}$.
	\end{theorem}
	
	\begin{wrapfigure}{r}{0.3\textwidth}
		\vskip -0.25in
		\centering
		\includegraphics[width=\linewidth]{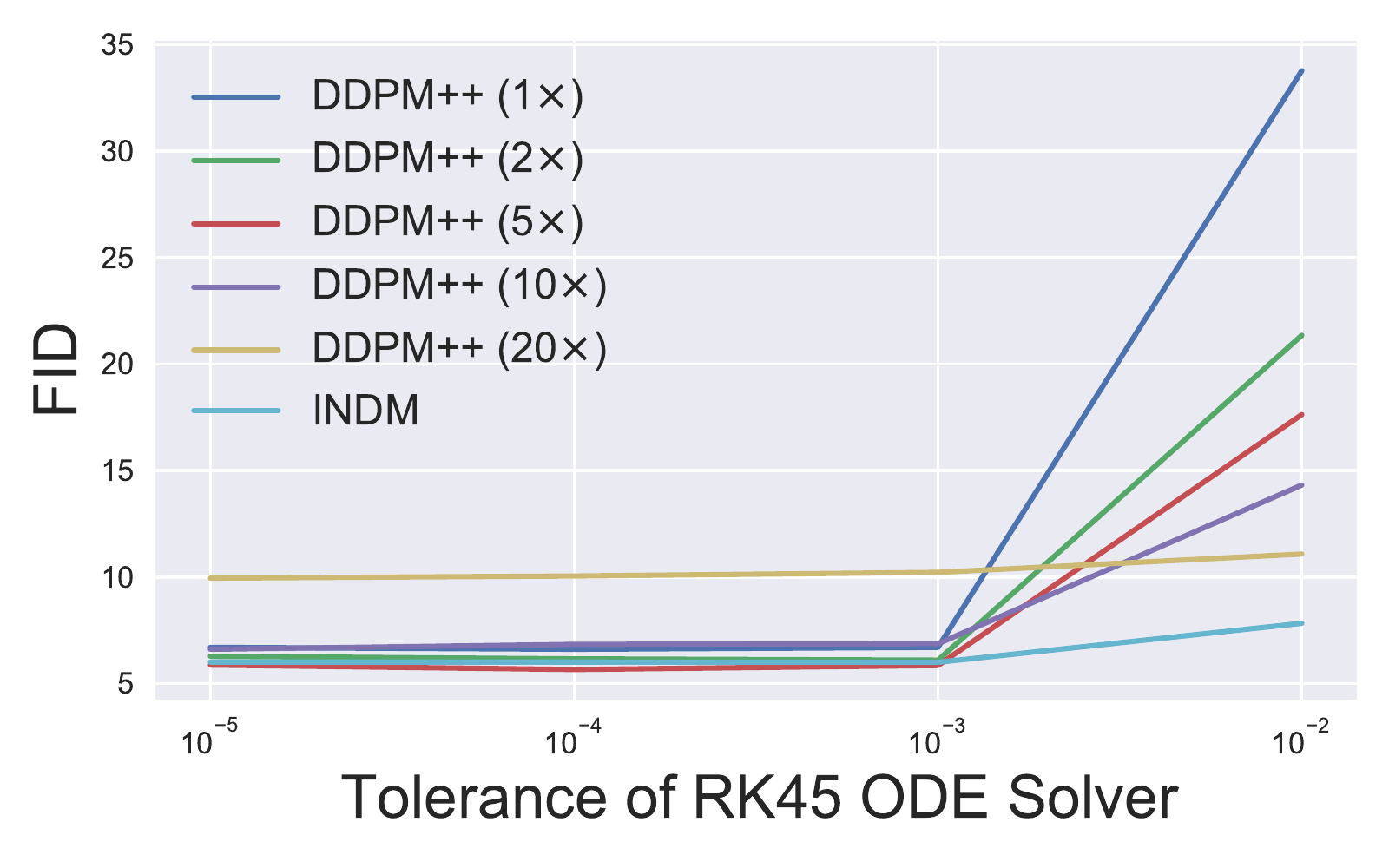}
		\vskip -0.1in
		\caption{Sensitivity analysis on scaled-up scenario.}
		\label{fig:sensitivity_analysis}
		\vskip -0.1in
	\end{wrapfigure}
	There are a pair of implications from Theorem \ref{thm:dsb}.
	\vspace{-2mm}
	\begin{itemize}\setlength\itemsep{0.2em}
	\item[\checkmark] $E_{pri}(\bm{\phi})$ and $E_{est}(\bm{\phi},\bm{\theta})$ are independent of the discretization steps.
	\item[\checkmark] $E_{dis}(\bm{\phi})/\sqrt{\delta}$ is the discretization sensitivity, entirely determined by the latent distribution's smoothness.
	\end{itemize}
	\vspace{-2mm}
	To the deep understanding of the second implication, let us assume $\mathbf{h}_{\bm{\phi}_{a}}(\mathbf{x})=a\mathbf{x}$ for some scalar $a>1$, then the sensitivity is anti-proportional to $a$ with the identical discretizations, i.e., $E_{dis}(\bm{\phi}_{a})\approx\frac{1}{a}E_{dis}(\bm{\phi}_{id})$. With a smaller sensitivity of $\bm{\phi}_{a}$, there is more room to reduce the number of discretization steps for $\bm{\phi}_{a}$. Figure \ref{fig:sensitivity_analysis} empirically supports the theory, showing that the sampler (at the large tolerance with $10^{-2}$) becomes more robust as $a$ increases, on CIFAR-10.
	
	\begin{figure*}[t]
		\centering
		\begin{subfigure}{0.32\linewidth}
			\includegraphics[width=\linewidth]{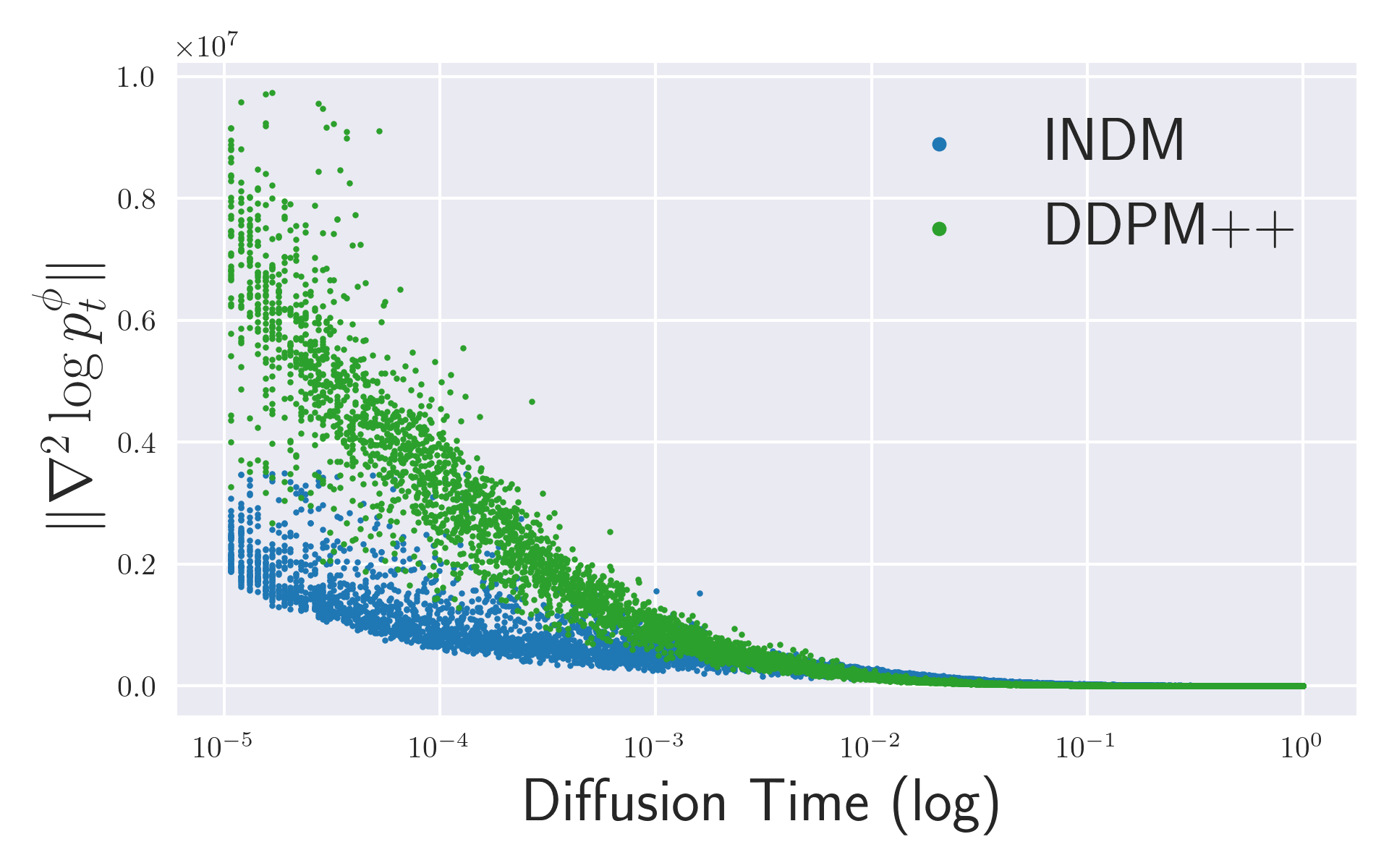}
			\vskip -0.05in
			\subcaption{$\Vert\nabla^{2}\log{p_{t}^{\bm{\phi}}(\mathbf{z})}\Vert$}
		\end{subfigure}
		\hfill
		\begin{subfigure}{0.32\linewidth}
			\includegraphics[width=\linewidth]{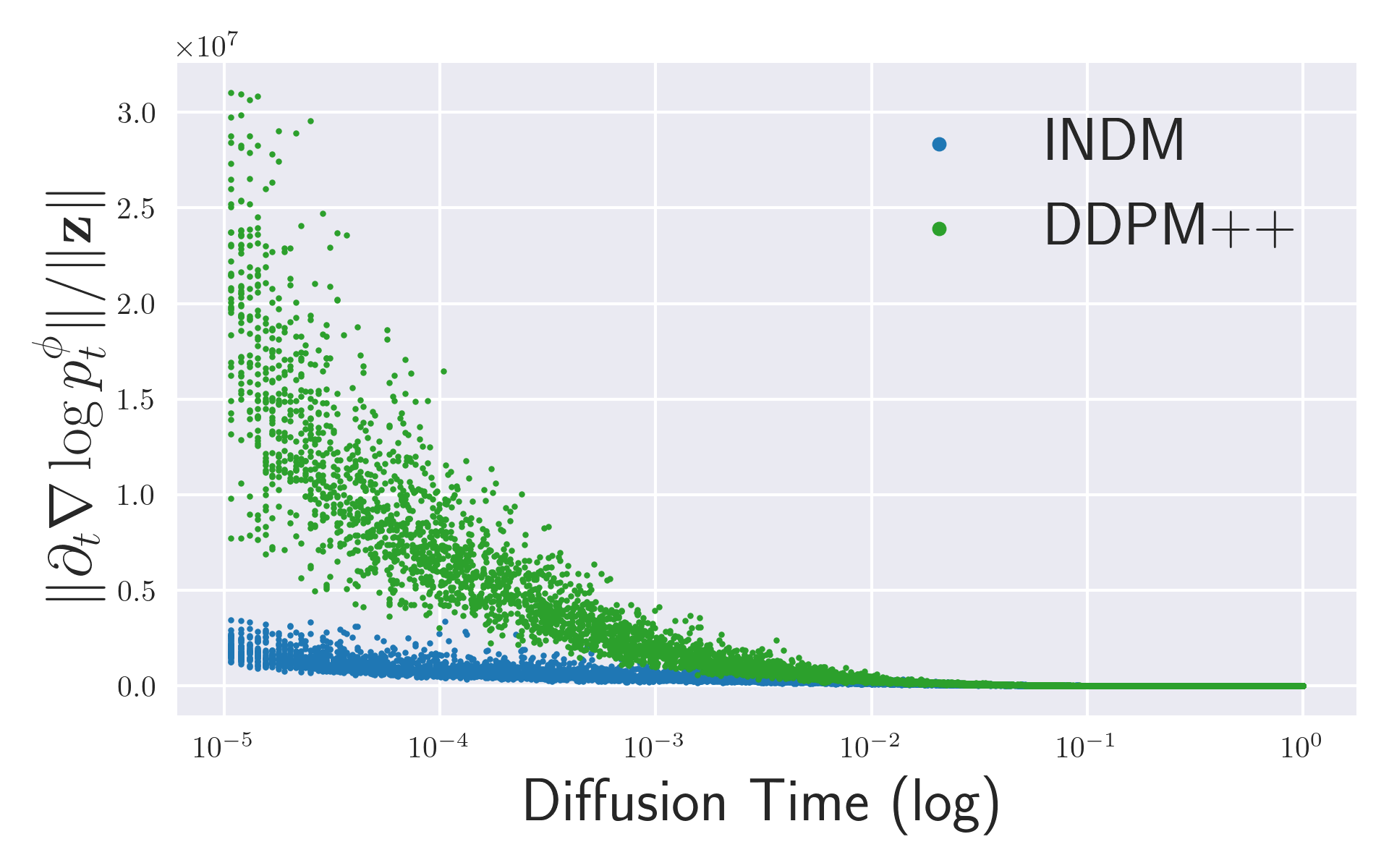}
			\vskip -0.05in
			\subcaption{$\Vert\partial_{t}\nabla\log{p_{t}^{\bm{\phi}}(\mathbf{z})}\Vert/\Vert\mathbf{z}\Vert$}
		\end{subfigure}
		\hfill
		\begin{subfigure}{0.32\linewidth}
			\includegraphics[width=\linewidth]{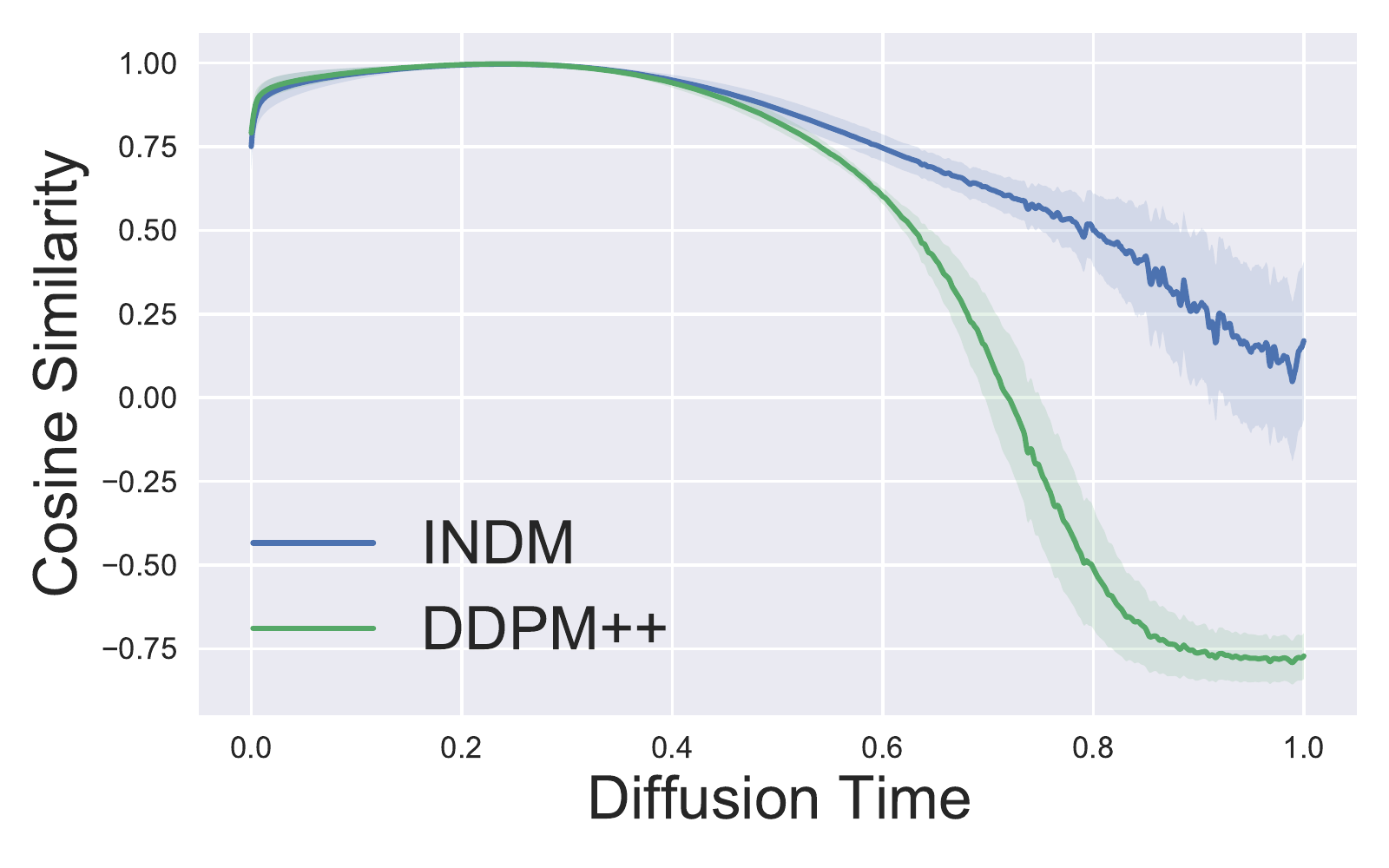}
			\vskip -0.05in
			\subcaption{Cosine Similarity}
		\end{subfigure}
		\vskip -0.05in
		\caption{(a,b) Comparison of INDM with DDPM++ for $K,L,\alpha$. (c) Cosine similarity of forward diffusion trajectories on CIFAR-10.}
		\label{fig:latent_geometry}
		\vskip -0.2in
	\end{figure*}
	
	Before we derive a concrete result from the implication, observe that the flow $\mathbf{h}_{\bm{\phi}}$ is maximizing $\text{det}(\nabla\mathbf{h}_{\bm{\phi}})$ in Eq. \eqref{main_eq:training_loss}. To understand the effect of flow training on the discretization sensitivity, let us restrict the hypothesis class of the transformation to be linear mappings of $\mathbf{h}_{\bm{\phi}_{a}}(\mathbf{x})=a\mathbf{x}$. Then, as the determinant increases by $a$, the trained diffusion model would be insensitive to the discretizations. Now, for the general case, Figure \ref{fig:latent_geometry}-(a,b) illustrate $\Vert\nabla^{2}\log{p_{t}^{\bm{\phi}}(\mathbf{z})}\Vert$ and $\Vert\partial_{t}\nabla\log{p_{t}^{\bm{\phi}}(\mathbf{z})}\Vert/\Vert\mathbf{z}\Vert$ on CIFAR-10, respectively. Also, $\mathbb{E}[\Vert\mathbf{z}\Vert^{4}]^{1/4}$ of INDM is slightly larger (1.3x) than DDPM++. Therefore, with these observations combined, we conclude that INDM is less sensitive to the discretization steps than DDPM++, from its loss design.
	
	\begin{wraptable}{r}{0.16\textwidth}
		\vskip -0.05in
		\caption{Average $L_{2}^{2}$ Norm.}
		\label{tab:manifold_geometry}
		\vskip -0.05in
		\tiny
		\centering
		\begin{tabular}{cc}
			\toprule
			Manifold & Norm \\\midrule
			Data & 776 \\
			Latent & 5,385\\
			Prior & 3,072\\
			\bottomrule
		\end{tabular}
		\vskip -0.2in
	\end{wraptable}
	Second, the robustness could originate from the geometry of the diffusion trajectory. The forward solution of VPSDE is $\mathbf{x}_{t}=\mu(t)\mathbf{x}_{0}+\sqrt{1-\mu^{2}(t)}\bm{\epsilon}$ with $\mu(t)=e^{-\frac{1}{2}\int_{0}^{t}\beta(s)\diff s}$, where the first term is a contraction mapping and the second term is the random perturbation. The contraction mapping points toward the origin, but the overall vector field of the diffusion path points outward because the prior manifold lies outside the data manifold, as shown in Table \ref{tab:manifold_geometry}. This contrastive force leads the drift and volatility coefficients works in repulsive and raises a highly nonlinear diffusion trajectory in DDPM++, see Figure \ref{fig:sample_trajectory_2d} for a toy illustration. On the other hand, the flow mapping of INDM pushes the latent manifold outside the prior manifold, and the drift and volatility coefficients act coherently. Hence, INDM has the relatively linear diffusion path; see Figures \ref{fig:linear_diffusion_by_range}, \ref{fig:diffusion_bridge}, and \ref{fig:2d_toy} for a quick intuition. Figure \ref{fig:latent_geometry}-(c) measures the cosine similarity of the ODE's diffusion trajectory with the straight line connecting the initial-final points of each trajectory. Figure \ref{fig:latent_geometry}-(c) implies that DDPM++ is under an inefficient nonlinear trajectory that reverts backward near the end of the trajectory, as in Figure \ref{fig:sample_trajectory_2d}. In contrast, INDM trajectory is relatively efficient and linear (Figure \ref{fig:illustrative_particle_trajectory}), which yields robust sampling by discretization steps; see Appendix \ref{appendix:latent_manifold} for details.
	
	\section{Experiments}\label{sec:experiments}
	
	This section quantitatively analyzes suggested INDM on CIFAR-10 and CelebA $64\times 64$. Throughout the experiments, we use NCSN++ with VESDE and DDPM++ with VPSDE \cite{song2020score} as the backbones of diffusion models, and a ResNet-based flow model \citep{chen2019residual, ma2020decoupling} as the backbone of the flow model. See Appendix \ref{appendix:experimental_details} for experimental details. We experiment with a pair of weighting functions for the score training. One is the likelihood weighting \cite{song2021maximum} with $\lambda(t)=g^{2}(t)$, and we denote INDM (NLL) for this weighing choice. The other is the variance weighting \cite{ho2020denoising} $\lambda(t)=\sigma^{2}(t)$ with an emphasis on FID, and we denote INDM (FID) for this weighting choice. 
	
	We use either the Predictor-Corrector (PC) sampler \cite{song2020score} or a numerical ODE solver (RK45 \cite{dormand1980family}) of the probability flow ODE \cite{song2020score}. For a better FID, we find the optimal signal-to-noise value (Table \ref{tab:snr_search}), sampling temperature (Table \ref{tab:temperature_search}), and stopping time (Table \ref{tab:line_search}). Moreover, sampling from $\mathbf{z}_{T}^{\bm{\phi}}$ rather than $\pi$ improves FID because $E_{pri}(\bm{\phi})$ collapses to zero in Theorem \ref{thm:dsb}, see Appendix \ref{appendix:sampling_tricks}. We compute NLL/NELBO for performances of density estimation with Bits Per Dimension (BPD). We compute NLL with the uniform dequantization, instead of the variational dequantization \cite{ho2019flow++} because it requires training an auxiliary network \cite{song2021maximum} only for the evaluation after the model training.
	
	\subsection{Correction on Likelihood Evaluation}
	
	A continuous diffusion model truncates the time horizon from $[0,T]$ to $[\epsilon,T]$ to avoid training instability \cite{kim2022soft}. In the model evaluation, this positive truncation could potentially be the primary source of poor evaluation (Figure 1-(c) of \citet{kim2022soft}), so we fix $\epsilon=10^{-5}$ as default in our training and evaluation. In the model evaluation, as the score network is untrained on $[0,\epsilon)$, we calculate NLL by the Right-Hand-Side (RHS) of Eq. \eqref{eq:nll_computation},
	\begin{align}\label{eq:nll_computation}
	\text{NLL}=\mathbb{E}_{\mathbf{x}_{0}}[-\log{p_{0}^{m}(\mathbf{x}_{0})}]\le\mathbb{E}_{\mathbf{x}_{0},\mathbf{x}_{\epsilon}}\bigg[-\log{p_{\epsilon}^{m}(\mathbf{x}_{\epsilon})}+\log{\frac{p_{\epsilon 0}^{m}(\mathbf{x}_{0}\vert\mathbf{x}_{\epsilon})}{p_{0\epsilon}(\mathbf{x}_{\epsilon}\vert\mathbf{x}_{0})}}\bigg].
	\end{align}
	Here, $p_{0}^{m}$ and $p_{\epsilon}^{m}$ are the model probability distributions at $t=0$ and $t=\epsilon$, respectively; and $p_{\epsilon 0}^{m}(\cdot\vert\mathbf{x}_{\epsilon})$ is the model's reconstruction probability of $\mathbf{x}_{0}$ given $\mathbf{x}_{\epsilon}$. RHS of Eq. \eqref{eq:nll_computation} is a generic formula to compute NLL in continuous diffusion models, including DDPM++, LSGM, and INDM. Previous continuous models \cite{song2020score, song2021maximum, vahdat2021score} have approximated $\mathbb{E}_{\mathbf{x}_{0}}[-\log{p_{0}^{m}(\mathbf{x}_{0})}]$ by $\mathbb{E}_{\mathbf{x}_{0}}[-\log{p_{\epsilon}^{m}(\mathbf{x}_{0})}]$. 
	
	There are two significant differences between our and the previous calculation: 1) the input of $p_{\epsilon}^{m}$ is replaced with $\mathbf{x}_{\epsilon}$ from $\mathbf{x}_{0}$ (Table \ref{tab:difference}); 2) the residual term of $\log{\frac{p_{\epsilon 0}^{m}(\mathbf{x}_{0}\vert\mathbf{x}_{\epsilon})}{p_{0\epsilon}(\mathbf{x}_{\epsilon}\vert\mathbf{x}_{0})}}$ is added. With this modification, our NLL differs from the previous NLL of $\mathbb{E}_{\mathbf{x}_{0}}[-\log{p_{\epsilon}^{m}(\mathbf{x}_{0})}]$ by about 0.03-0.06 in VPSDE, see Table \ref{tab:performance_cifar10}. We report both previous/corrected ways in Table \ref{tab:performance_cifar10} and report corrected NLL/NELBO as default; see Appendix \ref{appendix:correction_of_nll} for theoretical justification of our NLL/NELBO corrections.
	
	\subsection{Quantitative Results on Image Generation}
	
	\begin{wraptable}{r}{0.25\textwidth}
		\vskip -0.2in
		\caption{Effect of Pre-training.}
		\label{tab:pretraining_}
		\vskip -0.05in
		\tiny
		\centering
		\begin{tabular}{ccc}
			\toprule
			Model & NLL & FID \\\midrule
			DDPM++ & 3.03 & 6.70 \\
			\cc{15}INDM (w/ pre) & \cc{15}2.98 & \cc{15}6.01 \\
			\cc{15}INDM (w/o pre) & \cc{15}2.98 & \cc{15}8.49 \\
			\bottomrule
		\end{tabular}
		\vskip -0.1in
	\end{wraptable}
	\textbf{FID Boost with Pre-training} Training INDM from scratch improves NLL with the sacrifice of FID compared to DDPM++ in Table \ref{tab:pretraining_}. Therefore, we pre-train the score network by DDPM++ as default. This pre-training is intended to search the data nonlinearity near well-trained linear diffusions. Table \ref{tab:pretraining_} shows that training INDM after 500k of pre-training steps performs better than DDPM++ on both NLL and FID. Appendix \ref{appendix:pretraining} conducts the ablation study of pre-training steps.
	
	\begin{table*}[t]
		\vskip -0.05in
		\caption{Performance comparison to linear/nonlinear diffusion models on CIFAR-10. We report the performance of linear diffusions by training our PyTorch implementation based on \citet{song2020score, song2021maximum} with identical hyperparameters and score networks on both linear/nonlinear diffusions to quantify the effect of nonlinearity in a fair setting. Boldface numbers represent the best performance in a column.}
		\label{tab:performance_cifar10}
		\vskip -0.05in
		\tiny
		\centering
		\begin{adjustbox}{max width=\textwidth}
			\begin{tabular}{c|lc@{\hskip 0.3cm}r|cccccccc}
				\toprule
				\multirow{3}{*}{SDE} & \multirow{3}{*}{Model} & \multirow{3}{*}{\shortstack{Nonlinear Data\\Diffusion}} & \multirow{3}{*}{$\#$ Params} & \multicolumn{2}{c|}{NLL ($\downarrow$)} & \multicolumn{2}{c|}{NELBO ($\downarrow$)} & \multicolumn{2}{c|}{Gap ($\downarrow$)} & \multicolumn{2}{c}{FID ($\downarrow$)} \\
				&&&& \multirow{2}{*}{\shortstack{after\\correction}} & \multicolumn{1}{c|}{\multirow{2}{*}{\shortstack{before\\correction}}} &\multicolumn{1}{c}{w/ residual} & \multicolumn{1}{c|}{w/o residual} &\multicolumn{2}{c|}{(=NELBO-NLL)}& \multirow{2}{*}{ODE} & \multirow{2}{*}{PC} \\
				&&&& & \multicolumn{1}{c|}{} & \multicolumn{1}{c}{(after)} & \multicolumn{1}{c|}{(before)} & after & \multicolumn{1}{c|}{before} & &\\\midrule
				\multirow{2}{*}[-0pt]{VE} & NCSN++ (FID) & \xmark & 63M & 4.86 & 3.66 & 4.89 & 4.45 & 0.03 & 0.79 & - & 2.38 \\
				& \cc{15}INDM (FID) & \cc{15}\cmark & \cc{15}76M & \cc{15}3.22 & \cc{15}3.13 & \cc{15}3.28 & \cc{15}3.24 & \cc{15}0.06 & \cc{15}0.11 & \cc{15}- & \cc{15}\textbf{2.29} \\\midrule
				\multirow{5}{*}[-0pt]{VP} & DDPM++ (FID) & \xmark & 62M & 3.21 & 3.16 & 3.34 & 3.32 & 0.13 & 0.16 & 3.90 & 2.89 \\
				& \cc{15}INDM (FID) & \cc{15}\cmark & \cc{15}75M & \cc{15}3.17 & \cc{15}3.11 & \cc{15}3.23 & \cc{15}3.18 & \cc{15}0.06 & \cc{15}0.07 & \cc{15}\textbf{3.61} & \cc{15}2.90 \\\cmidrule(lr){2-12}
				& DDPM++ (NLL) & \xmark & 62M & 3.03 & 2.97 & 3.13 & 3.11 & 0.10 & 0.14 & 6.70 & 5.17 \\
				& \cc{15}INDM (NLL) & \cc{15}\cmark & \cc{15}75M & \cc{15}\textbf{2.98} & \cc{15}\textbf{2.95} & \cc{15}\textbf{2.98} & \cc{15}\textbf{2.97} & \cc{15}\textbf{0.00} & \cc{15}\textbf{0.02} & \cc{15}6.01 & \cc{15}5.30 \\
				\bottomrule
			\end{tabular}
		\end{adjustbox}
		\vskip -0.1in
	\end{table*}
	
	\begin{table}[t]
		\vskip -0.1in
		\begin{minipage}[c]{0.5\textwidth}
			\centering
			\caption{Performance comparison on CIFAR-10.}
			\label{tab:comparison}
			\tiny
			\centering
			\begin{tabular}{ccl@{\hskip 0.0cm}r@{\hskip 0.35cm}c@{\hskip 0.2cm}c}
				\toprule
				SDE&Type&Model&$\#$ Params& NLL & FID\\\midrule
				\multirow{3}{*}[-7pt]{Linear} && NCSN++ (FID) \cite{song2020score} & 108M & 4.85 & 2.20\\
				&& DDPM++ (FID) \cite{song2020score} & 108M & 3.19 & 2.64\\
				&& DDPM++ (NLL) \cite{song2020score} & 108M & 3.01 & 4.88\\
				&& VDM \cite{kingma2021variational} & - & \textbf{2.65}\tnote{\textdagger} & 7.41 \\
				&& CLD-SGM \cite{dockhorn2021score} & 108M & 3.31 & 2.25\\\midrule
				\multirow{10}{*}[-8pt]{Nonlinear} & SBP& SB-FBSDE \cite{chen2021likelihood} & 102M & 2.98 & 3.18\\\cmidrule(lr){2-6}
				&\multirow{5}{*}{\shortstack[c]{VAE\\-based}}& LSGM (FID) \cite{vahdat2021score} & 476M & 3.45 & \textbf{2.10} \\
				&& LSGM (NLL) \cite{vahdat2021score} & 269M & 2.97 & 6.15 \\
				&& LSGM (NLL) \cite{vahdat2021score} & 506M & 2.87 & 6.89 \\
				&& LSGM (balanced) \cite{vahdat2021score} & 109M & 2.96 & 4.60 \\
				&& LSGM (balanced) \cite{vahdat2021score} & 476M & 2.98 & 2.17 \\\cmidrule(lr){2-6}
				&\multirow{4}{*}[-2pt]{\shortstack[c]{Flow\\-based}} & DiffFlow (FID) \cite{zhang2021diffusion} & $\approx$36M & 3.04 & 14.14 \\\cmidrule(lr){3-6}
				&& \cc{15}INDM (FID) & \cc{15}118M & \cc{15}3.09 & \cc{15}2.28\\
				&& \cc{15}INDM (NLL) & \cc{15}121M & \cc{15}2.97 & \cc{15}4.79\\
				&& \cc{15}INDM (NLL) + ST & \cc{15}75M & \cc{15}3.01 & \cc{15}3.25 \\
				\bottomrule
			\end{tabular}
			\vskip -0.2in
		\end{minipage}
		\hfill
		\begin{minipage}[c]{0.44\textwidth}
			\centering
			\caption{Performance comparison on CelebA $64\times 64$.}
			\label{tab:performance_celeba}
			\centering
			\tiny
			\begin{tabular}{lcccc}
				\toprule
				Model & NLL & NELBO & Gap & FID\\\midrule
				UNCSN++ \cite{kim2022soft} & \textbf{1.93} & - & - & 1.92 \\
				DDGM \cite{nachmani2021non} & - & - & - & 2.92 \\
				Efficient-VDVAE \cite{hazami2022efficient} & - & \textbf{1.83} & - & - \\
				CR-NVAE \cite{sinha2021consistency} & - & 1.86 & - & - \\
				DenseFlow-74-10 \cite{grcic2021densely} & 1.99 & - & - & - \\
				StyleFormer \cite{park2022styleformer} & - & - & - & 3.66\\\midrule
				NCSN++ (VE, FID) & 3.41 & 3.42 & 0.01 & 3.95 \\
				\cc{15}INDM (VE, FID) & \cc{15}2.31 & \cc{15}2.33 & \cc{15}0.02 & \cc{15}2.54\\\cmidrule(lr){1-1}
				DDPM++ (VP, FID) & 2.14 & 2.21 & 0.07 & 2.32 \\
				\cc{15}INDM (VP, FID) & \cc{15}2.27 & \cc{15}2.31 & \cc{15}0.04 & \cc{15}\textbf{1.75} \\\cmidrule(lr){1-1}
				DDPM++ (VP, NLL) & 2.00 & 2.09 & 0.09 & 3.95 \\
				\cc{15}INDM (VP, NLL) & \cc{15}2.05 & \cc{15}2.05 & \cc{15}\textbf{0.00} & \cc{15}3.06 \\
				\bottomrule
			\end{tabular}
		\end{minipage}
		\vskip -0.15in
	\end{table}
	
	\textbf{Effect of Flow Training} Table \ref{tab:performance_cifar10} investigates how the flow training affects to performances, under the various linear diffusions and weighting functions. It compares the pre-trained NCSN++/DDPM++ with INDM, of which these pre-trained models initialize the score network of INDM. Experiments in Table \ref{tab:performance_cifar10} presents that INDM improves NELBO in any setting, and minimizes the variational gap to zero if we train the score network with the likelihood weighting.
	
	\textbf{SOTA on CelebA} Tables \ref{tab:comparison} and \ref{tab:performance_celeba} compare INDM to baseline models. With the emphasis on FID, LSGM is the State-Of-The-Art (SOTA) model on CIFAR-10, but INDM-118M (FID) is on par with LSGM-476M (FID). Moreover, we use Soft Truncation \cite{kim2022soft} to compare with LSGM (balanced). Soft Truncation softens the smallest diffusion time as a random variable in the training stage to boost sample performance by improving the score accuracy, particularly on large diffusion time. In the inference stage, Soft Truncation uses the fixed smallest diffusion time ($\epsilon$). INDM (NLL) + ST outperforms LSGM-109M (balanced) in terms of FID with comparable NLL. Also, INDM-121M (NLL) outperforms LSGM-269M (NLL) in FID with identical NLLs. We achieve SOTA FID of 1.75 on CelebA in Table \ref{tab:performance_celeba}. See Appendix \ref{appendix:quantitative_tables} for an extended comparison and Appendix \ref{appendix:samples} for samples.
	
	\subsection{Application Task: Dataset Interpolation}\label{sec:dataset_interpolation}
	
	\begin{figure*}[t]
		\vskip -0.05in
		\centering
		\includegraphics[width=0.9\linewidth]{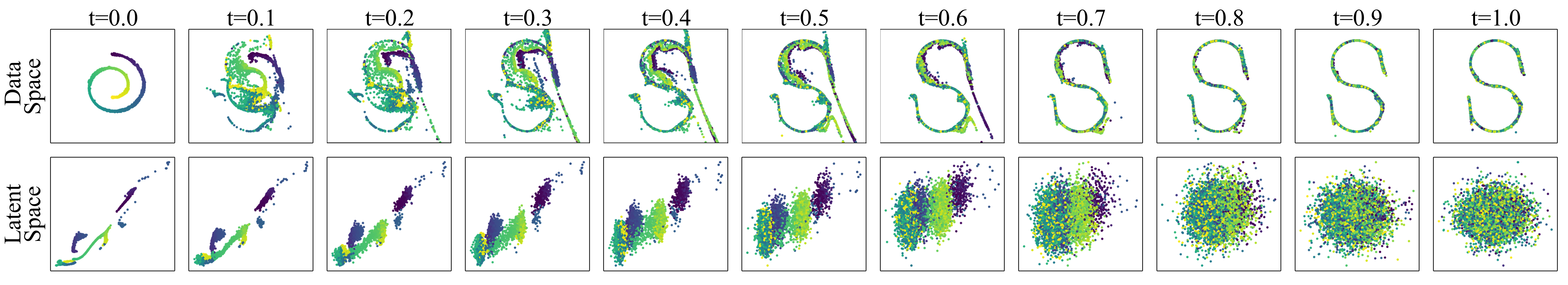}
		\vskip -0.05in
		\caption{INDM enables to learn a diffusion bridge between two distinctive data distributions.}
		\label{fig:translation_2d}
		\vskip -0.1in
	\end{figure*}
	
	The linear SDEs fixedly perturb data variables, so such SDEs should have the end distribution $p_{T}(\mathbf{x}_{T})$ as an easy-to-compute distribution. With the nonlinear extension, a complex diffusion process exists to transport $p_{r}^{(1)}$ to another arbitrary $p_{r}^{(2)}$. However, a common practice of diffusion models constrains only the starting variable by $\mathbf{x}_{0}^{\bm{\phi}}\sim p_{r}^{(1)}$, so the ending variable of $\mathbf{x}_{T}^{\bm{\phi}}$ is free to deviate from $p_{r}^{(2)}$. Previous works have tackled this data interpolation task by using a conditional diffusion model \cite{saharia2021palette} or a couple of jointly trained source-and-target diffusion models \cite{sasaki2021unit} on paired datasets. Among unconditional diffusion models using unpaired datasets, SBP \cite{de2021diffusion} is a natural approach for the task by imposing bi-constraints with $\mathcal{P}(p_{r},\pi)$ replaced by $\mathcal{P}(p_{r}^{(1)},p_{r}^{(2)})$.
	
	\begin{wrapfigure}{r}{0.5\textwidth}
		\vskip -0.15in
		\centering
		\includegraphics[width=\linewidth]{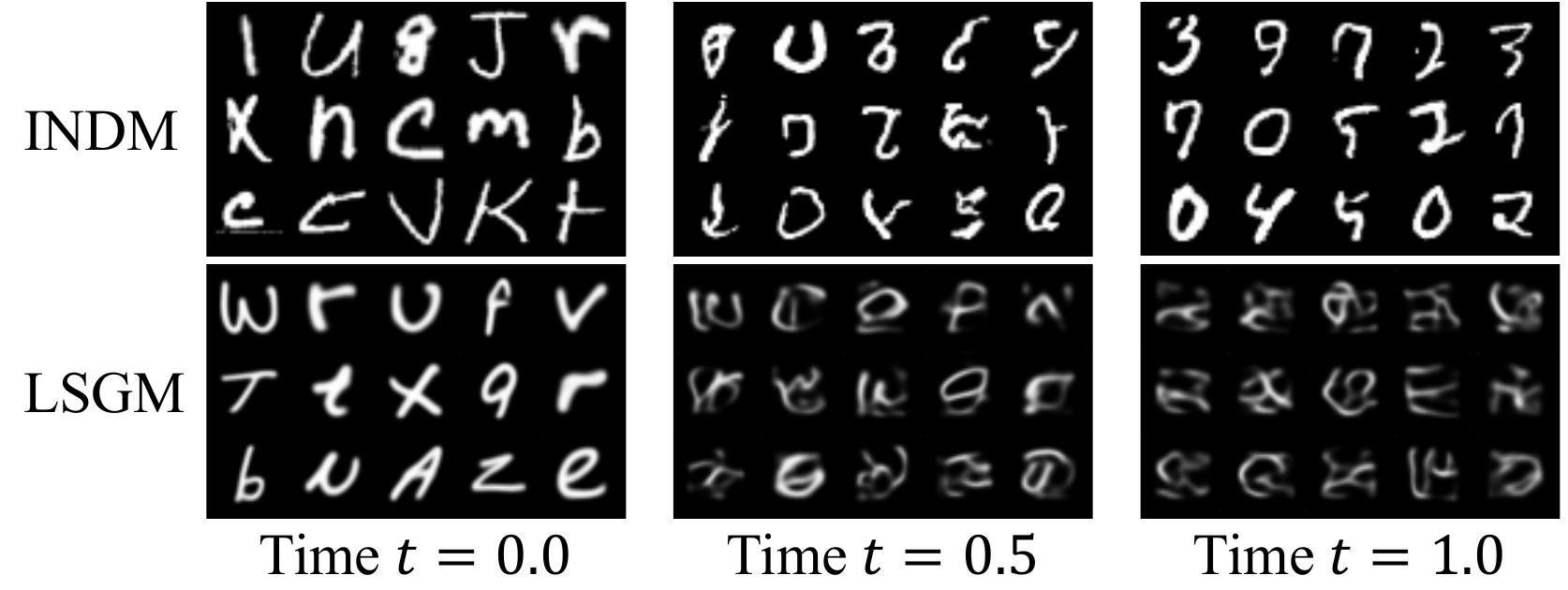}
		\vskip -0.05in
		\caption{Image-to-image translation from EMNIST letters dataset to MNIST digits dataset.}
		\label{fig:EMNIST_MNIST}
		\vskip -0.25in
	\end{wrapfigure}
	INDM can alternatively train the nonlinear diffusion from $p_{r}^{(1)}$ to $p_{r}^{(2)}$ with unpaired datasets. We train the score and flow networks with a loss
	\begin{align*}
	\underbrace{D_{KL}(\bm{\mu}_{\bm{\phi}}\Vert\bm{\nu}_{\bm{\phi},\bm{\theta}})}_{\text{INDM NELBO}}+\underbrace{D_{KL}(p_{r}^{(2)}\Vert p_{\bm{\phi}})}_{\text{Interpolation Loss}},
	\end{align*}
	where $p_{\bm{\phi}}$ is the probability distribution of $\mathbf{x}_{T}^{\bm{\phi}}$, which is calculated by a single feed-forward computation of the flow network, see Algorithm \ref{alg:interpolation} in Appendix \ref{appendix:interpolation__}. The additional interpolation loss forces the diffusion bridge $\{\mathbf{x}_{t}^{\bm{\phi}}\}_{t=0}^{T}$ to ahead towards $p_{r}^{(2)}$ by minimizing the KL divergence between $\mathbf{x}_{T}^{\bm{\phi}}\sim p_{\bm{\phi}}$ and $p_{r}^{(2)}$. As the destined variable of the bridge becomes $\mathbf{x}_{T}^{\bm{\phi}}=\mathbf{h}_{\bm{\phi}}^{-1}(\mathbf{z}_{T})\approx\mathbf{h}_{\bm{\phi}}^{-1}(\pi)$, the flow network is what constructs the interpolated bridge between a couple of data variables, see Figures \ref{fig:translation_2d} and \ref{fig:EMNIST_MNIST}. Particularly, Figure \ref{fig:EMNIST_MNIST} shows that LSGM fails to interpolate a letter to a digit, and we attribute this failure to the non-existence of a diffusion bridge in LSGM. Also, INDM models $p_{r}^{(1)}$ with $p_{\bm{\phi},\bm{\theta}}$ and $p_{r}^{(2)}$ with $p_{\bm{\phi}}$, so we could compute NLL of each dataset \textit{separately}: 1.10 BPD for MNIST and 1.56 BPD for EMNIST. In contrast, SBP cannot separately estimate densities on each dataset. We emphasize that no additional neural network is needed for the interpolation task with INDM.
	
	\section{Conclusion}\label{sec:conclusion}
	
	This paper expands the linear diffusion to trainable nonlinear diffusion by combining an invertible transformation and a diffusion model, where the nonlinear diffusion learns the forward diffusion out of variational family of inference measures. A limitation of INDM lies in the training/evaluation time. Potential risk from this work is the negative use of deep generative models, such as deepfake images.
	
	\section*{Acknowledgements}
	This research was supported by AI Technology Development for Commonsense Extraction, Reasoning, and Inference from Heterogeneous Data(IITP) funded by the Ministry of Science and ICT(2022-0-00077). Also, this work was supported by the National Research Foundation of Korea (NRF) grant funded by the Korea government(MSIT) (NRF-2019R1A5A1028324).
	
	\bibliography{references}

\begin{thebibliography}{79}
\providecommand{\natexlab}[1]{#1}
\providecommand{\url}[1]{\texttt{#1}}
\expandafter\ifx\csname urlstyle\endcsname\relax
  \providecommand{\doi}[1]{doi: #1}\else
  \providecommand{\doi}{doi: \begingroup \urlstyle{rm}\Url}\fi

\bibitem[Song et~al.(2020{\natexlab{a}})Song, Sohl-Dickstein, Kingma, Kumar,
  Ermon, and Poole]{song2020score}
Yang Song, Jascha Sohl-Dickstein, Diederik~P Kingma, Abhishek Kumar, Stefano
  Ermon, and Ben Poole.
\newblock Score-based generative modeling through stochastic differential
  equations.
\newblock In \emph{International Conference on Learning Representations},
  2020{\natexlab{a}}.

\bibitem[Dhariwal and Nichol(2021)]{dhariwal2021diffusion}
Prafulla Dhariwal and Alexander Nichol.
\newblock Diffusion models beat gans on image synthesis.
\newblock \emph{Advances in Neural Information Processing Systems}, 34, 2021.

\bibitem[Karras et~al.(2019)Karras, Laine, and Aila]{karras2019style}
Tero Karras, Samuli Laine, and Timo Aila.
\newblock A style-based generator architecture for generative adversarial
  networks.
\newblock In \emph{Proceedings of the IEEE/CVF Conference on Computer Vision
  and Pattern Recognition}, pages 4401--4410, 2019.

\bibitem[Grci{\'c} et~al.(2021)Grci{\'c}, Grubi{\v{s}}i{\'c}, and
  {\v{S}}egvi{\'c}]{grcic2021densely}
Matej Grci{\'c}, Ivan Grubi{\v{s}}i{\'c}, and Sini{\v{s}}a {\v{S}}egvi{\'c}.
\newblock Densely connected normalizing flows.
\newblock \emph{Advances in Neural Information Processing Systems}, 34, 2021.

\bibitem[Parmar et~al.(2018)Parmar, Vaswani, Uszkoreit, Kaiser, Shazeer, Ku,
  and Tran]{parmar2018image}
Niki Parmar, Ashish Vaswani, Jakob Uszkoreit, Lukasz Kaiser, Noam Shazeer,
  Alexander Ku, and Dustin Tran.
\newblock Image transformer.
\newblock In \emph{International Conference on Machine Learning}, pages
  4055--4064. PMLR, 2018.

\bibitem[Vahdat and Kautz(2020)]{vahdat2020nvae}
Arash Vahdat and Jan Kautz.
\newblock Nvae: A deep hierarchical variational autoencoder.
\newblock \emph{Advances in Neural Information Processing Systems},
  33:\penalty0 19667--19679, 2020.

\bibitem[Song and Ermon(2020)]{song2020improved}
Yang Song and Stefano Ermon.
\newblock Improved techniques for training score-based generative models.
\newblock \emph{Advances in neural information processing systems},
  33:\penalty0 12438--12448, 2020.

\bibitem[Ho et~al.(2020)Ho, Jain, and Abbeel]{ho2020denoising}
Jonathan Ho, Ajay Jain, and Pieter Abbeel.
\newblock Denoising diffusion probabilistic models.
\newblock \emph{Advances in Neural Information Processing Systems},
  33:\penalty0 6840--6851, 2020.

\bibitem[Vahdat et~al.(2021)Vahdat, Kreis, and Kautz]{vahdat2021score}
Arash Vahdat, Karsten Kreis, and Jan Kautz.
\newblock Score-based generative modeling in latent space.
\newblock \emph{Advances in Neural Information Processing Systems}, 34, 2021.

\bibitem[Anderson(1982)]{anderson1982reverse}
Brian~DO Anderson.
\newblock Reverse-time diffusion equation models.
\newblock \emph{Stochastic Processes and their Applications}, 12\penalty0
  (3):\penalty0 313--326, 1982.

\bibitem[Song et~al.(2021)Song, Durkan, Murray, and Ermon]{song2021maximum}
Yang Song, Conor Durkan, Iain Murray, and Stefano Ermon.
\newblock Maximum likelihood training of score-based diffusion models.
\newblock \emph{Advances in Neural Information Processing Systems}, 34, 2021.

\bibitem[Huang et~al.(2021)Huang, Lim, and Courville]{huang2021variational}
Chin-Wei Huang, Jae~Hyun Lim, and Aaron~C Courville.
\newblock A variational perspective on diffusion-based generative models and
  score matching.
\newblock \emph{Advances in Neural Information Processing Systems}, 34, 2021.

\bibitem[Zhang and Chen(2021)]{zhang2021diffusion}
Qinsheng Zhang and Yongxin Chen.
\newblock Diffusion normalizing flow.
\newblock \emph{Advances in Neural Information Processing Systems}, 34, 2021.

\bibitem[Vargas et~al.(2021)Vargas, Thodoroff, Lamacraft, and
  Lawrence]{vargas2021solving}
Francisco Vargas, Pierre Thodoroff, Austen Lamacraft, and Neil Lawrence.
\newblock Solving schr{\"o}dinger bridges via maximum likelihood.
\newblock \emph{Entropy}, 23\penalty0 (9):\penalty0 1134, 2021.

\bibitem[De~Bortoli et~al.(2021)De~Bortoli, Thornton, Heng, and
  Doucet]{de2021diffusion}
Valentin De~Bortoli, James Thornton, Jeremy Heng, and Arnaud Doucet.
\newblock Diffusion schr{\"o}dinger bridge with applications to score-based
  generative modeling.
\newblock \emph{Advances in Neural Information Processing Systems}, 34, 2021.

\bibitem[Chen et~al.(2022)Chen, Liu, and Theodorou]{chen2021likelihood}
Tianrong Chen, Guan-Horng Liu, and Evangelos Theodorou.
\newblock Likelihood training of schr\"odinger bridge using forward-backward
  {SDE}s theory.
\newblock In \emph{International Conference on Learning Representations}, 2022.

\bibitem[Oksendal(2013)]{oksendal2013stochastic}
Bernt Oksendal.
\newblock \emph{Stochastic differential equations: an introduction with
  applications}.
\newblock Springer Science \& Business Media, 2013.

\bibitem[Heusel et~al.(2017)Heusel, Ramsauer, Unterthiner, Nessler, and
  Hochreiter]{heusel2017gans}
Martin Heusel, Hubert Ramsauer, Thomas Unterthiner, Bernhard Nessler, and Sepp
  Hochreiter.
\newblock Gans trained by a two time-scale update rule converge to a local nash
  equilibrium.
\newblock \emph{Advances in neural information processing systems}, 30, 2017.

\bibitem[Krizhevsky et~al.(2009)Krizhevsky, Hinton,
  et~al.]{krizhevsky2009learning}
Alex Krizhevsky, Geoffrey Hinton, et~al.
\newblock Learning multiple layers of features from tiny images.
\newblock 2009.

\bibitem[Dockhorn et~al.(2022)Dockhorn, Vahdat, and Kreis]{dockhorn2021score}
Tim Dockhorn, Arash Vahdat, and Karsten Kreis.
\newblock Score-based generative modeling with critically-damped langevin
  diffusion.
\newblock In \emph{International Conference on Learning Representations}, 2022.

\bibitem[L{\'e}ger and Li(2021)]{leger2021hopf}
Flavien L{\'e}ger and Wuchen Li.
\newblock Hopf--cole transformation via generalized schr{\"o}dinger bridge
  problem.
\newblock \emph{Journal of Differential Equations}, 274:\penalty0 788--827,
  2021.

\bibitem[Guth et~al.(2022)Guth, Coste, De~Bortoli, and Mallat]{guth2022wavelet}
Florentin Guth, Simon Coste, Valentin De~Bortoli, and Stephane Mallat.
\newblock Wavelet score-based generative modeling.
\newblock \emph{Advances in Neural Information Processing Systems}, 35, 2022.

\bibitem[Chen et~al.(2019)Chen, Behrmann, Duvenaud, and
  Jacobsen]{chen2019residual}
Ricky~TQ Chen, Jens Behrmann, David~K Duvenaud, and J{\"o}rn-Henrik Jacobsen.
\newblock Residual flows for invertible generative modeling.
\newblock \emph{Advances in Neural Information Processing Systems}, 32, 2019.

\bibitem[Ma et~al.(2021)Ma, Kong, Zhang, and Hovy]{ma2020decoupling}
Xuezhe Ma, Xiang Kong, Shanghang Zhang, and Eduard~H Hovy.
\newblock Decoupling global and local representations via invertible generative
  flows.
\newblock In \emph{International Conference on Learning Representations}, 2021.

\bibitem[Ho et~al.(2019)Ho, Chen, Srinivas, Duan, and Abbeel]{ho2019flow++}
Jonathan Ho, Xi~Chen, Aravind Srinivas, Yan Duan, and Pieter Abbeel.
\newblock Flow++: Improving flow-based generative models with variational
  dequantization and architecture design.
\newblock In \emph{International Conference on Machine Learning}, pages
  2722--2730. PMLR, 2019.

\bibitem[Kim et~al.(2022)Kim, Shin, Song, Kang, and Moon]{kim2022soft}
Dongjun Kim, Seungjae Shin, Kyungwoo Song, Wanmo Kang, and Il-Chul Moon.
\newblock Soft truncation: A universal training technique of score-based
  diffusion model for high precision score estimation.
\newblock In \emph{International Conference on Machine Learning}, pages
  11201--11228. PMLR, 2022.

\bibitem[Dormand and Prince(1980)]{dormand1980family}
John~R Dormand and Peter~J Prince.
\newblock A family of embedded runge-kutta formulae.
\newblock \emph{Journal of computational and applied mathematics}, 6\penalty0
  (1):\penalty0 19--26, 1980.

\bibitem[Kingma et~al.(2021)Kingma, Salimans, Poole, and
  Ho]{kingma2021variational}
Diederik~P Kingma, Tim Salimans, Ben Poole, and Jonathan Ho.
\newblock Variational diffusion models.
\newblock In \emph{Advances in Neural Information Processing Systems}, 2021.

\bibitem[Nachmani et~al.(2021)Nachmani, Roman, and Wolf]{nachmani2021non}
Eliya Nachmani, Robin~San Roman, and Lior Wolf.
\newblock Non gaussian denoising diffusion models.
\newblock \emph{arXiv preprint arXiv:2106.07582}, 2021.

\bibitem[Hazami et~al.(2022)Hazami, Mama, and
  Thurairatnam]{hazami2022efficient}
Louay Hazami, Rayhane Mama, and Ragavan Thurairatnam.
\newblock Efficient-vdvae: Less is more.
\newblock \emph{arXiv preprint arXiv:2203.13751}, 2022.

\bibitem[Sinha and Dieng(2021)]{sinha2021consistency}
Samarth Sinha and Adji~Bousso Dieng.
\newblock Consistency regularization for variational auto-encoders.
\newblock \emph{Advances in Neural Information Processing Systems}, 34, 2021.

\bibitem[Saharia et~al.(2021)Saharia, Chan, Chang, Lee, Ho, Salimans, Fleet,
  and Norouzi]{saharia2021palette}
Chitwan Saharia, William Chan, Huiwen Chang, Chris~A. Lee, Jonathan Ho, Tim
  Salimans, David~J. Fleet, and Mohammad Norouzi.
\newblock Palette: Image-to-image diffusion models.
\newblock In \emph{NeurIPS 2021 Workshop on Deep Generative Models and
  Downstream Applications}, 2021.

\bibitem[Sasaki et~al.(2021)Sasaki, Willcocks, and Breckon]{sasaki2021unit}
Hiroshi Sasaki, Chris~G Willcocks, and Toby~P Breckon.
\newblock Unit-ddpm: Unpaired image translation with denoising diffusion
  probabilistic models.
\newblock \emph{arXiv preprint arXiv:2104.05358}, 2021.

\bibitem[S{\"a}rkk{\"a} and Solin(2019)]{sarkka2019applied}
Simo S{\"a}rkk{\"a} and Arno Solin.
\newblock \emph{Applied stochastic differential equations}, volume~10.
\newblock Cambridge University Press, 2019.

\bibitem[Duchi(2016)]{duchi2016lecture}
John Duchi.
\newblock Lecture notes for statistics 311/electrical engineering 377.
\newblock \emph{URL: https://stanford. edu/class/stats311/Lectures/full\_notes.
  pdf. Last visited on}, 2:\penalty0 23, 2016.

\bibitem[Avron and Toledo(2011)]{avron2011randomized}
Haim Avron and Sivan Toledo.
\newblock Randomized algorithms for estimating the trace of an implicit
  symmetric positive semi-definite matrix.
\newblock \emph{Journal of the ACM (JACM)}, 58\penalty0 (2):\penalty0 1--34,
  2011.

\bibitem[Rudin et~al.(1964)]{rudin1964principles}
Walter Rudin et~al.
\newblock \emph{Principles of mathematical analysis}, volume~3.
\newblock McGraw-hill New York, 1964.

\bibitem[Rezende and Mohamed(2015)]{rezende2015variational}
Danilo Rezende and Shakir Mohamed.
\newblock Variational inference with normalizing flows.
\newblock In \emph{International conference on machine learning}, pages
  1530--1538. PMLR, 2015.

\bibitem[Burda et~al.(2020)Burda, Grosse, and
  Salakhutdinov]{burda2015importance}
Yuri Burda, Roger Grosse, and Ruslan Salakhutdinov.
\newblock Importance weighted autoencoders.
\newblock 2020.

\bibitem[Gl{\"o}tzl and Richters(2020)]{glotzl2020helmholtz}
Erhard Gl{\"o}tzl and Oliver Richters.
\newblock Helmholtz decomposition and rotation potentials in n-dimensional
  cartesian coordinates.
\newblock \emph{arXiv preprint arXiv:2012.13157}, 2020.

\bibitem[Song et~al.(2020{\natexlab{b}})Song, Garg, Shi, and
  Ermon]{song2020sliced}
Yang Song, Sahaj Garg, Jiaxin Shi, and Stefano Ermon.
\newblock Sliced score matching: A scalable approach to density and score
  estimation.
\newblock In \emph{Uncertainty in Artificial Intelligence}, pages 574--584.
  PMLR, 2020{\natexlab{b}}.

\bibitem[Hutchinson(1989)]{hutchinson1989stochastic}
Michael~F Hutchinson.
\newblock A stochastic estimator of the trace of the influence matrix for
  laplacian smoothing splines.
\newblock \emph{Communications in Statistics-Simulation and Computation},
  18\penalty0 (3):\penalty0 1059--1076, 1989.

\bibitem[Hunter(2007)]{Hunter:2007}
J.~D. Hunter.
\newblock Matplotlib: A 2d graphics environment.
\newblock \emph{Computing in Science \& Engineering}, 9\penalty0 (3):\penalty0
  90--95, 2007.
\newblock \doi{10.1109/MCSE.2007.55}.

\bibitem[Blum et~al.(2020)Blum, Hopcroft, and Kannan]{blum2020foundations}
Avrim Blum, John Hopcroft, and Ravindran Kannan.
\newblock \emph{Foundations of data science}.
\newblock Cambridge University Press, 2020.

\bibitem[Flamary et~al.(2021)Flamary, Courty, Gramfort, Alaya, Boisbunon,
  Chambon, Chapel, Corenflos, Fatras, Fournier, Gautheron, Gayraud, Janati,
  Rakotomamonjy, Redko, Rolet, Schutz, Seguy, Sutherland, Tavenard, Tong, and
  Vayer]{flamary2021pot}
R{\'e}mi Flamary, Nicolas Courty, Alexandre Gramfort, Mokhtar~Z. Alaya,
  Aur{\'e}lie Boisbunon, Stanislas Chambon, Laetitia Chapel, Adrien Corenflos,
  Kilian Fatras, Nemo Fournier, L{\'e}o Gautheron, Nathalie~T.H. Gayraud,
  Hicham Janati, Alain Rakotomamonjy, Ievgen Redko, Antoine Rolet, Antony
  Schutz, Vivien Seguy, Danica~J. Sutherland, Romain Tavenard, Alexander Tong,
  and Titouan Vayer.
\newblock Pot: Python optimal transport.
\newblock \emph{Journal of Machine Learning Research}, 22\penalty0
  (78):\penalty0 1--8, 2021.

\bibitem[Khrulkov and Oseledets(2022)]{khrulkov2022understanding}
Valentin Khrulkov and Ivan Oseledets.
\newblock Understanding ddpm latent codes through optimal transport.
\newblock \emph{arXiv preprint arXiv:2202.07477}, 2022.

\bibitem[Bredon(2013)]{bredon2013topology}
Glen~E Bredon.
\newblock \emph{Topology and geometry}, volume 139.
\newblock Springer Science \& Business Media, 2013.

\bibitem[Higham(2001)]{higham2001algorithmic}
Desmond~J Higham.
\newblock An algorithmic introduction to numerical simulation of stochastic
  differential equations.
\newblock \emph{SIAM review}, 43\penalty0 (3):\penalty0 525--546, 2001.

\bibitem[Kingma and Welling(2014)]{kingma2013auto}
Diederik~P. Kingma and Max Welling.
\newblock Auto-encoding variational bayes.
\newblock In Yoshua Bengio and Yann LeCun, editors, \emph{2nd International
  Conference on Learning Representations, {ICLR} 2014, Banff, AB, Canada, April
  14-16, 2014, Conference Track Proceedings}, 2014.

\bibitem[Chen et~al.(2016)Chen, Georgiou, and Pavon]{chen2016relation}
Yongxin Chen, Tryphon~T Georgiou, and Michele Pavon.
\newblock On the relation between optimal transport and schr{\"o}dinger
  bridges: A stochastic control viewpoint.
\newblock \emph{Journal of Optimization Theory and Applications}, 169\penalty0
  (2):\penalty0 671--691, 2016.

\bibitem[Ruschendorf(1995)]{ruschendorf1995convergence}
Ludger Ruschendorf.
\newblock Convergence of the iterative proportional fitting procedure.
\newblock \emph{The Annals of Statistics}, pages 1160--1174, 1995.

\bibitem[Tancik et~al.(2020)Tancik, Srinivasan, Mildenhall, Fridovich-Keil,
  Raghavan, Singhal, Ramamoorthi, Barron, and Ng]{tancik2020fourier}
Matthew Tancik, Pratul Srinivasan, Ben Mildenhall, Sara Fridovich-Keil, Nithin
  Raghavan, Utkarsh Singhal, Ravi Ramamoorthi, Jonathan Barron, and Ren Ng.
\newblock Fourier features let networks learn high frequency functions in low
  dimensional domains.
\newblock \emph{Advances in Neural Information Processing Systems},
  33:\penalty0 7537--7547, 2020.

\bibitem[Vaswani et~al.(2017)Vaswani, Shazeer, Parmar, Uszkoreit, Jones, Gomez,
  Kaiser, and Polosukhin]{vaswani2017attention}
Ashish Vaswani, Noam Shazeer, Niki Parmar, Jakob Uszkoreit, Llion Jones,
  Aidan~N Gomez, {\L}ukasz Kaiser, and Illia Polosukhin.
\newblock Attention is all you need.
\newblock In \emph{Advances in neural information processing systems}, pages
  5998--6008, 2017.

\bibitem[Ronneberger et~al.(2015)Ronneberger, Fischer, and
  Brox]{ronneberger2015u}
Olaf Ronneberger, Philipp Fischer, and Thomas Brox.
\newblock U-net: Convolutional networks for biomedical image segmentation.
\newblock In \emph{International Conference on Medical image computing and
  computer-assisted intervention}, pages 234--241. Springer, 2015.

\bibitem[Kingma and Dhariwal(2018)]{kingma2018glow}
Durk~P Kingma and Prafulla Dhariwal.
\newblock Glow: Generative flow with invertible 1x1 convolutions.
\newblock \emph{Advances in neural information processing systems}, 31, 2018.

\bibitem[Ioffe and Szegedy(2015)]{ioffe2015batch}
Sergey Ioffe and Christian Szegedy.
\newblock Batch normalization: Accelerating deep network training by reducing
  internal covariate shift.
\newblock In \emph{International conference on machine learning}, pages
  448--456. PMLR, 2015.

\bibitem[Lu et~al.(2021)Lu, Chen, Li, Wang, and Zhu]{lu2021implicit}
Cheng Lu, Jianfei Chen, Chongxuan Li, Qiuhao Wang, and Jun Zhu.
\newblock Implicit normalizing flows.
\newblock In \emph{International Conference on Learning Representations}, 2021.

\bibitem[Ramachandran et~al.(2017)Ramachandran, Zoph, and
  Le]{ramachandran2017searching}
Prajit Ramachandran, Barret Zoph, and Quoc~V Le.
\newblock Searching for activation functions.
\newblock \emph{arXiv preprint arXiv:1710.05941}, 2017.

\bibitem[Nichol and Dhariwal(2021)]{nichol2021improved}
Alexander~Quinn Nichol and Prafulla Dhariwal.
\newblock Improved denoising diffusion probabilistic models.
\newblock In \emph{International Conference on Machine Learning}, pages
  8162--8171. PMLR, 2021.

\bibitem[Shampine(1986)]{shampine1986some}
Lawrence~F Shampine.
\newblock Some practical runge-kutta formulas.
\newblock \emph{Mathematics of computation}, 46\penalty0 (173):\penalty0
  135--150, 1986.

\bibitem[Jolicoeur-Martineau et~al.(2020)Jolicoeur-Martineau,
  Pich{\'e}-Taillefer, Mitliagkas, and des Combes]{jolicoeur2020adversarial}
Alexia Jolicoeur-Martineau, R{\'e}mi Pich{\'e}-Taillefer, Ioannis Mitliagkas,
  and Remi~Tachet des Combes.
\newblock Adversarial score matching and improved sampling for image
  generation.
\newblock In \emph{International Conference on Learning Representations}, 2020.

\bibitem[Song and Ermon(2019)]{song2019generative}
Yang Song and Stefano Ermon.
\newblock Generative modeling by estimating gradients of the data distribution.
\newblock \emph{Advances in Neural Information Processing Systems}, 32, 2019.

\bibitem[Parmar et~al.(2022)Parmar, Zhang, and Zhu]{parmar2022aliased}
Gaurav Parmar, Richard Zhang, and Jun-Yan Zhu.
\newblock On aliased resizing and surprising subtleties in gan evaluation,
  2022.

\bibitem[Benamou and Brenier(2000)]{benamou2000computational}
Jean-David Benamou and Yann Brenier.
\newblock A computational fluid mechanics solution to the monge-kantorovich
  mass transfer problem.
\newblock \emph{Numerische Mathematik}, 84\penalty0 (3):\penalty0 375--393,
  2000.

\bibitem[Villani(2009)]{villani2009optimal}
C{\'e}dric Villani.
\newblock \emph{Optimal transport: old and new}, volume 338.
\newblock Springer, 2009.

\bibitem[Sriperumbudur et~al.(2010)Sriperumbudur, Gretton, Fukumizu,
  Sch{\"o}lkopf, and Lanckriet]{sriperumbudur2010hilbert}
Bharath~K Sriperumbudur, Arthur Gretton, Kenji Fukumizu, Bernhard
  Sch{\"o}lkopf, and Gert~RG Lanckriet.
\newblock Hilbert space embeddings and metrics on probability measures.
\newblock \emph{The Journal of Machine Learning Research}, 11:\penalty0
  1517--1561, 2010.

\bibitem[Evans(1998)]{evans1998partial}
Lawrence~C Evans.
\newblock Partial differential equations.
\newblock \emph{Graduate studies in mathematics}, 19\penalty0 (2), 1998.

\bibitem[De~Bortoli(2022)]{de2022convergence}
Valentin De~Bortoli.
\newblock Convergence of denoising diffusion models under the manifold
  hypothesis.
\newblock \emph{arXiv preprint arXiv:2208.05314}, 2022.

\bibitem[Karras et~al.(2020)Karras, Aittala, Hellsten, Laine, Lehtinen, and
  Aila]{karras2020training}
Tero Karras, Miika Aittala, Janne Hellsten, Samuli Laine, Jaakko Lehtinen, and
  Timo Aila.
\newblock Training generative adversarial networks with limited data.
\newblock In H.~Larochelle, M.~Ranzato, R.~Hadsell, M.~F. Balcan, and H.~Lin,
  editors, \emph{Advances in Neural Information Processing Systems}, volume~33,
  pages 12104--12114. Curran Associates, Inc., 2020.

\bibitem[Park and Kim(2022)]{park2022styleformer}
Jeeseung Park and Younggeun Kim.
\newblock Styleformer: Transformer based generative adversarial networks with
  style vector.
\newblock In \emph{Proceedings of the IEEE/CVF Conference on Computer Vision
  and Pattern Recognition}, pages 8983--8992, 2022.

\bibitem[Ansari et~al.(2020)Ansari, Ang, and Soh]{ansari2020refining}
Abdul~Fatir Ansari, Ming~Liang Ang, and Harold Soh.
\newblock Refining deep generative models via discriminator gradient flow.
\newblock In \emph{International Conference on Learning Representations}, 2020.

\bibitem[Jiang et~al.(2021)Jiang, Chang, and Wang]{jiang2021transgan}
Yifan Jiang, Shiyu Chang, and Zhangyang Wang.
\newblock Transgan: Two pure transformers can make one strong gan, and that can
  scale up.
\newblock \emph{Advances in Neural Information Processing Systems}, 34, 2021.

\bibitem[Van~Oord et~al.(2016)Van~Oord, Kalchbrenner, and
  Kavukcuoglu]{van2016pixel}
Aaron Van~Oord, Nal Kalchbrenner, and Koray Kavukcuoglu.
\newblock Pixel recurrent neural networks.
\newblock In \emph{International Conference on Machine Learning}, pages
  1747--1756. PMLR, 2016.

\bibitem[Child et~al.(2019)Child, Gray, Radford, and
  Sutskever]{child2019generating}
Rewon Child, Scott Gray, Alec Radford, and Ilya Sutskever.
\newblock Generating long sequences with sparse transformers.
\newblock \emph{arXiv preprint arXiv:1904.10509}, 2019.

\bibitem[Chen et~al.(2020)Chen, Lu, Chenli, Zhu, and Tian]{chen2020vflow}
Jianfei Chen, Cheng Lu, Biqi Chenli, Jun Zhu, and Tian Tian.
\newblock Vflow: More expressive generative flows with variational data
  augmentation.
\newblock In \emph{International Conference on Machine Learning}, pages
  1660--1669. PMLR, 2020.

\bibitem[Child(2020)]{child2020very}
Rewon Child.
\newblock Very deep vaes generalize autoregressive models and can outperform
  them on images.
\newblock In \emph{International Conference on Learning Representations}, 2020.

\bibitem[Razavi et~al.(2018)Razavi, van~den Oord, Poole, and
  Vinyals]{razavi2018preventing}
Ali Razavi, Aaron van~den Oord, Ben Poole, and Oriol Vinyals.
\newblock Preventing posterior collapse with delta-vaes.
\newblock In \emph{International Conference on Learning Representations}, 2018.

\bibitem[Parmar et~al.(2021)Parmar, Li, Lee, and Tu]{parmar2021dual}
Gaurav Parmar, Dacheng Li, Kwonjoon Lee, and Zhuowen Tu.
\newblock Dual contradistinctive generative autoencoder.
\newblock In \emph{Proceedings of the IEEE/CVF Conference on Computer Vision
  and Pattern Recognition}, pages 823--832, 2021.

\bibitem[Song et~al.(2020{\natexlab{c}})Song, Meng, and
  Ermon]{song2020denoising}
Jiaming Song, Chenlin Meng, and Stefano Ermon.
\newblock Denoising diffusion implicit models.
\newblock In \emph{International Conference on Learning Representations},
  2020{\natexlab{c}}.

\end{thebibliography}

	\newpage
	\appendix
	
	\tableofcontents
	\newpage
	\parttoc
	
	\section{Derivations}\label{appendix:data_diffusion_derivation}

	\subsection{Derivation of Variational Bound for Nonlinear Diffusion}\label{appendix:derivation_of_variational_bound}
	
	The variational bound derived in \citet{song2021maximum} is only applicable when $\mathbf{G}(\mathbf{x}_{t},t)$ is reduced to $g(t)\mathbf{I}$. This section, therefore, derives the variational bound of a general diffusion SDE of $\diff\mathbf{x}_{t}=\mathbf{f}(\mathbf{x}_{t},t)\diff t+\mathbf{G}(\mathbf{x}_{t},t)\diff\mathbf{w}_{t}$, and we analyze why learning $\mathbf{f}$ and $\mathbf{G}$ is infeasible if 1) the transition probability of $p_{0t}(\mathbf{x}_{t}\vert\mathbf{x}_{0})$ is intractable and 2) $\mathbf{G}$ is anisotropic by $\mathbf{x}_{t}$. 
	
	From \citet{anderson1982reverse}, the corresponding reverse SDE of $\diff\mathbf{x}_{t}=\mathbf{f}(\mathbf{x}_{t},t)\diff t+\mathbf{G}(\mathbf{x}_{t},t)\diff\mathbf{w}_{t}$ is
	\begin{align}\label{eq:general_reverse}
	\diff\mathbf{x}_{t}=\big[\mathbf{f}(\mathbf{x}_{t},t)-\text{div}(\mathbf{G}\mathbf{G}^{T})-\mathbf{G}\mathbf{G}^{T}\nabla\log{p_{t}(\mathbf{x}_{t})}\big]\diff \bar{t}+\mathbf{G}(\mathbf{x}_{t},t)\diff\mathbf{\bar{w}}_{t},
	\end{align}
	and the generative SDE becomes
	\begin{align}
	\diff\mathbf{x}_{t}=\big[\mathbf{f}(\mathbf{x}_{t},t)-\text{div}(\mathbf{G}\mathbf{G}^{T})-\mathbf{G}\mathbf{G}^{T}\mathbf{s}(\mathbf{x}_{t},t)\big]\diff \bar{t}+\mathbf{G}(\mathbf{x}_{t},t)\diff\mathbf{\bar{w}}_{t}.\label{eq:general_generative}
	\end{align}	
	Then, from the Girsanov theorem \cite{sarkka2019applied} and the martingale property \cite{oksendal2013stochastic}, using the disintegration property of the KL divergence, we have
	\begin{align}
	\begin{split}\label{eq:generalized_diffusion_loss}
	&D_{KL}(\bm{\mu}\Vert\bm{\nu})=D_{KL}(p_{T}(\mathbf{x}_{T})\Vert \pi(\mathbf{x}_{T}))\\
	&\quad+\frac{1}{2}\int_{0}^{T}\mathbb{E}_{\mathbf{x}_{t}}\Big[\big(\mathbf{s}(\mathbf{x}_{t},t)-\nabla\log{p_{t}(\mathbf{x}_{t})}\big)^{T}\mathbf{G}\mathbf{G}^{T}\big(\mathbf{s}(\mathbf{x}_{t},t)-\nabla\log{p_{t}(\mathbf{x}_{t})}\big)\Big]\diff t,
	\end{split}
	\end{align}
	where $\bm{\mu}$ is the path measure of Eq. \eqref{eq:general_reverse} and $\bm{\nu}$ is the path measure of Eq. \eqref{eq:general_generative}. Therefore, from the data processing inequality \cite{duchi2016lecture}, we get
	\begin{eqnarray*}
		\lefteqn{D_{KL}(p_{r}\Vert p)\le D_{KL}(\bm{\mu}\Vert\bm{\nu})=D_{KL}(p_{T}(\mathbf{x}_{T})\Vert \pi(\mathbf{x}_{T}))}&\\
		&&+\frac{1}{2}\int_{0}^{T}\mathbb{E}_{\mathbf{x}_{t}}\Big[\big(\mathbf{s}(\mathbf{x}_{t},t)-\nabla\log{p_{t}(\mathbf{x}_{t})}\big)^{T}\mathbf{G}\mathbf{G}^{T}\big(\mathbf{s}(\mathbf{x}_{t},t)-\nabla\log{p_{t}(\mathbf{x}_{t})}\big)\Big]\diff t,
	\end{eqnarray*}
	where $p$ is the generative distribution at $t=0$.
	
	Now, from the Fokker-Planck equation, the density function satisfies
	\begin{align*}
	\frac{\partial p_{t}}{\partial t}=-\sum_{j}\frac{\partial}{\partial x_{t,j}}\Big[f_{j}(\mathbf{x}_{t},t)p_{t}(\mathbf{x}_{t})-\sum_{i}\frac{\partial}{\partial x_{t,j}}(H_{ij}(\mathbf{x}_{t},t)p_{t}(\mathbf{x}_{t}))\Big],
	\end{align*}
	where $\mathbf{H}(\mathbf{x}_{t},t)=\frac{1}{2}\mathbf{G}(\mathbf{x}_{t},t)\mathbf{G}(\mathbf{x}_{t},t)^{T}$. Then, analogous to Theorem 4 of \citet{song2021maximum}, the entropy becomes
	\begin{eqnarray*}
		\lefteqn{\mathcal{H}(p_{r})-\mathcal{H}(p_{T})=-\int_{0}^{T}\frac{\partial}{\partial t}\mathcal{H}(p_{t})\diff t}&\\
		&&=\int_{0}^{T}\int\frac{\partial p_{t}}{\partial t}\log{p_{t}(\mathbf{x}_{t})}\diff \mathbf{x}_{t}\diff t\\
		&&=-\int_{0}^{T}\int\sum_{j}\frac{\partial}{\partial x_{t,j}}\Big[f_{j}(\mathbf{x}_{t},t)p_{t}(\mathbf{x}_{t})-\sum_{i}\frac{\partial}{\partial x_{t,i}}(H_{ij}(\mathbf{x}_{t},t)p_{t}(\mathbf{x}_{t}))\Big]\log{p_{t}(\mathbf{x}_{t})}\diff\mathbf{x}_{t}\diff t\\
		&&=\int_{0}^{T}\int \sum_{j}\Big[f_{j}(\mathbf{x}_{t},t)p_{t}(\mathbf{x}_{t})-\sum_{i}\frac{\partial}{\partial x_{t,i}}(H_{ij}(\mathbf{x}_{t},t)p_{t}(\mathbf{x}_{t}))\Big]\frac{\partial \log{p_{t}(\mathbf{x}_{t})}}{\partial x_{t,j}}\diff\mathbf{x}_{t}\diff t\\
		&&=\int_{0}^{T}\int p_{t}(\mathbf{x}_{t})\sum_{j} f_{j}(\mathbf{x}_{t},t)\frac{\partial\log{p_{t}(\mathbf{x}_{t})}}{\partial x_{t,j}}\diff\mathbf{x}_{t}\diff t\\
		&&\quad-\int_{0}^{T}\int \sum_{j}\sum_{i}\bigg(\frac{\partial H_{ij}}{\partial x_{t,i}}p_{t}+H_{ij}\frac{\partial p_{t}}{\partial x_{t,i}}\bigg)\frac{\partial\log{p_{t}}}{\partial x_{t,j}}\diff\mathbf{x}_{t}\diff t\\
		&&=-\int_{0}^{T}\mathbb{E}_{\mathbf{x}_{t}}[\text{div}(\mathbf{f}(\mathbf{x}_{t},t))]\diff t\\
		&&\quad-\int_{0}^{T}\mathbb{E}_{\mathbf{x}_{t}}[\text{div}(\mathbf{H}(\mathbf{x}_{t},t))^{T}\nabla\log{p_{t}(\mathbf{x}_{t})}]+\mathbb{E}_{\mathbf{x}_{t}}[\nabla\log{p_{t}(\mathbf{x}_{t})}^{T}\mathbf{H}(\mathbf{x}_{t},t)\nabla\log{p_{t}(\mathbf{x}_{t})}]\diff t\\
		&&=-\frac{1}{2}\int_{0}^{T}\mathbb{E}_{\mathbf{x}_{t}}\Big[ 2\text{div}(\mathbf{f}(\mathbf{x}_{t},t))+\text{div}(\mathbf{G}(\mathbf{x}_{t},t)\mathbf{G}(\mathbf{x}_{t},t)^{T})^{T}\nabla\log{p_{t}(\mathbf{x}_{t})}\\
		&&\quad\quad\quad\quad\quad\quad\quad\quad+\nabla\log{p_{t}(\mathbf{x}_{t})}^{T}\mathbf{G}(\mathbf{x}_{t},t)\mathbf{G}(\mathbf{x}_{t},t)^{T}\nabla\log{p_{t}(\mathbf{x}_{t})} \Big]\diff t.
	\end{eqnarray*}
	Therefore, the variational bound of the model log-likelihood is derived by
	\begin{eqnarray*}
		\lefteqn{\mathbb{E}_{p_{r}(\mathbf{x}_{0})}\big[-\log{p(\mathbf{x}_{0})}\big]=D_{KL}(p_{r}\Vert p)+\mathcal{H}(p_{r})\le D_{KL}(\bm{\mu}\Vert\bm{\nu})+\mathcal{H}(p_{r})}&\\
		&&=\frac{1}{2}\int_{0}^{T}\mathbb{E}_{\mathbf{x}_{t}}\big[(\nabla\log{p_{t}(\mathbf{x}_{t})}-\mathbf{s}(\mathbf{x}_{t},t))^{T}\mathbf{G}\mathbf{G}^{T}(\nabla\log{p_{t}(\mathbf{x}_{t})}-\mathbf{s}(\mathbf{x}_{t},t))\big]\diff t\\
		&&\quad+\mathbb{E}_{\mathbf{x}_{T}}\big[-\log{\pi(\mathbf{x}_{T})}\big]+\mathcal{H}(p_{r})-\mathcal{H}(p_{T})\\
		&&=\frac{1}{2}\int_{0}^{T}\mathbb{E}_{\mathbf{x}_{t}}\Big[\big(\mathbf{s}(\mathbf{x}_{t},t)-\nabla\log{p_{t}(\mathbf{x}_{t})}\big)^{T}\mathbf{G}\mathbf{G}^{T}\big(\mathbf{s}(\mathbf{x}_{t},t)-\nabla\log{p_{t}(\mathbf{x}_{t})}\big)\\
		&&\quad-\nabla\log{p_{t}(\mathbf{x}_{t})}^{T}\mathbf{G}\mathbf{G}^{T}\nabla\log{p_{t}(\mathbf{x}_{t})}-\text{div}(\mathbf{G}\mathbf{G}^{T})^{T}\nabla\log{p_{t}(\mathbf{x}_{t})}-2\text{div}(\mathbf{f}(\mathbf{x}_{t},t))\Big]\diff t\\
		&&\quad+\mathbb{E}_{\mathbf{x}_{T}}\big[-\log{\pi(\mathbf{x}_{T})}\big].
	\end{eqnarray*}
	Using the integration by parts, this variational bound transforms to
	\begin{align*}
	\mathbb{E}_{p_{r}(\mathbf{x}_{0})}\big[-\log{p_{\bm{\theta}}(\mathbf{x}_{0})}\big]\le&\frac{1}{2}\int_{0}^{T}\mathbb{E}_{\mathbf{x}_{t}}\Big[\mathbf{s}^{T}\mathbf{G}\mathbf{G}^{T}\mathbf{s}+2\text{div}(\mathbf{G}\mathbf{G}^{T}\mathbf{s})\\
	&-\text{div}(\mathbf{G}\mathbf{G}^{T})\nabla\log{p_{t}}-2\text{div}(\mathbf{f})\Big]\diff t+\mathbb{E}_{\mathbf{x}_{T}}\big[-\log{\pi(\mathbf{x}_{T})}\big]
	\end{align*}
	Also, this variational bound is equivalently formulated as
	\begin{align*}
	&\mathbb{E}_{p_{r}(\mathbf{x}_{0})}\big[-\log{p(\mathbf{x}_{0})}\big]\\
	&\quad\le \frac{1}{2}\int_{0}^{T}\mathbb{E}_{\mathbf{x}_{0},\mathbf{x}_{t}}\Big[\big(\mathbf{s}(\mathbf{x}_{t},t)-\nabla\log{p_{0t}(\mathbf{x}_{t}\vert\mathbf{x}_{0})}\big)^{T}\mathbf{G}\mathbf{G}^{T}\big(\mathbf{s}(\mathbf{x}_{t},t)-\nabla\log{p_{0t}(\mathbf{x}_{t}\vert\mathbf{x}_{0})}\big)\\
	&\quad\quad-\nabla\log{p_{0t}(\mathbf{x}_{t}\vert\mathbf{x}_{0})}^{T}\mathbf{G}\mathbf{G}^{T}\nabla\log{p_{0t}(\mathbf{x}_{t}\vert\mathbf{x}_{0})}-\text{div}(\mathbf{G}\mathbf{G}^{T})^{T}\nabla\log{p_{0t}(\mathbf{x}_{t}\vert\mathbf{x}_{0})}-2\text{div}(\mathbf{f})\Big]\diff t\\
	&\quad\quad+\mathbb{E}_{\mathbf{x}_{T}}\big[-\log{\pi(\mathbf{x}_{T})}\big].
	\end{align*}
	Therefore, optimizing the nonlinear drift ($\mathbf{f}$) and diffusion ($\mathbf{G}$) terms are intractable in general for two reasons. First, the transition probability of $p_{0t}(\mathbf{x}_{t}\vert\mathbf{x}_{0})$ is intractable for nonlinear SDEs. To compute $p_{0t}(\mathbf{x}_{t}\vert\mathbf{x}_{0})$, one needs the Feynman-Kac formula \cite{huang2021variational} which requires expectation on every sample paths, see Appendix \ref{appendix:nll_correction}.
	
	Second, even if $p_{0t}(\mathbf{x}_{t}\vert\mathbf{x}_{0})$ is tractable, computing the above variational bound would not be scalable due to the matrix multiplication of $\mathbf{G}\mathbf{G}^{T}$ that is of $O(d^{2})$ complexity and the divergence computation \cite{avron2011randomized}. These would become the main source of training bottleneck if dimension increases.
	
	\subsection{Derivation of Nonlinear Drift and Volatility Terms for INDM}\label{appendix:derivation_of_diffusion}
	
	Throughout this section, we omit $\bm{\phi}$ for notational simplicity. With the linear SDE on latent space
	\begin{align}\label{eq:linear_diffusion_ap}
	\diff\mathbf{z}_{t}=-\frac{1}{2}\beta(t)\mathbf{z}_{t}\diff t+g(t)\diff \mathbf{w}_{t},\quad\mathbf{z}_{0}=\mathbf{h}(\mathbf{x}_{0})\text{ with }\mathbf{x}_{0}\sim p_{r},
	\end{align}
	from $\mathbf{x}_{t}=\mathbf{h}^{-1}(\mathbf{z}_{t})$, the $k$-th component of the induced variable satisfies
	\begin{align}\label{eq:ito_ap}
	\diff \mathrm{x}_{t,k}= \frac{\partial \mathrm{h}_{k}^{-1}}{\partial t} \diff t + \big[\nabla_{\mathbf{z}_{t}}\mathrm{h}_{k}^{-1}(\mathbf{z}_{t})\big]^{T} \diff\mathbf{z}_{t}+ \frac{1}{2} \text{tr} \left ( \nabla_{\mathbf{z}_{t}}^{2} \mathrm{h}_{k}^{-1}(\mathbf{z}_{t})\diff\mathbf{z}_{t}\diff\mathbf{z}_{t}^{T}  \right )
	\end{align}
	by the multivariate Ito's Lemma \citep{oksendal2013stochastic}. Plugging the linear SDE of Eq. \eqref{eq:linear_diffusion_ap}, Eq. \eqref{eq:ito_ap} is transformed to
	\begin{align}\label{eq:derived_ito_ap}
	\begin{split}
	\diff\mathrm{x}_{t,k} &= \big[\nabla_{\mathbf{z}_t}\mathrm{h}_{k}^{-1}(\mathbf{z}_{t})\big]^{T} \left \{ -\frac{1}{2}\beta(t)\mathbf{z}_{t}\diff t+g(t)\diff \mathbf{w}_{t} \right \} + \frac{1}{2} g^{2}(t)\text{tr} \left ( \nabla_{\mathbf{z}_{t}}^{2} \mathrm{h}_{k}^{-1}(\mathbf{z}_{t})   \right )\diff t \\
	&= \left \{ -\frac{1}{2} \beta(t) \big[\nabla_{\mathbf{z}_{t}}\mathrm{h}_{k}^{-1}(\mathbf{z}_{t})\big]^{T}\mathbf{z}_{t} + \frac{1}{2} g^{2}(t) \text{tr} \left ( \nabla_{\mathbf{z}_{t}}^{2} \mathrm{h}_{k}^{-1}(\mathbf{z}_{t})    \right ) \right \} \diff t + g(t) \big[\nabla_{\mathbf{z}_t}\mathrm{h}_{k}^{-1}(\mathbf{z}_{t})\big]^{T} \diff \mathbf{w}_{t},
	\end{split}
	\end{align}
	because $\frac{\partial\mathrm{h}_{k}^{-1}}{\partial t}=0$. Then, Eq. \eqref{eq:derived_ito_ap} in vector form becomes
	\begin{align*}
	\diff\mathbf{x}_{t}=\mathbf{f}(\mathbf{x}_{t},t)\diff t+\mathbf{G}(\mathbf{x}_{t},t)\diff\mathbf{w}_{t},
	\end{align*}
	where the vector-valued drift and the matrix-valued volatility terms are given by
	\begin{align}\label{eq:drift_ap}
	\begin{split}
	\left\{\begin{array}{ll}
	\mathbf{f}(\mathbf{x}_{t},t)=-\frac{1}{2} \beta(t) \nabla_{\mathbf{z}_{t}}\mathbf{h}^{-1}(\mathbf{z}_{t}) \mathbf{z}_{t} + \frac{1}{2} g^{2}(t) \text{tr} \left ( \nabla_{\mathbf{z}_{t}}^{2} \mathbf{h}^{-1}(\mathbf{z}_{t})    \right )\\[0.05cm]
	\mathbf{G}(\mathbf{x}_{t},t)=g(t) \nabla_{\mathbf{z}_t}\mathbf{h}^{-1}(\mathbf{z}_{t}).
	\end{array}\right.
	\end{split}
	\end{align}
	Here, $\nabla_{\mathbf{z}_{t}}^{2}\mathbf{h}^{-1}(\mathbf{z}_{t})$ is a 3-dimensional tensor with $(i,j,k)$-th element to be $\big(\nabla_{\mathbf{z}_{t}}^{2}\mathrm{h}_{k}^{-1}(\mathbf{z}_{t})\big)_{i,j}$, and the trace operator applied on this tensor is defined as a vector of $\Big[\text{tr}\big(\nabla_{\mathbf{z}_{t}}^{2}\mathrm{h}_{1}^{-1}(\mathbf{z}_{t})\big),...,\text{tr}\big(\nabla_{\mathbf{z}_{t}}^{2}\mathrm{h}_{d}^{-1}(\mathbf{z}_{t})\big)\Big]^{T}$. From the inverse function theorem \citep{rudin1964principles}, the Jacobian of the inverse function $\nabla_{\mathbf{z}_{t}}\mathbf{h}^{-1}(\mathbf{z}_{t})$ equals to the inverse Jacobian $\big[\nabla_{\mathbf{x}_{t}}\mathbf{h}(\mathbf{x}_{t})\big]^{-1}$. Therefore, Eq. \eqref{eq:drift_ap} is transformed to
	\begin{align}\label{eq:drift_ap_ver2}
	\begin{split}
	\left\{\begin{array}{ll}
	\mathbf{f}(\mathbf{x}_{t},t)=-\frac{1}{2} \beta(t) \big[\nabla_{\mathbf{x}_{t}}\mathbf{h}(\mathbf{x}_{t})\big]^{-1} \mathbf{h}(\mathbf{x}_{t}) + \frac{1}{2} g^{2}(t) \text{tr} \left ( \nabla_{\mathbf{z}_{t}}^{2} \mathbf{h}^{-1}(\mathbf{z}_{t})    \right )\\[0.1cm]
	\mathbf{G}(\mathbf{x}_{t},t)=g(t) \big[\nabla_{\mathbf{x}_{t}}\mathbf{h}(\mathbf{x}_{t})\big]^{-1}.
	\end{array}\right.
	\end{split}
	\end{align}	
	Now, we derive the second term of $\mathbf{f}$ in terms of $\mathbf{x}_{t}$ as follows: observe that $\sum_{k}\frac{\partial \mathrm{h}_{i}^{-1}}{\partial\mathrm{z}_{t,k}}\frac{\partial\mathrm{h}_{k}}{\partial\mathrm{x}_{t,j}}=\delta_{i,j}$, where $\delta_{i,j}=1$ if $i=j$ and $0$ otherwise. Differentiating both sides with respect to $\mathrm{x}_{t,l}$, we have
	\begin{align*}
	\sum_{k} \left \{  \frac{\partial}{\partial \mathrm{x}_{t,l}} \left ( \frac{\partial \mathrm{h}^{-1}_{i}}{\partial \mathrm{z}_{t,k}} \right ) \right \} \frac{\partial \mathrm{h}_{k}}{\partial \mathrm{x}_{t,j}} + \frac{\partial \mathrm{h}^{-1}_{i}}{\partial \mathrm{z}_{t,k}} \left \{  \frac{\partial}{\partial \mathrm{x}_{t,l}} \left (\frac{\partial \mathrm{h}_{k}}{\partial \mathrm{x}_{t,j}} \right ) \right \}=0,
	\end{align*}
	where the first term is
	\begin{align*}
	\sum_{k,m} \frac{\partial \mathrm{h}_{m}}{\partial \mathrm{x}_{t,l}} \left \{  \frac{\partial}{\partial \mathrm{z}_{t,m}} \left ( \frac{\partial \mathrm{h}^{-1}_{i}}{\partial \mathrm{z}_{t,k}} \right ) \right \} \frac{\partial \mathrm{h}_{k}}{\partial \mathrm{x}_{t,j}} = \sum_{k,m} \big( \nabla_{\mathbf{x}_{t}} \mathbf{h}(\mathbf{x}_{t}) \big)^{T}_{l,m} \big( \nabla_{\mathbf{z}_{t}}^{2} \mathbf{h}_{i}^{-1}(\mathbf{z}_{t}) \big)_{m,k} \big(\nabla_{\mathbf{x}_{t}} \mathbf{h}(\mathbf{x}_{t})\big)_{k,j},
	\end{align*}
	using the chain rule, and the second term becomes
	\begin{align*}
	\sum_{k}\frac{\partial \mathrm{h}^{-1}_{i}}{\partial \mathrm{z}_{t,k}} \left \{  \frac{\partial}{\partial \mathrm{x}_{t,l}} \left (\frac{\partial \mathrm{h}_{k}}{\partial \mathrm{x}_{t,j}} \right ) \right \} =  \sum_{k} \big(\nabla_{\mathbf{z}_{t}} \mathbf{h}^{-1}(\mathbf{z}_{t})\big)_{i,k} \big(\nabla_{\mathbf{x}_{t}}^{2} \mathbf{h}_{k}(\mathbf{x}_{t})\big)_{l,j}.
	\end{align*}
	From the above, we derive the trace term of $\mathbf{f}$ in Eq. \eqref{eq:drift_ap_ver2} as
	\begin{align*}
	\text{tr} \left ( \nabla_{\mathbf{z}_{t}}^{2} \mathbf{h}^{-1}(\mathbf{z}_{t})    \right ) = - \text{tr}  \left (  \big[ \nabla_{\mathbf{x}_{t}}\mathbf{h}(\mathbf{x}_{t})\big]^{-T} \left ( \big[\nabla_{\mathbf{x}_{t}}\mathbf{h}(\mathbf{x}_{t})\big]^{-1} * \nabla_{\mathbf{x}_{t}}^{2}\mathbf{h}(\mathbf{x}_{t})  \right )  \big[ \nabla_{\mathbf{x}_{t}}\mathbf{h}(\mathbf{x}_{t})\big]^{-1} \right ),
	\end{align*}
	where $\nabla_{\mathbf{x}_{t}}^{2}\mathbf{h}(\mathbf{x}_{t})$ is a 3-dimensional tensor with $(i,j,k)$-th element to be $\big(\nabla_{\mathbf{x}_{t}}^{2}\mathbf{h}_{k}(\mathbf{x}_{t})\big)_{i,j}$. Also, we define $*$ operation as the element-wise matrix multiplication given by
	\begin{align*}
	\Big(\big[\nabla_{\mathbf{x}_{t}}\mathbf{h}(\mathbf{x}_{t})\big]^{-1} * \nabla_{\mathbf{x}_{t}}^{2}\mathbf{h}(\mathbf{x}_{t})\Big)_{i,j}:=\nabla_{\mathbf{x}_{t}}\big[\mathbf{h}(\mathbf{x}_{t})\big]^{-1}\Big(\nabla_{\mathbf{x}_{t}}^{2}\mathbf{h}(\mathbf{x}_{t})\Big)_{i,j}.
	\end{align*}
	Combining all together, thus, we derive the nonlinear drift term in Eq. \eqref{eq:drift_ap_ver2} as a function of $\mathbf{x}_{t}$ given by
	\begin{align*}
	\mathbf{f}(\mathbf{x}_{t},t)=&-\frac{1}{2} \beta(t) \big[\nabla_{\mathbf{x}_{t}}\mathbf{h}(\mathbf{x}_{t})\big]^{-1} \mathbf{h}(\mathbf{x}_{t})\\
	& - \frac{1}{2} g^{2}(t) \text{tr}  \left (  \big[ \nabla_{\mathbf{x}_{t}}\mathbf{h}(\mathbf{x}_{t})\big]^{-T} \left ( \big[\nabla_{\mathbf{x}_{t}}\mathbf{h}(\mathbf{x}_{t})\big]^{-1} * \nabla_{\mathbf{x}_{t}}^{2}\mathbf{h}(\mathbf{x}_{t})  \right )  \big[ \nabla_{\mathbf{x}_{t}}\mathbf{h}(\mathbf{x}_{t})\big]^{-1} \right ).
	\end{align*}

	\section{Details on Section \ref{sec:better_optimization}}\label{appendix:variational_gap}
	It is one of central topics in the community of VAE to obtain a tighter ELBO \cite{rezende2015variational,burda2015importance}. This section analyzes the variational gap and further theoretical analysis in diffusion models. Before we start, we remind the generalized Helmholtz decomposition in Lemma \ref{lemma:1}.
	\begin{lemma}[Helmholtz Decomposition \cite{glotzl2020helmholtz}]\label{lemma:1}
		Any twice continuously differentiable vector field $\mathbf{s}$ that decays faster than $\Vert\mathbf{z}\Vert_{2}^{-c}$ for $\Vert\mathbf{z}\Vert_{2}\rightarrow\infty$ and $c>0$ can be uniquely decomposed into two vector fields, one rotation-free and one divergence-free: $\mathbf{s}=\nabla\log{p}+\mathbf{u}$.
	\end{lemma}
	A rotation-free vector field $\nabla\log{p}$, or the divergence part, is a score function of a probability density $p$, and a divergence-free vector field $\mathbf{u}$, or the rotation part, satisfies $\text{div}(\mathbf{u})\equiv 0$. From this decomposition, any score network is decomposed by $\mathbf{s}_{\bm{\theta}}(\mathbf{z}_{t},t)=\nabla\log{p_{t}^{\bm{\theta}}(\mathbf{z}_{t})}+\mathbf{u}_{t}^{\bm{\theta}}(\mathbf{z}_{t})$ for some probability $p_{t}^{\bm{\theta}}$ and vector field $\mathbf{u}_{t}^{\bm{\theta}}$, for any $t\in(0,T]$. Then, the generative SDE of the full score network
	\begin{align}\label{eq:full_score}
	\diff\mathbf{z}_{t}=\left[\mathbf{f}(\mathbf{z}_{t},t)-g^{2}(t)\mathbf{s}_{\bm{\theta}}(\mathbf{z}_{t},t)\right]\diff\bar{t}+g(t)\diff\mathbf{\bar{w}}_{t}
	\end{align}
	and the generative SDE of the divergence part
	\begin{align}\label{eq:rotation_free_score}
	\diff\mathbf{z}_{t}=\left[\mathbf{f}(\mathbf{z}_{t},t)-g^{2}(t)\nabla\log{p_{t}^{\bm{\theta}}(\mathbf{z}_{t})}\right]\diff\bar{t}+g(t)\diff\mathbf{\bar{w}}_{t}
	\end{align}
	has distinctive path measures. Throughout this section, $\mathbf{f}(\mathbf{z}_{t},t)$ does not have to be a linear vector field, such as $-\frac{1}{2}\beta(t)\mathbf{z}_{t}$. If $\bm{\nu}_{\bm{\theta}}$ and $\bm{\rho}_{\bm{\theta}}$ are the path measures of SDEs of Eqs. \eqref{eq:full_score} and \eqref{eq:rotation_free_score}, respectively, then using the Girsanov theorem \cite{song2021maximum,sarkka2019applied} and the martingale property \cite{oksendal2013stochastic}, we have
	\begin{align}
	\begin{split}\label{eq:u_t}
	D_{KL}(\bm{\nu}_{\bm{\theta}}\Vert\bm{\rho}_{\bm{\theta}})=&\frac{1}{2}\int_{0}^{T}g^{2}(t)\mathbb{E}_{\mathbf{z}_{t}\sim\bm{\nu}_{\bm{\theta}}\vert_{t}}\left[\Vert \nabla\log{p_{t}^{\bm{\theta}}(\mathbf{z}_{t})}-\mathbf{s}_{\bm{\theta}}(\mathbf{z}_{t},t) \Vert_{2}^{2}\right]\diff t\\
	=&\frac{1}{2}\int_{0}^{T}g^{2}(t)\mathbb{E}_{\mathbf{z}_{t}\sim\bm{\nu}_{\bm{\theta}}\vert_{t}}\left[\Vert \mathbf{u}_{\bm{\theta}}(\mathbf{z}_{t},t) \Vert_{2}^{2}\right]\diff t.
	\end{split}
	\end{align}
	This KL divergence of two path measures quantifies how much the score network contains the rotation part $\mathbf{u}_{t}^{\bm{\theta}}$. Recall that the forward SDE satisfies
	\begin{align*}
	\diff\mathbf{z}_{t}=\mathbf{f}(\mathbf{z}_{t},t)\diff t+g(t)\diff\mathbf{w}_{t},
	\end{align*}
	which starts at $p_{0}^{\bm{\phi}}$, and the marginal distribution of its path measure $\bm{\mu}_{\bm{\phi}}$ at $t$ is $p_{t}^{\bm{\phi}}$. As NELBO is equivalent to
	\begin{align}\label{eq:nelbo_score}
	D_{KL}(\bm{\mu}_{\bm{\phi}}\Vert\bm{\nu}_{\bm{\phi},\bm{\theta}})=\frac{1}{2}\int_{0}^{T}g^{2}(t)\mathbb{E}_{p_{t}^{\bm{\phi}}(\mathbf{z}_{t})}\left[\Vert\mathbf{s}_{\bm{\theta}}(\mathbf{z}_{t},t)-\nabla\log{p_{t}^{\bm{\phi}}(\mathbf{z}_{t})}\Vert_{2}^{2}\right]\diff t+D_{KL}\big(p_{T}^{\bm{\phi}}(\mathbf{z}_{T})\Vert\pi(\mathbf{z}_{T})\big),
	\end{align}
	for $\bm{\theta}$-optimization, the optimal $\bm{\theta}^{*}$ satisfies $\mathbf{s}_{\bm{\theta}^{*}}(\mathbf{z}_{t},t)=\nabla\log{p_{t}^{\bm{\phi}}(\mathbf{z}_{t})}$. At this optimality, $\bm{\theta}$ should satisfy a pair of constraints: 1) the zero-rotation part $\mathbf{u}_{t}^{\bm{\theta}^{*}}\equiv 0$, which is equivalent to $D_{KL}(\bm{\nu}_{\bm{\theta}^{*}}\Vert\bm{\rho}_{\bm{\theta}^{*}})=0$; 2) the coincidence of $\nabla\log{p_{t}^{\bm{\phi}}(\mathbf{z}_{t})}\equiv \nabla\log{p_{t}^{\bm{\theta}}(\mathbf{z}_{t})}$. The starting point to analyze the variational gap with respect to the Helmholtz decomposition is the next theorem. We defer the proofs in Section \ref{appendix:proofs}. 
	
	\begin{proposition}\label{thm:2}
		Suppose $q_{t}^{\bm{\theta}}$ is the marginal distribution of $\bm{\nu}_{\bm{\theta}}$ at $t$. The variational gap is
		\begin{align*}
		\textup{Gap}\big(\bm{\mu}_{\bm{\phi}}(\{\mathbf{x}_{t}\}),\bm{\nu}_{\bm{\phi},\bm{\theta}}(\{\mathbf{x}_{t}\})\big):=&D_{KL}\big(\bm{\mu}_{\bm{\phi}}(\{\mathbf{x}_{t}\})\Vert\bm{\nu}_{\bm{\phi},\bm{\theta}}(\{\mathbf{x}_{t}\})\big)-D_{KL}\big(p_{0}^{\bm{\phi}}(\mathbf{x}_{0})\Vert q_{0}^{\bm{\theta}}(\mathbf{x}_{0})\big)\\
		=&\frac{1}{2}\int_{0}^{T}g^{2}(t)\mathbb{E}_{p_{t}^{\bm{\phi}}(\mathbf{z}_{t})}\big[\underbrace{\Vert\nabla\log{q_{t}^{\bm{\theta}}(\mathbf{z}_{t})}-\mathbf{s}_{\bm{\theta}}(\mathbf{z}_{t},t)\Vert_{2}^{2}}_{\text{Score-only error}}\big]\diff t.
		\end{align*}
	\end{proposition}
	
	\begin{remark}
		The generative SDE of $\diff\mathbf{z}_{t}^{\bm{\theta}}=[-\frac{1}{2}\beta(t)\mathbf{z}_{t}^{\bm{\theta}}-g^{2}(t)\mathbf{s}_{\bm{\theta}}(\mathbf{z}_{t}^{\bm{\theta}},t)]\diff\bar{t}+g(t)\diff\mathbf{\bar{w}}_{t}$ does not necessarily start from the prior $\pi$. Proposition \ref{thm:2} holds for an arbitrary distribution $p_{T}^{\bm{\theta}}$. At the same spirit, Proposition \ref{thm:2} holds for any distribution $p_{0}^{\bm{\phi}}$.
	\end{remark}
	\begin{remark}
		Throughout the section, we follow the assumptions made in Appendix A of \citet{song2021maximum}. On top of that, we assume that both $\mathbf{s}_{\bm{\theta}}$ and $q_{t}^{\bm{\theta}}$ are continuously differentiable.
	\end{remark}
	
	The variational gap derived in Proposition \ref{thm:2} does not include the forward score, $\nabla\log{p_{t}^{\bm{\phi}}(\mathbf{z}_{t})}$, except for taking the expectation, $\mathbb{E}_{p_{t}^{\bm{\phi}}(\mathbf{z}_{t})}$. Therefore, the gap is intuitively connected to the score training, rather than the flow training. To elucidate the logic, we decompose the variational gap in Proposition \ref{thm:2} into
	\begin{eqnarray}
	\lefteqn{\textup{Gap}\big(\bm{\mu}_{\bm{\phi}},\bm{\nu}_{\bm{\phi},\bm{\theta}}\big)=\frac{1}{2}\int_{0}^{T}g^{2}(t)\mathbb{E}_{p_{t}^{\bm{\phi}}(\mathbf{z}_{t})}\big[\Vert\nabla\log{q_{t}^{\bm{\theta}}(\mathbf{z}_{t})}-\mathbf{s}_{\bm{\theta}}(\mathbf{z}_{t},t)\Vert_{2}^{2}\big]\diff t}&\label{eq:gap_decomposition}\\
	&&\le\frac{1}{2}\int_{0}^{T}g^{2}(t)\mathbb{E}_{p_{t}^{\bm{\phi}}(\mathbf{z}_{t})}\Big[\big\Vert\nabla\log{q_{t}^{\bm{\theta}}(\mathbf{z}_{t})}-\nabla\log{p_{t}^{\bm{\theta}}(\mathbf{z}_{t})}\big\Vert_{2}^{2}+\big\Vert\nabla\log{p_{t}^{\bm{\theta}}(\mathbf{z}_{t})}-\mathbf{s}_{\bm{\theta}}(\mathbf{z}_{t},t)\big\Vert_{2}^{2}\Big]\diff t.\notag\\
	&&=\frac{1}{2}\int_{0}^{T}g^{2}(t)\mathbb{E}_{p_{t}^{\bm{\phi}}(\mathbf{z}_{t})}\Big[\underbrace{\big\Vert\nabla\log{q_{t}^{\bm{\theta}}(\mathbf{z}_{t})}-\nabla\log{p_{t}^{\bm{\theta}}(\mathbf{z}_{t})}\big\Vert_{2}^{2}}_{\text{How close is $\bm{\nu}_{\bm{\theta}}$ to forward measure}}+\underbrace{\big\Vert\mathbf{u}_{t}^{\bm{\theta}}(\mathbf{z}_{t})\big\Vert_{2}^{2}}_{\text{How close is $\mathbf{u}_{t}^{\bm{\theta}}$ to zero}}\Big]\diff t.\notag
	\end{eqnarray}
	The second term, $\Vert\nabla\log{p_{t}^{\bm{\theta}}(\mathbf{z}_{t})}-\mathbf{s}_{\bm{\theta}}(\mathbf{z}_{t},t)\Vert_{2}^{2}$ (see Eq. \eqref{eq:u_t}), equals to the $L^{2}$-norm of the rotation term, $\Vert\mathbf{u}_{t}^{\bm{\theta}}(\mathbf{z}_{t},t)\Vert_{2}^{2}$, so it measures how close is the score network to the space of 
	\begin{align*}
	\mathbf{S}_{div}:=\{\mathbf{s}:\mathbb{R}^{d}\rightarrow\mathbb{R}^{d}\vert\text{ the rotation term of $\mathbf{s}$ is zero}\}.
	\end{align*}
	On the other hand, having assumed the rotation part to be zero, the first term, $\Vert\nabla\log{q_{t}^{\bm{\theta}}(\mathbf{z}_{t})}-\nabla\log{p_{t}^{\bm{\theta}}(\mathbf{z}_{t})}\Vert_{2}^{2}$, becomes zero only if the generative score $\nabla\log{q_{t}^{\bm{\theta}}}$ equals to a forward score $\nabla\log{q_{t}}$ with certain initial distribution $q_{0}$, meaning that if there exists a $q_{0}$ and $q_{t}$ is a marginal density of the forward SDE starting from $q_{0}$, then 
	\begin{align*}
	\nabla\log{q_{t}(\mathbf{z}_{t})}=\nabla\log{q_{t}^{\bm{\theta}}(\mathbf{z}_{t})}\text{ is equivalent to }\nabla\log{q_{t}(\mathbf{z}_{t})}=\nabla\log{p_{t}^{\bm{\theta}}(\mathbf{z}_{t})},
	\end{align*}
	and only in that case,  $\nabla\log{q_{t}^{\bm{\theta}}(\mathbf{z}_{t})}=\nabla\log{p_{t}^{\bm{\theta}}(\mathbf{z}_{t})}$. Therefore, the gap becomes tight if 1) $\mathbf{u}_{t}^{\bm{\theta}}\equiv 0$ and 2) $\nabla\log{q_{t}^{\bm{\theta}}}=\nabla\log{q_{t}}$ for some $q_{t}$ following the forward SDE, which is concretely proved in Lemma \ref{lemma:2}. It turns out that this is the only case of the gap being zero proved in Theorem \ref{cor:2}. For that, we provide a rigorous definition of the class of score functions of interest as below. 
	\begin{definition}\label{def:1}
		Let $\mathbf{S}_{sol}\subseteq\mathbf{S}_{div}$ be a sub-family of rotation-free score functions $\mathbf{s}:\mathbb{R}^{d}\rightarrow\mathbb{R}^{d}$ such that $\mathbf{s}(\mathbf{z}_{t},t)=\nabla\log{p_{t}(\mathbf{z}_{t})}$ almost everywhere for $p_{t}$ that is the marginal distribution of the path measure of $\diff\mathbf{z}_{t}=\mathbf{f}(\mathbf{z}_{t},t)\diff t+g(t)\diff\mathbf{w}_{t}$ at $t$.
	\end{definition}
	\begin{remark}
		Analogous to Theorem \ref{thm:2}, no condition for the starting and ending variables is imposed in Definition \ref{def:1}.
	\end{remark}
	\begin{remark}
		$\mathbf{S}_{sol}$ is the space of score functions of the forward SDE $\diff\mathbf{z}_{t}=\mathbf{f}(\mathbf{z}_{t},t)\diff t+g(t)\diff\mathbf{w}_{t}$ with arbitrary initial variable.
	\end{remark}
	
	Although \citet{song2021maximum} focused on the data diffusion, their theory is applicable for a diffusion process that starts with an arbitrary initial distribution. Lemma \ref{lemma:2} describes the theoretic analysis done by \citet{song2021maximum}.
	\begin{lemma}[Theorem 2 of \citet{song2021maximum}]\label{lemma:2}
		$\textup{Gap}(\bm{\mu}_{\bm{\phi}},\bm{\nu}_{\bm{\phi},\bm{\theta}})=0$ if $\mathbf{s}_{\bm{\theta}}\in\mathbf{S}_{sol}$.
	\end{lemma}
	
	With Lemma \ref{lemma:2}, however, we cannot certainly be sure that the score network $\mathbf{s}_{\bm{\theta}}$ of INDM falls to $\mathbf{S}_{sol}$ when the variational gap is zero. Thus, we take a step further to identify the connection of zero variational gap and the class of rotation-free score functions $\mathbf{S}_{sol}$ in Theorem \ref{cor:2}. This Theorem \ref{cor:2} completely characterizes all admissible score networks that achieve the zero variational gaps, and we are certain that the zero variational gap implies $\mathbf{s}_{\bm{\theta}}\in\mathbf{S}_{sol}$, which turns out to be a solution space in Theorem \ref{thm:3}.
	\begingroup
	\renewcommand\thetheorem{2}
	\begin{theorem}
		$\textup{Gap}(\bm{\mu}_{\bm{\phi}},\bm{\nu}_{\bm{\phi},\bm{\theta}})=0$ if and only if $\mathbf{s}_{\bm{\theta}}\in\mathbf{S}_{sol}$. 
	\end{theorem}
	\endgroup
	
	\begin{wrapfigure}{r}{0.6\textwidth}
		\vskip -0.25in
		\centering
		\includegraphics[width=\linewidth]{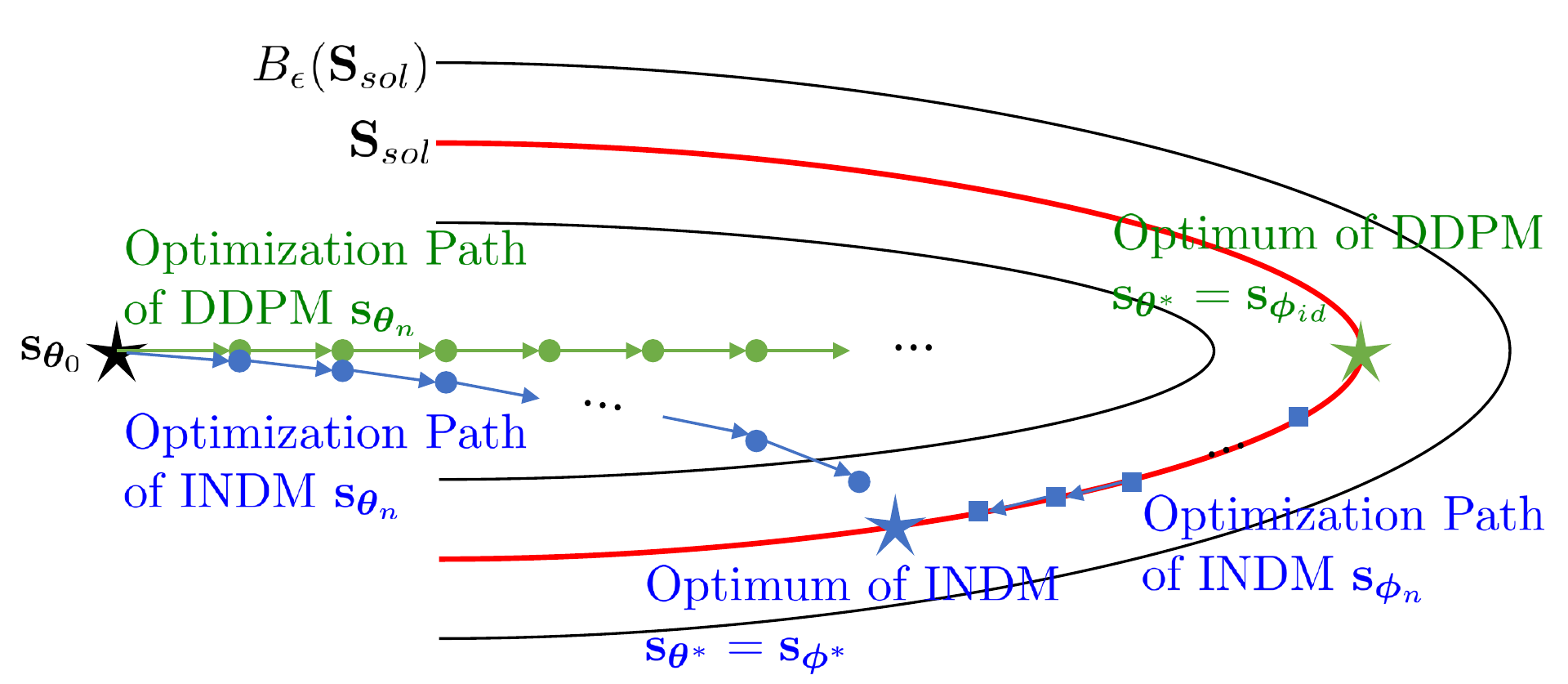}
		\vskip -0.05in
		\caption{Descriptive Illustration On Nearly MLE Training.}
		\label{fig:training_curve}
		\vskip -0.1in
	\end{wrapfigure}
	From Theorem \ref{cor:2}, the variational gap is strictly positive as long as the rotation part of the score network remains to be nonzero. NELBO of Eq. \eqref{eq:nelbo_score} optimizes its score network towards $\mathbf{s}_{\bm{\theta}}(\mathbf{z}_{t},t)\rightarrow\nabla\log{p_{t}^{\bm{\phi}}(\mathbf{z}_{t})}:=\mathbf{s}_{\bm{\phi}}(\mathbf{z}_{t},t)$, which is equivalent to $\log{p_{t}^{\bm{\theta}}(\mathbf{z}_{t})}\rightarrow\nabla\log{p_{t}^{\bm{\phi}}(\mathbf{z}_{t})}$ (or equivalently, $\log{q_{t}^{\bm{\theta}}(\mathbf{z}_{t})}\rightarrow\nabla\log{p_{t}^{\bm{\phi}}(\mathbf{z}_{t})}$) and $\mathbf{u}_{t}^{\bm{\theta}}(\mathbf{z}_{t},t)\rightarrow 0$. In contrast to DDPM++ with fixed $\bm{\phi}$, optimizing $D_{KL}(\bm{\mu}_{\bm{\phi}}\Vert\bm{\nu}_{\bm{\phi},\bm{\theta}})$ w.r.t $\bm{\phi}$ finds the closest $\mathbf{s}_{\bm{\phi}}$ among $\mathbf{S}_{sol}$ to $\mathbf{s}_{\bm{\theta}}$. Thus, if $\mathbf{s}_{\bm{\theta}}\in\mathbf{S}_{sol}$, then $\mathbf{s}_{\bm{\phi}^{*}}=\mathbf{s}_{\bm{\theta}}$, which is proved in Theorem \ref{thm:3}. If $\mathbf{s}_{\bm{\theta}}\notin\mathbf{S}_{sol}$, then $\mathbf{s}_{\bm{\phi}^{*}}$ is not equal to $\mathbf{s}_{\bm{\theta}}$, anymore, but $\mathbf{s}_{\bm{\phi}^{*}}$ will be the closest among $\mathbf{S}_{sol}$ to $\mathbf{s}_{\bm{\theta}}$ because $D_{KL}(\bm{\mu}_{\bm{\phi}}\Vert\bm{\nu}_{\bm{\phi},\bm{\theta}})$ is the weighted $L^{2}$-norm of $\mathbf{s}_{\bm{\phi}}-\mathbf{s}_{\bm{\theta}}$. 
	
	\begingroup
	\renewcommand\thetheorem{3}
	\begin{theorem}
		For any fixed $\mathbf{s}_{\bm{\bar{\theta}}}\in\mathbf{S}_{sol}$, if $\bm{\phi}^{*}\in\argmin_{\bm{\phi}}{D_{KL}(\bm{\mu}_{\bm{\phi}}\Vert \bm{\nu}_{\bm{\phi},\bm{\bar{\theta}}})}$, then $\mathbf{s}_{\bm{\phi}^{*}}(\mathbf{z}_{t},t)=\nabla\log{p_{t}^{\bm{\phi}^{*}}(\mathbf{z}_{t})}=\mathbf{s}_{\bm{\bar{\theta}}}(\mathbf{z}_{t},t)$, and $D_{KL}(\bm{\mu}_{\bm{\phi}^{*}}\Vert \bm{\nu}_{\bm{\phi}^{*},\bm{\bar{\theta}}})=D_{KL}(p_{r}\Vert p_{\bm{\phi}^{*},\bm{\bar{\theta}}})=\text{Gap}(\bm{\mu}_{\bm{\phi}^{*}},\bm{\nu}_{\bm{\phi}^{*},\bm{\bar{\theta}}})=0$.
	\end{theorem}
	\endgroup
	
	Indeed, Theorem \ref{thm:3} implies that the whole class of $\mathbf{S}_{sol}$ is the solution space, which means that any $\mathbf{s}_{\bm{\theta}}$ in $\mathbf{S}_{sol}$ is a candidate for an optimal score function as there always exists $\bm{\phi}^{*}$ corresponding to a given $\bm{\theta}$ that achieves the perfect match of the model distribution to the data distribution. This is contrastive to DDPM++ that only has a unique optimal point of $\mathbf{s}_{\bm{\theta}^{*}}(\mathbf{z}_{t},t)=\nabla\log{p_{t}^{\bm{\phi}_{id}}(\mathbf{z}_{t})}\in\mathbf{S}_{sol}$. Figure \ref{fig:training_curve} illustrates that the optimal point of DDPM is a single point in $\mathbf{S}_{sol}$, whereas any $\mathbf{s}_{\bm{\theta}}\in\mathbf{S}_{sol}$ is a candidate for the optimal point of INDM by Theorem \ref{thm:3}. In other words, the number of DDPM optimality is one, while INDM has infinite number of optimalities.
	
	\subsection{Restricting Search Space of \texorpdfstring{$\mathbf{s}_{\bm{\theta}}$}{TEXT} into \texorpdfstring{$\mathbf{S}_{div}$}{TEXT}}\label{appendix:regularization}
	
	Due to the space limit, the argument in this section has not been included in the main paper. Below, we provide the rationale that it is the number of optimal points that affect the NLL performance. For that, we optimize DDPM++ with a regularization, suggested in Proposition \ref{prop:4}. This regularization restricts the score network from not deviating $\mathbf{S}_{div}$ too far by keeping the rotation term, $\mathbf{u}_{t}^{\bm{\theta}}$, being consistently small. Consequently, a fastly converging rotation term is advantageous in reducing the variational gap (see Inequality \eqref{eq:gap_decomposition}), and this regularization helps the MLE training of DDPM++. 
	
	Proposition \ref{prop:1} proves that $\mathbf{S}_{div}$ is identical to a class of score functions that have symmetric derivatives. From this, Proposition \ref{prop:2} provides a motivation of the regularization by proving that a symmetric matrix satisfies a certain equality.  Then, Proposition \ref{prop:3} implies that the formula suggested in Proposition \ref{prop:2} indeed measures how close is the matrix symmetric. Lastly, Proposition \ref{prop:4} provides the minimum variance estimator of the formula. With these propositions, we conclude that the constraint of
	\begin{align}\label{eq:symmetric}
	\mathbb{E}_{\bm{\epsilon}_{1},\bm{\epsilon}_{2}}\left[\Big(\bm{\epsilon}_{2}^{T}\big(\nabla\mathbf{s}_{\bm{\theta}}(\mathbf{z}_{t},t)-(\nabla\mathbf{s}_{\bm{\theta}})^{T}(\mathbf{z}_{t},t)\big)\bm{\epsilon}_{1}\Big)^{2}\right]=0
	\end{align}
	with $\bm{\epsilon}_{1}$ and $\bm{\epsilon}_{2}$ sampled from the random variable suggested in Proposition \ref{prop:4} would optimize $\mathbf{s}_{\bm{\theta}}$ in the space of $\mathbf{S}_{div}$. Using the Lagrangian form, we could add the left-hand-side of Eq. \eqref{eq:symmetric} as a regularization term in NELBO to force the score network not deviate from $\mathbf{S}_{div}$ too much. 
	
	With the clear mathematical properties, however, obtaining the full matrix of $\nabla\mathbf{s}_{\bm{\theta}}$ is a bottleneck in the computation of the regularization term. Specifically, each row of $\nabla\mathbf{s}_{\bm{\theta}}$ needs to be computed separately \cite{song2020sliced}, so it takes $O(d)$ complexity to compute $\nabla\mathbf{s}_{\bm{\theta}}$, which is prohibitively expensive. Therefore, we use a trick to reduce $O(d)$ to $O(1)$ motivated from the Hutchinson's estimator \cite{hutchinson1989stochastic, chen2019residual}: first, we compute the gradient of $\bm{\epsilon}_{2}^{T}\mathbf{s}_{\bm{\theta}}$ and $\bm{\epsilon}_{1}^{T}\mathbf{s}_{\bm{\theta}}$, separately. Afterwards, we apply vector multiplication between $\bm{\epsilon}_{1}$ and $\nabla(\bm{\epsilon}_{2}^{T}\mathbf{s}_{\bm{\theta}})$, which gives us $\bm{\epsilon}_{2}^{T}\nabla\mathbf{s}_{\bm{\theta}}\bm{\epsilon}_{1}$; and analogously, the multiplication of $\bm{\epsilon}_{2}$ with $\nabla(\bm{\epsilon}_{1}^{T}\mathbf{s}_{\bm{\theta}})$ yields $\bm{\epsilon}_{2}^{T}(\nabla\mathbf{s}_{\bm{\theta}})^{T}\bm{\epsilon}_{1}$. This trick requires only second time of gradient computations to estimate the regularization. Hence, the computational complexity of $\bm{\epsilon}_{2}^{T}(\nabla\mathbf{s}_{\bm{\theta}}-\nabla\mathbf{s}_{\bm{\theta}}^{T})\bm{\epsilon}_{1}$ is $O(1)$.
	
	\begin{proposition}\label{prop:1}
		$\mathbf{s}_{\bm{\theta}}\in\mathbf{S}_{div}$ if and only if $\nabla_{\mathbf{z}_{t}}\mathbf{s}_{\bm{\theta}}(\mathbf{z}_{t},t)$ is symmetric.
	\end{proposition}
	
	\begin{proposition}\label{prop:2}
		A matrix $A\in\mathbb{R}^{d\times d}$ is symmetric if and only if $\mathbb{E}_{\bm{\epsilon}_{1},\bm{\epsilon}_{2}\sim\mathcal{N}(0,\mathbf{I})}\left[ (\bm{\epsilon}_{2}^{T}(A-A^{T})\bm{\epsilon}_{1})^{2}\right]=0$.
	\end{proposition}
	
	In fact, we can prove a bit stronger results in the next propositions.
	
	\begin{proposition}\label{prop:3}
		Let $\bm{\epsilon}_{1}$ and $\bm{\epsilon}_{2}$ be vectors of $d$ independent samples from a random variable $U$ with mean zero. Then
		\begin{align*}
		\mathbb{E}_{\bm{\epsilon}_{1},\bm{\epsilon}_{2}}[(\bm{\epsilon}_{2}^{T}(A-A^{T})\bm{\epsilon}_{1})^{2}]=\mathbb{E}_{U}[U^{2}]^{2}\Vert A-A^{T}\Vert_{F}^{2}
		\end{align*}
		and
		\begin{align*}
		&\text{Var}\Big(\big(\bm{\epsilon}_{2}^{T}(A-A^{T})\bm{\epsilon}_{1}\big)^{2}\Big)=\text{Var}(U^{2})\Big(\text{Var}(U^{2})+2\big(\text{Var}(U)+\mathbb{E}_{U}[U]^{2}\big)^{2}\Big)\sum_{a,b}(\Delta A)_{ab}^{4}\\
		&\quad+2\big(\text{Var}(U)+\mathbb{E}_{U}[U]^{2}\big)^{2}\Big(3\text{Var}(U^{2})+2\big(\text{Var}(U)+\mathbb{E}_{U}[U]^{2}\big)^{2}\Big)\sum_{a}\sum_{b\ne d}(\Delta A)_{ab}^{2}(\Delta A)_{ad}^{2}\\
		&\quad+2\big(\text{Var}(U)+\mathbb{E}_{U}[U]^{2}\big)^{4}\Big(\sum_{a\neq c}\sum_{b\neq d}(\Delta A)_{ab}^{2}(\Delta A)_{cd}^{2}\\
		&\quad\quad\quad\quad\quad\quad\quad\quad\quad\quad\quad\quad+3\sum_{a\ne c}\sum_{b\ne d}(\Delta A)_{ab}(\Delta A)_{ad}(\Delta A)_{cb}(\Delta A)_{cd}\Big),
		\end{align*}
		where $(\Delta A)_{ab}:=A_{ab}-A_{ba}$.
	\end{proposition}
	
	\begin{proposition}\label{prop:4}
		Let $U$ be the discrete random variable which takes the values $1,-1$ each with probability $1/2$. Then $(\bm{\epsilon}_{2}^{T}(A-A^{T})\bm{\epsilon}_{1})^{2}$ is the unbiased estimator of $\Vert A-A^{T}\Vert_{F}^{2}$.
		Moreover, $U$ is the unique random variable amongst zero-mean random variables for which the estimator is an unbiased estimator, and attains a minimum variance.
	\end{proposition}
	
	Summing altogether, if it is the main focus to eliminate the rotation term in the score estimation, we could optimize $D_{KL}(\bm{\mu}_{\bm{\phi}}\Vert\bm{\nu}_{\bm{\phi},\bm{\theta}})+\lambda\mathbb{E}_{\bm{\bm{\epsilon}}_{1},\bm{\bm{\epsilon}}_{2}}\left[(\bm{\bm{\epsilon}}_{2}^{T}(\nabla\mathbf{s}_{\bm{\theta}}-\nabla\mathbf{s}_{\bm{\theta}}^{T})\bm{\bm{\epsilon}}_{1}^{T})^{2}\right]$, where $\bm{\bm{\epsilon}}_{1}$ and $\bm{\bm{\epsilon}}_{2}$ are the random variables of minimum variance, as proposed in Proposition \ref{prop:4}. In practice, we find that the above regularized training loss is unnecessary for INDM because we already achieves the nearly MLE training, but it helps DDPM++ to reduce the variational gap at the expense of $4\times$ slower training speed than the training with unregularized loss in DDPM++. Even with reduced variational gap, we find that NLL of DDPM++ is improved only marginally only on certain training scenarios, and has no effect in most trials, so we leave the detailed effect of MLE training in diffusion models as a future work. Notably, therefore, we conclude that the NLL gain in INDM, compared to DDPM++, essentially originates from $\bm{\phi}$-training and its consequential expanded solution space to $\mathbf{S}_{sol}$.
	
	\section{Details on Section \ref{sec:sampling_robustness}}\label{appendix:sample_robustness}
	
	\subsection{Full Statement of Theorem \ref{thm:dsb}}\label{appendix:dsb}
	We provide a full statement of Theorem \ref{thm:dsb}. Theorem \ref{thm:dsb} is heavily influenced by the theoretic analysis of \citet{de2021diffusion, guth2022wavelet}, and it could be considered as merely an application of their results. It is possible that the inequality in Theorem \ref{thm:dsb} could not be tight, but empirically the robustness is significantly connected to the initial distribution's smoothness.
	\begingroup
	\renewcommand\thetheorem{4}
	\begin{theorem}
		Assume that there exists $M\ge 0$ such that for any $t\in[0,T]$ and $\mathbf{z}\in\mathbb{R}^{d}$, the score estimation is close enough to the forward score by $M$, $\Vert\mathbf{s}_{\bm{\theta}}(\mathbf{x},t)-\nabla\log{p_{t}^{\bm{\phi}}(\mathbf{x})}\Vert\le M$, with $\mathbf{s}_{\bm{\theta}}\in C([0,T]\times\mathbb{R}^{d},\mathbb{R}^{d})$. Assume that $\nabla\log{p_{t}^{\bm{\phi}}(\mathbf{z})}$ is $C^{2}$ in both $t$ and $\mathbf{z}$, and that $\sup_{\mathbf{z},t}\Vert\nabla^{2}\log{p_{t}^{\bm{\phi}}(\mathbf{z})}\Vert\le K\quad\text{and}\quad\Vert\frac{\partial}{\partial t}\nabla\log{p_{t}^{\bm{\phi}}(\mathbf{z})}\Vert\le M e^{-\alpha t}\Vert\mathbf{z}\Vert$ for some $K,M,\alpha>0$. Suppose $(\mathbf{h}_{\bm{\phi}}^{-1})_{\#}$ s a push-forward map. Then $\Vert p_{r}-(\mathbf{h}_{\bm{\phi}}^{-1})_{\#}p_{0,N}^{\bm{\theta}}\Vert_{TV}\le E_{pri}(\bm{\phi})+E_{dis}(\bm{\phi})+E_{est}(\bm{\phi},\bm{\theta})$, where $E_{pri}(\bm{\phi})=\sqrt{2}e^{-T}D_{KL}(p_{T}^{\bm{\phi}}\Vert\pi)^{1/2}$ is the error originating from the prior mismatch; $E_{dis}(\bm{\phi})=6\sqrt{\delta}(1+\mathbb{E}_{p_{0}^{\bm{\phi}}(\mathbf{z})}[\Vert\mathbf{z}\Vert^{4}]^{1/4})(1+K+M(1+\frac{1}{\sqrt{2\alpha}}))$ is the discretization error with $\delta=\frac{\max{\gamma_{k}}^{2}}{\min{\gamma_{k}}}$; $E_{est}(\bm{\phi},\bm{\theta})=2TM^{2}$ is the score estimation error.
	\end{theorem}
	\endgroup
	
	\subsection{Geometric Interpretation of Latent Diffusion}\label{appendix:latent_manifold}
	
	\begin{wrapfigure}{r}{0.5\textwidth}
		\vskip -0.2in
		\centering
		\includegraphics[width=\linewidth]{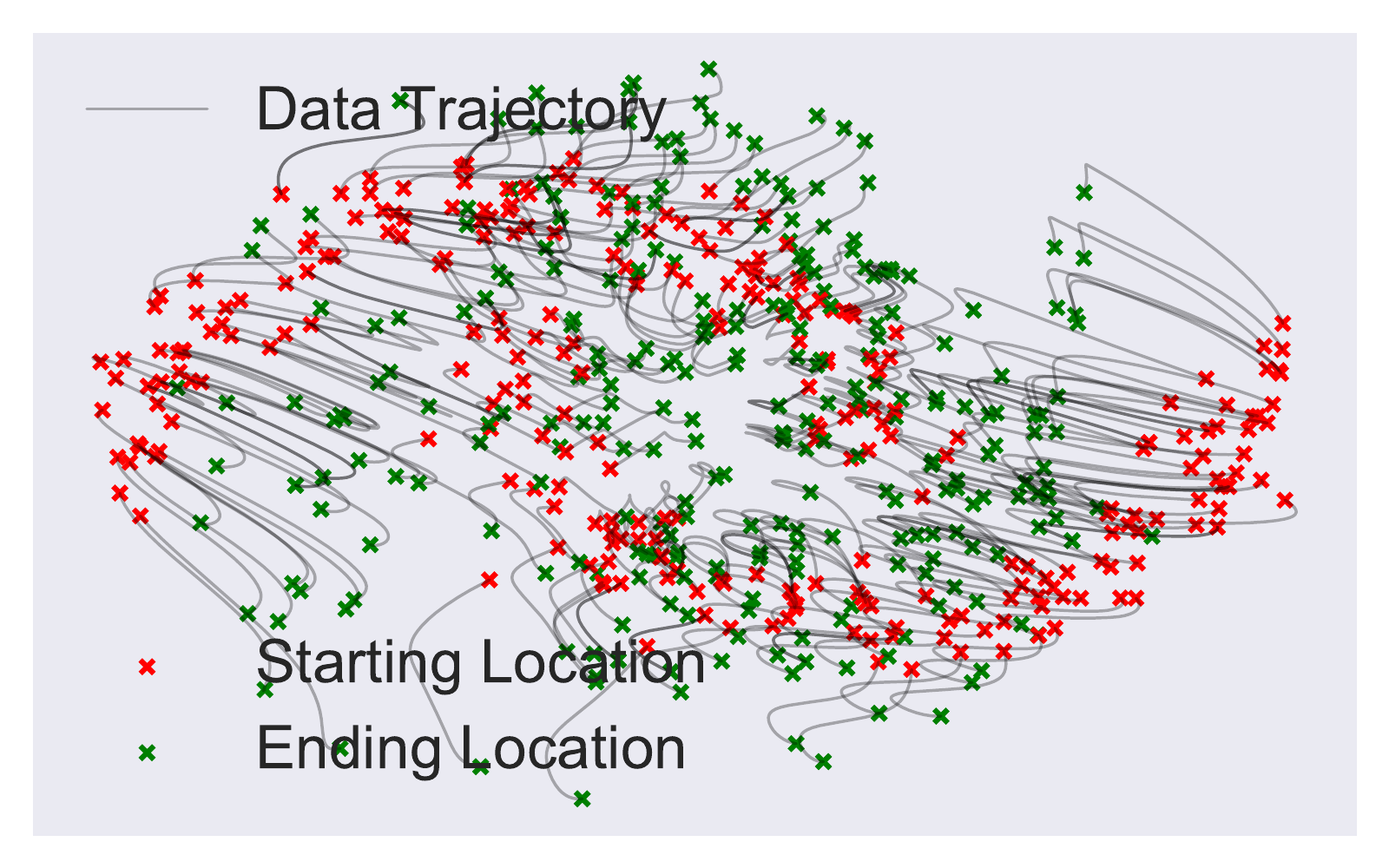}
		\vskip -0.05in
		\caption{Particle trajectories of the probability flow ODE for VPSDE on the synthetic two moons 2d dataset.}
		\label{fig:sample_trajectory_2d}
		\vskip -0.1in
	\end{wrapfigure}
	Figure \ref{fig:sample_trajectory_2d} illustrates the diffusion trajectories of the probability flow ODE of VPSDE. It shows that the trajectories are highly nonlinear, and this section is devoted to analyze why such nonlinear trajectory occurs. Figure \ref{fig:linear_diffusion_by_range} shows two diffusion paths differing only on their scales on (a) the two moons dataset and (b) the ring dataset. The standard Gaussian distribution at $T$ has a larger variance than the initial data at the top row and has a smaller one at the bottom row on each dataset. For the visualization purpose, we zoom in the top row, and we zoom out the bottom row for each dataset, but we fix the xlim and ylim arguments in the matplotlib package \cite{Hunter:2007} row-wisely. With this discrepancy of the initial data scale, the particle trajectory at the bottom row is much more straightforward than in the top row, and it implies that the scale of initial data matters to the straightness of the bridge even if the diffusion SDE is identically linear.
	
	\begin{figure*}[t]
		\vskip -0.1in
		\centering
		\begin{subfigure}{\linewidth}
			\includegraphics[width=\linewidth]{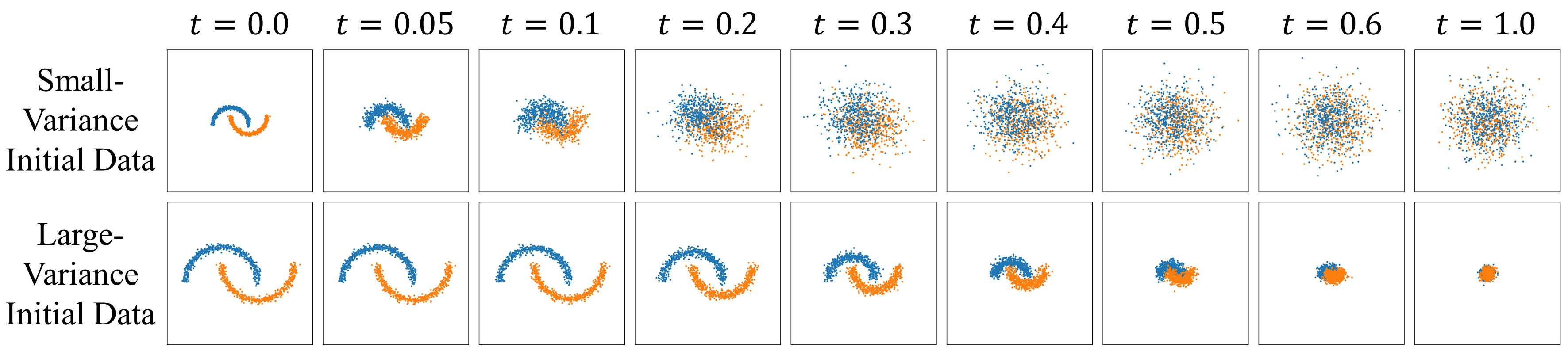}
			\subcaption{Two Moons Dataset}
		\end{subfigure}
		\begin{subfigure}{\linewidth}
			\includegraphics[width=\linewidth]{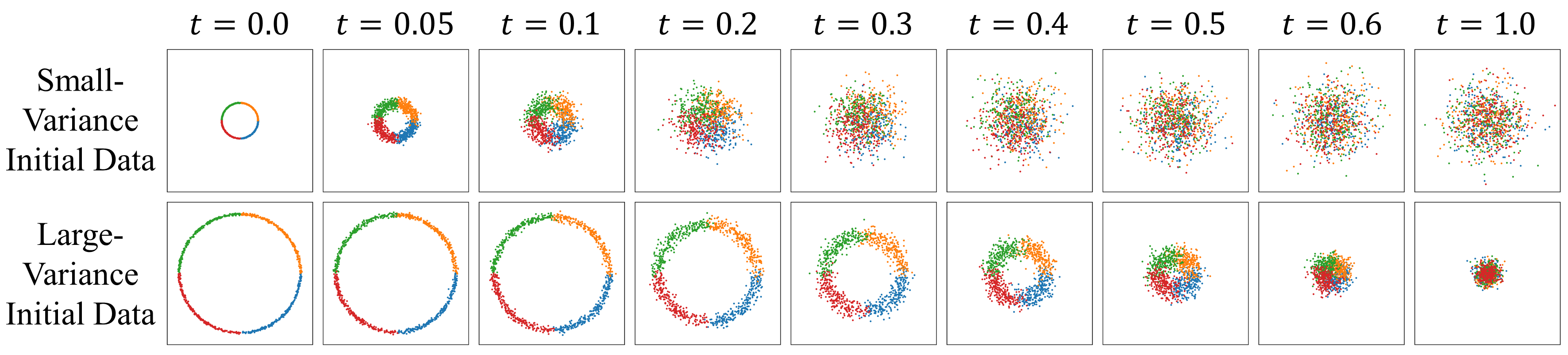}
			\subcaption{Ring Dataset}
		\end{subfigure}
		\vskip -0.05in
		\caption{Comparison of linear diffusion bridges on data and latent spaces in diverse datasets.}
		\label{fig:linear_diffusion_by_range}
		\vskip -0.2in
	\end{figure*}
	
	A behind rationale for this observation comes from the closed-form solution of VPSDE. Suppose the forward diffusion follows VPSDE of $\diff\mathbf{x}_{t}=-\frac{1}{2}\beta(t)\mathbf{x}_{t}\diff t+\sqrt{\beta(t)}\diff\mathbf{w}_{t}$. Then, the solution of this SDE becomes
	\begin{align}\label{eq:sde_solution}
	\mathbf{x}_{t}=\underbrace{e^{-\frac{1}{2}\int_{0}^{t}\beta(s)\diff s}\mathbf{x}_{0}}_{\text{linearly contraction mapping}}+\underbrace{\sqrt{1-e^{-\int_{0}^{t}\beta(s)\diff s}}\bm{\epsilon}}_{\text{random perturbation}},
	\end{align}
	\begin{wrapfigure}{r}{0.6\textwidth}
		\vskip -0.2in
		\centering
		\includegraphics[width=\linewidth]{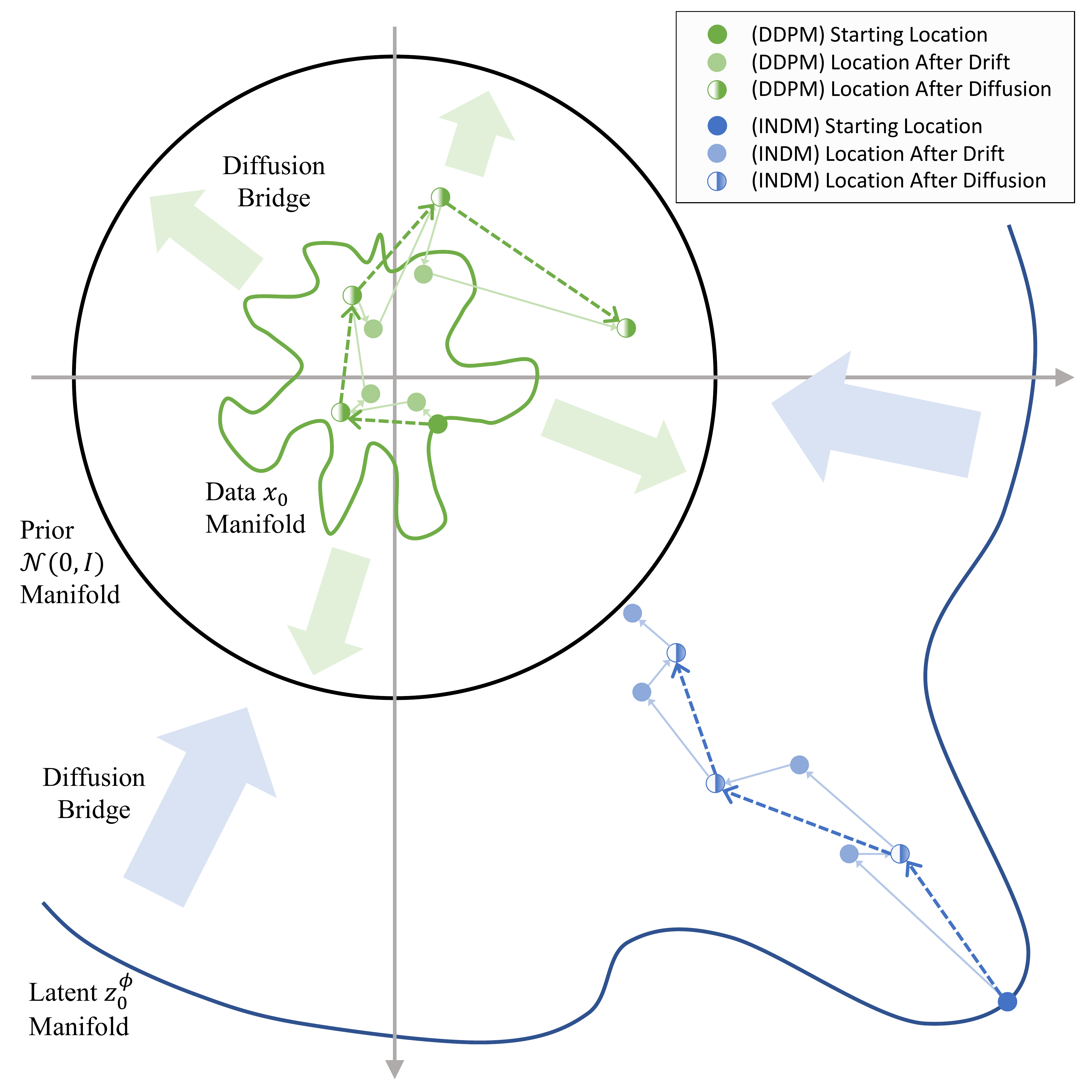}
		\vskip -0.05in
		\caption{Descriptive Illustration On Diffusion Bridge.}
		\label{fig:diffusion_bridge}
		\vskip -0.2in
	\end{wrapfigure}
	where $\bm{\epsilon}\sim\mathcal{N}(0,\mathbf{I})$. As the drift term $-\frac{1}{2}\beta(t)\mathbf{x}_{t}$ ahead towards the origin of $\mathbb{R}^{d}$, the solution in Eq. \eqref{eq:sde_solution} is a summation of the contraction mapping to the origin, $0\in\mathbb{R}^{d}$, with a random noise function, where the magnitude of the random perturbation depends solely on the diffusion coefficient, $g(t)=\sqrt{\beta(t)}$. If $\mathbf{x}_{0}$ is inflated by $c\mathbf{x}_{0}$, then it becomes $\mathbf{x}_{t}=c\times e^{-\frac{1}{2}\int_{0}^{t}\beta(s)\diff s}\mathbf{x}_{0}+\sqrt{1-e^{-\int_{0}^{t}\beta(s)\diff s}}\bm{\epsilon}$ with contraction mapping multiplied by $c$. Therefore, as $c$ increases, the contraction force outweighs the random perturbing effect, and the particle trajectory is becoming straight. 
	
	\begin{table}[t]
		\caption{Statistics of data variable and latent variable on CIFAR-10. All statistics are averaged by dimension.}
		\label{tab:statistics}
		\centering
		\tiny
		\begin{tabular}{ccccc}
			\toprule
			& Mean & Variance & Min & Max \\\midrule
			DDPM++ ($\mathbf{x}_{0}=\mathbf{z}_{0}^{\bm{\phi}_{id}}$) & -0.05 & 0.25 & -1 & 1 \\
			INDM ($\mathbf{z}_{0}^{\bm{\phi}}$) & 0.70 & 9.74 & -8.66 & 12.17 \\
			\bottomrule
		\end{tabular}
	\end{table}

	On a high-dimensional dataset, most of the mass of the standard Gaussian $\pi=\mathcal{N}(0,\mathbf{I})$, which is the prior, is concentrated on a thin spherical shell with squared radius of $d$, according to the Gaussian annulus theorem \cite{blum2020foundations}, as described in the black circle of Figure \ref{fig:diffusion_bridge}. On CIFAR-10, the data distribution has the smaller average square radius of $\mathbb{E}_{p_{r}(\mathbf{x}_{0})}\big[\Vert\mathbf{x}_{0}\Vert_{2}^{2}\big]=776<3072=d$, whereas the latent distribution has a larger average square radius of $\mathbb{E}_{p_{r}(\mathbf{x}_{0})}\big[\Vert\mathbf{h}_{\bm{\phi}}(\mathbf{x}_{0})\Vert_{2}^{2}\big]=\mathbb{E}_{p_{0}^{\bm{\phi}}(\mathbf{z}_{0}^{\bm{\phi}})}\big[\Vert\mathbf{z}_{0}^{\bm{\phi}}\Vert_{2}^{2}\big]>d$ than a standard Gaussian distribution. The latent radius varies from $5,385$ to $31,399$ by experimental settings. Thus, the latent manifold is located outside of the prior on CIFAR-10 as depicted in Figure \ref{fig:diffusion_bridge}. 
	
	When the latent manifold envelops the prior manifold, i.e., $\Vert\mathbf{z}_{0}^{\bm{\phi}}\Vert_{2}>\Vert\mathbf{z}_{T}^{\bm{\phi}}\Vert_{2}$, the drift term, $-\frac{1}{2}\beta(t)\mathbf{z}_{t}^{\bm{\phi}}$, and the vector of $\mathbf{z}_{T}^{\bm{\phi}}-\mathbf{z}_{0}^{\bm{\phi}}$ aligns towards the origin. On the other hand, if the initial manifold is located inside the prior manifold, i.e., $\Vert\mathbf{x}_{0}\Vert_{2}<\Vert\mathbf{x}_{T}\Vert_{2}$, then the drift term points towards the opposite direction of $\mathbf{x}_{T}-\mathbf{x}_{0}$. This leads that the contraction mapping disturbs the particle to move towards $\mathbf{x}_{T}$, and it is the random perturbation that leads the particle to converge to $\mathbf{x}_{T}$. In latent trajectory, the contraction mapping driven by the drift term helps the particle moving towards $\mathbf{z}_{T}^{\bm{\phi}}$. Therefore, the particle trajectory is more straightforward in the latent trajectory, which moves \textit{outside} of the prior manifold, compared to the data trajectory that lives \textit{inside} of the prior manifold. This clarifies why the sampling-friendly bridge is constructed in INDM.
	
	\begin{wrapfigure}{r}{0.6\textwidth}
		\vskip -0.2in	
		\begin{subfigure}{0.48\linewidth}
			\centering
			\includegraphics[width=\linewidth]{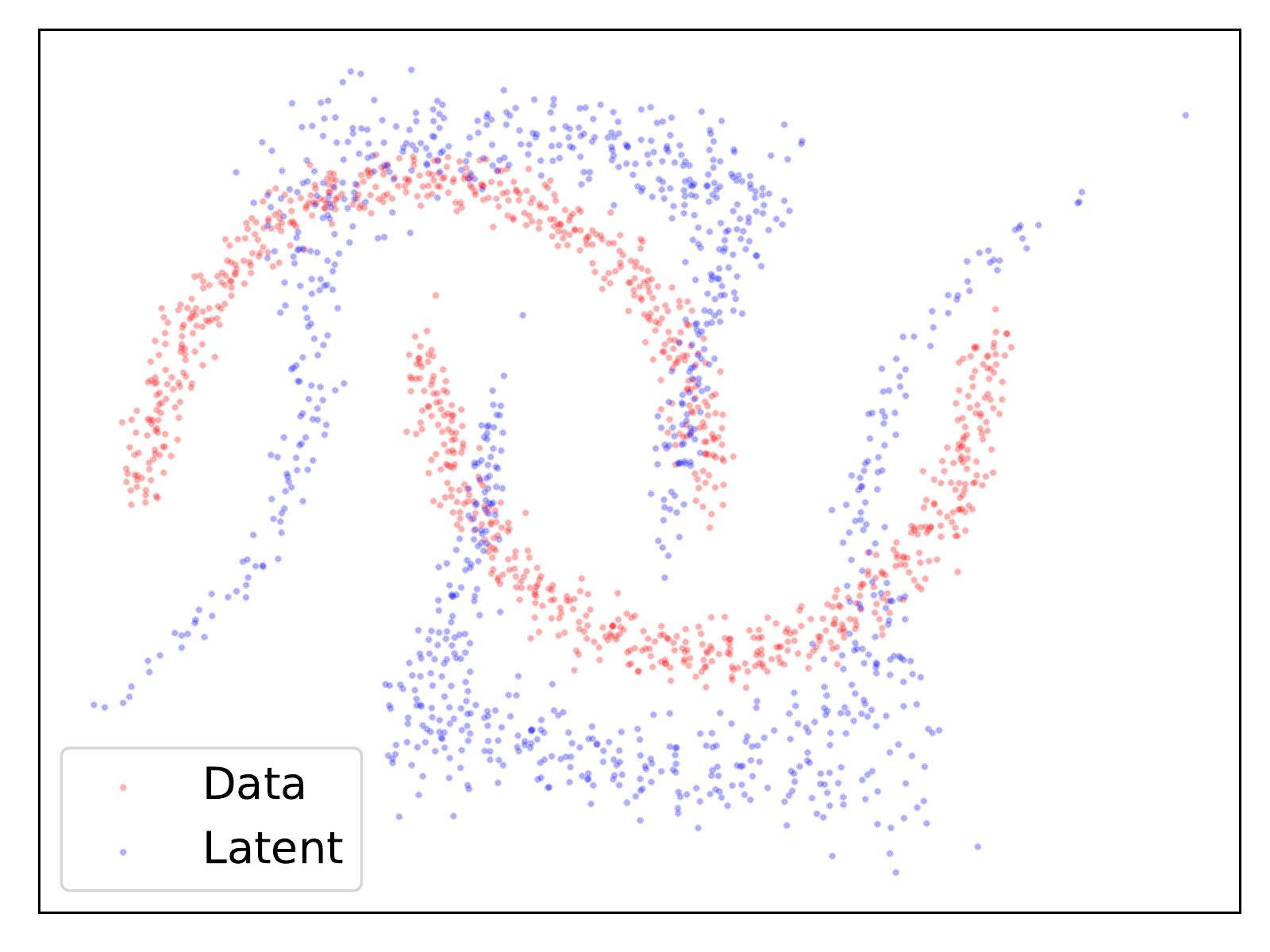}
			\subcaption{Data and Latent Manifolds At Initial Stage of Training}
		\end{subfigure}
		\hfill
		\begin{subfigure}{0.48\linewidth}
			\centering
			\includegraphics[width=\linewidth]{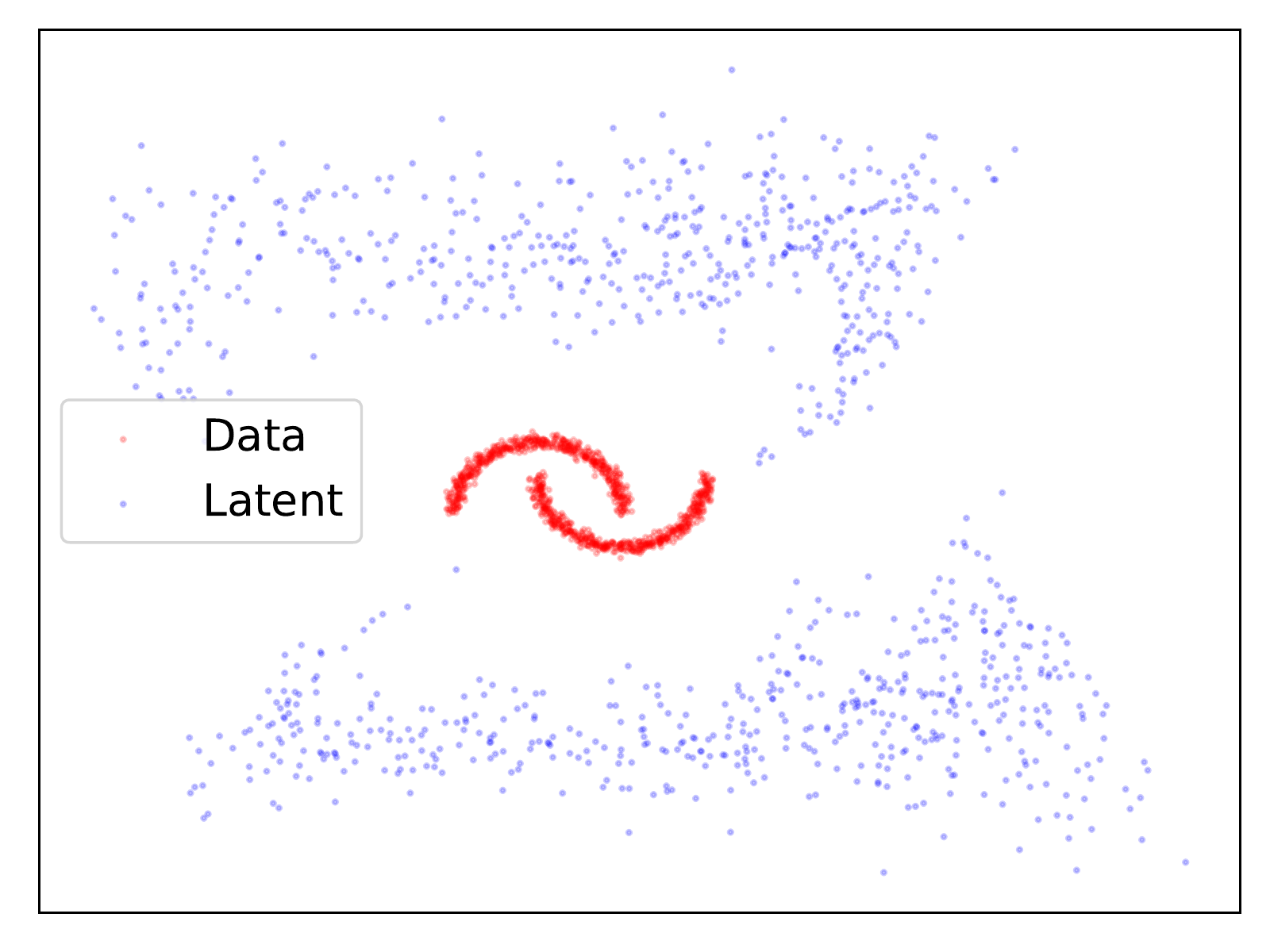}
			\subcaption{Data and Latent Manifolds Afer Training of 10k Steps}
		\end{subfigure}
		\bigskip 
		\begin{subfigure}{0.48\linewidth}
			\centering
			\includegraphics[width=\linewidth]{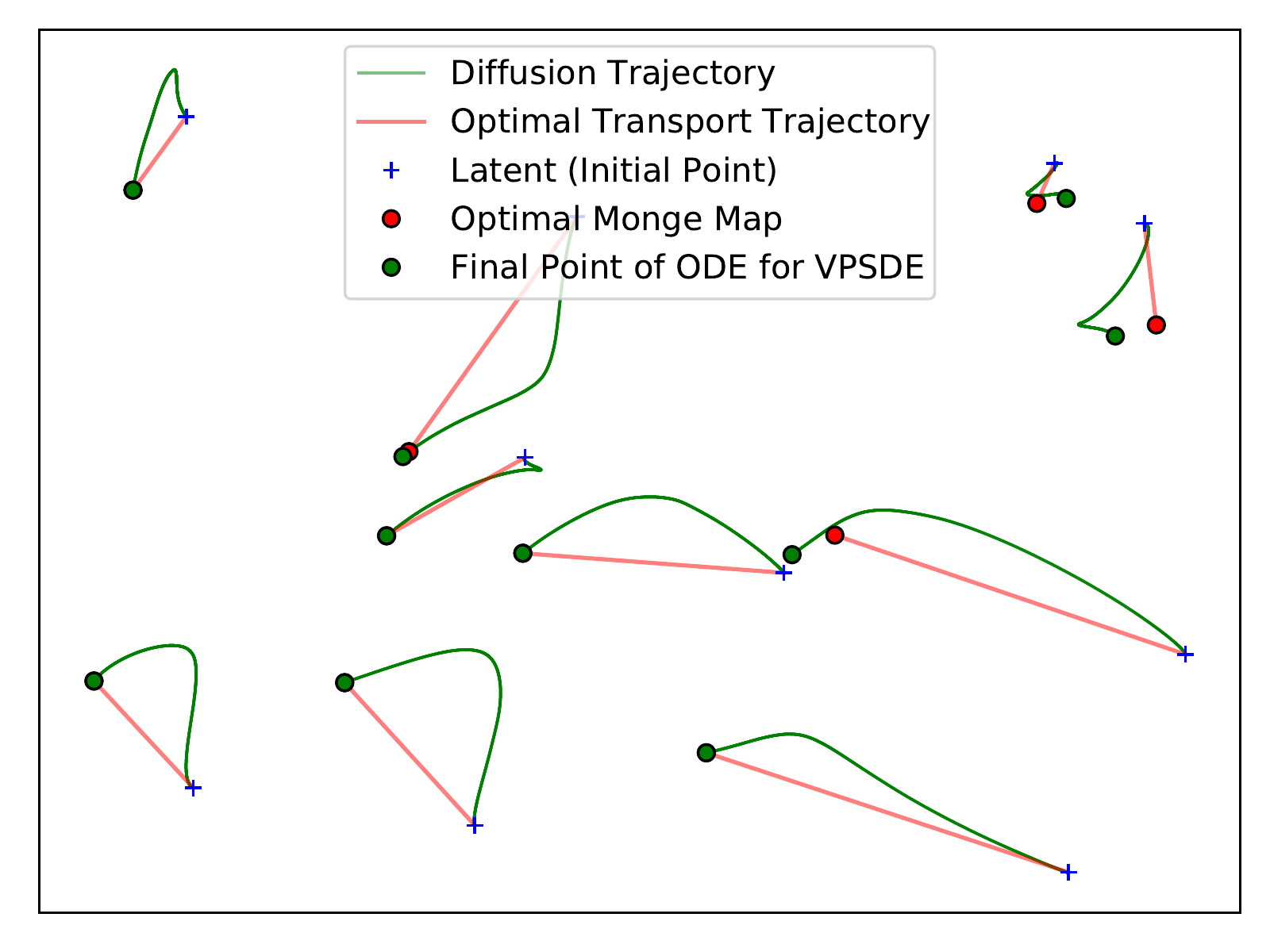}
			\subcaption{Diffusion and (optimal) Monge Trajectories At Initial Stage of Training}
		\end{subfigure}
		\hfill
		\begin{subfigure}{0.48\linewidth}
			\centering
			\includegraphics[width=\linewidth]{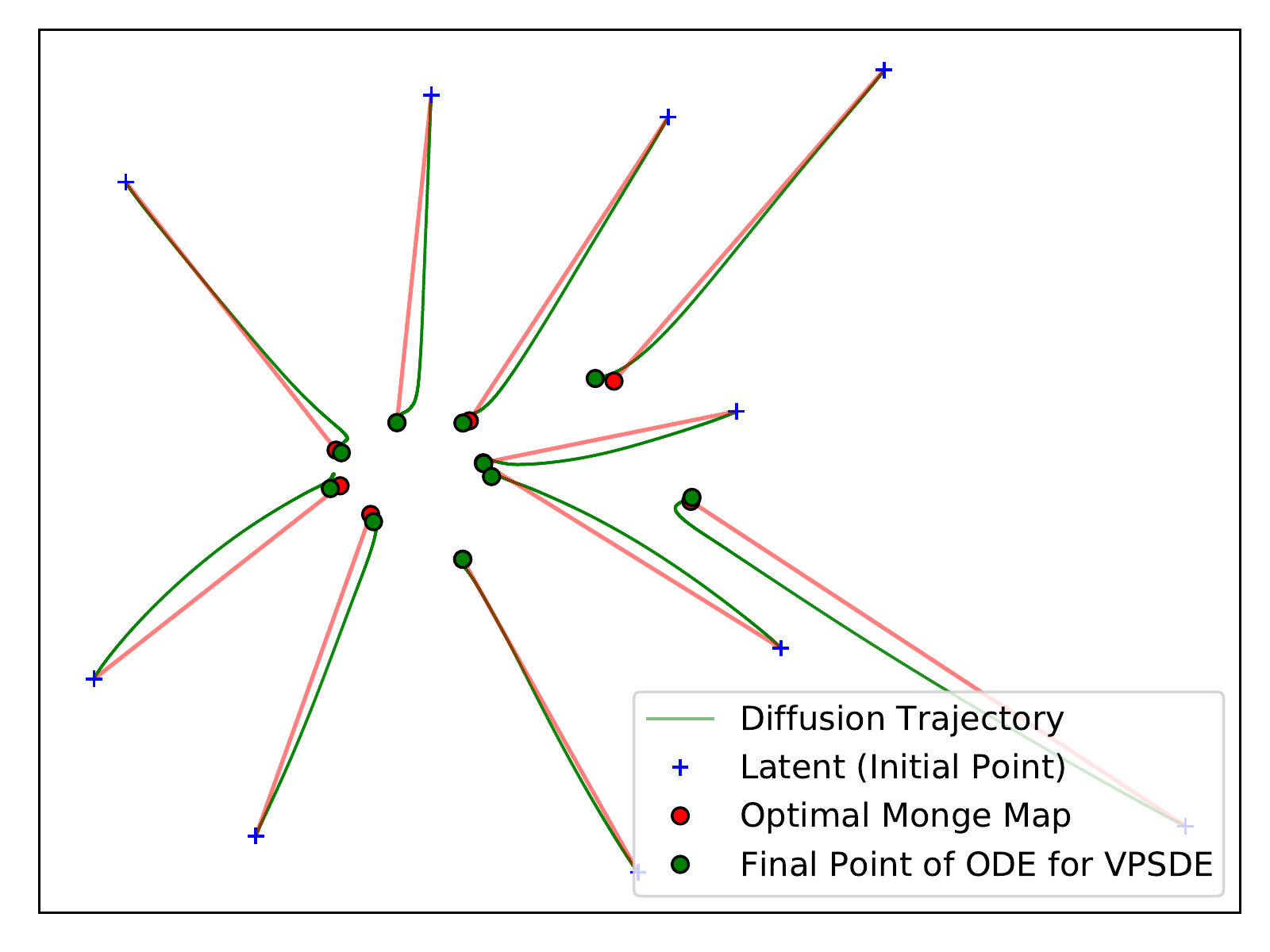}
			\subcaption{Diffusion and (optimal) Monge Trajectories After Training of 10k Steps}
		\end{subfigure}
		\vskip -0.15in
		\caption{(a,b) Latent manifold by training iterations (c,d) Diffusion trajectories by training iterations. We use Python Optimal Transport (POT) library \cite{flamary2021pot} to obtain the optimally transported Monge map between 1,000 samples from the latent starting variable and the latent ending variable. We only visualize 10 samples out of 1,000 transport maps for a clear implication. In (c), we train the score network further until converged (with the fixed flow) to visualize accurate diffusion paths.}
		\label{fig:2d_toy}
		\vskip -0.1in
	\end{wrapfigure}
	Figure \ref{fig:2d_toy} presents the 2d toy case of the two moons dataset. It illustrates a simple visualization of the flow training. Figure \ref{fig:2d_toy} shows that even though the latent manifold is located near the data manifold at the initial phase of training in Figure \ref{fig:2d_toy}-(a), after the training, the latent manifold is inflated to the outside of the real data in Figure \ref{fig:2d_toy}-(b). Therefore, the probability flow ODE (deterministic trajectory), after the training, transports the initial mass to the final mass with a nearly linear line in Figure \ref{fig:2d_toy}-(d), in contrast to the curvy VPSDE trajectory at the initial phase of training in Figure \ref{fig:2d_toy}-(c). In this example, the flow training puts the latent manifold out of the data manifold, and this helps the robust sampling.
	
	In addition, Figure \ref{fig:2d_toy} illustrates the Monge trajectories between the latent initial distribution and the prior distribution. As theoretically demonstrated in Gaussian and empirically shown in general distribution in \citet{khrulkov2022understanding}, the encoder map of VPSDE is nearly optimal transport under the squared Euclidean cost function, where the encoder map is the mapping from the initial point to the final point passed through the probability flow ODE. Figure \ref{fig:2d_toy} supports this, and the diffusion trajectory becomes more straight alike to the optimal Monge map after the training.
	
	\begin{wrapfigure}{r}{0.6\textwidth}
		\vskip -0.25in
		\centering
		\includegraphics[width=\linewidth]{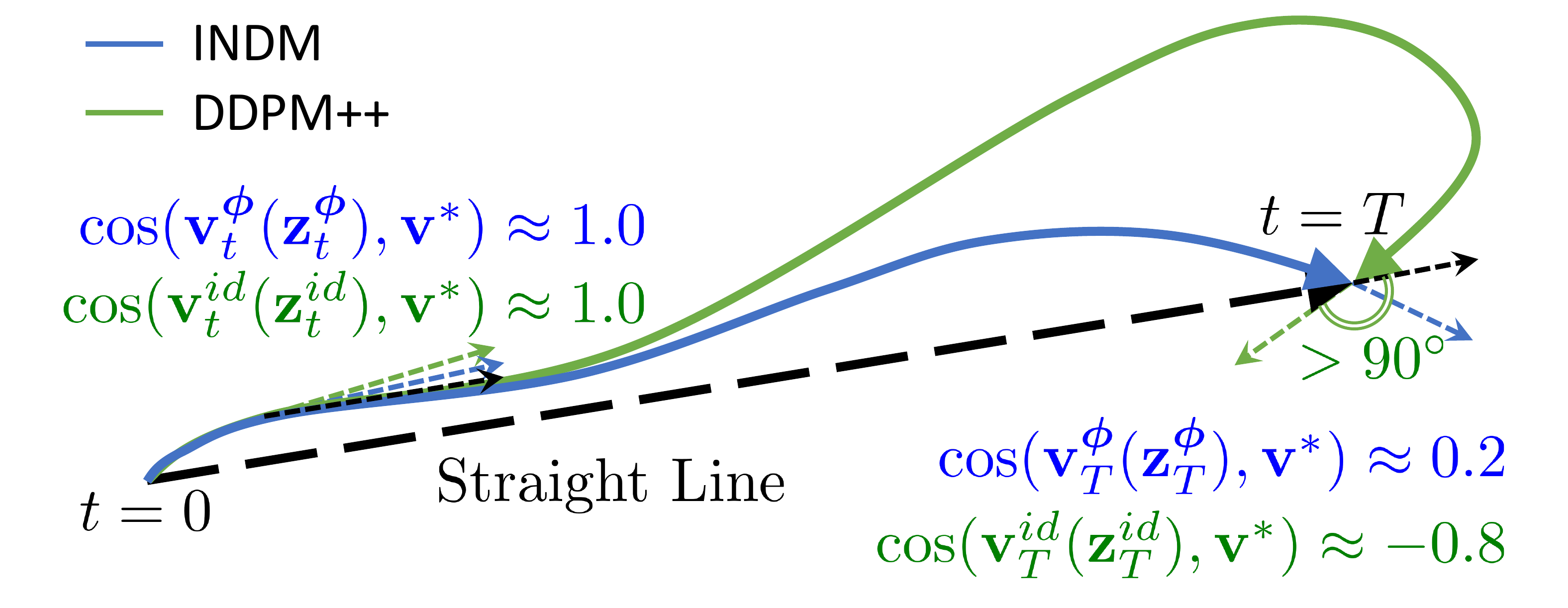}
		\vskip -0.05in
		\caption{Illustrative Particle Trajectory.}
		\label{fig:illustrative_particle_trajectory}
		\vskip -0.35in
	\end{wrapfigure}
	Figure \ref{fig:illustrative_particle_trajectory} illustrates the concept of linearized diffusion path. As the flow inflates the latent manifold, the diffusion trajectory becomes more linear, and Figure \ref{fig:latent_geometry} supports the conceptual illustration of Figure \ref{fig:illustrative_particle_trajectory} on CIFAR-10.
	
	\section{Related Work}\label{appendix:related_work_appendix}

	\subsection{Latent Score-based Generative Model (LSGM)}\label{appendix:LSGM}
	
	The diffusion process on latent space is firstly introduced in LSGM. LSGM transforms the data variable to a latent variable, and estimates the prior distribution with a diffusion model. Suppose $\bm{\theta}$, $\bm{\phi}$, and $\bm{\psi}$ represent for the parameters for the score network, the encoder network, and the decoder network, respectively. Then, LSGM optimizes the loss of
	\begin{eqnarray*}
		\lefteqn{D_{KL}(p_{r}\Vert p_{\bm{\theta},\bm{\psi}})\le D_{KL}\big(p_{r}(\mathbf{x}_{0})q_{\bm{\phi}}(\mathbf{z}_{0}\vert\mathbf{x}_{0})\Vert p_{\bm{\theta}}(\mathbf{z}_{0})p_{\bm{\psi}}(\mathbf{x}_{0}\vert\mathbf{z}_{0})\big)}&\\
		&&= D_{KL}\big(p_{r}(\mathbf{x}_{0})q_{\bm{\phi}}(\mathbf{z}_{0}\vert\mathbf{x}_{0})\Vert q_{\bm{\phi}}(\mathbf{z}_{0})p_{\bm{\psi}}(\mathbf{x}_{0}\vert\mathbf{z}_{0})\big)+D_{KL}\big(q_{\bm{\phi}}(\mathbf{z}_{0})\Vert p_{\bm{\theta}}(\mathbf{z}_{0})\big)\\
		&&\le D_{KL}\big(p_{r}(\mathbf{x}_{0})q_{\bm{\phi}}(\mathbf{z}_{0}\vert\mathbf{x}_{0})\Vert q_{\bm{\phi}}(\mathbf{z}_{0})p_{\bm{\psi}}(\mathbf{x}_{0}\vert\mathbf{z}_{0})\big)+D_{KL}\big(\bm{\mu}_{\bm{\phi}}(\{\mathbf{z}_{t}\}_{t=0}^{T})\Vert \bm{\nu}_{\bm{\theta}}(\{\mathbf{z}_{t}\}_{t=0}^{T})\big)\\
		&&=\mathcal{L}_{LSGM}(\bm{\theta},\bm{\phi},\bm{\psi})
	\end{eqnarray*}
	where $q_{\bm{\phi}}(\mathbf{z}_{0})$ is the marginal distribution of the encoder posterior, $q_{\bm{\phi}}(\mathbf{z}_{0})=\int p_{r}(\mathbf{x}_{0})q_{\bm{\phi}}(\mathbf{z}_{0}\vert\mathbf{x}_{0})\diff\mathbf{x}_{0}$.
	
	As well as INDM, LSGM also optimizes the log-likelihood of the model distribution by using a diffusion model in the latent space. Though both INDM and LSGM losses include a denoising score loss on the latent space (which is the KL divergence between path measures on the latent space), $\mathcal{L}_{LSGM}(\bm{\theta},\bm{\phi},\bm{\psi})$ is not equivalent to the KL divergence between the forward and generative path measures on the data space, in contrast to INDM with $D_{KL}(\bm{\mu}_{\bm{\phi}}(\{\mathbf{x}_{t}\}_{t=0}^{T}\Vert\bm{\nu}_{\bm{\phi},\bm{\theta}}(\{\mathbf{x}_{t}\}_{t=0}^{T})$ as its loss function. In fact, there is no forward SDE (green path in Figure 3) on the data space in LSGM according to Lemma \ref{lemma:3}, which is a direct application of the Borsuk-Ulam theorem \cite{bredon2013topology}. 
	\begin{lemma}[\textbf{$\mathbb{R}^{n}$ is not homeomorphic to $\mathbb{R}^{m}$} \cite{bredon2013topology}]\label{lemma:3}
		If $n\neq m$, there is no continuous map $E:\mathbb{R}^{n}\rightarrow\mathbb{R}^{m}$ that has the continuous inverse map $E^{-1}:\mathbb{R}^{m}\rightarrow\mathbb{R}^{n}$.
	\end{lemma}
	Lemma \ref{lemma:3} implies that there is no inverse function of the encoder as long as the latent dimension is different from the data dimension (and the activation function is continuous, such as ReLU). From this, LSGM cannot define a random variable on the data space by $\mathbf{x}_{t}^{\bm{\phi}}=E_{\bm{\phi}}^{-1}(\mathbf{z}_{t})$, in contrast to INDM that defines $\mathbf{x}_{t}^{\bm{\phi}}=\mathbf{h}_{\bm{\phi}}^{-1}(\mathbf{z}_{t})$. This non-existence of random variables on the data space implies that \textit{the forward diffusion process does not exists as long as the latent dimension differs to the data dimension}.
	
	With the above theoretic dilemma of LSGM, one could build a generative diffusion process on the data space. If $\mathbf{x}_{t}^{\bm{\psi},\bm{\theta}}:=D_{\bm{\psi}}(\mathbf{z}_{t}^{\bm{\theta}})$, where $\mathbf{z}_{t}^{\bm{\theta}}$ is a generative random variable on the latent space, and $D_{\bm{\psi}}$ is a decoder map, then we could build a generative diffusion process on the data space through the Ito's formula in the same way as we did in INDM. Inspired by this, one could argue that the forward diffusion could be constructed by $\mathbf{x}_{t}^{\bm{\psi}}:=D_{\bm{\psi}}(\mathbf{z}_{t})$, where $\mathbf{z}_{t}$ is a forward random variable on the latent space. This construction enables to construct a forward diffusion process on the latent space, but there are a couple of caveats to this construction.

	\begin{wraptable}{r}{0.4\textwidth}
		\vskip -0.2in
		\caption{LSGM training fails when using the variance weighting function.}
		\label{tab:lsgm_comparison}
		\scriptsize
		\centering
		\begin{tabular}{lccc}
			\toprule
			& NLL & NELBO & FID\\\midrule
			LSGM (VP, FID) & NaN & NaN & NaN \\
			INDM (VP, FID) & 3.23 & 3.17 & 2.90 \\
			\bottomrule
		\end{tabular}
		\vskip -0.1in
	\end{wraptable}
	Theoretically, this forward diffusion process starts from the reconstructed variable, $\mathbf{x}_{0}^{\bm{\psi}}=D_{\bm{\psi}}(\mathbf{z}_{0})=D_{\bm{\psi}}(E_{\bm{\phi}}(\mathbf{x}_{0}))=\mathbf{x}_{rec}$, where $\mathbf{x}_{0}$ and $\mathbf{x}_{rec}$ differs throughout the training procedure. In addition, even if we admit $\{\mathbf{x}_{t}^{\bm{\psi}}\}$ as a forward diffusion, $\mathcal{L}_{LSGM}(\bm{\theta},\bm{\phi},\bm{\psi})$ cannot be derived as the KL divergence of path measures for the forward diffusion (admittably $\{\mathbf{x}_{t}^{\bm{\psi}}\}$, but not true to be precise) and the generative diffusion ($\mathbf{x}_{t}^{\bm{\psi},\bm{\theta}}$) on the data space. Instead, the loss contains the encoder parameters to optimize, and the loss diverges from the KL divergence on the data space. Also, hypothetically, even if the loss is the KL divergence of the forward and generative path measures on the data space, the optimization could be drifted away from the optimal point because the forward diffusion starts from untrained reconstructed variable, $\mathbf{x}_{rec}$, which is not close to the data variable, $\mathbf{x}_{0}$. This analysis provides a clue to explain the training instability of LSGM as reported in \citet{vahdat2021score} and \citet{dockhorn2021score}, in contrast to INDM that is stable to train in any training configuration. Table \ref{tab:lsgm_comparison} shows a fast comparison of LSGM and INDM with variance weighting function, sampled from $t\in\mathcal{U}[0,1]$. \textit{NaN} indicates experiments that fail due to training instability, see Section 5.2 and Table 6 of \citet{vahdat2021score} and Section E.2.7 of \citet{dockhorn2021score}. 
	
	\begin{table*}[t]
		\caption{Comparison of latent dimension of INDM and LSGM.}
		\label{tab:dimension}
		\scriptsize
		\centering
		\begin{tabular}{lccc}
			\toprule
			Datset & Data Dimension & Latent Dimension of INDM & Latent Dimension of INDM \\\midrule
			MNIST & 784 & 784 & 2,560 \\
			CIFAR-10 & 3,072 & 3,072 & 46,080 \\
			CelebA-HQ 256 & 196,608 & 196,608 & 819,200 \\
			\bottomrule
		\end{tabular}
	\end{table*}
	Moreover, Table \ref{tab:dimension} compares INDM with LSGM in terms of the latent dimension. We compute the latent dimension of LSGM, according to their paper and released checkpoint. Contrary to the dimensional reduction property which is the crux of the auto-encoding structure, LSGM maps data into a latent space of a much higher dimension than the data dimension. LSGM is known to perform well, but having observed 15x higher latent dimension than the data dimension on CIFAR-10, the good performance was not gained for free. On the other hand, INDM always retains the same dimension to the data, while keeping the invertibility.
	
	\subsection{Diffusion Normalizing Flow (DiffFlow)}\label{appendix:DiffFlow}
	
	The Girsanov theorem \cite{sarkka2019applied} proves that the variational bound is derived by
	\begin{align}\label{eq:ncsn_loss}
	D_{KL}(p_{r}\Vert p_{\bm{\theta}})\le \frac{1}{2}\int_{0}^{T}g^{2}(t)\mathbb{E}_{p_{r}(\mathbf{x}_{0})}\mathbb{E}_{p_{0t}(\mathbf{x}_{t}\vert\mathbf{x}_{0})}\big[\Vert\mathbf{s}_{\bm{\theta}}(\mathbf{x}_{t},t)-\nabla_{\mathbf{x}_{t}}\log{p_{0t}(\mathbf{x}_{t}\vert\mathbf{x}_{0})}\Vert_{2}^{2}\big]\diff t+D_{KL}(p_{T}\Vert\pi).
	\end{align}
	When the forward diffusion is given as $\diff\mathbf{x}_{t}=\mathbf{f}_{\bm{\phi}}(\mathbf{x}_{t},t)\diff t+g(t)\diff\mathbf{w}_{t}$, where $\mathbf{f}_{\bm{\phi}}$ is an explicit parametrization of the drift term by a normalizing flow with parameters $\bm{\phi}$, then the transition probability, $p_{0t}(\mathbf{x}_{t}\vert\mathbf{x}_{0})$, becomes intractable. Therefore, optimizing the continuous variational bound is not feasible. One might detour this issue by alternatively optimizing the continuous DDPM++ loss of
	\begin{align}\label{eq:ddpm_loss}
	\int_{0}^{T}\tilde{\lambda}(t)\mathbb{E}_{p_{r}(\mathbf{x}_{0})}\mathbb{E}_{\bm{\epsilon}\sim\mathcal{N}(0,\mathbf{I})}\big[\Vert\bm{\epsilon}-\bm{\hat{\epsilon}}_{\bm{\theta}}(\mathbf{x}_{t},t)\Vert_{2}^{2}\big]\diff t,
	\end{align}
	but the denoising score loss of Eq. \eqref{eq:ncsn_loss} is not equivalent to the continuous DDPM++ loss of Eq. \eqref{eq:ddpm_loss} when the transition probability is no longer a Gaussian distribution.
	
	DiffFlow detours the intractability issue of the continuous loss of Eq. \eqref{eq:ncsn_loss} by discretizing the nonlinear SDE in the Euler-Maruyama (EM) fashion \cite{higham2001algorithmic}. We construct the discrete random variables that approximate the nonlinear SDE by the induction. If $\mathbf{x}_{t_{0}}^{\bm{\phi},\text{EM}}:=\mathbf{x}_{0}^{\bm{\phi}}$ and $\Delta t_{i}:=t_{i}-t_{i-1}$, where $\{t_{i}\}_{t=0}^{N}$ are discretization timesteps with $t_{0}=0$ and $t_{N}=T$, then the solution of the nonlinear SDE that starts from $\mathbf{x}_{t_{i-1}}^{\bm{\phi},\text{EM}}$ is
	\begin{align}\label{eq:approximate_diffflow}
	\mathbf{x}_{t_{i}}^{\bm{\phi}}-\mathbf{x}_{t_{i-1}}^{\bm{\phi},\text{EM}}=&\int_{t_{i-1}}^{t_{i}}\mathbf{f}_{\bm{\phi}}(\mathbf{x}_{t}^{\bm{\phi}},t)\diff t+\int_{t_{i-1}}^{t_{i}}g(t)\diff\mathbf{w}_{t}.
	\end{align}
	Here, the integral of the drift term is 
	\begin{align*}
	\int_{t_{i-1}}^{t_{i}}\mathbf{f}_{\bm{\phi}}(\mathbf{x}_{t}^{\bm{\phi}},t)\diff t=&\int_{t_{i-1}}^{t_{i}}\mathbf{f}_{\bm{\phi}}\big(\mathbf{x}_{t_{i-1}}^{\bm{\phi},\text{EM}}+(\mathbf{x}_{t}^{\bm{\phi}}-\mathbf{x}_{t_{i-1}}^{\bm{\phi},\text{EM}}),t_{i-1}+(t-t_{i-1})\big)\diff t\\
	=&\int_{t_{i-1}}^{t_{i}}\mathbf{f}_{\bm{\phi}}(\mathbf{x}_{t_{i-1}}^{\bm{\phi},\text{EM}},t_{i-1})\diff t+O(\Delta t_{i}^{3/2})\\
	=&\mathbf{f}_{\bm{\phi}}(\mathbf{x}_{t_{i-1}}^{\bm{\phi},\text{EM}},t_{i-1})\Delta t_{i}+O(\Delta t_{i}^{3/2}),
	\end{align*}
	and the integral of the volatility term is 
	\begin{align*}
	\int_{t_{i-1}}^{t_{i}}g(t)\diff\mathbf{w}_{t}=g(t_{i-1})(\mathbf{w}_{t_{i}}-\mathbf{w}_{t_{i-1}})+O(\Delta t_{i}^{3/2})=g(t_{i-1})\bm{\epsilon}\sqrt{\Delta t_{i}}+O(\Delta t_{i}^{3/2}),
	\end{align*}
	where $\bm{\epsilon}\sim\mathcal{N}(0,\mathbf{I})$. Therefore, DiffFlow defines the next discretized random variable, $\mathbf{x}_{t_{i}}^{\bm{\phi},\text{EM}}$, to be
	\begin{align*}
	\mathbf{x}_{t_{i}}^{\bm{\phi}}=&\mathbf{x}_{t_{i-1}}^{{\bm{\phi}},\text{EM}}+\mathbf{f}_{\bm{\phi}}(\mathbf{x}_{t_{i-1}}^{{\bm{\phi}},\text{EM}},t_{i-1})\Delta t_{i}+g(t_{i-1})\bm{\epsilon}\sqrt{\Delta t_{i}}+O(\Delta t_{i}^{2/3})\\
	\approx&\mathbf{x}_{t_{i-1}}^{{\bm{\phi}},\text{EM}}+\mathbf{f}_{\bm{\phi}}(\mathbf{x}_{t_{i-1}}^{{\bm{\phi}},\text{EM}},t_{i-1})\Delta t_{i}+g(t_{i-1})\bm{\epsilon}\sqrt{\Delta t_{i}}\\
	=&\mathbf{x}_{t_{i}}^{{\bm{\phi}},\text{EM}},
	\end{align*}
	and this Euler-Maruyama random variable $\mathbf{x}_{t_{i}}^{\bm{\phi},\text{EM}}$ follows a Gaussian distribution of mean $\mathbf{x}_{t_{i-1}}^{\bm{\phi},\text{EM}}+\mathbf{f}_{\bm{\phi}}(\mathbf{x}_{t_{i-1}}^{\bm{\phi},\text{EM}},t_{i-1})\Delta t_{i}$ and variance $g^{2}(t_{i-1})\Delta t_{i}$. Note that this discretization approximates the nonlinear SDE with a finite Markov chain of $\{\mathbf{x}_{t_{i}}^{\text{EM}}\}_{i=0}^{N}$. 
	
	DiffFlow constructs the generative process as
	\begin{align*}
	\mathbf{x}_{t_{i-1}}^{\bm{\theta}}=\mathbf{x}_{t_{i}}^{\bm{\theta}}-\big[\mathbf{f}_{\bm{\phi}}(\mathbf{x}_{t_{i}}^{\bm{\phi}},t_{i})-g^{2}(t_{i})\mathbf{s}_{\bm{\theta}}(\mathbf{x}_{t_{i}}^{\bm{\phi}},t_{i})\big]\Delta t_{i}+g(t_{i})\bm{\epsilon}\sqrt{\Delta t_{i}}.
	\end{align*}
	Then, from the Jensen's inequality, the discrete DDPM loss satisfies
	\begin{align}\label{eq:discrete_ddpm_loss}
	D_{KL}(p_{r}\Vert p_{\bm{\phi},\bm{\theta}})\le\sum_{i=1}^{N-1}\mathbb{E}_{p_{r}(\mathbf{x}_{t_{0}}^{\text{EM}})}\mathbb{E}_{p_{\bm{\phi}}(\mathbf{x}_{t_{i}}^{\text{EM}},\mathbf{x}_{t_{i-1}}^{\text{EM}}\vert\mathbf{x}_{t_{0}}^{\text{EM}})}\big[D_{KL}(p_{\bm{\phi}}(\mathbf{x}_{t_{i-1}}^{\text{EM}}\vert\mathbf{x}_{t_{i}}^{\text{EM}},\mathbf{x}_{t_{0}}^{\text{EM}})\Vert p_{\bm{\theta}}(\mathbf{x}_{t_{i-1}}^{\text{EM}}\vert\mathbf{x}_{t_{i}}^{\text{EM}}))\big].
	\end{align}
	While the true inference distribution on the continuous variables, $p_{\bm{\phi}}(\mathbf{x}_{t_{i-1}}\vert\mathbf{x}_{t_{i}},\mathbf{x}_{t_{0}})$, is not a Gaussian distribution due to terms related to $O(\Delta t_{i}^{3/2})$, the inference distribution on the \textit{discretized} variables, $p_{\bm{\phi}}(\mathbf{x}_{t_{i-1}}^{\text{EM}}\vert\mathbf{x}_{t_{i}}^{\text{EM}},\mathbf{x}_{t_{0}}^{\text{EM}})$, becomes a Gaussian distribution by the Euler-Maruyama-style discretization. Therefore, Eq. \eqref{eq:discrete_ddpm_loss} reduces to a tractable loss that does not need to compute the transition probability:
	\begin{align}\label{eq:ddpm_style_loss}
	\begin{split}
	D_{KL}(p_{r}\Vert p_{\bm{\phi},\bm{\theta}})&\le\sum_{i=1}^{N-1}\mathbb{E}_{p_{r}(\mathbf{x}_{t_{0}}^{\text{EM}})}\mathbb{E}_{p_{\bm{\phi}}(\mathbf{x}_{t_{i}}^{\text{EM}},\mathbf{x}_{t_{i-1}}^{\text{EM}}\vert\mathbf{x}_{t_{0}}^{\text{EM}})}\big[D_{KL}(p_{\bm{\phi}}(\mathbf{x}_{t_{i-1}}^{\text{EM}}\vert\mathbf{x}_{t_{i}}^{\text{EM}},\mathbf{x}_{t_{0}}^{\text{EM}})\Vert p_{\bm{\theta}}(\mathbf{x}_{t_{i-1}}^{\text{EM}}\vert\mathbf{x}_{t_{i}}^{\text{EM}}))\big]\\
	&=\frac{1}{2}\sum_{i=1}^{N-1}\mathbb{E}_{p_{r}(\mathbf{x}_{t_{0}}^{\text{EM}})}\mathbb{E}_{p_{\bm{\phi}}(\mathbf{x}_{t_{i}}^{\text{EM}},\mathbf{x}_{t_{i-1}}^{\text{EM}}\vert\mathbf{x}_{t_{0}}^{\text{EM}})}\bigg[\frac{1}{g^{2}(t_{i})\Delta t_{i}}\Big\Vert\mathbf{x}_{t_{i-1}}^{\text{EM}}-\mathbf{x}_{t_{i}}^{\text{EM}}\\
	&\quad\quad\quad\quad\quad\quad\quad+\big[\mathbf{f}_{\bm{\phi}}(\mathbf{x}_{t_{i}}^{\text{EM}},t_{i})-g^{2}(t_{i})\mathbf{s}_{\bm{\theta}}(\mathbf{x}_{t_{i}}^{\text{EM}},t_{i})\big]\Delta t_{i}\Big\Vert_{2}^{2}\bigg]\\
	&=\mathcal{L}_{\text{DiffFlow}}(\bm{\phi},\bm{\theta})
	\end{split}
	\end{align}
	While Eq. \eqref{eq:ddpm_style_loss} does not need to compute the transition probability, another issue of optimizing the variational bound originates from the expectation of $\mathbb{E}_{p_{\bm{\phi}}(\mathbf{x}_{t_{i}}^{\text{EM}},\mathbf{x}_{t_{i-1}}^{\text{EM}}\vert\mathbf{x}_{t_{0}}^{EM})}$. The empirical Monte-Carlo estimation is too expensive because a realization of $\mathbf{x}_{t_{i}}^{\text{EM}}$ needs $i$ number of flow evaluations. In total, summing $i$ over $i=1$ to $N$ requires $O(N^{2})$ flow evaluations to estimate the discrete variational bound of Eq. \eqref{eq:ddpm_style_loss}. Therefore, DiffFlow exchanges the summation and the expectation to reduce the number of flow evaluations by
	\begin{align}
	\begin{split}\label{eq:exchange_sum_and_expectation}
	&\mathcal{L}_{\text{DiffFlow}}(\bm{\phi},\bm{\theta})=\frac{1}{2}\mathbb{E}_{\{\mathbf{x}_{t_{i}}\}_{i=0}^{N-1}\sim p_{\bm{\phi}}(\mathbf{x}_{t_{0}},...,\mathbf{x}_{t_{N-1}})}\bigg[\sum_{i=1}^{N-1}\frac{1}{g^{2}(t_{i})\Delta t_{i}}\Big\Vert\mathbf{x}_{t_{i-1}}-\mathbf{x}_{t_{i}}\\
	&\quad\quad\quad\quad\quad\quad\quad+\big[\mathbf{f}_{\bm{\phi}}(\mathbf{x}_{t_{i}},t_{i})-g^{2}(t_{i})\mathbf{s}_{\bm{\theta}}(\mathbf{x}_{t_{i}},t_{i})\big]\Delta t_{i}\Big\Vert_{2}^{2}\bigg].
	\end{split}
	\end{align}
	This reformulated Eq. \eqref{eq:exchange_sum_and_expectation} estimates $\mathcal{L}_{\text{DiffFlow}}$ with a single sample path from the Markov chain of $\{\mathbf{x}_{t_{i}}^{\text{EM}}\}_{t=1}^{N}$, so it requires $O(N)$ flow evaluations to estimate $\mathcal{L}_{\text{DiffFlow}}(\bm{\phi},\bm{\theta})$. Therefore, DiffFlow takes $O(N)$ computational complexity in total for every optimization step.
	
	There are five differences between DiffFlow and INDM. Basically, these differences arise from the different usage of the flow transformation between DiffFlow and INDM. First, INDM enables to train the continuous diffusion model without the sacrifice on training time, while DiffFlow is limited on the discrete diffusion model at the expense of slower training time. DiffFlow approximates the forward nonlinear SDE with a finite Markov chain. Suppose $\mathbf{x}_{t}^{\text{EM}}$ to be the continuous-time random variable defined by $\mathbf{x}_{t}^{\text{EM}}=\mathbf{x}_{t_{i-1}}^{\text{EM}}+\mathbf{f}_{\bm{\phi}}(\mathbf{x}_{t_{i-1}}^{\text{EM}},t_{i-1})(t-t_{i-1})+g(t_{i-1})\bm{\epsilon}\sqrt{t-t_{i-1}}$ on time range of $t\in [t_{i-1},t_{i})$, then we have
	\begin{align}\label{eq:error}
	\mathbb{E}\big[\Vert\mathbf{x}_{t}-\mathbf{x}_{t}^{\text{EM}}\Vert_{2}\big]\le C\sqrt{\Delta t_{i}},
	\end{align}
	where $C=C(T, K, \mathbb{E}[\Vert\mathbf{x}_{0}\Vert_{2}^{2}])\ge O(K^{2})$ is a constant with $K$ being a Lipschits constant of
	\begin{align*}
	\Vert\mathbf{f}_{\bm{\phi}}(\mathbf{x},t)-\mathbf{f}_{\bm{\phi}}(\mathbf{y},t)\Vert_{2}\le K\Vert\mathbf{x}-\mathbf{y}\Vert_{2}
	\end{align*}
	and
	\begin{align*}
	\Vert\mathbf{f}_{\bm{\phi}}(\mathbf{x},t)\Vert_{2}+\vert g(t)\vert\le K(1+\Vert\mathbf{x}\Vert_{2})
	\end{align*}
	for all $\mathbf{x},\mathbf{y}\in\mathbb{R}^{d}$ and $t\in[t_{i-1},t_{i})$. Having that $\Delta t_{i}$ is fixed a-priori, the upper bound in Inequality \eqref{eq:error} could be arbitrarily large becuase it depends on $K$ that represents the magnitude of nonlinearity of $\mathbf{f}_{\bm{\phi}}$. For instance, if $\mathbf{f}_{\bm{\phi}}(\mathbf{x}_{t},t)=\mathbf{x}_{t}^{2}$, then there does not exist any $K>0$ that satisfies above Lipschitz bounds. In such case, it is unable to guarantee the tightness of the discretized Markov chain to the continuous nonlinear SDE in the classical sense. Therefore, the Euler-Maruyama approximation of the nonlinear SDE should take $N$ as many as possible if we want to regard the finite Markov chain as a discretized nonlinear SDE, which would eventually increase the training, evaluation, and sampling time.
	
	Second, the computational complexity of INDM is $O(1)$ because the flow is evaluated only once at every optimization step. This is because the INDM loss is simply an addition of the flow loss and the linear diffusion loss. The training time of DiffFlow will be prohibitive as $N$ increases.
	
	Third, our INDM jointly models both drift and volatility terms nonlinearly, whereas DiffFlow nonlinearly models only the drift term. As illustrated in Figure 1 and 2-(c) in the main paper, nonlinearizing the volatility term brings a different diffusion to the overall process, compared to a diffusion that arises from a nonlinear drift. In particular, Figure 2-(c) depicts that the data-dependent volatility term yields an ellipsoidal covariance in the noise distribution, which was assumed to have a fixed diagonal covariance in previous research, as illustrated in Figure 6. In INDM, this covariance becomes the subject of matter to optimize.
	
	DiffFlow, as its current form, cannot impose nonlinearity to the volatility term because the discretized Markov chain is not a Gaussian distribution, anymore. To clarify, suppose a SDE of $\diff\mathbf{x}_{t}=\mathbf{f}_{\bm{\phi}}(\mathbf{x}_{t},t)\diff t+\mathbf{G}_{\bm{\phi}}(\mathbf{x}_{t},t)\diff\mathbf{w}_{t}$ (think of the green path of Figure 3 in the main paper) starts from a random variable $\mathbf{x}_{t_{i-1}}^{\text{EM}}$. The next discrete random variable of the Euler-Maruyama discretization is the approximate solution of this SDE at $t=t_{i}$, so let us approximate the right-hand-side of Eq. \eqref{eq:approximate_}:
	\begin{align}\label{eq:approximate_}
	\mathbf{x}_{t_{i}}-\mathbf{x}_{t_{i-1}}^{\text{EM}}=&\int_{t_{i-1}}^{t_{i}}\mathbf{f}_{\bm{\phi}}(\mathbf{x}_{t},t)\diff t+\int_{t_{i-1}}^{t_{i}}\mathbf{G}_{\bm{\phi}}(\mathbf{x}_{t},t)\diff\mathbf{w}_{t}.
	\end{align}
	The integral of the volatility term is
	\begin{align*}
	\int_{t_{i-1}}^{t_{i}}\mathbf{G}_{\bm{\phi}}(\mathbf{x}_{t},t)\diff\mathbf{w}_{t}=\int_{t_{i-1}}^{t_{i}}\mathbf{G}_{\bm{\phi}}\big(\mathbf{x}_{t_{i-1}}^{\text{EM}}+(\mathbf{x}_{t}-\mathbf{x}_{t_{i-1}}^{\text{EM}}),t_{i-1}+(t-t_{i-1})\big)\diff\mathbf{w}_{t}
	\end{align*}
	and since $\mathbf{x}_{t}-\mathbf{x}_{t_{i-1}}^{\text{EM}}=\mathbf{G}_{\bm{\phi}}(\mathbf{x}_{t_{i-1}}^{\text{EM}},t_{i-1})(\mathbf{w}_{t}-\mathbf{w}_{t_{i-1}})+O(\Delta t_{i})$, we get
	\begin{eqnarray*}
		\lefteqn{\int_{t_{i-1}}^{t_{i}}\mathbf{G}_{\bm{\phi}}(\mathbf{x}_{t},t)\diff\mathbf{w}_{t}}&\\
		&&=\mathbf{G}_{\bm{\phi}}(\mathbf{x}_{t_{i-1}}^{\text{EM}},t_{i-1})(\mathbf{w}_{t_{i}}-\mathbf{w}_{t_{i-1}})\\
		&&\quad+\mathbf{G}_{\bm{\phi}}(\mathbf{x}_{t_{i-1}}^{\text{EM}},t_{i-1})\frac{\partial\mathbf{G}_{\bm{\phi}}(\mathbf{x}_{t},t)}{\partial\mathbf{x}_{t}}\vert_{\mathbf{x}_{t_{i-1}}^{\text{EM}}}\int_{t_{i-1}}^{t_{i}}\mathbf{w}_{t}-\mathbf{w}_{t_{i-1}}\diff\mathbf{w}_{t}+O(\Delta t_{i}^{2})\\
		&&=\mathbf{G}_{\bm{\phi}}(\mathbf{x}_{t_{i-1}}^{\text{EM}},t_{i-1})(\mathbf{w}_{t_{i}}-\mathbf{w}_{t_{i-1}})\\
		&&\quad+\mathbf{G}_{\bm{\phi}}(\mathbf{x}_{t_{i-1}}^{\text{EM}},t_{i-1})\nabla_{\mathbf{x}_{t_{i-1}}^{\text{EM}}}\mathbf{G}_{\bm{\phi}}(\mathbf{x}_{t_{i-1}}^{\text{EM}},t_{i-1})\frac{1}{2}\big((\mathbf{w}_{t_{i}}-\mathbf{w}_{t_{i-1}})^{2}-\Delta t_{i}\big)+O(\Delta t_{i}^{2})\\
		&&=\mathbf{G}_{\bm{\phi}}(\mathbf{x}_{t_{i-1}}^{\text{EM}},t_{i-1})\bm{\epsilon}\sqrt{\Delta t_{i}}\\
		&&\quad+\frac{1}{2}\mathbf{G}_{\bm{\phi}}(\mathbf{x}_{t_{i-1}}^{\text{EM}},t_{i-1})\nabla_{\mathbf{x}_{t_{i-1}}^{\text{EM}}}\mathbf{G}_{\bm{\phi}}(\mathbf{x}_{t_{i-1}}^{\text{EM}},t_{i-1})\big(\bm{\epsilon}^{2}-1\big)\Delta t_{i}+O(\Delta t_{i}^{2}),
	\end{eqnarray*}
	where $\bm{\epsilon}\sim\mathcal{N}(0,\mathbf{I})$ and $\int_{t_{i-1}}^{t_{i}}\mathbf{w}_{t}-\mathbf{w}_{t_{i-1}}\diff\mathbf{w}_{t}=\int_{t_{i-1}}^{t_{i}}\mathbf{w}_{t}\diff\mathbf{w}_{t}-\mathbf{w}_{t_{i-1}}(\mathbf{w}_{t_{i}}-\mathbf{w}_{t_{i-1}})=\int_{t_{i-1}}^{t_{i}}\frac{1}{2}\diff(\mathbf{w}_{t}^{2})-\int_{t_{i-1}}^{t_{i}}\frac{1}{2}\diff t-\mathbf{w}_{t_{i-1}}(\mathbf{w}_{t_{i}}-\mathbf{w}_{t_{i-1}})=\frac{1}{2}(\mathbf{w}_{t_{i}}^{2}-\mathbf{w}_{t_{i-1}}^{2}-\Delta t_{i})-\mathbf{w}_{t_{i-1}}(\mathbf{w}_{t_{i}}-\mathbf{w}_{t_{i-1}})=\frac{1}{2}\big((\mathbf{w}_{t_{i}}-\mathbf{w}_{t_{i-1}})^{2}-\Delta t_{i}\big)$ is according to the Ito's formula \cite{oksendal2013stochastic}. As $\mathbf{G}_{\bm{\phi}}(\mathbf{x}_{t},t)$ now depends on $\mathbf{x}_{t}$, the term including $(\bm{\epsilon}^{2}-1)$ does not vanish. Therefore, $\mathbf{x}_{t_{i}}^{\text{EM}}$ is approximated by
	\begin{align}\label{eq:exact_discretization}
	\begin{split}
	\mathbf{x}_{t_{i}}=&\mathbf{x}_{t_{i-1}}^{\text{EM}}+\mathbf{f}_{\bm{\phi}}(\mathbf{x}_{t_{i-1}}^{\text{EM}},t_{i-1})\Delta t_{i}+\mathbf{G}_{\bm{\phi}}(\mathbf{x}_{t_{i-1}}^{\text{EM}},t_{i-1})\bm{\epsilon}\sqrt{\Delta t_{i}}\\
	&+\frac{1}{2}\mathbf{G}_{\bm{\phi}}(\mathbf{x}_{t_{i-1}}^{\text{EM}},t_{i-1})\nabla_{\mathbf{x}_{t_{i-1}}^{\text{EM}}}\mathbf{G}_{\bm{\phi}}(\mathbf{x}_{t_{i-1}}^{\text{EM}},t_{i-1})\big(\bm{\epsilon}^{2}-1\big)\Delta t_{i}+O(\Delta t_{i}^{3/2})\\
	\approx&\mathbf{x}_{t_{i-1}}^{\text{EM}}+\mathbf{f}_{\bm{\phi}}(\mathbf{x}_{t_{i-1}}^{\text{EM}},t_{i-1})\Delta t_{i}+\mathbf{G}(\mathbf{x}_{t_{i-1}}^{\text{EM}},t_{i-1})\bm{\epsilon}\sqrt{\Delta t_{i}}\\
	&+\frac{1}{2}\mathbf{G}_{\bm{\phi}}(\mathbf{x}_{t_{i-1}}^{\text{EM}},t_{i-1})\nabla_{\mathbf{x}_{t_{i-1}}^{\text{EM}}}\mathbf{G}_{\bm{\phi}}(\mathbf{x}_{t_{i-1}}^{\text{EM}},t_{i-1})\big(\bm{\epsilon}^{2}-1\big)\Delta t_{i}\\
	:=&\mathbf{x}_{t_{i}}^{\text{EM}}.
	\end{split}
	\end{align}
	The order of the term $\frac{1}{2}\mathbf{G}_{\bm{\phi}}(\mathbf{x}_{t_{i-1}}^{\text{EM}},t_{i-1})\nabla_{\mathbf{x}_{t_{i-1}}^{\text{EM}}}\mathbf{G}_{\bm{\phi}}(\mathbf{x}_{t_{i-1}}^{\text{EM}},t_{i-1})\big(\bm{\epsilon}^{2}-1\big)\Delta t_{i}$ is $O(\Delta t_{i})$, which is the same order of the term $\mathbf{f}_{\bm{\phi}}(\mathbf{x}_{t_{i-1}}^{\text{EM}},t_{i-1})\Delta t_{i}$. Thus, this last term including $\bm{\epsilon}^{2}$ cannot be ignored in the approximation. 
	
	With this approximation, the discretized random variable, $\mathbf{x}_{t_{i}}^{\text{EM}}$, includes a term of $\bm{\epsilon}^{2}$, which is the square of the Brownian motion that does not follow a Gaussian distribution. Therefore, the variational bound of Eq. \eqref{eq:discrete_ddpm_loss} is no longer reduced to a tractable loss, such as Eq. \eqref{eq:ddpm_style_loss}, and as a consequence, Eq. \eqref{eq:discrete_ddpm_loss} is not optimizable even though the nonlinear SDE is discretized. Therefore, we have to ignore the last term, $\frac{1}{2}\mathbf{G}_{\bm{\phi}}\nabla\mathbf{G}_{\bm{\phi}}(\bm{\epsilon}^{2}-1)\Delta t_{i}$, to tractably optimize the variational bound, but such ingorance equals to the approximation of DiffFlow, which would incur a large approximation error if $\mathbf{G}_{\bm{\phi}}$ nonlinearly depends on $\mathbf{x}_{t}$. This leads DiffFlow limited on $\mathbf{G}_{\bm{\phi}}(\mathbf{x}_{t},t)=g_{\bm{\phi}}(t)$, at its maximal capacity. This is contrastive to the result of INDM illustrated in Figure 6.
	
	Fourth, as the generative process of DiffFlow starts from an easy-to-sample prior distribution, the flexibility of $\mathbf{f}_{\bm{\phi}}$ is severely restricted to constrain $p_{T}^{\bm{\phi}}(\mathbf{x}_{T}^{\bm{\phi}})\approx \pi(\mathbf{x}_{T}^{\bm{\phi}})$. The feasible space of nonlinear $\mathbf{f}_{\bm{\phi}}$ that satisfies this constraint does not seem to be derived explicitly. Contrastive to DiffFlow, the data diffusion does not have to end at $\pi$ in INDM. Instead, INDM assumes the linear diffusion on the latent variable, so the ending variable on the latent space, $\mathbf{z}_{T}^{\bm{\phi}}$, is already close to the prior distribution. Therefore, the space of admissible nonlinear drift in INDM, which is \textit{explicitly} desribed in Eq. \eqref{eq:drift_ap}, should be larger than the space of DiffFlow. A lesson from this is that the explicit parametrization seems to be intuitive, but underneath the surface, not many properties could be uncovered explicitly, whereas the implicit parametrization using the invertible transformation enjoys its explicit derivations that enable to analyze concrete properties.
	
	Fifth, DiffFlow estimates its loss of Eq. \eqref{eq:exchange_sum_and_expectation} using a single (or multiple) path to update the parameters with the reparametrization trick \cite{kingma2013auto}. On the other hand, the discretized diffusion model estimates its loss with Eq. \eqref{eq:ddpm_style_loss}, where the sampling from $p_{0t}(\mathbf{x}_{t}\vert\mathbf{x}_{0})$ is inexpensive because the transition probability is a Gaussian distribution. Therefore, the losses of Eqs. \eqref{eq:exchange_sum_and_expectation} and \eqref{eq:ddpm_style_loss} coincide in the expectation sense, but they are estimated differently between DiffFlow and diffusion models with analytic transition probabilities. Taking $\frac{1}{g^{2}(t_{i})\Delta t_{i}}\big\Vert\mathbf{x}_{t_{i-1}}^{\text{EM}}-\mathbf{x}_{t_{i}}^{\text{EM}}+\big[\mathbf{f}_{\bm{\phi}}(\mathbf{x}_{t_{i}}^{\text{EM}},t_{i})-g^{2}(t_{i})\mathbf{s}_{\bm{\theta}}(\mathbf{x}_{t_{i}}^{\text{EM}},t_{i})\big]\Delta t_{i}\big\Vert_{2}^{2}$ as a random variable $X_{i}$, Eq. \eqref{eq:ddpm_style_loss} is reduced to $\frac{1}{2}\sum\mathbb{E}[X_{i}]$, and Eq. \eqref{eq:exchange_sum_and_expectation} is reduced to $\frac{1}{2}\mathbb{E}_{\text{sample-path}}[\sum X_{i}]$. Therefore, the variance of the Monte-Carlo estimation of Eq. \eqref{eq:ddpm_style_loss} becomes $\frac{1}{2}\sum \text{Var}(X_{i})$, whereas the variance of the Monte-Carlo estimation of Eq. \eqref{eq:exchange_sum_and_expectation} becomes 
	\begin{align*}
	\frac{1}{2}\text{Var}\Big(\sum X_{i}\Big)=\frac{1}{2}\Big[\sum\text{Var}(X_{i})+2\sum\text{Cov}(X_{i},X_{j})\Big],
	\end{align*}
	where $\text{Cov}(X_{i},X_{j})$ represents the covariance of two random variables $X_{i}$ and $X_{j}$. Table \ref{tab:variance} represents the ratio of these two variances,
	\begin{align*}
	\text{Ratio}:=\frac{\text{Var}(\sum X_{i})}{\sum\text{Var}(X_{i})}=\frac{\sum\text{Var}(X_{i})+2\sum\text{Cov}(X_{i},X_{j})}{\sum\text{Var}(X_{i})}=1+2\frac{\sum\text{Cov}(X_{i},X_{j})}{\sum\text{Var}(X_{i})},
	\end{align*}
	and it shows that the DiffFlow loss has prohibitively large variance as $N$ increases, compared to the INDM loss, which computes its Monte-Carlo estimation in spirit of Eq. \eqref{eq:ddpm_style_loss} with $N=\infty$.
	
	Note that throughout our argument, we have omitted the prior and reconstruction terms on the variational bounds in this section.
	
	\begin{table}[t]
		\caption{The variance ratio between the variances of the analytic transition probability-based estimation of Eq. \eqref{eq:ddpm_style_loss} and the sample-based estimation of Eq. \eqref{eq:exchange_sum_and_expectation}.}
		\label{tab:variance}
		\scriptsize
		\centering
		\begin{tabular}{lccccc}
			\toprule
			& \multicolumn{5}{c}{Number of Random Variables ($N$)}\\\cmidrule(lr){2-6}
			& 1 & 10 & 100 & 1000 & 10000\\\midrule
			\multirow{2}{*}{\shortstack{Estimation\\Variance Ratio}} & \multirow{2}{*}{1.00} & \multirow{2}{*}{1.02} & \multirow{2}{*}{2.08} & \multirow{2}{*}{16.68} & \multirow{2}{*}{76.08}\\
			&&&&&\\
			\bottomrule
		\end{tabular}
		\vskip -0.1in
	\end{table}
	
	\subsection{Schr\"odinger Bridge Problem (SBP)}\label{appendix:SBP}
	
	Schr\"odinger Bridge Problem (SBP) \cite{vargas2021solving, de2021diffusion, chen2021likelihood} has recently been highlighted in machine learning for its connection to the score-based diffusion model. Schr\"odinger Bridge Problem is a bi-constrained problem of
	\begin{align*}
	\min_{\bm{\rho}\in\mathcal{P}(p_{r},\pi)}D_{KL}(\bm{\rho}\Vert\bm{\mu}),
	\end{align*}
	where $\mathcal{P}(p_{r},\pi)$ is a family of path measure with bi-constraints of $p_{r}$ and $\pi$ as its marginal distributions at $t=0$ and $t=T$, respectively, and $\bm{\mu}$ is a reference path measure that is governed by
	\begin{align}\label{eq:reference_SDE}
	\diff\mathbf{x}_{t}=\mathbf{f}(\mathbf{x}_{t},t)\diff t+g(t)\diff\mathbf{w}_{t}, \quad\mathbf{x}_{0}\sim p_{r}.
	\end{align}
	As the KL divergence becomes infinite if the diffusion coefficient of $\bm{\rho}$ is not equal to $g(t)$ (because quadratic variations of $\bm{\mu}$ and $\bm{\rho}$ becomes different), SBP is equivalently formulated as
	\begin{align*}
	\min_{\bm{\rho}\in\mathcal{P}(p_{r},\pi)}D_{KL}(\bm{\rho}\Vert\bm{\mu})=\min_{\bm{\rho}_{\mathbf{v}}\in\mathcal{P}(p_{r},\pi)}D_{KL}(\bm{\rho}_{\mathbf{v}}\Vert\bm{\mu}),
	\end{align*}
	where the path measure $\bm{\rho}_{\mathbf{v}}\in\mathcal{P}(p_{r},\pi)$ follows the SDE of
	\begin{align}\label{eq:forward_SDE_SBP}
	\diff\mathbf{x}_{t}=\big[\mathbf{f}(\mathbf{x}_{t},t)+g^{2}(t)\mathbf{v}(\mathbf{x}_{t},t)\big]\diff t+g(t)\mathbf{w}_{t}.
	\end{align}
	
	From the Girsanov theorem and the Martingale property \cite{chen2016relation}, we have
	\begin{align*}
	D_{KL}(\bm{\rho}_{\mathbf{v}}\Vert\bm{\mu})=\frac{1}{2}\int_{0}^{T}g^{2}(t)\mathbb{E}_{\bm{\rho}_{\mathbf{v}}}[\Vert\mathbf{v}(\mathbf{x}_{t},t)\Vert_{2}^{2}]\diff t+D_{KL}(\pi\Vert p_{T}),
	\end{align*}
	where $p_{T}$ is the marginal distribution of $\bm{\mu}$ at $t=T$. If $\mathcal{V}(p_{r},\pi)$ is the space of all vector fields $\mathbf{v}$ of which forward SDE with Eq. \eqref{eq:forward_SDE_SBP} satisfies the boundary conditions, then SBP is equivalent to
	\begin{align}\label{eq:SBP_equivalent}
	\min_{\bm{\nu}\in\mathcal{P}(p_{r},\pi)}D_{KL}(\bm{\nu}\Vert\bm{\mu})=\min_{\mathbf{v}\in\mathcal{V}(p_{r},\pi)}\frac{1}{2}\int_{0}^{T}g^{2}(t)\mathbb{E}_{\bm{\rho}_{\mathbf{v}}}[\Vert\mathbf{v}(\mathbf{x}_{t},t)\Vert_{2}^{2}]\diff t,
	\end{align}
	where $\bm{\rho}_{\mathbf{v}}$ is the associated path measure of Eq. \eqref{eq:forward_SDE_SBP}. Eq. \eqref{eq:SBP_equivalent} interprets the solution of SBP as the least energy (weighted by $g^{2}$) of the auxiliary vector field ($\mathbf{v}$) among admissible space of vector fields ($\mathcal{V}(p_{r},\pi)$). Hence, if $\bm{\mu}\in\mathcal{P}(p_{r},\pi)$, then the trivial vector field, $\mathbf{v}\equiv 0$, is the solution of SBP. When the reference SDE of Eq. \eqref{eq:reference_SDE} is one of the family of linear SDEs, such as VESDE or VPSDE, then $\bm{\mu}\notin\mathcal{P}(p_{r},\pi)$, so the trivial vector field is not the solution of SBP, anymore. Instead, $\bm{\mu}$'s ending variable is close enough to $\pi$ (e.g., $D_{KL}(p_{T}\Vert\pi)\approx 10^{-5}$ in bpd scale \cite{ho2020denoising}), so the closest path measure in $\mathcal{V}(p_{r},\pi)$ to $\bm{\mu}$ is nearly identical to a trivial vector field, $\mathbf{v}^{*}\approx 0$, and the nonlinearity of SBP is limited.

	\citet{chen2021likelihood} connects the optimal solution of SBP with PDEs. At the optimal point, if we denote by $\bm{\rho}^{*}=\argmin_{\bm{\rho}\in\mathcal{P}(p_{r},\pi)}D_{KL}(\bm{\rho}\Vert\bm{\mu})$, then this \textit{optimal} diffusion process follows a forward diffusion SDE \cite{chen2021likelihood} of
	\begin{align*}
	\diff\mathbf{x}_{t}=\big[\mathbf{f}(\mathbf{x}_{t},t)+g^{2}(t)\nabla_{\mathbf{x}_{t}}\log{\Psi(\mathbf{x}_{t},t)}\big]\diff t+g(t)\diff\mathbf{w}_{t}, \quad\mathbf{x}_{0}\sim p_{r},
	\end{align*}
	with the corresponding reverse diffusion as
	\begin{align*}
	\diff\mathbf{x}_{t}=\big[\mathbf{f}(\mathbf{x}_{t},t)-g^{2}(t)\nabla_{\mathbf{x}_{t}}\log{\hat{\Psi}(\mathbf{x}_{t},t)}\big]\diff \bar{t}+g(t)\diff\mathbf{\bar{w}}_{t}, \quad\mathbf{x}_{T}\sim \pi,
	\end{align*}
	where $\Psi(\mathbf{x}_{t},t)$ and $\hat{\Psi}(\mathbf{x}_{t},t)$ are the solutions of a system of PDEs \cite{leger2021hopf}:
	\begin{align}\label{eq:hopf_cole_transform}
	\begin{split}
	\frac{\partial\Psi}{\partial t}=&-\nabla_{\mathbf{x}_{t}}\Psi^{T}\mathbf{f}-\frac{1}{2}\text{tr}(g^{2}\nabla_{\mathbf{x}_{t}}^{2}\Psi)\\
	\frac{\partial\hat{\Psi}}{\partial t}=&-\text{div}(\hat{\Psi}\mathbf{f})+\frac{1}{2}\text{tr}(g^{2}\nabla_{\mathbf{x}_{t}}^{2}\hat{\Psi}),
	\end{split}
	\end{align}
	such that $\Psi(\mathbf{x}_{0},0)\hat{\Psi}(\mathbf{x}_{0},0)=p_{r}(\mathbf{x}_{0})$ and $\Psi(\mathbf{x}_{T},T)\hat{\Psi}(\mathbf{x}_{T},T)=\pi(\mathbf{x}_{T})$. With $\Psi$ and $\hat{\Psi}$, the forward diffusion SDE ends exactly at $\pi$, and the corresponding reverse SDE ends at $p_{r}$. Therefore, SBP is equivalent to solve the system of PDEs given by Eq. \eqref{eq:hopf_cole_transform}.
	
	\citet{chen2021likelihood} solves the system of coupled PDEs with Eq. \eqref{eq:hopf_cole_transform} using a theory of forward-backward SDEs, which requires a deep understanding of PDE theory. SB-FBSDE \cite{chen2021likelihood} uses the fact that the solution $(\Psi,\hat{\Psi})$ of Hopf-Cole transform in Eq. \eqref{eq:hopf_cole_transform} is derived from the solution of the forward-backward SDEs of
	\begin{align}\label{eq:FBSDE}
	\left\{\begin{array}{l}
	\diff\mathbf{x}_{t}=\big[\mathbf{f}(\mathbf{x}_{t},t)+g(t)\mathbf{z}_{t}(\mathbf{x}_{t},t)\big]\diff t+g(t)\diff\mathbf{w}_{t}\\[0.5em]
	\diff\mathbf{y}_{t}=\frac{1}{2}(\mathbf{z}_{t}^{T}\mathbf{z}_{t})(\mathbf{x}_{t},t)\diff t+\mathbf{z}_{t}^{T}(\mathbf{x}_{t},t)\diff\mathbf{w}_{t}\\[0.5em]
	\diff\mathbf{\hat{y}}_{t}=\Big[\frac{1}{2}(\mathbf{\hat{z}}_{t}^{T}\mathbf{\hat{z}}_{t})(\mathbf{x}_{t},t)+\text{div}\big(g(t)\mathbf{\hat{z}}_{t}(\mathbf{x}_{t},t)-\mathbf{f}(\mathbf{x}_{t},t)\big)+(\mathbf{\hat{z}}_{t}^{T}\mathbf{z}_{t})(\mathbf{x}_{t},t)\Big]\diff t+\mathbf{\hat{z}}_{t}^{T}(\mathbf{x}_{t},t)\diff\mathbf{w}_{t},
	\end{array}\right.
	\end{align}
	where the boundary conditions are given by $\mathbf{x}(0)=\mathbf{x}_{0}$ and $\mathbf{y}_{T}+\mathbf{\hat{y}}_{T}=\log{\pi(\mathbf{x}_{T})}$. The solution of the above system of forward-backward SDEs satisfies $\mathbf{z}_{t}(\mathbf{x}_{t},t)=g(t)\nabla\log{\Psi(\mathbf{x}_{t},t)}$ and $\mathbf{\hat{z}}_{t}(\mathbf{x}_{t},t)=g(t)\nabla\log{\hat{\Psi}(\mathbf{x}_{t},t)}$, where $(\Psi,\hat{\Psi})$ is the solution of Eq. \eqref{eq:hopf_cole_transform}. SB-FBSDE parametrizes $(\mathbf{z}_{t},\mathbf{\hat{z}}_{t})$ as $\bm{\theta}$ and $\bm{\phi}$, and it estimates the solution $(\mathbf{z}_{t},\mathbf{\hat{z}}_{t})$ of Eq. \eqref{eq:FBSDE} from MLE training of the log-likelihood $\log{p_{\bm{\phi},\bm{\theta}}(\mathbf{x}_{0})}$.

	Other than the PDE-driven approach \cite{chen2021likelihood}, SBP has been traditionally solved via Iterative Proportional Fitting (IPF) \cite{ruschendorf1995convergence}. Concretely, suppose
	\begin{align}\label{eq:forward_SBP}
	\diff\mathbf{x}_{t}^{\bm{\phi}}=\big[\mathbf{f}(\mathbf{x}_{t}^{\bm{\phi}},t)+g^{2}(t)\mathbf{s}_{\bm{\phi}}(\mathbf{x}_{t}^{\bm{\phi}},t)\big]\diff t+g(t)\diff\mathbf{w}_{t}, \quad\mathbf{x}_{0}^{\bm{\phi}}\sim p_{r},
	\end{align}
	is a forward diffusion with a parametrized vector field of $\mathbf{s}_{\bm{\phi}}$, and
	\begin{align*}
	\diff\mathbf{x}_{t}^{\bm{\theta}}=\big[\mathbf{f}(\mathbf{x}_{t}^{\bm{\theta}},t)-g^{2}(t)\mathbf{s}_{\bm{\theta}}(\mathbf{x}_{t}^{\bm{\theta}},t)\big]\diff \bar{t}+g^{2}(t)\mathbf{\bar{w}}_{t},\quad\mathbf{x}_{T}^{\bm{\theta}}\sim\pi,
	\end{align*}
	is a generative diffusion with a parametrized vector field of $\mathbf{s}_{\bm{\theta}}$. Then, IPF get its optimal vector fields by alternatively solving below half-bridge problems
	\begin{align}
	\bm{\nu}_{\bm{\phi}_{n}}=&\argmin_{\bm{\nu}_{\bm{\phi}}\in\mathcal{P}(p_{r},\cdot)}D_{KL}(\bm{\nu}_{\bm{\phi}}\Vert\bm{\nu}_{\bm{\theta}_{n-1}}),\label{eq:SBP1}\\
	\bm{\nu}_{\bm{\theta}_{n}}=&\argmin_{\bm{\nu}_{\bm{\theta}}\in\mathcal{P}(\cdot,\pi)}D_{KL}(\bm{\nu}_{\bm{\theta}}\Vert\bm{\nu}_{\bm{\phi}_{n}}),\label{eq:SBP2}
	\end{align}
	where the convergence of $\bm{\nu}_{\bm{\phi}^{n}}\rightarrow\bm{\nu}_{\bm{\phi}^{*}}$ and $\bm{\mu}_{\bm{\theta}^{n}}\rightarrow\bm{\nu}_{\bm{\theta}^{*}}$ is guaranteed in \citet{de2021diffusion}. Here, analogously, $\mathcal{P}(\cdot,\pi)$ is a family of path measure with $\pi$ as its marginal distribution at $t=T$. Notably, each of the half-bridge problem is a diffusion problem with the KL divergence replaced with the reverse KL divergence. Since SBP learns the forward SDE, sampling particle paths is expensive as it requires to solve an SDE numerically, so the training of IPF is slow.

	\section{Correction of Density Estimation Metrics of Diffusion Models with Time Truncation}\label{appendix:correction_of_nll}
	
	\subsection{Equivalent Reverse SDEs}\label{appendix:equivalent_reverse_sdes}
	
	Throughout Section \ref{appendix:correction_of_nll}, the diffusion process is assumed to follow a SDE of $\diff\mathbf{x}_{t}=\mathbf{f}(\mathbf{x}_{t},t)\diff t+g(t)\diff\mathbf{w}_{t}$ because the below argument is generally applicable for any continuous diffusion models. For INDM, we apply the below argument on the latent space, which has a linear drift term. Let $\diff\mathbf{x}_{t}=\big[\mathbf{f}(\mathbf{x}_{t},t)-\frac{1+\lambda^{2}}{2}g^{2}(t)\nabla_{\mathbf{x}_{t}}\log{p_{t}^{\lambda}(\mathbf{x}_{t})}\big]\diff \bar{t}+\lambda g(t)\bar{\mathbf{w}}_{t}$ be the reverse SDEs starting from $p_{T}$, where $p_{t}^{\lambda}$ is the probability law of the solution at $t$. Then, the reverse Kolmogorov equation (or Fokker-Planck equation) becomes
	\begin{align*}
	\frac{\partial p_{t}^{\lambda}(\mathbf{x}_{t},t)}{\partial t}=&-\sum_{i=1}^{d}\frac{\partial}{\partial x_{i}}\left(\left[ f_{i}(\mathbf{x}_{t},t)-\frac{1+\lambda^{2}}{2}g^{2}(t)\big(\nabla_{\mathbf{x}_{t}}\log{p_{t}^{\lambda}(\mathbf{x}_{t},t)}\big)_{i} \right]p_{t}^{\lambda}(\mathbf{x}_{t},t)\right)\\
	&-\frac{\lambda^{2}g^{2}(t)}{2}\sum_{i=1}^{d}\frac{\partial^{2}}{\partial x_{i}^{2}}\big[p_{t}^{\lambda}(\mathbf{x}_{t},t)\big]\\
	=&-\sum_{i=1}^{d}\frac{\partial}{\partial x_{i}}\left( f_{i}(\mathbf{x}_{t},t)p_{t}^{\lambda}(\mathbf{x}_{t},t)-\frac{1+\lambda^{2}}{2}g^{2}(t)\frac{\partial p_{t}^{\lambda}(\mathbf{x}_{t},t)}{\partial x_{i}} \right)\\
	&-\frac{\lambda^{2}g^{2}(t)}{2}\sum_{i=1}^{d}\frac{\partial^{2}}{\partial x_{i}^{2}}\big[p_{t}^{\lambda}(\mathbf{x}_{t},t)\big]\\
	=&-\sum_{i=1}^{d}\frac{\partial}{\partial x_{i}}\big[f_{i}(\mathbf{x}_{t},t)p_{t}^{\lambda}(\mathbf{x}_{t},t)\big]+\frac{1}{2}g^{2}(t)\sum_{i=1}^{d}\frac{\partial^{2}}{\partial x_{i}^{2}}\big[p_{t}^{\lambda}(\mathbf{x}_{t},t)\big],
	\end{align*}
	which is independent of $\lambda$. Therefore, it satisfies $p_{t}^{\lambda}=p_{t}^{\lambda'}$ for any $\lambda\neq\lambda'$. 
	
	For any $\lambda\in[0,1]$, the generative SDE is constructed by plugging $\mathbf{s}_{\bm{\theta}}(\mathbf{x}_{t},t)$ in place of $\nabla_{\mathbf{x}_{t}}\log{p_{t}^{\lambda}(\mathbf{x}_{t})}$ in the reverse SDE as $\diff\mathbf{x}_{t}=\big[\mathbf{f}(\mathbf{x}_{t},t)-\frac{1+\lambda^{2}}{2}g^{2}(t)\mathbf{s}_{\bm{\theta}}(\mathbf{x}_{t},t)\big]\diff \bar{t}+\lambda g(t)\bar{\mathbf{w}}_{t}$. Suppose we denote $p_{t}^{\lambda,\bm{\theta}}$ as the marginal distribution of the model at $t$. Then, the generative SDEs with different $\lambda$ have distinctive marginal distributions: $p_{t}^{\lambda,\bm{\theta}}\neq p_{t}^{\lambda',\bm{\theta}}$ for $\lambda\neq\lambda'$.
	
	\subsection{Log-Likelihood for Diffusion Models with Time Truncation}
	Due to the unbounded score loss illustrated in \cite{kim2022soft}, a diffusion model truncates the diffusion time to be $[\epsilon,1]$ for small enough $\epsilon>0$. However, since the small range of diffusion time contributes significant portion of the log-likelihood \cite{kim2022soft}, the effect of truncation should be counted both on training and evaluation. To describe, as we have no knowledge on the score estimation at $t\in[0,\epsilon)$, we have estimate the data log-likelihood by using the variational inferecence:
	\begin{eqnarray*}
		\lefteqn{\log{p_{0}^{\lambda,\bm{\theta}}(\mathbf{x}_{0})}=\log{\int p_{0}^{\lambda,\bm{\theta}}(\mathbf{x}_{0},\mathbf{x}_{\epsilon}) \diff\mathbf{x}_{\epsilon}}}&\\
		&&\ge \int p_{0\epsilon}(\mathbf{x}_{\epsilon}\vert\mathbf{x}_{0})\log{\frac{p_{\epsilon}^{\lambda,\bm{\theta}}(\mathbf{x}_{\epsilon})p_{\epsilon 0}^{\bm{\theta}}(\mathbf{x}_{0}\vert\mathbf{x}_{\epsilon})}{p_{0\epsilon}(\mathbf{x}_{\epsilon}\vert\mathbf{x}_{0})}}\diff\mathbf{x}_{\epsilon}\\
		&&=\mathbb{E}_{p_{0\epsilon}(\mathbf{x}_{\epsilon}\vert\mathbf{x}_{0})}\big[\log{p_{\epsilon}^{\lambda,\bm{\theta}}(\mathbf{x}_{\epsilon})}\big]+\mathbb{E}_{p_{0\epsilon}(\mathbf{x}_{\epsilon}\vert\mathbf{x}_{0})}\bigg[\log{\frac{p_{\epsilon 0}^{\bm{\theta}}(\mathbf{x}_{0}\vert\mathbf{x}_{\epsilon})}{p_{0\epsilon}(\mathbf{x}_{\epsilon}\vert\mathbf{x}_{0})}}\bigg],
	\end{eqnarray*}
	where $p_{\epsilon}^{\lambda,\bm{\theta}}$ is the generative distribution perturbed by $\epsilon$, and $p_{\epsilon 0}^{\bm{\theta}}(\mathbf{x}_{0}\vert\mathbf{x}_{\epsilon})$ is the reconstruction transition probability given $\mathbf{x}_{\epsilon}$. Then, we have
	\begin{eqnarray*}
		\lefteqn{\mathbb{E}_{p_{r}(\mathbf{x}_{0})}[-\log{p_{0}^{\lambda,\bm{\theta}}(\mathbf{x}_{0})}]\le\mathbb{E}_{\mathbf{x}_{\epsilon}}\big[-\log{p_{\epsilon}^{\lambda,\bm{\theta}}(\mathbf{x}_{\epsilon})}\big]-\mathbb{E}_{\mathbf{x}_{0},\mathbf{x}_{\epsilon}}\bigg[\log{\frac{p_{\epsilon 0}^{\bm{\theta}}(\mathbf{x}_{0}\vert\mathbf{x}_{\epsilon})}{p_{0\epsilon}(\mathbf{x}_{\epsilon}\vert\mathbf{x}_{0})}}\bigg]}&\\
		&&=D_{KL}(p_{\epsilon}\Vert p_{\epsilon}^{\lambda,\bm{\theta}})-\mathbb{E}_{\mathbf{x}_{0},\mathbf{x}_{\epsilon}}\bigg[\log{\frac{p_{\epsilon 0}^{\bm{\theta}}(\mathbf{x}_{0}\vert\mathbf{x}_{\epsilon})}{p_{0\epsilon}(\mathbf{x}_{\epsilon}\vert\mathbf{x}_{0})}}\bigg]+\mathcal{H}(p_{\epsilon}),
	\end{eqnarray*}
	which is equivalent to
	\begin{eqnarray*}
		\lefteqn{D_{KL}(p_{r}\Vert p_{0}^{\lambda,\bm{\theta}})\le D_{KL}(p_{\epsilon}\Vert p_{\epsilon}^{\lambda,\bm{\theta}})-\mathbb{E}_{\mathbf{x}_{0},\mathbf{x}_{\epsilon}}\bigg[\log{\frac{p_{\epsilon 0}^{\bm{\theta}}(\mathbf{x}_{0}\vert\mathbf{x}_{\epsilon})}{p_{0\epsilon}(\mathbf{x}_{\epsilon}\vert\mathbf{x}_{0})}}\bigg]+\mathcal{H}(p_{\epsilon})-\mathcal{H}(p_{r})}&\\
		&&=D_{KL}(p_{\epsilon}\Vert p_{\epsilon}^{\lambda,\bm{\theta}})+D_{KL}\big(p_{r}(\mathbf{x}_{0})p_{0\epsilon}(\mathbf{x}_{\epsilon}\vert\mathbf{x}_{0})\Vert p_{\epsilon}(\mathbf{x}_{\epsilon})p_{\epsilon 0}^{\bm{\theta}}(\mathbf{x}_{0}\vert\mathbf{x}_{\epsilon})\big).
	\end{eqnarray*}
	
	\subsection{NELBO Correction}
	
	Suppose $\bm{\mu}_{\epsilon}$ is the path measure of $\diff\mathbf{x}_{t}=\mathbf{f}(\mathbf{x}_{t},t)\diff t+g(t)\diff\mathbf{w}_{t}$ on $[\epsilon,T]$, and $\bm{\nu}_{\bm{\theta},\epsilon}^{\lambda}$ is the path measure of $\diff\mathbf{x}_{t}=\big[\mathbf{f}(\mathbf{x}_{t},t)-\frac{1+\lambda^{2}}{2}g^{2}(t)\nabla_{\mathbf{x}_{t}}\log{p_{t}^{\lambda}(\mathbf{x}_{t})}\big]\diff \bar{t}+\lambda g(t)\bar{\mathbf{w}}_{t}$ on $[\epsilon,T]$. Then, the continuous variational bound on the truncated diffusion model becomes
	\begin{eqnarray*}
		\lefteqn{\mathbb{E}_{p_{r}(\mathbf{x}_{0})}\big[-\log{p_{0}^{\lambda,\bm{\theta}}(\mathbf{x}_{0})}\big]\le D_{KL}(p_{\epsilon}\Vert p_{\epsilon}^{\lambda,\bm{\theta}})-\mathbb{E}_{\mathbf{x}_{0},\mathbf{x}_{\epsilon}}\bigg[\log{\frac{p_{\epsilon 0}^{\bm{\theta}}(\mathbf{x}_{0}\vert\mathbf{x}_{\epsilon})}{p_{0\epsilon}(\mathbf{x}_{\epsilon}\vert\mathbf{x}_{0})}}\bigg]+\mathcal{H}(p_{\epsilon})}&\\
		&&\le D_{KL}(\bm{\mu}_{\epsilon}\Vert \bm{\nu}_{\bm{\theta},\epsilon}^{\lambda})-\mathbb{E}_{\mathbf{x}_{0},\mathbf{x}_{\epsilon}}\bigg[\log{\frac{p_{\epsilon 0}^{\bm{\theta}}(\mathbf{x}_{0}\vert\mathbf{x}_{\epsilon})}{p_{0\epsilon}(\mathbf{x}_{\epsilon}\vert\mathbf{x}_{0})}}\bigg]+\mathcal{H}(p_{\epsilon})\\
		&&=\frac{1}{2}\frac{(1+\lambda^{2})^{2}}{4\lambda^{2}}\int_{\epsilon}^{T}g^{2}(t)\mathbb{E}_{\mathbf{x}_{t}}\big[\Vert\mathbf{s}_{\bm{\theta}}(\mathbf{x}_{t},t)-\log{p_{t}(\mathbf{x}_{t})}\Vert_{2}^{2}\big]\diff t-\mathbb{E}_{\mathbf{x}_{T}}\big[\log{\pi(\mathbf{x}_{T})}\big]\\
		&&\quad-\mathbb{E}_{\mathbf{x}_{0},\mathbf{x}_{\epsilon}}\bigg[\log{\frac{p_{\epsilon 0}^{\bm{\theta}}(\mathbf{x}_{0}\vert\mathbf{x}_{\epsilon})}{p_{0\epsilon}(\mathbf{x}_{\epsilon}\vert\mathbf{x}_{0})}}\bigg]+\mathcal{H}(p_{\epsilon})-\mathcal{H}(p_{T})\\
		&&=\frac{1}{2}\int_{\epsilon}^{T}\mathbb{E}_{\mathbf{x}_{0},\mathbf{x}_{t}}\bigg[\frac{(1+\lambda^{2})^{2}}{4\lambda^{2}}g^{2}(t)\Vert\log{p_{0t}(\mathbf{x}_{t}\vert\mathbf{x}_{0})}-\mathbf{s}_{\bm{\theta}}(\mathbf{x}_{t},t)\Vert_{2}^{2}-g^{2}(t)\Vert\nabla_{\mathbf{x}_{t}}\log{p_{0t}(\mathbf{x}_{t}\vert\mathbf{x}_{0})}\Vert_{2}^{2}\\
		&&\quad\quad\quad\quad\quad\quad\quad-2\nabla_{\mathbf{x}_{t}}\cdot \mathbf{f}(\mathbf{x}_{t},t)\bigg]\diff t-\mathbb{E}_{\mathbf{x}_{T}}\big[\log{\pi(\mathbf{x}_{T})}\big]-\mathbb{E}_{\mathbf{x}_{0},\mathbf{x}_{\epsilon}}\bigg[\log{\frac{p_{\epsilon 0}^{\bm{\theta}}(\mathbf{x}_{0}\vert\mathbf{x}_{\epsilon})}{p_{0\epsilon}(\mathbf{x}_{\epsilon}\vert\mathbf{x}_{0})}}\bigg],
	\end{eqnarray*}
	where $\mathcal{H}(p_{\epsilon})-\mathcal{H}(p_{T})$ is derived to be $-\frac{1}{2}\int_{\epsilon}^{T}\mathbb{E}\big[g^{2}\Vert\nabla\log{p_{0t}}\Vert_{2}^{2}+2\nabla\cdot \mathbf{f}\big]\diff t$ by Theorem 4 of \citet{song2021maximum}. The residual term, $\mathbb{E}_{p_{r}(\mathbf{x}_{0})p_{0\epsilon}(\mathbf{x}_{\epsilon}\vert\mathbf{x}_{0})}\big[\log{\frac{p_{\epsilon 0}^{\bm{\theta}}(\mathbf{x}_{0}\vert\mathbf{x}_{\epsilon})}{p_{0\epsilon}(\mathbf{x}_{\epsilon}\vert\mathbf{x}_{0})}}\big]$, has been ignored both on training and evaluation in previous research. Therefore, we report the \textit{correct} NELBO (denoted by \textit{w/ residual} in the main paper) by counting the residual term $\mathbb{E}_{p_{r}(\mathbf{x}_{0})p_{0\epsilon}(\mathbf{x}_{\epsilon}\vert\mathbf{x}_{0})}\big[\log{\frac{p_{\epsilon 0}^{\bm{\theta}}(\mathbf{x}_{0}\vert\mathbf{x}_{\epsilon})}{p_{0\epsilon}(\mathbf{x}_{\epsilon}\vert\mathbf{x}_{0})}}\big]$ into account. Note that $\frac{(1+\lambda^{2})^{2}}{4\lambda^{2}}$ is minimized when $\lambda=1$, so our reported NELBO is based on the generative SDE at $\lambda=1$.
	
	\subsection{NLL Correction}\label{appendix:nll_correction}
	
	NLL of the generative SDE can be computed through the Feynman-Kac formula \cite{huang2021variational} by
	\begin{align}\label{eq:feynman-kac_formula}
	p_{\bm{\theta},\epsilon}^{\lambda}(\mathbf{x}_{\epsilon})=\mathbb{E}_{\{\mathbf{x}_{t}\}_{t>\epsilon}\vert\mathbf{x}_{\epsilon}}\bigg[\pi(\mathbf{x}_{T})\exp{\bigg(-\int_{\epsilon}^{T} \text{tr}\Big( \nabla_{\mathbf{x}_{t}}\Big[\mathbf{f}(\mathbf{x}_{t},t) - \frac{1+\lambda^{2}}{2}g^{2}(t)\mathbf{s}_{\bm{\theta}}(\mathbf{x}_{t},t)\Big] \Big) \diff t\bigg)}\bigg].
	\end{align}
	However, the expectation is intractable because there are infinitely-many sample paths. Fortunately, the sample variance diminishes as $\lambda\rightarrow 0$, and the generative SDE collapses to a generative ODE when $\lambda=0$ \cite{song2020score}, i.e., the generative SDE of $\lambda=0$ becomes
	\begin{align*}
	\diff\mathbf{x}_{t}=\bigg[\mathbf{f}(\mathbf{x}_{t},t)-\frac{1}{2}g^{2}(t)\mathbf{s}_{\bm{\theta}}(\mathbf{x}_{t},t)\bigg]\diff \bar{t},
	\end{align*}
	which corresponds to the generative ODE of forward time as
	\begin{align}\label{eq:generative_ode_forward}
	\diff\mathbf{x}_{t}=\bigg[\mathbf{f}(\mathbf{x}_{t},t)-\frac{1}{2}g^{2}(t)\mathbf{s}_{\bm{\theta}}(\mathbf{x}_{t},t)\bigg]\diff t.
	\end{align}
	Then, the sample path becomes deterministic, and the expectation in Eq. \eqref{eq:feynman-kac_formula} is degenerated as the single sample path of ODE with Eq. \eqref{eq:generative_ode_forward} starting $\mathbf{x}_{\epsilon}$. The instantaneous change-of-variable formula \cite{song2020score}, which is a collapsed Feynman-Kac formula in Eq. \eqref{eq:feynman-kac_formula}, guarantees that there is a corresponding ODE of Eq. \eqref{eq:generative_ode_forward} as
	\begin{align}\label{eq:instantaneous_change_of_variable_ode}
	\frac{\diff \log{p_{t}^{0,\bm{\theta}}(\mathbf{y}_{t})}}{\diff t}=-\text{tr}\bigg(\nabla_{\mathbf{y}_{t}}\Big[\mathbf{f}(\mathbf{y}_{t},t)-\frac{1}{2}g^{2}(t)\mathbf{s}_{\bm{\theta}}(\mathbf{y}_{t},t)\Big]\bigg).
	\end{align}
	From the fact that the reverse SDEs have the identical marginal distributions described in Section \ref{appendix:equivalent_reverse_sdes}, we approximate the model log-likelihood at $\lambda=1$ by the log-likelihood at $\lambda=0$ at the expense of slight difference between the model distributions of different $\lambda$s. When computing the model log-likelihood at $\lambda=0$, we integrate the ODE of Eq. \eqref{eq:instantaneous_change_of_variable_ode} over $[\epsilon,1]$ using an ODE solver, such as the Runge-Kutta 45 method \cite{dormand1980family}.
	
	There are minor subtleties in computing the log-likelihood at $\lambda=0$ that significantly affects to bpd evaluation. To the best of our knowledge, all the current practice on continuous diffusion models computes bpd by integrating
	\begin{align*}
	\frac{\diff \log{p_{t}^{0,\bm{\theta}}(\mathbf{y}_{t})}}{\diff t}=-\text{tr}\bigg(\nabla_{\mathbf{y}_{t}}\Big[\mathbf{f}(\mathbf{y}_{t},t)-\frac{1}{2}g^{2}(t)\mathbf{s}_{\bm{\theta}}(\mathbf{y}_{t},t)\Big]\bigg),
	\end{align*}
	on $t\in[\epsilon,T]$, where $\{\mathbf{y}_{t}\}_{t=\epsilon}^{T}$ is a sample path starting from $\mathbf{y}_{\epsilon}:=\mathbf{x}_{0}$. This is equivalent of computing $\log{p_{\epsilon}^{0,\bm{\theta}}(\mathbf{x}_{0})}$. However, strarting from $\mathbf{x}_{0}$ incurs large discrepancy on the NLL output, compared to starting from instead of $\mathbf{x}_{\epsilon}$. Since the integration is on $[\epsilon,1]$, the starting variable should follow $\mathbf{x}_{\epsilon}$, which is a slightly perturbed variable.
	
	To fix this subtlety, we solve the below alternative differential equation of
	\begin{align}\label{eq:modified_ode}
	\frac{\diff \log{p_{t}^{0,\bm{\theta}}}}{\diff t}=-\text{tr}\bigg(\nabla_{\mathbf{y}_{t}}\Big[\mathbf{f}(\mathbf{y}_{t},t)-\frac{1}{2}g^{2}(t)\mathbf{s}_{\bm{\theta}}(\mathbf{y}_{t},t)\Big]\bigg),
	\end{align}
	on $t\in[\epsilon,T]$, where $\{\mathbf{y}_{t}^{0}\}_{t=\epsilon}^{T}$ is a sample path starting from $\mathbf{y}_{\epsilon}:=\mathbf{x}_{\epsilon}$. By replacing the initial value to $\mathbf{x}_{\epsilon}$ from $\mathbf{x}_{0}$, we could correctly compute $\log{p_{\epsilon}^{\bm{\theta}}(\mathbf{x}_{\epsilon})}$. Table \ref{tab:difference} presents the difference of $\mathbb{E}_{\mathbf{x}_{\epsilon}}\big[-\log{p_{\epsilon}^{0,\bm{\theta}}(\mathbf{x}_{\epsilon})}\big]-\mathbb{E}_{\mathbf{x}_{0}}\big[-\log{p_{\epsilon}^{0,\bm{\theta}}(\mathbf{x}_{0})}\big]$ with various $\epsilon$.
	\begin{table}[t]
		\caption{The difference of the integration with different initial points of $\mathbf{x}_{\epsilon}$ and $\mathbf{x}_{0}$ on DDPM++ (VP, NLL). The difference increases by $\epsilon$.}
		\label{tab:difference}
		\centering
		\tiny
		\begin{tabular}{ccccc}
			\toprule
			& \multicolumn{4}{c}{$\epsilon$}\\\cmidrule(lr){2-5}
			& $10^{-2}$ & $10^{-3}$ & $10^{-4}$ & $10^{-5}$\\\midrule
			$\mathbb{E}_{\mathbf{x}_{\epsilon}}\big[-\log{p_{\epsilon}^{\bm{\theta}}(\mathbf{x}_{\epsilon})}\big]-\mathbb{E}_{\mathbf{x}_{0}}\big[-\log{p_{\epsilon}^{\bm{\theta}}(\mathbf{x}_{0})}\big]$ & 1.13 & 0.73 & 0.24 & 0.05 \\
			\bottomrule
		\end{tabular}
	\end{table}
	We report the \textit{correct} NLL as
	\begin{align*}
	\mathbb{E}_{\mathbf{x}_{\epsilon}}\big[-\log{p_{\epsilon}^{0,\bm{\theta}}(\mathbf{x}_{\epsilon})}\big]-\mathbb{E}_{\mathbf{x}_{0},\mathbf{x}_{\epsilon}}\bigg[\log{\frac{p_{\epsilon 0}^{\bm{\theta}}(\mathbf{x}_{0}\vert\mathbf{x}_{\epsilon})}{p_{0\epsilon}(\mathbf{x}_{\epsilon}\vert\mathbf{x}_{0})}}\bigg],
	\end{align*}
	where $\log{p_{\epsilon}^{0,\bm{\theta}}(\mathbf{x}_{\epsilon})}$ is computed based on the initial point of $\mathbf{x}_{\epsilon}$ when $\lambda=0$.
	
	\subsection{Calculating the Residual Term}
	
	This section calculates the residual term, $\mathbb{E}_{p_{r}(\mathbf{x}_{0})p_{0\epsilon}(\mathbf{x}_{\epsilon}\vert\mathbf{x}_{0})}\big[\log{\frac{p_{\epsilon 0}^{\bm{\theta}}(\mathbf{x}_{0}\vert\mathbf{x}_{\epsilon})}{p_{0\epsilon}(\mathbf{x}_{\epsilon}\vert\mathbf{x}_{0})}}\big]$.	The transition probability of $p_{0\epsilon}(\mathbf{x}_{\epsilon}\vert\mathbf{x}_{0})$ is the Gaussian distribution of $\mathcal{N}(\mathbf{x}_{\epsilon};\mu(\epsilon)\mathbf{x}_{0},\sigma^{2}(\epsilon)\mathbf{I})$ if $\mathbf{f}(\mathbf{x}_{t},t)=-\frac{1}{2}\beta(t)\mathbf{x}_{t}$ with
	\begin{align*}
	\mu(\epsilon)=e^{-\frac{1}{2}\int_{0}^{\epsilon}\beta(s)\diff s}\text{ and }\sigma^{2}(\epsilon)=\Big(\frac{g^{2}(t)}{\beta(t)}-\frac{g^{2}(0)}{\beta(0)}+1-e^{-\int_{0}^{t}\beta(s)\diff s}\Big),
	\end{align*}
	see Appendix A.1 of \citet{kim2022soft} for detailed computation. On the other hand, the generative distribution of $p_{\bm{\theta},\epsilon 0}(\mathbf{x}_{0}\vert\mathbf{x}_{\epsilon})$ is assumed to be a Gaussian distribution of $\mathcal{N}\big(\mathbf{x}_{0};\bm{\mu}_{\bm{\theta},\epsilon 0}(\mathbf{x}_{\epsilon}),\sigma_{\bm{\theta},\epsilon 0}^{2}\mathbf{I}\big)$, where $\bm{\mu}_{\bm{\theta},\epsilon 0}(\mathbf{x}_{\epsilon})=\frac{1}{\mu(\epsilon)}\big(\mathbf{x}_{\epsilon}+\sigma^{2}(\epsilon)\mathbf{s}_{\bm{\theta}}(\mathbf{x}_{\epsilon},\epsilon)\big)$ \cite{kingma2021variational}. Then, we have
	\begin{eqnarray*}
		\lefteqn{\mathbb{E}_{p_{0\epsilon}(\mathbf{x}_{\epsilon}\vert\mathbf{x}_{0})}\bigg[\log{\frac{p_{\epsilon 0}^{\bm{\theta}}(\mathbf{x}_{0}\vert\mathbf{x}_{\epsilon})}{p_{0\epsilon}(\mathbf{x}_{\epsilon}\vert\mathbf{x}_{0})}}\bigg]}&\\
		&&=\log{\mu(\epsilon)}-\frac{1}{2\sigma_{\bm{\theta},\epsilon 0}^{2}(\epsilon)}\mathbb{E}_{p_{0\epsilon}(\mathbf{x}_{\epsilon}\vert\mathbf{x}_{0})}\bigg[\Big\Vert\mathbf{x}_{0}-\frac{1}{\mu(\epsilon)}\Big(\mathbf{x}_{t}+\sigma^{2}(\epsilon)\mathbf{s}_{\bm{\theta}}(\mathbf{x}_{\epsilon},\epsilon)\Big)\Big\Vert_{2}^{2}\bigg]+\frac{1}{2}.
	\end{eqnarray*}
	
	We could approximate the variance of $p_{\epsilon 0}^{\bm{\theta}}(\mathbf{x}_{0}\vert\mathbf{x}_{\epsilon})$ to be the variance of $p_{\epsilon 0}(\mathbf{x}_{0}\vert\mathbf{x}_{\epsilon})$, if $p_{\epsilon 0}(\mathbf{x}_{0}\vert\mathbf{x}_{\epsilon})$ is derived as a closed-form. For that, let us assume $\mathbf{x}_{0}\sim\mathcal{N}(0,\sigma^{2})$. Then,
	\begin{eqnarray*}
		\lefteqn{p(\mathbf{x}_{0},\mathbf{x}_{\epsilon})=p(\mathbf{x}_{0})p_{0\epsilon}(\mathbf{x}_{\epsilon}\vert\mathbf{x}_{0})}&\\
		&&\propto\exp{\left( -\frac{\Vert\mathbf{x}_{0}\Vert_{2}^{2}}{2\sigma^{2}} - \frac{\Vert\mathbf{x}_{\epsilon}-\mu(\epsilon)\mathbf{x}_{0}\Vert_{2}^{2}}{2\sigma^{2}(\epsilon)} \right)}\\
		&&=\exp{\left( -\frac{1}{2}\left( \frac{1}{\sigma^{2}}+\frac{\mu^{2}(\epsilon)}{\sigma^{2}(\epsilon)} \right) \left\Vert\mathbf{x}_{0}-\frac{\mu(\epsilon)\sigma^{2}}{\sigma^{2}(\epsilon)+\mu^{2}(\epsilon)\sigma^{2}}\mathbf{x}_{\epsilon}\right\Vert_{2}^{2}+O(\Vert\mathbf{x}_{\epsilon}\Vert_{2}^{2}) \right)}.
	\end{eqnarray*}
	Therefore, $p_{\epsilon 0}(\mathbf{x}_{0}\vert\mathbf{x}_{\epsilon})=\mathcal{N}\big(\mathbf{x}_{0}\big\vert\frac{\mu(\epsilon)\sigma^{2}}{\sigma^{2}(\epsilon)+\mu^{2}(\epsilon)\sigma^{2}}\mathbf{x}_{\epsilon},1/(\frac{1}{\sigma^{2}}+\frac{\mu^{2}(\epsilon)}{\sigma^{2}(\epsilon)})\big)$. When $\sigma$ is sufficiently large compared to $\frac{\sigma^{2}(\epsilon)}{\mu^{2}(\epsilon)}$, the variance of $p_{\epsilon 0}(\mathbf{x}_{0}\vert\mathbf{x}_{\epsilon})$ is approximately $\frac{\sigma^{2}(\epsilon)}{\mu^{2}(\epsilon)}$. Now, if $\mathbf{x}_{0}\sim p_{r}$, then the variance of $\mathbf{x}_{0}$ is large enough compared to $\frac{\sigma^{2}(\epsilon)}{\mu^{2}(\epsilon)}$, so we could approximate $\sigma_{\bm{\theta},\epsilon 0}^{2}(\epsilon)$ to be $\frac{\sigma^{2}(\epsilon)}{\mu^{2}(\epsilon)}$. Note that DDPM \cite{ho2020denoising} assumes the variance to be $\sigma^{2}(\epsilon)$. We compute the residual term with $\frac{\sigma^{2}(\epsilon)}{\mu^{2}(\epsilon)}$ variance for both VESDE and VPSDE. Note that this residual term is inspired from the released code of \citet{song2021maximum}.
	
	\section{Experimental Details and Additional Results}\label{appendix:experimental_details}

	\subsection{Model Architecture}
	
	\textbf{Diffusion Model} We implement two diffusion models as backbone: NCSN++ (VE) \citep{song2020score} and DDPM++ (VP) \citep{song2020score}, where two backbones are one of the best performers in CIFAR-10 dataset. In our setting, NCSN++ assumes the score network with parametrization of $\mathbf{s}_{\bm{\theta}}(\mathbf{z}_{t},\log{\sigma^{2}(t)})$, where $\sigma^{2}(t)=\sigma_{min}^{2}(\frac{\sigma_{max}}{\sigma_{min}})^{2t}$ is the variance of the transition probability $p_{0t}(\mathbf{z}_{t}\vert\mathbf{z}_{0})$ with VESDE. As introduced in \citet{song2020score}, we use the Gaussian Fourier embeddings \citep{tancik2020fourier} to model the high frequency details across the temporal embedding. DDPM++ models the score network with parametrization of $\mathbf{\epsilon}_{\bm{\theta}}(\mathbf{z}_{t},t)$, which targets to estimate $-\sigma(t)\nabla_{\mathbf{z}_{t}}\log{p_{t}(\mathbf{z}_{t})}$. We use the Transformer sinusoidal temporal embedding \citep{vaswani2017attention}.
	
	We use the U-Net \citep{ronneberger2015u} architecture for the score networks on both NCSN++ and DDPM++ based on \citep{ho2020denoising}. We stack U-Net resblocks of up-and-down convolutions with skip connections that give input information to the output layer. Also, we follow \citet{ho2020denoising} by applying the global attention at the resolution of $16\times 16$. We use four U-Net resblocks with four feature map resolutions ($32\times 32$ to $4\times 4$). On CIFAR-10, we use four and eight resblocks for shallow and deep settings, respectively. The performances of shallow and deep models turn out to be insignificant, so we use four resblocks on CelebA. We provide the identical diffusion model structures to compare the baseline linear diffusion model and the INDM model in a fair setting.
	
	\textbf{Flow Model} We build a normalizing flow model as follows. \citet{ma2020decoupling} uses the autoencoding structure of decouple the global information and the local representation. The compression encoder extracts the global information, and the invertible decoder is a conditional flow conditioned by the encoded latent representation. \citet{ma2020decoupling} utilizes a shallow network for the compressive encoder, and it applies Glow \cite{kingma2018glow} for the invertible decoder. We empirically find that resnet-based flow network outperforms the Glow-based flow, so we replace Glow to ResFlow \cite{chen2019residual}.
	
	For the ResFlow, we drop three components from the original paper: 1) the activation normalization \cite{kingma2018glow}, 2) the batch normalization \cite{ioffe2015batch}, and 3) the fully connected layers. For the activation function, we use the sine function \cite{lu2021implicit} on quantitative comparisons in Section 7, and we use swish function \cite{ramachandran2017searching} on qualitative analysis in otherwise sections including Section 6. With the sine activation, the training becomes more stable, and the FID performance is significantly improved while maintaining the NLL performance. For the multi-GPU training, we use the Neumann log-determinant gradient estimator, instead of the memory-efficient estimator \cite{chen2019residual}.
	
	\subsection{Experimental details}
	
	\begin{table}[t]
		\caption{Despite of our implementation is built deeply based on \citet{song2020score}, our pytorch implementation and the jax implementation of \citet{song2021maximum} differs in their final performances.}
		\label{tab:comparison_with_reported}
		\scriptsize
		\centering
		\begin{tabular}{clccccccc}
			\toprule
			\multirow{2}{*}{SDE} & \multirow{2}{*}{Model} & \multicolumn{2}{c}{NLL} & \multicolumn{2}{c}{NELBO} & \multicolumn{2}{c}{Gap} & FID \\
			&& after & before & w/ residual & w/o residual & after & before & ODE \\\midrule
			\multirow{3}{*}[-3pt]{VP} & DDPM++ (NLL, reported) \cite{song2021maximum} & - & 2.95 & - & 3.08 & - & 0.13 & 6.03 \\\cmidrule(lr){2-9}
			& DDPM++ (NLL, ours) & 3.03 & 2.97 & 3.13 & 3.11 & 0.10 & 0.14 & 6.70 \\
			& \cc{15}INDM (NLL) & \cc{15}2.98 & \cc{15}2.95 & \cc{15}2.98 & \cc{15}2.97 & \cc{15}0.00 & \cc{15}0.02 & \cc{15}6.01 \\
			\bottomrule
		\end{tabular}
	\end{table}
	
	With the batch size of 128, we train the diffusion model with Exponential Moving Average (EMA) \cite{kingma2021variational} rate of 0.9999, and we do not use EMA on our flow model. Using EMA on the flow model is advantageous on NLL at the expense of FID, and we build our model with emphasis on FID. We train the model by two step. The pre-training stage trains the diffusion model about five days with a flow model fixed as an identity function on four P40 GPUs with 96Gb GPU memory for all experiments. After the pre-training, we train both flow and diffusion networks about five days. In this stage, we apply the learning rate scheduling to boost the FID score. We initiate the learning rate after the sample generation performance is saturated. For the diffusion model, we drop the learning rate from $2\times 10^{-4}$ to $10^{-5}$. For the flow model, we drop the learning rate from $10^{-3}$ to $10^{-5}$ for VPSDE and $5\times 10^{-5}$ to $10^{-5}$ for VESDE.
	
	VESDE and VPSDE have different training details. We apply INDM on VESDE with $\sigma_{min}=10^{-2}$. Throughout the experiments, VESDE has $\sigma_{max}=50$ on CIFAR-10 and $\sigma_{max}=90$ on CelebA. On the other hand, VPSDE assumes $\beta(t)=\beta_{min}+(\beta_{max}-\beta_{min})t$ with $\beta_{min}=0.1$ and $\beta_{max}=20$. Both VESDE and VPSDE truncate the diffusion time on $[\epsilon,T]$ in order to stabilize the diffusion model \cite{kim2022soft}, where $\epsilon=10^{-5}$ and $T=1$. 
	
	With all hyperparameters identical to \citet{song2021maximum}, however, we could not achieve the reported performance. Table \ref{tab:comparison_with_reported} compares the reported performance and the model performance trained on out implementation, of which structure is heavily based on the released code of \cite{song2020score}. Due to the discrepancy between the reported and the regenerated performances, we compare our INDM to the regenerated performance as default to investigate the effect of nonlinear diffusion in a fair setting. Throughout the training, we used 4$\times$ NVIDIA RTX 3090.
	
	\subsubsection{Variance Reduction}
	
	\textbf{Flow Training} When we train the flow network with $\mathcal{L}\big(\{\mathbf{x}_{t}\}_{t=0}^{T},g^{2};\{\bm{\phi},\bm{\theta}\})$, this NELBO contains the integration of 
	\begin{align*}
	\mathcal{L}\big(\{\mathbf{z}_{t}\}_{t=0}^{T},g^{2};\bm{\theta}\big)=\frac{1}{2}\int_{0}^{T}g^{2}(t)\mathbb{E}_{\mathbf{z}_{0},\mathbf{z}_{t}}\big[\Vert\mathbf{s}_{\bm{\theta}}(\mathbf{z}_{t},t)\big]\diff t,
	\end{align*}
	up to a constant. Suppose $\mathcal{L}_{t}\big(\{\mathbf{z}_{t}\}_{t=0}^{T};\bm{\theta}\big)$ to be $\frac{1}{2}\mathbb{E}_{\mathbf{z}_{0},\mathbf{z}_{t}}[\Vert\mathbf{s}_{\bm{\theta}}(\mathbf{z}_{t},t)-\nabla_{\mathbf{z}_{t}}\log{p_{0t}(\mathbf{z}_{t}\vert\mathbf{x}_{0})}\Vert_{2}^{2}]$. Previous works on diffusion models \citep{nichol2021improved,huang2021variational,song2021maximum, kim2022soft} show that the estimation variance is largely reduced with the importance sampling, which could improve the model performance \citep{song2021maximum}, and we apply this importance sampling throughout the experiments for NLL setting. Concretely, the importance sampling chooses an importance weight that is proportional to $\frac{g^{2}(t)}{\sigma^{2}(t)}$, and estimates the integration by $\mathcal{L}\big(\{\mathbf{z}_{t}\}_{t=0}^{T},g^{2};\bm{\theta}\big)=\int_{0}^{T}g^{2}(t)\mathcal{L}_{t}\big(\{\mathbf{z}_{t}\}_{t=0}^{T};\bm{\theta}\big)\diff t\approx\sum_{n=1}^{N}\sigma^{2}(t_{n})\mathcal{L}_{t_{n}}\big(\{\mathbf{z}_{t}\}_{t=0}^{T};\bm{\theta}\big)$, where $t_{n}$ is sampled from the importance distribution.
	
	For VESDE, it satisfies $\beta(t)=0$ and $g(t)=\sigma_{min}(\frac{\sigma_{max}}{\sigma_{min}})^{t}\sqrt{2\log{(\frac{\sigma_{max}}{\sigma_{min}})}}$. The transition probability becomes $p_{-\infty t}(\mathbf{z}_{t}\vert\mathbf{z}_{-\infty})=\mathcal{N}(\mathbf{z}_{t};\mathbf{z}_{-\infty},\sigma^{2}(t))$, where $\sigma^{2}(t)=\int_{-\infty}^{t}g^{2}(s)\diff s=\sigma_{min}^{2}\big(\frac{\sigma_{max}}{\sigma_{min}}\big)^{2t}$. Since $\sigma^{2}(t)$ is proportional to $g^{2}(t)$ in VESDE, the importance weight follows the uniform distribution, and the importance sampling is equivalent with choosing the uniform $t$. This is why there is no experimetal setting of VE with NLL.
	
	On the other hand, VPSDE satisfies $\beta(t)=\beta_{min}+(\beta_{max}-\beta_{min})t$ with $g(t)=\sqrt{\beta(t)}$. Then, the transition probability becomes $p_{0t}(\mathbf{z}_{t}\vert\mathbf{z}_{0})=\mathcal{N}(\mathbf{z}_{t};\mu(t)\mathbf{z}_{t},\sigma^{2}(t)\mathbf{I})$, where $\mu(t)=e^{-\frac{1}{2}\int_{0}^{t}\beta(s)\diff s}$ and $\sigma^{2}(t)=1-e^{-\int_{0}^{t}\beta(s)\diff s}$. Thus, VPSDE has the importance weight of $\frac{g^{2}(t)}{\sigma^{2}(t)}=\frac{\beta(t)}{1-e^{-\int_{0}^{t}\beta(s)\diff s}}$.
	
	The Monte-Carlo sample from this importance weight is the solution of the inverse Cumulative Distribution Function (CDF) of the importance distribution as
	\begin{align}\label{eq:inverse_cdf}
	\begin{split}
	&t=F^{-1}(u)\\
	\iff &u=F(t)=\frac{1}{Z}\int_{\epsilon}^{t}\frac{g^{2}(s)}{\sigma^{2}(s)}\diff s=\frac{1}{Z}\big(\mathcal{F}(t)-\mathcal{F}(\epsilon)\big),
	\end{split}
	\end{align}
	where $u$ is a uniform sample from $[0,1]$, $\mathcal{F}(t)$ is the antiderivative of the importance weight given by $\mathcal{F}(t)=\log{(1-e^{-0.5 t^{2}(\beta_{max}-\beta_{min})-t\beta_{min}})}+0.5t^{2}(\beta_{max}-\beta_{min})+t\beta_{min}$, and $Z$ is the normalizing constant given by
	\begin{align*}
	Z=&\int_{\epsilon}^{T}\frac{g^{2}(t)}{\sigma^{2}(t)}\diff t\\
	=&\Big[\log{(1-e^{-0.5 t^{2}(\beta_{max}-\beta_{min})-t\beta_{min}})}+0.5t^{2}(\beta_{max}-\beta_{min})+t\beta_{min}\Big]_{\epsilon}^{T}\\
	=&\log{(1-e^{-0.5T^{2}(\beta_{max}-\beta_{min})-T\beta_{min}})}-\log{(1-e^{-0.5\epsilon^{2}(\beta_{max}-\beta_{min})-\epsilon\beta_{min}})}\\
	&+0.5(T^{2}-\epsilon^{2})(\beta_{max}-\beta_{min})+(T-\epsilon)\beta_{min}\\
	\approx&23.86
	\end{align*}
	for $T=1$ and $\epsilon=10^{-5}$. The solution for the inverse CDF in Eq. \eqref{eq:inverse_cdf} becomes
	\begin{align*}
	&e^{\int_{0}^{t}\beta(s)\diff s}=1+\exp{(Zu+\mathcal{F}(\epsilon))}\\
	&\iff \int_{0}^{t}\beta(s)\diff s=\frac{1}{2}(\beta_{max}-\beta_{min})t^{2}+\beta_{min}t=\log{(1+\exp{(Zu+\mathcal{F}(\epsilon))})}\\
	&\iff t=\frac{-\beta_{min}+\sqrt{\beta_{min}^{2}+2(\beta_{max}-\beta_{min})\log{\big(1+\exp{(Zu+\mathcal{F}(\epsilon))}\big)}}}{\beta_{max}-\beta_{min}}.
	\end{align*}
	The variation of the Monte-Carlo diffusion time depends on the uniform sample of $u$.
	
	\begin{align*}
	&\int_{0}^{t}\beta(s)\diff s=\log{\bigg(\frac{1-\sigma_{min}^{2}}{1-\sigma_{min}^{2}\big(\frac{\sigma_{max}}{\sigma_{min}}\big)^{t}}\bigg)}=\log{\big(1+\exp{(Zu+\mathcal{F}(\epsilon))}\big)}\\
	&\iff 1-\sigma_{min}^{2}\Big( \frac{\sigma_{max}}{\sigma_{min}} \Big)^{t}=\frac{1-\sigma_{min}^{2}}{1+e^{Zu+\mathcal{F}(\epsilon)}}\\
	&\iff \sigma_{min}^{2}\Big(\frac{\sigma_{max}^{2}}{\sigma_{min}^{2}}\Big)^{t}=\frac{e^{Zu+\mathcal{F}(\epsilon)}+\sigma_{min}^{2}}{1+e^{Zu+\mathcal{F}(\epsilon)}}\\
	&\iff t\log{\frac{\sigma_{max}^{2}}{\sigma_{min}^{2}}}=\log{\Big( e^{Zu+\mathcal{F}(\epsilon)}+\sigma_{min}^{2} \Big)}-\log{\Big(1+e^{Zu+\mathcal{F}(\epsilon)}\Big)}-\log{\sigma_{min}^{2}}.
	\end{align*}
	
	\begin{table}[t]
		\caption{Ablation study on the stopping sampling time trained on DDPM++ (VP) in CIFAR-10.}
		\label{tab:sampling_eps_search}
		\scriptsize
		\centering
		\begin{tabular}{lccc}
			\toprule
			Model & FID ($t_{min}=10^{-3}$) & FID ($t_{min}=10^{-4}$) & FID ($t_{min}=10^{-5}$)\\
			\midrule
			INDM (deep, VP, NLL) & 5.94 & 5.74 & 5.71 \\
			\bottomrule
		\end{tabular}
	\end{table}
	
	\begin{table}[t]
		\caption{Ablation study on the SNR trained on NCSN++ (VE) in CIFAR-10. The performances are FID-5k scores.}
		\label{tab:snr_search}
		\tiny
		\centering
		\begin{tabular}{lccccc}
			\toprule
			\multirow{2}{*}[-2pt]{Model} & \multicolumn{5}{c}{Signal-to-Natio Ratio (SNR)}\\\cmidrule(lr){2-6}
			& 0.13 & 0.14 & 0.15 & 0.16 & 0.17\\\midrule
			INDM (VE, FID) & 7.24 & 7.12 & 7.20 & 7.25 & 7.34 \\
			\bottomrule
		\end{tabular}
	\end{table}
	
	\begin{table}[t]
		\caption{Ablation study on the temperature for PC sampler trained on DDPM++ (VP) in CIFAR-10. The performances are FID scores. Contrary to \citet{kingma2018glow}, the temperature bigger than 1 works the best.}
		\label{tab:temperature_search}
		\tiny
		\centering
		\begin{tabular}{lcccccc}
			\toprule
			\multirow{2}{*}[-2pt]{Model} & \multicolumn{6}{c}{Temperature}\\\cmidrule(lr){2-7}
			& 1 & 1.03 & 1.04 & 1.05 & 1.1 & 1.2\\\midrule
			INDM (VP, FID) & 2.92 & 2.91 & 2.90 & 2.90 & 2.91 & 3.09 \\
			\bottomrule
		\end{tabular}
	\end{table}
	
	\begin{table}[t]
		\caption{Ablation study on the final denoising step trained on VESDE in CIFAR-10. The performances are FID scores.}
		\label{tab:line_search}
		\tiny
		\centering
		\begin{tabular}{lccccccc}
			\toprule
			Model & FID ($\mathbf{x}_{\epsilon}$) & FID ($\mathbf{x}_{-0.25}$) & FID ($\mathbf{x}_{-0.5}$) & FID ($\mathbf{x}_{-0.75}$) & FID ($\mathbf{x}_{-1}$) & FID ($\mathbf{x}_{-1.5}$) & FID ($\mathbf{x}_{-\infty}$)\\
			\midrule
			NCSN++ (VE, FID) & 11.7 & 8.40 & 40.8 & 65.7 & 85.8 & 118 & 2.46 \\\midrule
			\cc{15}INDM (VE, FID) & \cc{15}2.40 & \cc{15}2.33 & \cc{15}2.31 & \cc{15}2.29 & \cc{15}2.29 & \cc{15}2.34 & \cc{15}2.37 \\\cmidrule(lr){1-1}
			\cc{15}INDM (VE, deep, FID) & \cc{15}2.35 & \cc{15}2.29 & \cc{15}2.28 & \cc{15}2.29 & \cc{15}2.29 & \cc{15}2.36 & \cc{15}2.33 \\
			\bottomrule
		\end{tabular}
	\end{table}
	
	\subsubsection{Sampling Tricks to Improve FID}\label{appendix:sampling_tricks}
	For ODE sampler, we use Runge Kutta 45 method \cite{shampine1986some} for the ODE solver. Since the score network was not trained beneath the truncation time, i.e., $\mathbf{s}_{\bm{\theta}}(\mathbf{z}_{t},t)$ has not been trained on $t\in[0,\epsilon)$, keep denoising up to the zero diffusion time would harm the sample fidelity. If $t_{min}$ is the stopping diffusion time of the ODE, one predictor step from $t_{min}$ to $0$ is applied to the noised sample, $\mathbf{z}_{t_{min}}$, in order to eliminate the residual noise in $\mathbf{z}_{t_{min}}$ to $\mathbf{z}_{0}$ \cite{jolicoeur2020adversarial}. Table \ref{tab:sampling_eps_search} searches the optimal stopping diffusion time, and it shows that the truncation time ($10^{-5}$) turns out to be the optimal stopping time. Throughout the paper, we report the FID (ODE) performance of our INDM with the training truncation time ($10^{-5}$). For VESDE, the ODE sampler fails to generate realistic images, so we do not report sample generation performance.
	
	\begin{wrapfigure}{r}{0.3\textwidth}
		\vskip -0.15in
		\begin{subfigure}{\linewidth}
			\centering
			\includegraphics[width=\linewidth]{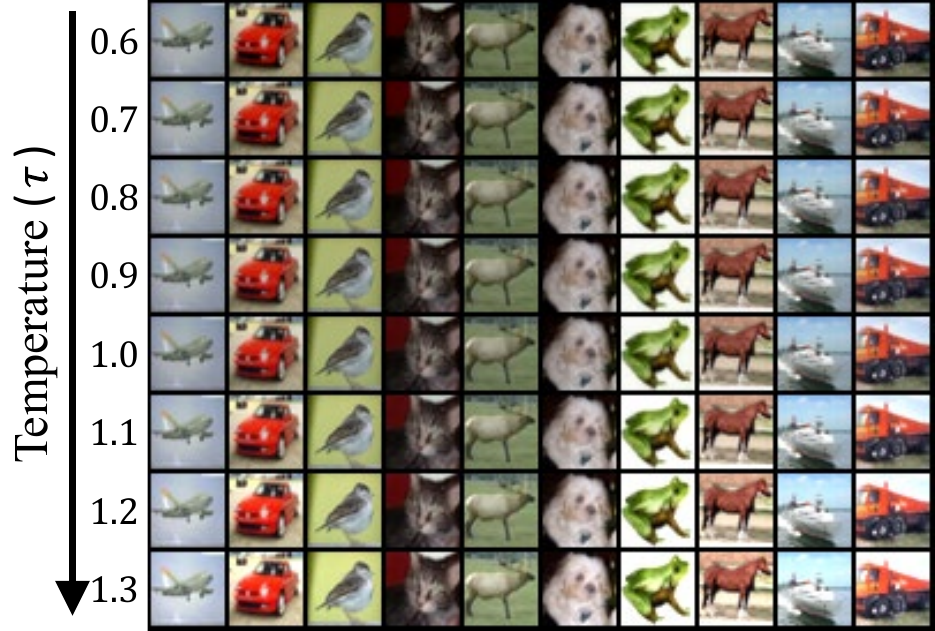}
		\end{subfigure}
		\caption{Ablation study for the flow temperature.}
		\label{fig:temperature}
		\vskip -0.3in
	\end{wrapfigure}
	In PC sampler, for the predictor, we use the Reverse Diffusion Predictor for VESDE and the Euler-Maruyama Predictor for VPSDE. For the corrector, we use the Langevin dynamics \cite{song2019generative} for VESDE, and we do not use any corrector for VPSDE. We use 1) Signal-to-Noise Ratio (SNR) scheduling, 2) temperature scheduling, 3) stopping time scheduling, and 4) data-adaptive prior than a fixed prior to improve FID. First, Table \ref{tab:snr_search} presents that the optimal SNR is 0.14, which is slightly different from the optimal SNR of 0.16 in the linear diffusion \cite{song2020score}. We use SNR of 0.14 as default in our PC sampling. 
	
	Second, as introduced in \citet{kingma2018glow}, we scale the generated latent, $\mathbf{z}_{0}^{\bm{\theta}}$, by multiplying the temperature. Table \ref{tab:temperature_search} presents that the optimal temperature for VPSDE is $1.04\sim 1.05$ in terms of FID on INDM (VP, FID) setting. We use the temperature of 1 for the remaining settings except INDM (VP, FID). With temperature $\tau$, the normalizing flow puts its latent input scaled by $\tau$ to the flow network. In Figure \ref{fig:temperature}, the image color with a higher temperature tends to be brighter, and we find that the optimal temperature depends on the experimental settings.
	
	Third, the stopping time scheduling is a method that manipulate the final denoising step. To attain the variance of VESDE as $\sigma^{2}(t)=\sigma_{min}^{2}(\frac{\sigma_{max}^{2}}{\sigma_{min}^{2}})^{2t}$, we should start the diffusion process of $\diff\mathbf{z}_{t}=\sigma^{2}(t)\diff\mathbf{w}_{t}$ at $t=-\infty$ because 
	\begin{align}\label{eq:vesde_variance}
	\sigma^{2}(t)=\int_{t_{0}}^{t}g^{2}(s)\diff s=\sigma_{min}^{2}\Big(\frac{\sigma_{max}^{2}}{\sigma_{min}^{2}}\Big)^{2t}
	\end{align}
	implies $t_{0}=-\infty$. If the generative SDE is $\diff\mathbf{z}_{t}=g^{2}(t)\mathbf{s}_{\bm{\theta}}(\mathbf{z}_{t},t)\diff \bar{t}+\sigma^{2}(t)\diff\mathbf{\bar{w}}_{t}$, then the Euler-Maruyama discretization is 
	\begin{align}\label{eq:Euler-Maruyama_discretization}
	\mathbf{z}_{t_{i}}\leftarrow\mathbf{z}_{t_{i+1}}+g^{2}(t_{i+1})(t_{i}-t_{i+1})\mathbf{s}_{\bm{\theta}}(\mathbf{z}_{t_{i+1}},t_{i+1})+\sqrt{\sigma(t_{i+1})^{2}-\sigma(t_{i})^{2}}\bm{\epsilon},
	\end{align}
	where $\bm{\epsilon}\sim\mathcal{N}(0,\mathbf{I})$. However, since the initial time of VESDE is $t=-\infty$, denoising the noised sample with the Euler-Maruyama discretization would incur arbitrary large error at the final step that denoises from $t=\epsilon$ to $t=-\infty$. Therefore, \citet{song2020score} suggested the reverse diffusion predictor that denoises by
	\begin{align}\label{eq:reverse_diffusion_discretization}
	\mathbf{z}_{\sigma^{-1}(\sigma_{i})}\leftarrow\mathbf{z}_{\sigma^{-1}(\sigma_{i+1})}+(\sigma_{i+1}^{2}-\sigma_{i}^{2})\mathbf{s}_{\bm{\theta}}(\mathbf{z}_{\sigma^{-1}(\sigma_{i+1})},\sigma_{i+1})+\sqrt{\sigma_{i+1}^{2}-\sigma_{i}^{2}}\bm{\epsilon},
	\end{align}
	which is equivalent to the Euler-Maruyama discretization if $t_{i}-t_{i+1}$ is small enough (because $\Delta \sigma^{2}(t)\approx g^{2}(t)\Delta t$ by Eq. \eqref{eq:vesde_variance}). The difference of Eqs. \eqref{eq:Euler-Maruyama_discretization} and \eqref{eq:reverse_diffusion_discretization} is minor as long as we denoise on the range of $[\epsilon,T]$, but only Eq. \eqref{eq:reverse_diffusion_discretization} enables to denoise from $t=\epsilon$ to $t=-\infty$.
	
	However, it turns out that the direction of the score network is not aligned to the direction of the data score near $t\approx 0$, so $\mathbf{s}_{\bm{\theta}}(\mathbf{z}_{\epsilon},\sigma_{min})$ would not be accurate enough to the perturbed data score. Therefore, the final denoising step of
	\begin{align*}
	\mathbf{z}_{-\infty}\leftarrow\mathbf{z}_{\epsilon}+\sigma_{min}^{2}\mathbf{s}_{\bm{\theta}}(\mathbf{z}_{\epsilon},\sigma_{min})+\sigma_{min}\bm{\epsilon},
	\end{align*}
	might not be mostly effective. This leads us to try the final step as
	\begin{align*}
	\mathbf{z}_{\epsilon-\delta}=\mathbf{z}_{\epsilon}+\frac{1}{2}\Big(\sigma^{2}(\epsilon)-\sigma^{2}(\epsilon-\delta)\Big)\mathbf{s}_{\bm{\theta}}(\mathbf{z}_{\epsilon},\sigma_{min}),
	\end{align*}
	for various $\delta\ge 0$. After the denoising up to $\mathbf{z}_{\epsilon-\delta}$, we apply the inverse of the flow network to obtain $\mathbf{x}_{\epsilon-\delta}=\mathbf{h}_{\bm{\phi}}^{-1}(\mathbf{z}_{\epsilon-\delta})$, and Table \ref{tab:line_search} presents that there is a sweet spot ($\mathbf{x}_{-0.5}\sim\mathbf{x}_{-0.75}$) that works the best in terms of FID. We report the line searched FID performance for each of VESDE setting. 
	
	Lastly, we use $p_{T}^{\bm{\phi}}$ instead of $\pi$ to sample from INDM (VE). This data-adaptive prior is particularly beneficial on the experiment of VESDE. In INDM (VE), the data-adaptive prior reduces FID-5k from 8.14 to 7.52, so we use this technique by default throughout out performance report. In INDM (VP), this technique is not effective, and we use the vanilla prior distribution. The reason why the data-adaptive prior is effective in VESDE is because the discrepancy of VESDE between $p_{T}^{\bm{\phi}}$ and $\pi$ is significantly larger than VPSDE, see Figure 5 of \citet{chen2021likelihood}.
	
	We compute FID \citep{heusel2017gans} for CIFAR-10 based on the statistics released by \citet{song2020score}\footnote{\url{https://github.com/yang-song/score_sde_pytorch}}, which used the modified Inception V1 network\footnote{\url{https://tfhub.dev/tensorflow/tfgan/eval/inception/1}} in order to compare INDM to the baselines \citep{song2020score, song2021maximum} in a fair setting. On the other hand, for the CelebA dataset, we compute the clean-FID \citep{parmar2022aliased} that provides consistently antialiased performance.
	
	\begin{figure}[t]
		\centering
		\includegraphics[width=\linewidth]{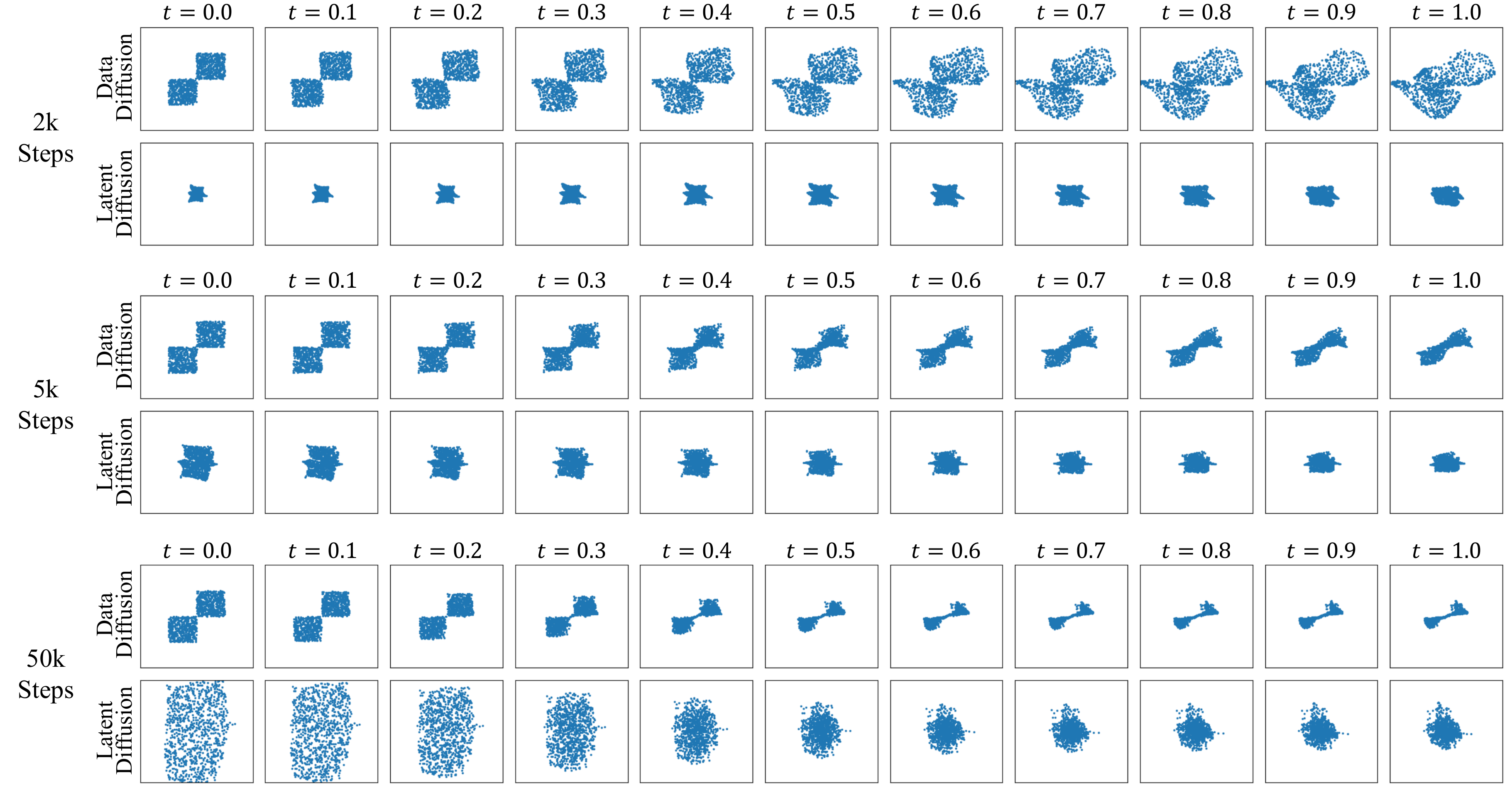}
		\caption{Comparison of data and latent diffusions by training steps of Checkerboard.}
		\label{fig:checkerboard}
	\end{figure}
	
	\subsubsection{Interpolation Task}\label{appendix:interpolation__}
	
	For the interpolation task, we provide the line-by-line algorithm in Algorithm \ref{alg:interpolation}. We train with the likelihood weighting as default for our experiment on the dataset interpolation. The interpolation loss of $\mathcal{L}_{int}$ consumes 0.2Gb of GPU memory, and the INDM loss of $\mathcal{L}_{INDM}$ takes 2.5Gb of GPU memory in the EMNIST-MNIST experiment.
	\begingroup
	\renewcommand\thealgorithm{2}
	\begin{algorithm}[H]
		\centering
		\caption{Data Interpolation of INDM}\label{alg:interpolation}
		\begin{algorithmic}[1]
			\Repeat
			\State Compute $\mathcal{L}_{INDM}=\mathcal{L}(\{\mathbf{x}_{t}\}_{t=0}^{T}, g^{2};\{\bm{\phi},\bm{\theta}\})$ for $\mathbf{x}_{0}\sim p_{r}^{(1)}$
			\State Compute $\mathcal{L}_{int}=\mathbb{E}_{p_{r}^{(2)}}[-\log{p_{\bm{\phi}}(\mathbf{y})}]$ for $\mathbf{y}\sim p_{r}^{(2)}$
			\State Compute $\mathcal{L}_{tot}=\mathcal{L}_{INDM}+\mathcal{L}_{int}$
			\State Update $\bm{\phi}\leftarrow \bm{\phi}-\frac{\partial\mathcal{L}_{tot}}{\partial\bm{\phi}}$
			\State Update $\bm{\theta}\leftarrow \bm{\theta}-\frac{\partial\mathcal{L}_{tot}}{\partial\bm{\theta}}$
			\Until {converged}
		\end{algorithmic}
	\end{algorithm}
	\endgroup
	
	\subsection{Effect of Pre-training}\label{appendix:pretraining}
	
	\begin{wraptable}{r}{0.5\textwidth}
		\vskip -0.38in
		\caption{Ablation study on pre-training.}
		\label{tab:pretraining}
		\scriptsize
		\centering
		\begin{tabular}{lccccc}
			\toprule
			$\#$ Pre-training Steps & 100k & 200k & 300k & 400k & 500k\\\midrule
			NLL & 3.00 & 2.99 & 2.99 & 2.99 & 2.98 \\
			FID & 7.39 & 7.31 & 6.80 & 6.65 & 6.22 \\
			\bottomrule
		\end{tabular}
		\vskip -0.23in
	\end{wraptable}
	We find that training INDM with a pre-trained score network of linear diffusion models improves FID. Table \ref{tab:pretraining} conducts the ablation study on the number of pre-training steps. We pre-train the score network with DDPM++ (VP, NLL) for five variations of pre-training steps (100k/200k/300k/400k/500k), and we train flow+score networks for 350k steps further with NLL setting ($\lambda=g^{2}$). Table \ref{tab:pretraining} empirically demonstrates that it is better to search the nonlinearity of the data process near the linear process. For this clear empirical advantage of pre-training, we report the quantitative performances in Section \ref{sec:experiments} with pre-training.
	
	\subsection{Training Time}\label{appendix:training_time}
	
	\begin{table*}[t]
		\caption{Elapsed time per a training step by discretization.}
		\label{tab:elapsed_time}
		\scriptsize
		\centering
		\begin{tabular}{lcccc}
			\toprule
			Model & Complexity & $N=100$ & $N=1,000$ & $N=\infty$ \\\midrule
			DDPM & $O(1)$ & 0.27 & 0.27 & 0.27 \\
			SBP (w/o experience memory) & $O(N)$ & 2.83 & 23.3 & $\infty$ \\
			SBP (w/ experience memory) & $O(N)$ & 0.52 & 2.39 & $\infty$ \\
			DiffFlow & $O(N)$ & 18.45 & 180.88 & $\infty$ \\
			INDM & $O(1)$ & 1.69 & 1.69 & 1.69 \\
			\bottomrule
		\end{tabular}
	\end{table*}
	
	\begin{table*}[t]
		\caption{Total training time in a single GPU days.}
		\label{tab:training_time}
		\scriptsize
		\centering
		\begin{tabular}{lcccccc}
			\toprule
			Model & Total Training Time (GPU Days) & Training Steps & GPU Spec & $\#$GPUs & NLL & FID \\\midrule
			DDPM++ & 5 & 500k & P40 & 1 & 3.03 & 6.70 \\
			LSGM & 44 & 450k & RTX 3090 & 8 & 2.87 & 6.89 \\
			SBP & 3 & 260k & RTX 3090 & - & 2.98 & 3.18 \\
			DiffFlow & 32 & 100k & RTX 2080 & 8 & 3.04 & 14.14 \\
			INDM (including pre-training time) & 25 & 700k & P40 & 4 & 2.98 & 6.01 \\
			INDM (w/o pre-training) & 60 & 600k & P40 & 4 & 2.98 & 8.49 \\
			\bottomrule
		\end{tabular}
	\end{table*}
	
	Table \ref{tab:elapsed_time} presents the elapsed time per a training step by the number of discretization on CIFAR-10. In contrast to INDM which is invariant on the choice of $N$, the training time of SBP and DiffFlow is not scalable for their $O(N)$ complexities. The training time is measured under the identical computing resource (1x NVIDIA RTX 3090/Intel I7 3.8GHz) and the same batch size (32) to compare INDM with baselines in a fair setting.
	
	Table \ref{tab:training_time} compares INDM with baselines with respect to a single GPU-time for the total training time on CIFAR-10. The remaining columns including training steps, GPU Spec, NLL, and FID are reported for the reference. For DiffFlow, we present the reported GPU days in the paper. For LSGM and SBP, we estimate the elapsed time with the released training configuration in their papers and GitHub repositories. For DDPM++ and INDM, we report the elapsed time from our own experiments. From the table, the overall training time of INDM/DiffFlow/LSGM remains at a similar scale. SBP is the fastest algorithm because of the experience replay memory technique. Note that a completely fair comparison between algorithms is infeasible because the training setup (e.g. $\#$GPUs, training steps, network size …) varies by algorithms. Also, P40 is strictly slower than RTX series GPUs.
	
	\begin{wrapfigure}{r}{0.5\textwidth}
		\vskip -0.45in	
		\begin{subfigure}{0.48\linewidth}
			\includegraphics[width=\linewidth]{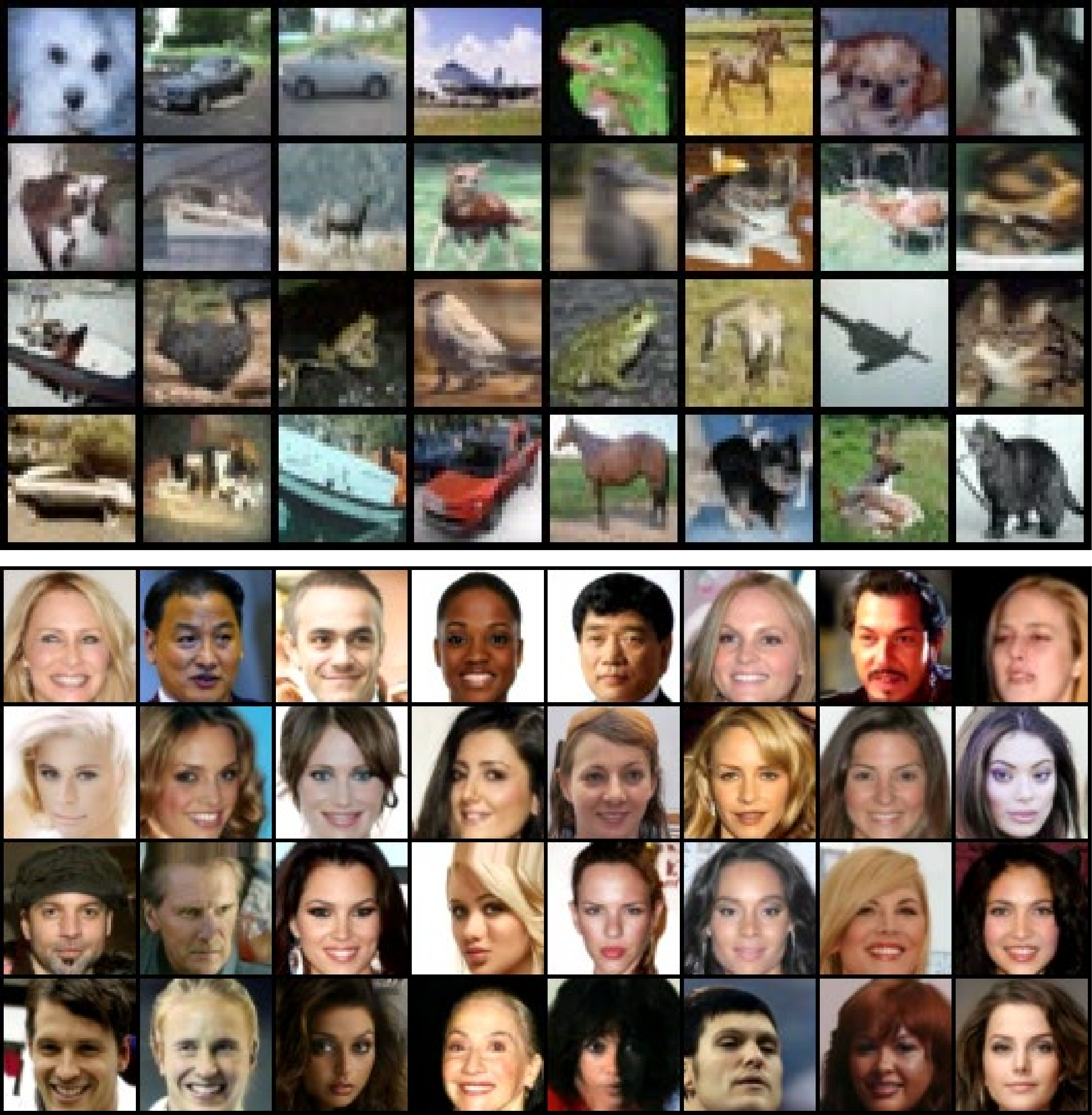}
			\subcaption{Samples from $\mathbf{x}_{0}^{\bm{\phi},\bm{\theta}}$}
		\end{subfigure}
		\begin{subfigure}{0.48\linewidth}
			\includegraphics[width=\linewidth]{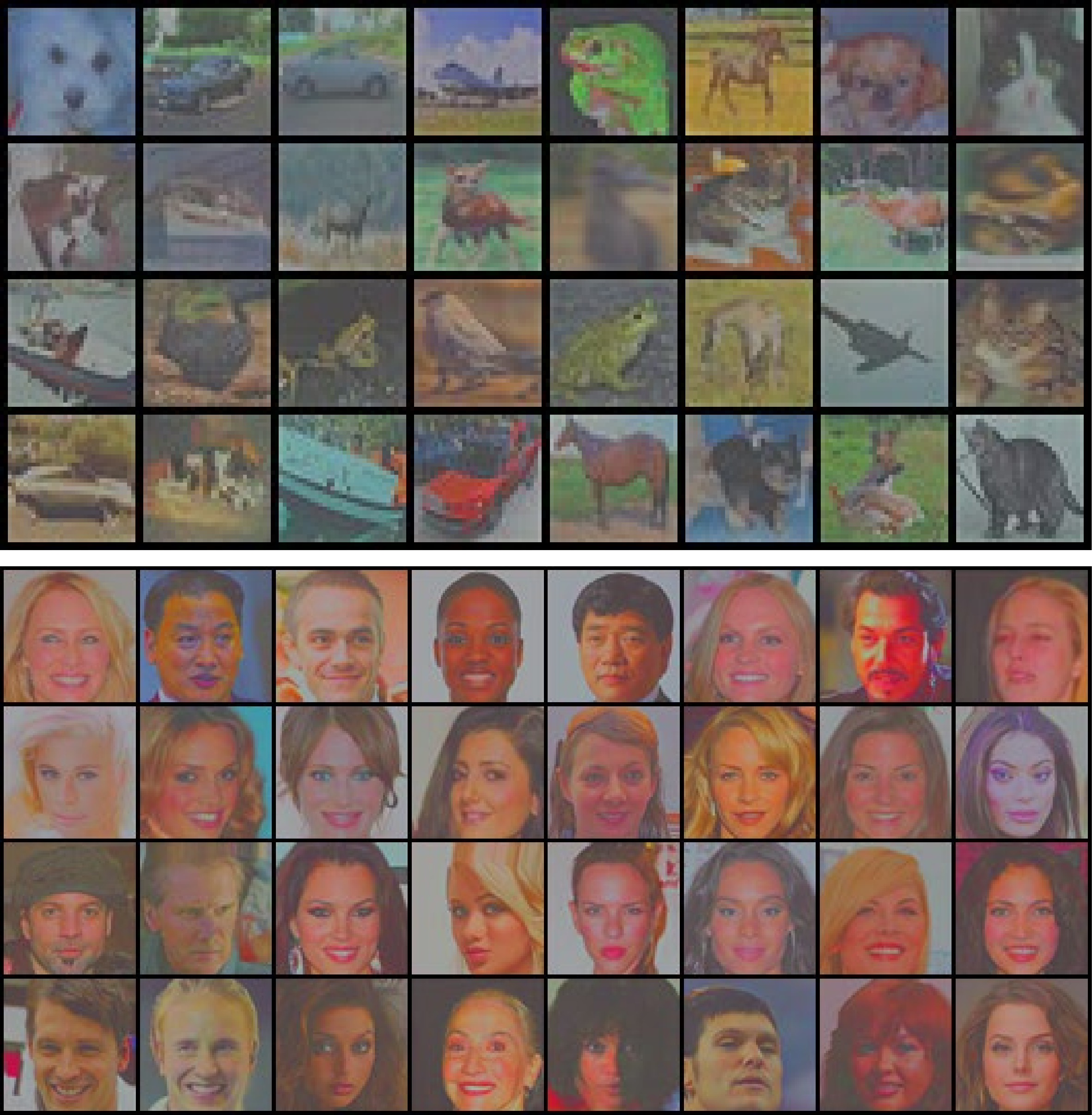}
			\subcaption{Samples from $\mathbf{z}_{0}^{\bm{\theta}}$}
		\end{subfigure}
		\vskip -0.05in
		\caption{Samples from the data space and latent space on CIFAR-10 and CelebA.}
		\label{fig:samples_}
		\vskip -0.15in
	\end{wrapfigure}
	\subsection{Visualization of Latent}
	
	\subsubsection{Visualization of 2d Latent Manifold}
	
	Figure \ref{fig:checkerboard} illustrates the data and latent manifolds of the 2d checkerboard dataset by training steps. The data manifold has the singularity at the origin, but this singularity disappears in the latent manifold after the training.
	
	\subsubsection{Visualization of High-dimensional Latent Vector on Benchmark Datasets}\label{appendix:visualization}
	
	Figure \ref{fig:samples_} illustrates the samples from (a) the data space and (b) the latent space. To visualize the latent vectors, we normalize the latent value into the $[0,1]^{d}$ space.

	\subsection{Nonlinear Diffusion Coefficient}\label{appendix:nonlinear_term}
	
	\begin{wrapfigure}{r}{0.6\textwidth}
		\vskip -0.2in	
		\begin{subfigure}{0.48\linewidth}
			\includegraphics[width=\linewidth]{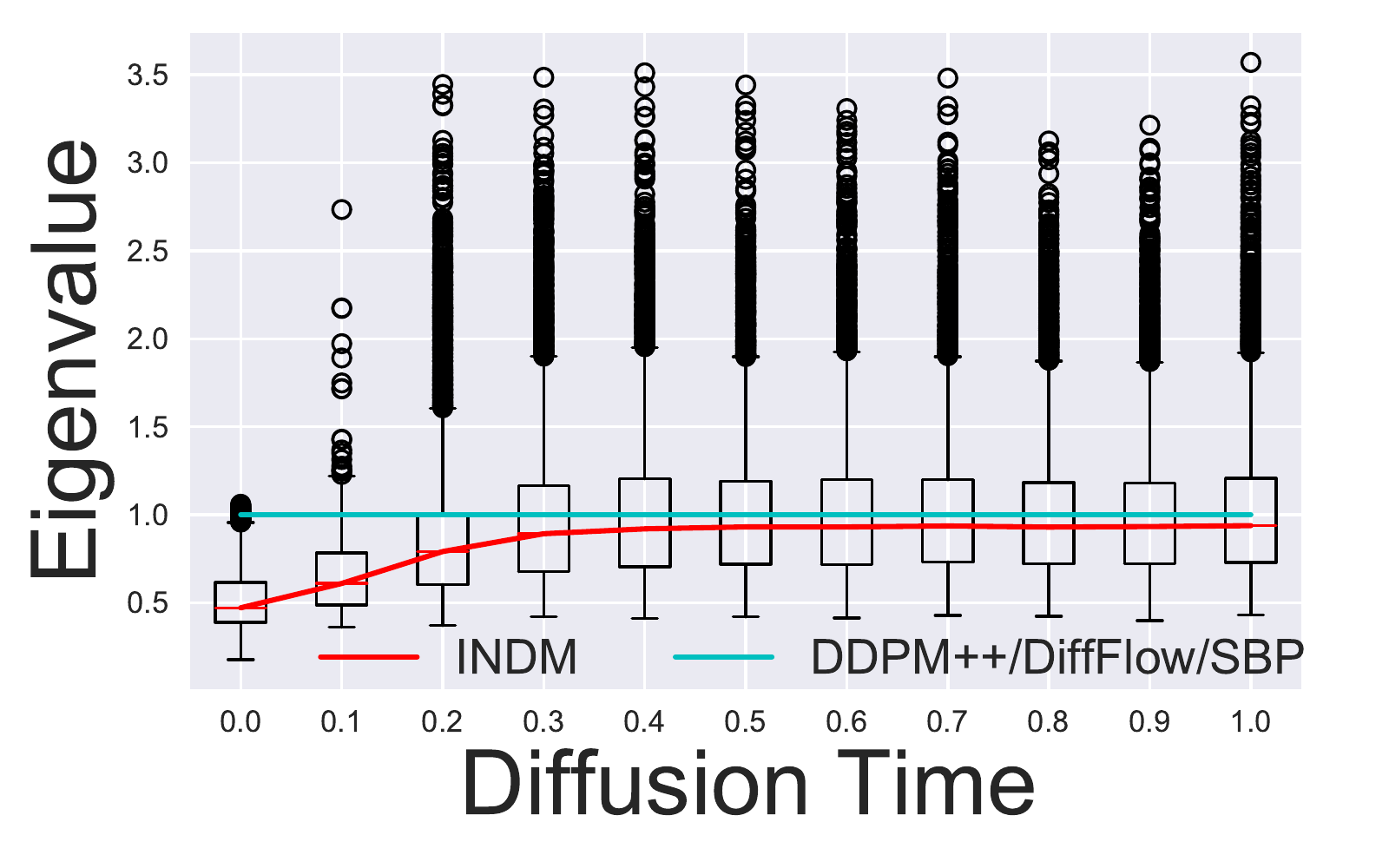}
			\subcaption{Eigenvalue by $t$}
		\end{subfigure}
		\begin{subfigure}{0.48\linewidth}
			\includegraphics[width=\linewidth]{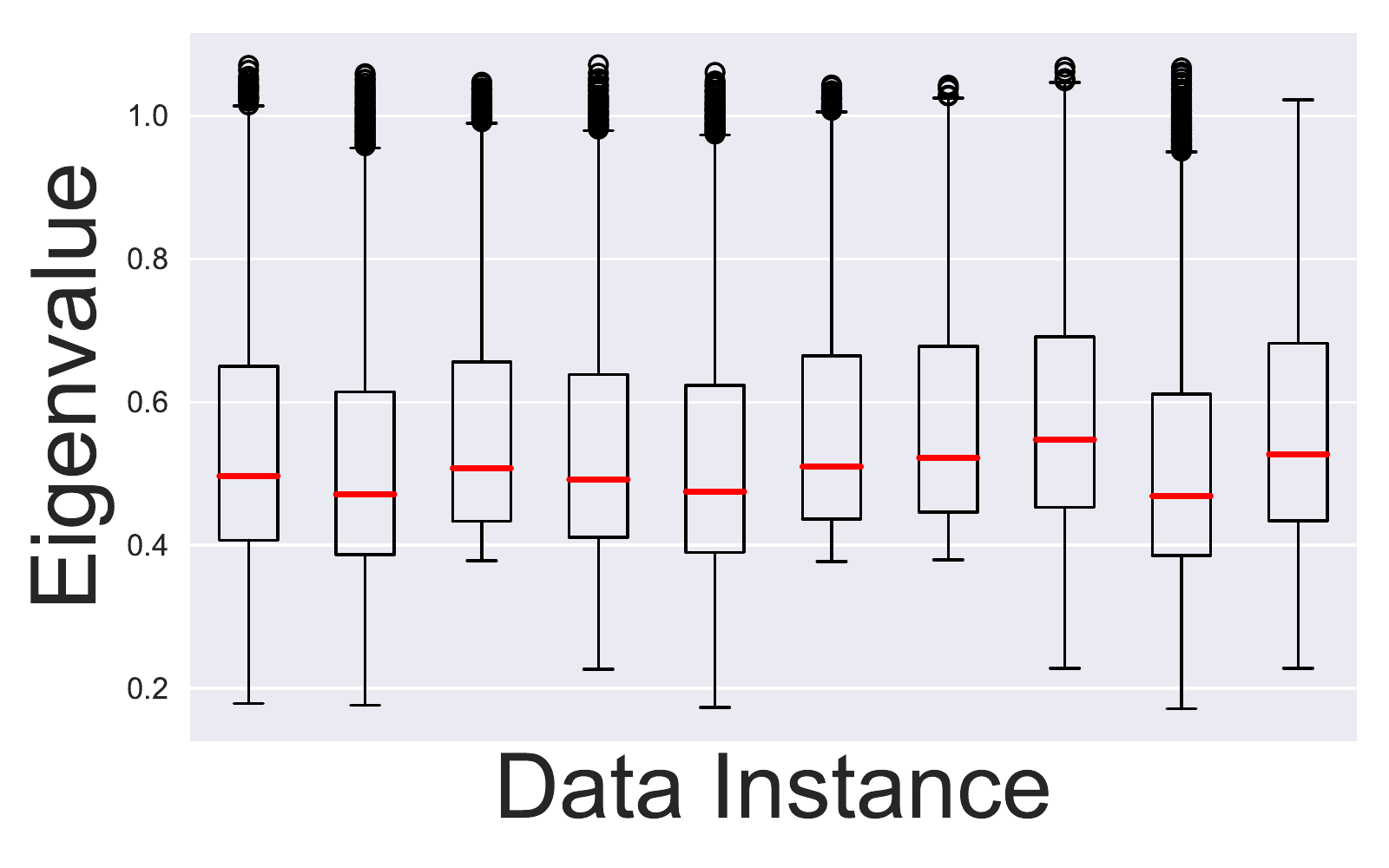}
			\subcaption{Eigenvalue by $\mathbf{x}_{t}$}
		\end{subfigure}
		\vskip -0.05in
		\caption{Eigenvalue of $\mathbf{G}\mathbf{G}^{T}/g^{2}$ on CIFAR-10.}
		\label{fig:eigenvalue}
		\vskip -0.15in
	\end{wrapfigure}
	INDM trains the volatility term, $\mathbf{G}_{\bm{\phi}}$, which was fixed across previous research, except LSGM. The exact form of $\mathbf{G}_{\bm{\phi}}$ in LSGM, however, is not derivable, so we exclude comparing LSGM in this section. As stated in Section \ref{sec:motivation}, the noise distribution of a diffusion process is $\mathcal{N}(0,\mathbf{G}(\mathbf{x}_{t},t)\mathbf{G}^{T}(\mathbf{x}_{t},t))$, which is anisotropic by the input data, $\mathbf{x}_{t}$, and time, $t$. The influence of diffusion time on this covariance matrix is illustrated in Figure \ref{fig:eigenvalue}-(a). It presents the box plot of the eigenvalue distribution of the (normalized) covariance, $\mathbf{G}_{\bm{\phi}}(\mathbf{x}_{t},t)\mathbf{G}_{\bm{\phi}}^{T}(\mathbf{x}_{t},t)/g^{2}(t)$, on CIFAR-10, from $t=0$ to $t=T$. All the eigenvalues of previous research collapse to one as they share the isotropic covariance matrix, $g^{2}(t)\mathbf{I}$. On the other hand, the eigenvalues of INDM is dispersive throughout the diffusion time. As the distribution becomes more dispersive, the covariance matrix becomes more unisotropic, and Figure \ref{fig:eigenvalue}-(a) implies that the learnt diffusion process is under a highly nonlinear noise perturbation in a range of large diffusion time. The covariance matrix also depends on the input data, and Figure \ref{fig:eigenvalue}-(b) illustrates the eigenvalue distribution of the covariance at distinctive data instances of $\mathbf{x}_{t}$ at $t=0$. The eigenvalue distribution varies by instance, implying that data is diffused inhomogeneously by its location.
	
	\subsection{Relative Energy}\label{appendix:admissibility}
	
	\begin{wraptable}{r}{0.25\textwidth}
		\vskip 0.00in
		\caption{\textit{Relative} Energy.}
		\label{tab:relative_energy}
		\vskip -0.05in
		\tiny
		\centering
		\begin{tabular}{cc}
			\toprule
			Model & \textit{Relative} Energy \\\midrule
			DDPM++ & 1.60 \\
			\cc{15}INDM & \cc{15}1.23 \\
			\bottomrule
		\end{tabular}
		\vskip -0.2in
	\end{wraptable}
	Each flow network parameter constructs a different latent trajectory, so training the flow network has the effect of shifting the diffusion bridge. To check if learning the flow network is helpful for the transportation cost or not, recall the Benamou-Brenier formula \cite{benamou2000computational, villani2009optimal}, which is a dual formulation of the Wasserstein distance \cite{villani2009optimal, sriperumbudur2010hilbert} that the optimal transportation cost is the least kinetic energy out of all admissible transportation plans: $W_{2}^{2}(p,q)=\inf_{\{p_{t},\mathbf{v}_{t}\}_{t}}\bigg\{\mathcal{K}(\{p_{t},\mathbf{v}_{t}\}_{t});\underbrace{\frac{\partial p_{t}}{\partial t}+\text{div}(p_{t}\mathbf{v}_{t})=0}_{\text{continuity equation}}, \underbrace{p_{0}=p, p_{T}=q}_{\text{boundary conditions}}\bigg\}$, where $\mathcal{K}(\{p_{t},\mathbf{v}_{t}\}_{t}):=\int_{0}^{T}\int p_{t}(\mathbf{x})\Vert\mathbf{v}_{t}(\mathbf{x})\Vert_{2}^{2}\diff\mathbf{x}\diff t/T$ is the kinetic energy of the transportation. The continuity equation (that guarantees the conservation of mass along time transition \cite{evans1998partial}) and the boundary conditions determine the set of admissible trajectories, and the forward diffusion constructs an admissible trajectory. We \textit{quantify} how much a trajectory is close to the optimal transport as the \textit{relative} energy, given by $R(\bm{\phi})=\frac{\mathcal{K}(\{p_{t}^{\bm{\phi}},\mathbf{v}_{t}^{\bm{\phi}}\}_{t})}{W_{2}^{2}(p_{0}^{\bm{\phi}},p_{T}^{\bm{\phi}})}$. Table \ref{tab:relative_energy} shows that INDM's latent diffusion is more close to the optimal transport than DDPM++ on CIFAR-10.
	
	\subsection{Full Quantitative Tables}\label{appendix:quantitative_tables}
	
	Tables \ref{tab:performance_cifar10_full}, \ref{tab:performance_cifar10_appendix}, and \ref{tab:performance_celeba_appendix} gives the full details of the quantitative comparisons to baseline models.
	
	\subsection{Random samples}\label{appendix:samples}
	
	Figures \ref{fig:CIFAR10_VE_FID_samples_tau_1.05} and \ref{fig:CelebA_VP_FID_samples_tau_1.11} show the non cherry-picked random samples from INDM (VE, FID) on CIFAR-10 and INDM (VP, FID) on CelebA, respectively.
	
	\section{Proofs of Theorems and Propositions}\label{appendix:proofs}
	
	\begingroup
	\renewcommand\thetheorem{1}
	\begin{theorem}\label{thm_app:1}
		Suppose that $p_{\bm{\phi},\bm{\theta}}(\mathbf{x}_{0})$ is the likelihood of a generative random variable $\mathbf{x}_{0}^{\bm{\phi},\bm{\theta}}$. Then, the negative log-likelihood is bounded by
		\begin{align*}
		\mathbb{E}_{p_{r}(\mathbf{x}_{0})}\big[-\log{p_{\bm{\phi},\bm{\theta}}(\mathbf{x}_{0})}\big]\le\mathcal{L}\big(\{\mathbf{x}_{t}\}_{t=0}^{T},g^{2};\{\bm{\phi},\bm{\theta}\}\big),
		\end{align*}
		where
		\begin{align*}
		\begin{split}
		&\mathcal{L}\big(\{\mathbf{x}_{t}\}_{t=0}^{T},g^{2};\{\bm{\phi},\bm{\theta}\}\big)=-\mathbb{E}_{p_{r}(\mathbf{x}_{0})}\big[\log{\big\vert\det\big(\nabla_{\mathbf{x}_{0}}\mathbf{h}_{\bm{\phi}}\big)\big\vert}\big]\\
		&\quad\quad\quad\quad\quad+\mathcal{L}\big(\{\mathbf{z}_{t}\}_{t=0}^{T},g^{2};\bm{\theta}\big)-\mathbb{E}_{\mathbf{z}_{T}}\big[\log{\pi(\mathbf{z}_{T})}\big]+\frac{d}{2}\int_{0}^{T}\beta(t)-\frac{g^{2}(t)}{\sigma^{2}(t)}\diff t,
		\end{split}
		\end{align*}
		with $\mathcal{L}\big(\{\mathbf{z}_{t}\}_{t=0}^{T},g^{2};\bm{\theta}\big):=\frac{1}{2}\int_{0}^{T}g^{2}(t)\mathbb{E}_{\mathbf{z}_{0},\mathbf{z}_{t}}\big[\Vert\mathbf{s}_{\bm{\theta}}(\mathbf{z}_{t},t)-\nabla_{\mathbf{z}_{t}}\log{p_{0t}(\mathbf{z}_{t}\vert\mathbf{z}_{0})}\Vert_{2}^{2}\big]\diff t$. Here, $p_{0t}(\mathbf{z}_{t}\vert\mathbf{z}_{0})$ is the transition probability of the forward linear diffusion process on latent space.
	\end{theorem}
	\endgroup
	
	\begin{proof}[Proof of Theorem \ref{thm_app:1}]
		From the change of variable, the transformation of $\mathbf{z}_{0}=\mathbf{h}_{\bm{\phi}}(\mathbf{x}_{0})$ induces
		\begin{align*}
		p_{r}(\mathbf{x}_{0})=\frac{p_{0}(\mathbf{z}_{0})}{\Big\vert\det\Big(\frac{\partial\mathbf{h}_{\bm{\phi}}}{\partial\mathbf{x}_{0}}\Big)\Big\vert^{-1}},
		\end{align*}
		and thus the entropy of the data distribution becomes
		\begin{eqnarray*}
			\lefteqn{\mathcal{H}(p_{r})=-\int p_{r}(\mathbf{x}_{0})\log{p_{r}(\mathbf{x}_{0})}\diff\mathbf{x}_{0}}\\
			&&=-\int p_{0}(\mathbf{z}_{0})\log{\frac{p_{0}(\mathbf{z}_{0})}{\Big\vert\det\Big(\frac{\partial\mathbf{h}_{\bm{\phi}}}{\partial\mathbf{x}_{0}}\Big)\Big\vert^{-1}}}\diff\mathbf{z}_{0}\\
			&&=-\int p_{r}(\mathbf{x}_{0})\log{\Big\vert\det\Big(\frac{\partial\mathbf{h}_{\bm{\phi}}}{\partial\mathbf{x}_{0}}\Big)\Big\vert}\diff\mathbf{x}_{0}-\int p_{0}(\mathbf{z}_{0})\log{p_{0}(\mathbf{z}_{0})}\diff\mathbf{z}_{0}\\
			&&=-\mathbb{E}_{p_{r}(\mathbf{x}_{0})}\Big[\log{\Big\vert\det\Big(\frac{\partial\mathbf{h}_{\bm{\phi}}}{\partial\mathbf{x}_{0}}\Big)\Big\vert}\Big]-\int p_{0}(\mathbf{z}_{0})\log{p_{0}(\mathbf{z}_{0})}\diff\mathbf{z}_{0}\\
			&&=-\mathbb{E}_{p_{r}(\mathbf{x}_{0})}\Big[\log{\Big\vert\det\Big(\frac{\partial\mathbf{h}_{\bm{\phi}}}{\partial\mathbf{x}_{0}}\Big)\Big\vert}\Big]+\mathcal{H}(p_{0}).
		\end{eqnarray*}
		
		From Theorem 4 of \citet{song2021maximum}, the entropy at $t=0$ equals to
		\begin{align*}
		\mathcal{H}(p_{0})=\mathcal{H}(p_{T})-\frac{1}{2}\int_{0}^{T}\mathbb{E}_{p_{t}(\mathbf{z}_{t})}\big[2\nabla_{\mathbf{z}_{t}}\cdot\mathbf{f}(\mathbf{z}_{t},t)+g^{2}(t)\Vert\nabla_{\mathbf{z}_{t}}\log{p_{t}(\mathbf{z}_{t})}\Vert_{2}^{2}\big]\diff t,
		\end{align*}
		where $\mathbf{f}(\mathbf{z}_{t},t)$ is a drift term of the diffusion for $\mathbf{z}_{t}$ and $p_{t}$ is the probability distribution of $\mathbf{z}_{t}$. Therefore, the negative log-likelihood becomes
		\begin{eqnarray*}
			\lefteqn{-\mathbb{E}_{p_{r}(\mathbf{x}_{0})}\big[\log{p_{\bm{\phi},\bm{\theta}}(\mathbf{x}_{0})}\big]=D_{KL}(p_{r}\Vert p_{\bm{\phi},\bm{\theta}})+\mathcal{H}(p_{r})}\\
			&&\le D_{KL}(\bm{\mu}_{\bm{\phi}}(\{\mathbf{x}_{t}\})\Vert\bm{\nu}_{\bm{\phi},\bm{\theta}}(\{\mathbf{x}_{t}\}))+\mathcal{H}(p_{r})\\
			&&= D_{KL}(\bm{\mu}_{\bm{\phi}}(\{\mathbf{x}_{t}\})\Vert\bm{\nu}_{\bm{\phi},\bm{\theta}}(\{\mathbf{x}_{t}\}))-\mathbb{E}_{p_{r}(\mathbf{x}_{0})}\Big[\log{\Big\vert\det\Big(\frac{\partial\mathbf{h}_{\bm{\phi}}}{\partial\mathbf{x}_{0}}\Big)\Big\vert}\Big]+\mathcal{H}(p_{0})\\
			&&= D_{KL}(\bm{\mu}_{\bm{\phi}}(\{\mathbf{z}_{t}\})\Vert\bm{\nu}_{\bm{\theta}}(\{\mathbf{z}_{t}\}))-\mathbb{E}_{p_{r}(\mathbf{x}_{0})}\Big[\log{\Big\vert\det\Big(\frac{\partial\mathbf{h}_{\bm{\phi}}}{\partial\mathbf{x}_{0}}\Big)\Big\vert}\Big]+\mathcal{H}(p_{T})\\
			&&\quad-\frac{1}{2}\int_{0}^{T}\mathbb{E}_{\mathbf{z}_{t}^{\bm{\phi}}}\big[-d\beta(t)+g^{2}(t)\Vert\nabla_{\mathbf{z}_{t}}\log{p_{t}(\mathbf{z}_{t})}\Vert_{2}^{2}\big]\diff t.
		\end{eqnarray*}
		Now, from Theorem 1 of \cite{song2021maximum}, the KL-divergence between the path measures becomes
		\begin{align}\label{eq:nelbo}
		D_{KL}(\bm{\mu}_{\bm{\phi}}(\{\mathbf{z}_{t}\})\Vert\bm{\nu}_{\bm{\theta}}(\{\mathbf{z}_{t}\}))&=D_{KL}(p_{T}(\mathbf{z}_{T})\Vert\pi(\mathbf{z}_{T}))\\
		&\quad+\frac{1}{2}\int_{0}^{T}g^{2}(t)\mathbb{E}_{p_{t}(\mathbf{z}_{t})}\big[\Vert\mathbf{s}_{\bm{\theta}}(\mathbf{z}_{t},t)-\nabla_{\mathbf{z}_{t}}\log{p_{t}(\mathbf{z}_{t})}\Vert_{2}^{2}\big]\diff t,
		\end{align}
		so if we plug in this into the negative log-likelihood, we yield the following:
		\begin{eqnarray*}
			\lefteqn{-\mathbb{E}_{p_{r}(\mathbf{x}_{0})}\big[\log{p_{\bm{\phi},\bm{\theta}}(\mathbf{x}_{0})}\big]}\\
			&&\le -\mathbb{E}_{p_{r}(\mathbf{x}_{0})}\Big[\log{\Big\vert\det\Big(\frac{\partial\mathbf{h}_{\bm{\phi}}}{\partial\mathbf{x}_{0}}\Big)\Big\vert}\Big] + \frac{1}{2}\int_{0}^{T}g^{2}(t)\mathbb{E}_{\mathbf{z}_{t}}\big[\Vert\mathbf{s}_{\bm{\theta}}(\mathbf{z}_{t},t)-\nabla_{\mathbf{z}_{t}}\log{p_{t}(\mathbf{z}_{t})}\Vert_{2}^{2}\big]\diff t\\
			&&\quad +D_{KL}(p_{T}\Vert\pi) +\mathcal{H}(p_{T})-\frac{1}{2}\int_{0}^{T}\mathbb{E}_{\mathbf{z}_{t}}\big[-d\beta(t)+g^{2}(t)\Vert\nabla_{\mathbf{z}_{t}}\log{p_{t}(\mathbf{z}_{t})}\Vert_{2}^{2}\big]\diff t\\
			&&= -\mathbb{E}_{p_{r}(\mathbf{x}_{0})}\Big[\log{\Big\vert\det\Big(\frac{\partial\mathbf{h}_{\bm{\phi}}}{\partial\mathbf{x}_{0}}\Big)\Big\vert}\Big] -\mathbb{E}_{\mathbf{z}_{T}}\big[\log{\pi(\mathbf{z}_{T})}\big] +\frac{d}{2}\int_{0}^{T}\beta(t)\diff t\\
			&&\quad+\frac{1}{2}\int_{0}^{T}g^{2}(t)\mathbb{E}_{\mathbf{z}_{t}}\big[\Vert\mathbf{s}_{\bm{\theta}}(\mathbf{z}_{t},t)-\nabla_{\mathbf{z}_{t}}\log{p_{t}(\mathbf{z}_{t})}\Vert_{2}^{2}-\Vert\nabla_{\mathbf{z}_{t}}\log{p_{t}(\mathbf{z}_{t})}\Vert_{2}^{2} \big]\diff t
		\end{eqnarray*}
		
		Also, we have
		\begin{eqnarray*}
			\lefteqn{\mathbb{E}_{\mathbf{z}_{t}}\big[\mathbf{s}_{\bm{\theta}}(\mathbf{z}_{t},t)\cdot\nabla_{\mathbf{z}_{t}}\log{p_{t}(\mathbf{z}_{t})}\big]=\int p_{t}(\mathbf{z}_{t})\mathbf{s}_{\bm{\theta}}(\mathbf{z}_{t},t)\cdot\nabla_{\mathbf{z}_{t}}\log{p_{t}(\mathbf{z}_{t})}\diff\mathbf{z}_{t}}\\
			&&=\int \mathbf{s}_{\bm{\theta}}(\mathbf{z}_{t},t)\cdot\nabla_{\mathbf{z}_{t}}p_{t}(\mathbf{z}_{t})\diff\mathbf{z}_{t}\\
			&&=\int \mathbf{s}_{\bm{\theta}}(\mathbf{z}_{t},t)\cdot \int p_{0}(\mathbf{z}_{0})\nabla_{\mathbf{z}_{t}}p_{0t}(\mathbf{z}_{t}\vert\mathbf{z}_{0})\diff\mathbf{z}_{0}\diff\mathbf{z}_{t}\\
			&&=\mathbb{E}_{\mathbf{z}_{0}}\mathbb{E}_{\mathbf{z}_{t}\vert\mathbf{z}_{0}}\big[\mathbf{s}_{\bm{\theta}}(\mathbf{z}_{t},t)\cdot\nabla_{\mathbf{z}_{t}}\log{p_{0t}(\mathbf{z}_{t}\vert\mathbf{z}_{0})}\big]
		\end{eqnarray*}
		
		Therefore,
		\begin{eqnarray*}
			\lefteqn{\frac{1}{2}\int_{0}^{T}g^{2}(t)\mathbb{E}_{\mathbf{z}_{t}}\big[\Vert\mathbf{s}_{\bm{\theta}}(\mathbf{z}_{t},t)-\nabla_{\mathbf{z}_{t}}\log{p_{t}(\mathbf{z}_{t})}\Vert_{2}^{2}-\Vert\nabla_{\mathbf{z}_{t}}\log{p_{t}(\mathbf{z}_{t})}\Vert_{2}^{2}\big]}\\
			&&=\int_{0}^{T}g^{2}(t)\mathbb{E}_{\mathbf{z}_{t}}\Big[\frac{1}{2}\Vert\mathbf{s}_{\bm{\theta}}(\mathbf{z}_{t},t)\Vert_{2}^{2}-\mathbf{s}_{\bm{\theta}}(\mathbf{z}_{t},t)\cdot\nabla_{\mathbf{z}_{t}}\log{p_{t}(\mathbf{z}_{t})}\Big]\\
			&&=\int_{0}^{T}g^{2}(t)\mathbb{E}_{\mathbf{z}_{0}}\mathbb{E}_{\mathbf{z}_{t}\vert\mathbf{z}_{0}}\Big[\frac{1}{2}\Vert\mathbf{s}_{\bm{\theta}}(\mathbf{z}_{t},t)\Vert_{2}^{2}-\mathbf{s}_{\bm{\theta}}(\mathbf{z}_{t},t)\cdot\nabla_{\mathbf{z}_{t}}\log{p_{0t}(\mathbf{z}_{t}\vert\mathbf{z}_{0})}\Big]\\
			&&=\frac{1}{2}\int_{0}^{T}g^{2}(t)\mathbb{E}_{\mathbf{z}_{0}}\mathbb{E}_{\mathbf{z}_{t}\vert\mathbf{z}_{0}}\Big[\Vert\mathbf{s}_{\bm{\theta}}(\mathbf{z}_{t},t)-\nabla_{\mathbf{z}_{t}}\log{p_{0t}(\mathbf{z}_{t}\vert\mathbf{z}_{0})}\Vert_{2}^{2}-\Vert\nabla_{\mathbf{z}_{t}}\log{p_{0t}(\mathbf{z}_{t}\vert\mathbf{z}_{0})}\Vert_{2}^{2}\Big].
		\end{eqnarray*}
		Now, since $p_{0t}(\mathbf{z}_{t}\vert\mathbf{z}_{0})=\mathcal{N}(\mathbf{z}_{t};\mu(t)\mathbf{z}_{t},\sigma^{2}(t)\mathbf{I})$ for $\mu(t)$ and $\sigma^{2}(t)$ determined by $\beta(t)$ and $g(t)$, we have
		\begin{align*}
		\mathbb{E}_{\mathbf{z}_{t}\vert\mathbf{z}_{0}}\big[\Vert\nabla_{\mathbf{z}_{t}}\log{p_{0t}(\mathbf{z}_{t}\vert\mathbf{z}_{0})}\Vert_{2}^{2}\big]=\mathbb{E}_{\mathbf{z}_{t}\vert\mathbf{z}_{0}}\Big[\Big\Vert\frac{\mathbf{z}_{t}-\mu(t)\mathbf{z}_{0}}{\sigma^{2}(t)}\Big\Vert_{2}^{2}\Big]=\mathbb{E}_{\mathcal{N}(\mathbf{z};0,\mathbf{I})}\Big[\frac{\Vert\mathbf{z}\Vert_{2}^{2}}{\sigma^{2}(t)}\Big]=\frac{d}{\sigma^{2}(t)},
		\end{align*}
		and we have the desired result.
	\end{proof}
	
	\begingroup
	\renewcommand\theproposition{1}
	\begin{proposition}
		Suppose $q_{t}^{\bm{\theta}}$ is the marginal distribution of $\nu_{\bm{\theta}}$ at $t$. The variational gap is
		\begin{align*}
		\textup{Gap}\big(\bm{\mu}_{\bm{\phi}}(\{\mathbf{x}_{t}\}),\bm{\nu}_{\bm{\phi},\bm{\theta}}(\{\mathbf{x}_{t}\})\big):=&D_{KL}\big(\bm{\mu}_{\bm{\phi}}(\{\mathbf{x}_{t}\})\Vert\bm{\nu}_{\bm{\phi},\bm{\theta}}(\{\mathbf{x}_{t}\})\big)-D_{KL}\big(p_{0}^{\bm{\phi}}(\mathbf{x}_{0})\Vert q_{0}^{\bm{\theta}}(\mathbf{x}_{0})\big)\\
		=&\frac{1}{2}\int_{0}^{T}g^{2}(t)\mathbb{E}_{p_{t}^{\bm{\phi}}(\mathbf{z}_{t})}\big[\underbrace{\Vert\nabla\log{q_{t}^{\bm{\theta}}(\mathbf{z}_{t})}-\mathbf{s}_{\bm{\theta}}(\mathbf{z}_{t},t)\Vert_{2}^{2}}_{\text{Score-only error}}\big]\diff t.
		\end{align*}
	\end{proposition}
	\endgroup
	\begin{proof}[\textbf{Proof of Proposition \ref{thm:2}}]
		Suppose $q_{t}^{\bm{\theta}}$ is a marginal distribution of the path measure of the generative SDE given by
		\begin{align}\label{eq:generative_latent_sde}
		\diff\mathbf{z}_{t}=\left[ \mathbf{f}(\mathbf{z}_{t},t)-g^{2}(t)\mathbf{s}_{\bm{\theta}}(\mathbf{z}_{t},t) \right]\diff \bar{t}+g(t)\diff\mathbf{\bar{w}}_{t}.
		\end{align}
		The Fokker-Planck equation of the above generative SDE satisfies
		\begin{align*}
		\frac{\partial q_{t}^{\bm{\theta}}}{\partial t}(\mathbf{z}_{t})=&-\sum_{i=1}^{d}\frac{\partial}{\partial z_{i}}\left(\left[ f_{i}(\mathbf{z}_{t},t)-g^{2}(t)\big(\mathbf{s}_{\bm{\theta}}(\mathbf{z}_{t},t)\big)_{i} \right]q_{t}^{\bm{\theta}}(\mathbf{z}_{t})\right)-\frac{g^{2}(t)}{2}\sum_{i=1}^{d}\frac{\partial^{2}}{\partial z_{i}^{2}}\big[q_{t}^{\bm{\theta}}(\mathbf{z}_{t})\big]\\
		=&\text{div}\left(\Big(-\mathbf{f}(\mathbf{z}_{t},t)+g^{2}(t)\mathbf{s}_{\bm{\theta}}(\mathbf{z}_{t},t)-\frac{g^{2}(t)}{2}\nabla \log{q_{t}^{\bm{\theta}}(\mathbf{z}_{t})}\Big)q_{t}^{\bm{\theta}}(\mathbf{z}_{t})\right)\numberthis\label{eq:generative_ode}.
		\end{align*}
		
		On the other hand, if $p_{t}^{\bm{\phi}}$ is the marginal distribution of the path measure of the forward SDE given by
		\begin{align*}
		\diff\mathbf{z}_{t}=\mathbf{f}(\mathbf{z}_{t},t)\diff t+g(t)\diff\mathbf{w}_{t},
		\end{align*}
		then the corresponding Fokker-Planck equation becomes
		\begin{align}\label{eq:reverse_ode}
		\frac{\partial p_{t}^{\bm{\phi}}}{\partial t}(\mathbf{z}_{t})=\text{div}\left( \Big(-\mathbf{f}(\mathbf{z}_{t},t)+\frac{g^{2}(t)}{2}\nabla \log{p_{t}^{\bm{\phi}}(\mathbf{z}_{t},t)}\Big)p_{t}^{\bm{\phi}}(\mathbf{z}_{t}) \right).
		\end{align}
		
		Combining Eq. \eqref{eq:generative_ode} with Eq. \eqref{eq:reverse_ode} and using the integration by parts, the derivative of the KL divergence becomes
		\begin{eqnarray*}
			\lefteqn{\frac{\partial D_{KL}(p_{t}^{\bm{\phi}}\Vert q_{t}^{\bm{\theta}})}{\partial t}=\frac{\partial}{\partial t}\int p_{t}^{\bm{\phi}}(\mathbf{z}_{t})\log{\frac{p_{t}^{\bm{\phi}}(\mathbf{z}_{t})}{q_{t}^{\bm{\theta}}(\mathbf{z}_{t})}}\diff \mathbf{z}_{t}}&\\
			&&=\int \frac{\partial p_{t}^{\bm{\phi}}}{\partial t}(\mathbf{z}_{t})\log{\frac{p_{t}^{\bm{\phi}}(\mathbf{z}_{t})}{q_{t}^{\bm{\theta}}(\mathbf{z}_{t})}}\diff \mathbf{z}_{t}-\int \frac{\partial q_{t}^{\bm{\theta}}}{\partial t}(\mathbf{z}_{t})\frac{p_{t}^{\bm{\phi}}(\mathbf{z}_{t})}{q_{t}^{\bm{\theta}}(\mathbf{z}_{t})}\diff\mathbf{z}_{t}\\
			&&=-\int p_{t}^{\bm{\phi}}(\mathbf{z}_{t})\left(-\mathbf{f}(\mathbf{z}_{t},t)+\frac{g^{2}(t)}{2}\nabla \log{p_{t}^{\bm{\phi}}(\mathbf{z}_{t})}\right)^{T}\nabla\log{\frac{p_{t}^{\bm{\phi}}(\mathbf{z}_{t})}{q_{t}^{\bm{\theta}}(\mathbf{z}_{t})}}\diff\mathbf{z}_{t}\\
			&&\quad+\int p_{t}^{\bm{\phi}}(\mathbf{z}_{t})\left(-\mathbf{f}(\mathbf{z}_{t},t)+g^{2}(t)\mathbf{s}_{\bm{\theta}}(\mathbf{z}_{t},t)-\frac{g^{2}(t)}{2}\nabla \log{q_{t}^{\bm{\theta}}(\mathbf{z}_{t})}\right)^{T}\nabla\log{\frac{p_{t}^{\bm{\phi}}(\mathbf{z}_{t})}{q_{t}^{\bm{\theta}}(\mathbf{z}_{t})}}\diff\mathbf{z}_{t}\\
			&&=\frac{g^{2}(t)}{2}\int p_{t}^{\bm{\phi}}(\mathbf{z}_{t})\left(\nabla\log{\frac{p_{t}^{\bm{\phi}}(\mathbf{z}_{t})}{q_{t}^{\bm{\theta}}(\mathbf{z}_{t})}}\right)^{T}\left(2\mathbf{s}_{\bm{\theta}}(\mathbf{z}_{t},t)-\nabla\log{p_{t}^{\bm{\phi}}(\mathbf{z}_{t})}-\nabla\log{q_{t}^{\bm{\theta}}(\mathbf{z}_{t})}\right)\diff\mathbf{z}_{t}.
		\end{eqnarray*}
		Integrating the above derivative, we get the KL divergence of
		\begin{align}
		&D_{KL}\big(p_{0}^{\bm{\phi}}(\mathbf{z}_{0})\Vert q_{0}^{\bm{\theta}}(\mathbf{z}_{0})\big)=-\int_{0}^{T} \frac{\partial D_{KL}(p_{t}^{\bm{\phi}}\Vert q_{t}^{\bm{\theta}})}{\partial t}\diff t+D_{KL}(p_{T}^{\bm{\phi}}\Vert q_{T}^{\bm{\theta}})\label{eq:indm_kl}\\
		&=\int_{0}^{T}\frac{g^{2}(t)}{2}\mathbb{E}_{\mathbf{z}_{t}^{\bm{\phi}}}\left[ (\nabla\log{p_{t}^{\bm{\phi}}}-\nabla\log{q_{t}^{\bm{\theta}}})^{T}(\nabla\log{p_{t}^{\bm{\phi}}}+\nabla\log{q_{t}^{\bm{\theta}}}-2\mathbf{s}_{\bm{\theta}}) \right]\diff t+D_{KL}(p_{T}^{\bm{\phi}}\Vert q_{T}^{\bm{\theta}}).\nonumber
		\end{align}
		
		Also, from Eq. \eqref{eq:nelbo}, we have 
		\begin{eqnarray}
		\begin{split}\label{eq:indm_nelbo}
		\lefteqn{D_{KL}\big(\bm{\mu}_{\bm{\phi}}(\{\mathbf{x}_{t}\})\Vert\bm{\nu}_{\bm{\phi},\bm{\theta}}(\{\mathbf{x}_{t}\})\big)=D_{KL}\big(\bm{\mu}_{\bm{\phi}}(\{\mathbf{z}_{t}\})\Vert\bm{\nu}_{\bm{\theta}}(\{\mathbf{z}_{t}\})\big)}&\\
		&&=\int_{0}^{T}\frac{g^{2}(t)}{2}\mathbb{E}_{p_{t}^{\bm{\phi}}(\mathbf{z}_{t})}\big[\Vert\nabla\log{p_{t}^{\bm{\phi}}(\mathbf{z}_{t})}-\mathbf{s}_{\bm{\theta}}(\mathbf{z}_{t},t)\Vert_{2}^{2}\big]\diff t+D_{KL}(p_{T}^{\bm{\phi}}\Vert q_{T}^{\bm{\theta}}).
		\end{split}
		\end{eqnarray}
		
		By subtracting Eq. \eqref{eq:indm_kl} from Eq. \eqref{eq:indm_nelbo}, we get the desired result:
		\begin{align*}
		&\text{Gap}(\bm{\mu}_{\bm{\phi}},\bm{\nu}_{\bm{\phi},\bm{\theta}})=D_{KL}\big(\bm{\mu}_{\bm{\phi}}(\{\mathbf{x}_{t}\})\Vert\bm{\nu}_{\bm{\phi},\bm{\theta}}(\{\mathbf{x}_{t}\})\big)-D_{KL}\big(p_{r}(\mathbf{x}_{0})\Vert p_{\bm{\phi},\bm{\theta}}(\mathbf{x}_{0})\big)\\
		&=D_{KL}\big(\bm{\mu}_{\bm{\phi}}(\{\mathbf{z}_{t}\})\Vert\bm{\nu}_{\bm{\theta}}(\{\mathbf{z}_{t}\})\big)-D_{KL}\big(p_{0}^{\bm{\phi}}(\mathbf{z}_{0})\Vert q_{0}^{\bm{\theta}}(\mathbf{z}_{0})\big)\\	
		&=\int\frac{g^{2}(t)}{2}\mathbb{E}_{\mathbf{z}_{t}^{\bm{\phi}}}\bigg[\Vert\nabla\log{p_{t}^{\bm{\phi}}}-\mathbf{s}_{\bm{\theta}}\Vert_{2}^{2}-(\nabla\log{p_{t}^{\bm{\phi}}}-\nabla\log{q_{t}^{\bm{\theta}}})^{T}(\nabla\log{p_{t}^{\bm{\phi}}}+\nabla\log{q_{t}^{\bm{\theta}}}-2\mathbf{s}_{\bm{\theta}})\bigg]\diff t\\
		&=\int\frac{g^{2}(t)}{2}\mathbb{E}_{p_{t}^{\bm{\phi}}(\mathbf{z}_{t})}\left[ \Vert\nabla\log{q_{t}^{\bm{\theta}}}-\mathbf{s}_{\bm{\theta}}\Vert_{2}^{2} \right]\diff t.
		\end{align*}
	\end{proof}
	
	\begingroup
	\renewcommand\thetheorem{2}
	\begin{theorem}\label{cor_app:2}
		$\textup{Gap}(\bm{\mu}_{\bm{\phi}},\bm{\nu}_{\bm{\phi},\bm{\theta}})=0$ if and only if $\mathbf{s}_{\bm{\theta}}\in\mathbf{S}_{sol}$. 
	\end{theorem}
	\endgroup
	\begin{proof}[\textbf{Proof of Theorem \ref{cor_app:2}}]
		($\Rightarrow$) Suppose the variational gap is zero. Then, as the support of $p_{t}^{\bm{\phi}}$ is the whole space of $\mathbb{R}^{d}$, Theorem \ref{thm:2} implies that $\mathbf{s}_{\bm{\theta}}(\mathbf{z}_{t},t)=\nabla\log{q_{t}^{\bm{\theta}}(\mathbf{z}_{t})}$ almost everywhere, for any $t>0$. To check if $\mathbf{s}_{\bm{\theta}}(\mathbf{z}_{0},0)=\nabla\log{q_{0}^{\bm{\theta}}(\mathbf{z}_{0})}$ at $t=0$, suppose $\mathbf{s}_{\bm{\theta}}(\mathbf{z}_{0},0)\neq\nabla\log{q_{0}^{\bm{\theta}}(\mathbf{z}_{0})}$ on a set of positive measure. Then, from the continuity of $\mathbf{s}_{\bm{\theta}}$ and $\nabla\log{q_{t}^{\bm{\theta}}}$, we have $\mathbf{s}_{\bm{\theta}}(\mathbf{z}_{s},s)\neq\nabla\log{q_{s}^{\bm{\theta}}(\mathbf{z}_{s})}$ on $s<t_{0}$ for some $t_{0}$. Therefore, for any $t\in[0,T]$, we conclude that $\mathbf{s}_{\bm{\theta}}(\mathbf{z}_{t},t)=\nabla\log{q_{t}^{\bm{\theta}}(\mathbf{z}_{t})}$ almost everywhere and Eq. \eqref{eq:generative_latent_sde} becomes
		\begin{align}\label{eq:cor_gen_sde}
		\diff\mathbf{z}_{t}=\left[\mathbf{f}(\mathbf{z}_{t},t)-g^{2}(t)\nabla\log{q_{t}^{\bm{\theta}}(\mathbf{z}_{t})}\right]\diff\bar{t}+g(t)\diff\mathbf{\bar{w}}_{t}.
		\end{align}
		As the Fokker-Planck equation of the SDE of Eq. \eqref{eq:cor_gen_sde} becomes
		\begin{align*}
		\frac{\partial q_{t}^{\bm{\theta}}}{\partial t}(\mathbf{z}_{t})=\text{div}\left(\Big(-\mathbf{f}(\mathbf{z}_{t},t)+\frac{g^{2}(t)}{2}\nabla \log{q_{t}^{\bm{\theta}}(\mathbf{z}_{t})}\Big)q_{t}^{\bm{\theta}}(\mathbf{z}_{t})\right),
		\end{align*}
		which coincide with the Fokker-Planck equation of the forward SDE of $\diff\mathbf{z}_{t}=\mathbf{f}(\mathbf{z}_{t},t)\diff t+g(t)\diff\mathbf{w}_{t}$, we conclude $\mathbf{s}_{\bm{\theta}}\in\mathbf{S}_{sol}$ by definition.
		
		($\Leftarrow$) holds from Lemma \ref{lemma:2}.
	\end{proof}
	
	\begingroup
	\renewcommand\thetheorem{3}
	\begin{theorem}
		For any fixed $\mathbf{s}_{\bm{\bar{\theta}}}\in\mathbf{S}_{sol}$, if $\bm{\phi}^{*}\in\argmin_{\bm{\phi}}{D_{KL}(\bm{\mu}_{\bm{\phi}}\Vert \bm{\nu}_{\bm{\phi},\bm{\bar{\theta}}})}$, then $\mathbf{s}_{\bm{\phi}^{*}}(\mathbf{z}_{t},t)=\nabla\log{p_{t}^{\bm{\phi}^{*}}(\mathbf{z}_{t})}=\mathbf{s}_{\bm{\bar{\theta}}}(\mathbf{z}_{t},t)$, and $D_{KL}(\bm{\mu}_{\bm{\phi}^{*}}\Vert \bm{\nu}_{\bm{\phi}^{*},\bm{\bar{\theta}}})=D_{KL}(p_{r}\Vert p_{\bm{\phi}^{*},\bm{\bar{\theta}}})=\text{Gap}(\bm{\mu}_{\bm{\phi}^{*}},\bm{\nu}_{\bm{\phi}^{*},\bm{\bar{\theta}}})=0$.
	\end{theorem}
	\endgroup
	\begin{proof}[\textbf{Proof of Theorem \ref{thm:3}}]
		If $\mathbf{s}_{\bm{\bar{\theta}}}\in\mathbf{S}_{sol}$, there exists $q_{0}$ such that $\mathbf{s}_{\bm{\bar{\theta}}}(\mathbf{z}_{t},t)=\nabla\log{q_{t}(\mathbf{z}_{t})}$, where $\mathbf{z}_{t}\sim q_{t}$ is governed by $\diff\mathbf{z}_{t}=\mathbf{f}(\mathbf{z}_{t},t)\diff t+g(t)\diff\mathbf{w}_{t}$ that starts from $\mathbf{z}_{0}\sim q_{0}$. This implies that the generative path measure of $\bm{\nu}_{\bm{\phi},\bm{\bar{\theta}}}$ coincides with some forward path measure. On the other hand, the forward latent diffusion is also governed by $\diff\mathbf{z}_{t}=\mathbf{f}(\mathbf{z}_{t},t)\diff t+g(t)\diff\mathbf{w}_{t}$ that starts from $\mathbf{z}_{0}\sim p_{0}^{\bm{\phi}}$. Therefore, if $p_{0}^{\bm{\phi}}=q_{0}$ almost everywhere, then the generative path measure of $\bm{\nu}_{\bm{\phi},\bm{\bar{\theta}}}$ coincides with the forward path measure of $\bm{\mu}_{\bm{\phi}}$, and it holds that $D_{KL}(\bm{\mu}_{\bm{\phi}}\Vert \bm{\nu}_{\bm{\phi},\bm{\bar{\theta}}})=\int_{0}^{T}\frac{g^{2}(t)}{2}\mathbb{E}[\Vert\nabla\log{p_{t}^{\bm{\phi}}}-\nabla\log{q_{t}}\Vert_{2}^{2}]\diff t+D_{KL}(p_{T}^{\bm{\phi}}\Vert q_{T})=0$. If $p_{0}^{\bm{\phi}}\neq q_{0}$ on a set of positive measure $A$, then $D_{KL}(\bm{\mu}_{\bm{\phi}}\Vert\bm{\nu}_{\bm{\phi},\bm{\bar{\theta}}})=\int_{0}^{T}\frac{g^{2}(t)}{2}\mathbb{E}[\Vert\nabla\log{p_{t}^{\bm{\phi}}}-\nabla\log{q_{t}}\Vert_{2}^{2}]\diff t+D_{KL}(p_{T}^{\bm{\phi}}\Vert q_{T})$ is strictly positive because $\Vert\nabla\log{p_{t}^{\bm{\phi}}}-\nabla\log{q_{t}}\Vert_{2}^{2}>0$ on $A$, for any $t$. This leads that if $\bm{\phi}^{*}\in\argmin_{\bm{\phi}}{D_{KL}(\bm{\mu}_{\bm{\phi}}\Vert \bm{\nu}_{\bm{\phi},\bm{\bar{\theta}}})}$, then $D_{KL}(\bm{\mu}_{\bm{\phi}^{*}}\Vert \bm{\nu}_{\bm{\phi}^{*},\bm{\bar{\theta}}})= 0$, and $p_{0}^{\bm{\phi}^{*}}=q_{0}$ almost everywhere. Therefore, we get the desired result because $0=D_{KL}(\bm{\mu}_{\bm{\phi}^{*}}\Vert\bm{\nu}_{\bm{\phi}^{*},\bm{\bar{\theta}}})\ge D_{KL}(p_{r}\Vert p_{\bm{\phi}^{*},\bm{\bar{\theta}}})\ge 0$.
	\end{proof}
	
	\begingroup
	\renewcommand\theproposition{2}
	\begin{proposition}
		$\mathbf{s}_{\bm{\theta}}\in\mathbf{S}_{div}$ if and only if $\nabla_{\mathbf{z}_{t}}\mathbf{s}_{\bm{\theta}}(\mathbf{z}_{t},t)$ is symmetric.
	\end{proposition}
	\endgroup
	\begin{proof}[\textbf{Proof of Proposition \ref{prop:1}}]
		If $\nabla_{\mathbf{z}_{t}}\mathbf{s}_{\bm{\theta}}(\mathbf{z}_{t},t)$ is symmetric, then $\mathbf{s}_{\bm{\theta}}(\mathbf{z}_{t},t)$ is a 1-form, and $\mathbf{s}_{\bm{\theta}}\in\mathbf{S}_{div}$. If $\mathbf{s}_{\bm{\theta}}\in\mathbf{S}_{div}$, then there exists $p_{t}$ such that $\mathbf{s}_{\bm{\theta}}(\mathbf{z}_{t},t)=\nabla\log{p_{t}(\mathbf{z}_{t})}$. Thus, $\nabla\mathbf{s}_{\bm{\theta}}=\nabla^{2}\log{p_{t}}$, which is symmetric.
	\end{proof}
	
	\begingroup
	\renewcommand\theproposition{3}
	\begin{proposition}
		A matrix $A\in\mathbb{R}^{d\times d}$ is symmetric if and only if $\mathbb{E}_{\bm{\epsilon}_{1},\bm{\epsilon}_{2}\sim\mathcal{N}(0,\mathbf{I})}\left[ (\bm{\epsilon}_{2}^{T}(A-A^{T})\bm{\epsilon}_{1})^{2}\right]=0$.
	\end{proposition}
	\endgroup
	\begin{proof}[\textbf{Proof of Proposition \ref{prop:2}}]
		As
		\begin{align*}
		\mathbb{E}_{\bm{\epsilon}_{1},\bm{\epsilon}_{2}\sim\mathcal{N}(0,\mathbf{I})}\big[(\bm{\epsilon}_{2}^{T}A\bm{\epsilon}_{1}-\bm{\epsilon}_{1}^{T}A\bm{\epsilon}_{2})^{2}\big]=&\mathbb{E}_{\bm{\epsilon}_{1},\bm{\epsilon}_{2}\sim\mathcal{N}(0,\mathbf{I})}\big[(\bm{\epsilon}_{2}^{T}(A-A^{T})\bm{\epsilon}_{1})^{2}\big],
		\end{align*}
		$A$ is symmetric if and only if $\mathbb{E}_{\bm{\epsilon}_{1},\bm{\epsilon}_{2}\sim\mathcal{N}(0,\mathbf{I})}\big[(\bm{\epsilon}_{2}^{T}A\bm{\epsilon}_{1}-\bm{\epsilon}_{1}^{T}A\bm{\epsilon}_{2})^{2}\big]=0$.
	\end{proof}
	
	\begingroup
	\renewcommand\theproposition{4}
	\begin{proposition}
		Let $\bm{\epsilon}_{1}$ and $\bm{\epsilon}_{2}$ be vectors of $d$ independent samples from a random variable $U$ with mean zero. Then
		\begin{align*}
		\mathbb{E}_{\bm{\epsilon}_{1},\bm{\epsilon}_{2}}[(\bm{\epsilon}_{2}^{T}(A-A^{T})\bm{\epsilon}_{1})^{2}]=\mathbb{E}_{U}[U^{2}]^{2}\Vert A-A^{T}\Vert_{F}^{2}
		\end{align*}
		and
		\begin{align*}
		&\text{Var}\Big(\big(\bm{\epsilon}_{2}^{T}(A-A^{T})\bm{\epsilon}_{1}\big)^{2}\Big)=\text{Var}(U^{2})\Big(\text{Var}(U^{2})+2\big(\text{Var}(U)+\mathbb{E}_{U}[U]^{2}\big)^{2}\Big)\sum_{a,b}(\Delta A)_{ab}^{4}\\
		&\quad+2\big(\text{Var}(U)+\mathbb{E}_{U}[U]^{2}\big)^{2}\Big(3\text{Var}(U^{2})+2\big(\text{Var}(U)+\mathbb{E}_{U}[U]^{2}\big)^{2}\Big)\sum_{a}\sum_{b\ne d}(\Delta A)_{ab}^{2}(\Delta A)_{ad}^{2}\\
		&\quad+2\big(\text{Var}(U)+\mathbb{E}_{U}[U]^{2}\big)^{4}\Big(\sum_{a\neq c}\sum_{b\neq d}(\Delta A)_{ab}^{2}(\Delta A)_{cd}^{2}\\
		&\quad\quad\quad\quad\quad\quad\quad\quad\quad\quad\quad\quad+3\sum_{a\ne c}\sum_{b\ne d}(\Delta A)_{ab}(\Delta A)_{ad}(\Delta A)_{cb}(\Delta A)_{cd}\Big),
		\end{align*}
		where $(\Delta A)_{ab}:=A_{ab}-A_{ba}$.
	\end{proposition}
	\endgroup
	\begin{proof}[\textbf{Proof of Proposition \ref{prop:3}}]
		\begin{align*}
		\mathbb{E}_{\bm{\epsilon}_{1},\bm{\epsilon}_{2}}\left[\left(\bm{\epsilon}_{2}^{T}(A-A^{T})\bm{\epsilon}_{1}\right)^{2}\right]=&\mathbb{E}_{\bm{\epsilon}_{1},\bm{\epsilon}_{2}}\Big[\big(\sum_{i,j}{\epsilon}_{1,i}{\epsilon}_{2,j}(A_{ij}-A_{ji})^{2}\big)^{2}\Big]\\
		=&\mathbb{E}_{\bm{\epsilon}_{1},\bm{\epsilon}_{2}}\Big[\sum_{i,j,r,s}{\epsilon}_{1,i}{\epsilon}_{2,j}{\epsilon}_{1,r}{\epsilon}_{2,s}(A_{ij}-A_{ji})(A_{rs}-A_{sr})\Big]\\
		=&\mathbb{E}_{\bm{\epsilon}_{1},\bm{\epsilon}_{2}}\Big[\sum_{i,j}{\epsilon}_{1,i}^{2}{\epsilon}_{2,j}^{2}(A_{ij}-A_{ji})^{2}\Big]\\
		=&\mathbb{E}_{U}[U^{2}]^{2}\sum_{i,j}(A_{ij}-A_{ji})^{2}\\
		=&\mathbb{E}_{U}[U^{2}]^{2}\Vert A-A^{T}\Vert_{F}^{2}.
		\end{align*}
		Also, if $B:=A-A^{T}$, then
		\begin{eqnarray*}
			\lefteqn{\mathbb{E}_{\bm{\epsilon}_{1},\bm{\epsilon}_{2}}\left[\left(\bm{\epsilon}_{2}^{T}(A-A^{T})\bm{\epsilon}_{1}\right)^{4}\right]}&\\
			&&=\mathbb{E}_{\bm{\epsilon}_{1},\bm{\epsilon}_{2}}\Big[\sum_{a,b,c,d,e,f,g,h}{\epsilon}_{1,a}{\epsilon}_{2,b}{\epsilon}_{1,c}{\epsilon}_{2,d}{\epsilon}_{1,e}{\epsilon}_{2,f}{\epsilon}_{1,g}{\epsilon}_{2,h}B_{ab}B_{cd}B_{ef}B_{gh}\Big]\\
			&&=\mathbb{E}_{\bm{\epsilon}_{1},\bm{\epsilon}_{2}}\Big[\sum_{b,d,f,h}{\epsilon}_{2,b}{\epsilon}_{2,d}{\epsilon}_{2,f}{\epsilon}_{2,h}\sum_{a,c}{\epsilon}_{1,a}{\epsilon}_{1,c}\sum_{e,g}{\epsilon}_{1,e}{\epsilon}_{1,g}B_{ab}B_{cd}B_{ef}B_{gh}\Big]\\
			&&=\mathbb{E}_{\bm{\epsilon}_{1},\bm{\epsilon}_{2}}\Big[\sum_{b,d,f,h}{\epsilon}_{2,b}{\epsilon}_{2,d}{\epsilon}_{2,f}{\epsilon}_{2,h}\sum_{a}{\epsilon}_{1,a}^{2}\sum_{e,g}{\epsilon}_{1,e}{\epsilon}_{1,g}B_{ab}B_{ad}B_{ef}B_{gh}\Big]\\
			&&\quad+\mathbb{E}_{\bm{\epsilon}_{1},\bm{\epsilon}_{2}}\Big[\sum_{b,d,f,h}{\epsilon}_{2,b}{\epsilon}_{2,d}{\epsilon}_{2,f}{\epsilon}_{2,h}\sum_{a\ne c}{\epsilon}_{1,a}{\epsilon}_{1,c}\sum_{e,g}{\epsilon}_{1,e}{\epsilon}_{1,g}B_{ab}B_{cd}B_{ef}B_{gh}\Big]\\
			&&=\mathbb{E}_{\bm{\epsilon}_{1},\bm{\epsilon}_{2}}\Big[\sum_{b,d,f,h}{\epsilon}_{2,b}{\epsilon}_{2,d}{\epsilon}_{2,f}{\epsilon}_{2,h}\sum_{a}{\epsilon}_{1,a}^{2}\sum_{e}{\epsilon}_{1,e}^{2}B_{ab}B_{ad}B_{ef}B_{eh}\Big]\\
			&&\quad+\mathbb{E}_{\bm{\epsilon}_{1},\bm{\epsilon}_{2}}\Big[\sum_{b,d,f,h}{\epsilon}_{2,b}{\epsilon}_{2,d}{\epsilon}_{2,f}{\epsilon}_{2,h}\sum_{a\ne c}{\epsilon}_{1,a}^{2}{\epsilon}_{1,c}^{2}B_{ab}B_{cd}(B_{af}B_{ch}+B_{cf}B_{ah})\Big]\\
			&&=\mathbb{E}_{\bm{\epsilon}_{1},\bm{\epsilon}_{2}}\Big[\sum_{b,d,f,h}{\epsilon}_{2,b}{\epsilon}_{2,d}{\epsilon}_{2,f}{\epsilon}_{2,h}\sum_{a}{\epsilon}_{1,a}^{4}B_{ab}B_{ad}B_{af}B_{ah}\Big]\\
			&&\quad+3\mathbb{E}_{\bm{\epsilon}_{1},\bm{\epsilon}_{2}}\Big[\sum_{b,d,f,h}{\epsilon}_{2,b}{\epsilon}_{2,d}{\epsilon}_{2,f}{\epsilon}_{2,h}\sum_{a\ne c}{\epsilon}_{1,a}^{2}{\epsilon}_{1,c}^{2}B_{ab}B_{af}B_{cd}B_{ch}\Big]\\
			&&=\mathbb{E}_{U}[U^{4}]\sum_{a}\mathbb{E}_{\bm{\epsilon}_{2}}\Big[\sum_{b,d,f,h}{\epsilon}_{2,b}{\epsilon}_{2,d}{\epsilon}_{2,f}{\epsilon}_{2,h}B_{ab}B_{ad}B_{af}B_{ah}\Big]\\
			&&\quad+3\mathbb{E}_{U}[U^{2}]^{2}\sum_{a\ne c}\mathbb{E}_{\bm{\epsilon}_{2}}\Big[\sum_{b,d,f,h}{\epsilon}_{2,b}{\epsilon}_{2,d}{\epsilon}_{2,f}{\epsilon}_{2,h}B_{ab}B_{af}B_{cd}B_{ch}\Big]\\
			&&=\mathbb{E}_{U}[U^{4}]\sum_{a}\mathbb{E}_{\bm{\epsilon}_{2}}\Big[\sum_{b}{\epsilon}_{2,b}^{4}B_{ab}^{4}+3\sum_{b\ne d}{\epsilon}_{2,b}^{2}{\epsilon}_{2,d}^{2}B_{ab}^{2}B_{ad}^{2}\Big]\\
			&&\quad+3\mathbb{E}_{U}[U^{2}]^{2}\sum_{a\ne c}\mathbb{E}_{\bm{\epsilon}_{2}}\Big[\sum_{b}{\epsilon}_{2,b}^{4}B_{ab}^{2}B_{cb}^{2}+\sum_{b\ne d}{\epsilon}_{2,b}^{2}{\epsilon}_{2,d}^{2}(B_{ab}^{2}B_{cd}^{2}+2B_{ab}B_{ad}B_{cb}B_{cd})\Big]\\
			&&=\mathbb{E}_{U}[U^{4}]^{2}\sum_{a,b}B_{ab}^{4}+3\mathbb{E}_{U}[U^{2}]^{2}\mathbb{E}_{U}[U^{4}]\Big[\Big(\sum_{a}\sum_{b\ne d}B_{ab}^{2}B_{ad}^{2}+\sum_{b}\sum_{a\ne c}B_{ab}^{2}B_{cb}^{2}\Big)\Big]\\
			&&\quad+3\mathbb{E}_{U}[U^{2}]^{4}\sum_{a\ne c}\sum_{b\ne d}\Big(B_{ab}^{2}B_{cd}^{2}+2B_{ab}B_{ad}B_{cb}B_{cd}\Big)\\
			&&=\mathbb{E}_{U}[U^{4}]^{2}\sum_{a,b}B_{ab}^{4}+6\mathbb{E}_{U}[U^{2}]^{2}\mathbb{E}_{U}[U^{4}]\Big[\sum_{a}\sum_{b\ne d}B_{ab}^{2}B_{ad}^{2}\Big]\\
			&&\quad+3\mathbb{E}_{U}[U^{2}]^{4}\sum_{a\ne c}\sum_{b\ne d}\Big(B_{ab}^{2}B_{cd}^{2}+2B_{ab}B_{ad}B_{cb}B_{cd}\Big)
		\end{eqnarray*}
		Also,
		\begin{eqnarray*}
			\lefteqn{\mathbb{E}_{\bm{\epsilon}_{1},\bm{\epsilon}_{2}}\left[\left(\bm{\epsilon}_{2}^{T}(A-A^{T})\bm{\epsilon}_{1}\right)^{2}\right]^{2}=\Big(\mathbb{E}_{U}[U^{2}]^{2}\sum_{i,j}B_{ij}^{2}\Big)^{2}}&\\
			&&=\mathbb{E}_{U}[U^{2}]^{4}\sum_{i,j,r,s}B_{ij}^{2}B_{rs}^{2}\\
			&&=\mathbb{E}_{U}[U^{2}]^{4}\Big(\sum_{i,j}\sum_{s}B_{ij}^{2}B_{is}^{2}+\sum_{j,s}\sum_{i\ne r}B_{ij}^{2}B_{rs}^{2}\Big)\\
			&&=\mathbb{E}_{U}[U^{2}]^{4}\Big(\sum_{i,j}\sum_{s}B_{ij}^{2}B_{is}^{2}+\sum_{j}\sum_{i\ne r}B_{ij}^{2}B_{rj}^{2}+\sum_{i\ne r}\sum_{j\ne s}B_{ij}^{2}B_{rs}^{2}\Big)\\
			&&=\mathbb{E}_{U}[U^{2}]^{4}\Big(\sum_{i,j}B_{ij}^{4}+\sum_{i}\sum_{j\ne s}B_{ij}^{2}B_{is}^{2}+\sum_{j}\sum_{i\ne r}B_{ij}^{2}B_{rj}^{2}+\sum_{i\ne r}\sum_{j\ne s}B_{ij}^{2}B_{rs}^{2}\Big)\\
		\end{eqnarray*}
		
		Therefore,
		\begin{eqnarray*}
			\lefteqn{\text{Var}\left(\left(\bm{\epsilon}_{2}^{T}(A-A^{T})\bm{\epsilon}_{1}\right)^{2}\right)=\mathbb{E}_{\bm{\epsilon}_{1},\bm{\epsilon}_{2}}\left[\left(\bm{\epsilon}_{2}^{T}(A-A^{T})\bm{\epsilon}_{1}\right)^{4}\right]-\mathbb{E}_{\bm{\epsilon}_{1},\bm{\epsilon}_{2}}\left[\left(\bm{\epsilon}_{2}^{T}(A-A^{T})\bm{\epsilon}_{1}\right)^{2}\right]^{2}}&\\
			&&=\Big(\mathbb{E}_{U}[U^{4}]^{2}-\mathbb{E}_{U}[U^{2}]^{4}\Big)\sum_{a,b}B_{ab}^{4}+2\mathbb{E}_{U}[U^{2}]^{2}\Big(3\mathbb{E}_{U}[U^{4}]-\mathbb{E}_{U}[U^{2}]^{2}\Big)\sum_{a}\sum_{b\ne d}B_{ab}^{2}B_{ad}^{2}\\
			&&\quad+2\mathbb{E}_{U}[U^{2}]^{4}\Big(\sum_{a\neq c}\sum_{b\neq d}B_{ab}^{2}B_{cd}^{2}+3\sum_{a\ne c}\sum_{b\ne d}B_{ab}B_{ad}B_{cb}B_{cd}\Big)\\
			&&=\text{Var}(U^{2})\Big(\text{Var}(U^{2})+2\big(\text{Var}(U)+\mathbb{E}_{U}[U]^{2}\big)^{2}\Big)\sum_{a,b}B_{ab}^{4}\\
			&&\quad+2\big(\text{Var}(U)+\mathbb{E}_{U}[U]^{2}\big)^{2}\Big(3\text{Var}(U^{2})+2\big(\text{Var}(U)+\mathbb{E}_{U}[U]^{2}\big)^{2}\Big)\sum_{a}\sum_{b\ne d}B_{ab}^{2}B_{ad}^{2}\\
			&&\quad+2\big(\text{Var}(U)+\mathbb{E}_{U}[U]^{2}\big)^{4}\Big(\sum_{a\neq c}\sum_{b\neq d}B_{ab}^{2}B_{cd}^{2}+3\sum_{a\ne c}\sum_{b\ne d}B_{ab}B_{ad}B_{cb}B_{cd}\Big)
		\end{eqnarray*}
	\end{proof}
	
	\begingroup
	\renewcommand\theproposition{5}
	\begin{proposition}
		Let $U$ be the discrete random variable which takes the values $1,-1$ each with probability $1/2$. Then $(\bm{\epsilon}_{2}^{T}(A-A^{T})\bm{\epsilon}_{1})^{2}$ is the unbiased estimator of $\Vert A-A^{T}\Vert_{F}^{2}$.
		Moreover, $U$ is the unique random variable amongst zero-mean random variables for which the estimator is an unbiased estimator, and attains a minimum variance.
	\end{proposition}
	\endgroup
	\begin{proof}[\textbf{Proof of Proposition \ref{prop:4}}]
		A random variable $U^{2}$ has strictly positive variance if $U^{2}$ attains more than two values on a nonzero measure. To make $\text{Var}(U^{2})=0$, the random variable should be a discrete variable which takes the values 1, -1 each with probability 1/2.
	\end{proof}
	
	\begingroup
	\renewcommand\thetheorem{4}
	\begin{theorem}[\citet{de2021diffusion} and \citet{guth2022wavelet}]
		Assume that there exists $M\ge 0$ such that for any $t\in[0,T]$ and $\mathbf{z}\in\mathbb{R}^{d}$, the score estimation is close enough to the forward score by $M$, $\Vert\mathbf{s}_{\bm{\theta}}(\mathbf{x},t)-\nabla\log{p_{t}^{\bm{\phi}}(\mathbf{x})}\Vert\le M$, with $\mathbf{s}_{\bm{\theta}}\in C([0,T]\times\mathbb{R}^{d},\mathbb{R}^{d})$. Assume that $\nabla\log{p_{t}^{\bm{\phi}}(\mathbf{z})}$ is $C^{2}$ in both $t$ and $\mathbf{z}$, and that $\sup_{\mathbf{z},t}\Vert\nabla^{2}\log{p_{t}^{\bm{\phi}}(\mathbf{z})}\Vert\le K\quad\text{and}\quad\Vert\frac{\partial}{\partial t}\nabla\log{p_{t}^{\bm{\phi}}(\mathbf{z})}\Vert\le M e^{-\alpha t}\Vert\mathbf{z}\Vert$ for some $K,M,\alpha>0$. Suppose $(\mathbf{h}_{\bm{\phi}}^{-1})_{\#}$ s a push-forward map. Then $\Vert p_{r}-(\mathbf{h}_{\bm{\phi}}^{-1})_{\#}p_{0,N}^{\bm{\theta}}\Vert_{TV}\le E_{pri}(\bm{\phi})+E_{dis}(\bm{\phi})+E_{est}(\bm{\phi},\bm{\theta})$, where $E_{pri}(\bm{\phi})=\sqrt{2}e^{-T}D_{KL}(p_{T}^{\bm{\phi}}\Vert\pi)^{1/2}$ is the error originating from the prior mismatch; $E_{dis}(\bm{\phi})=6\sqrt{\delta}(1+\mathbb{E}_{p_{0}^{\bm{\phi}}(\mathbf{z})}[\Vert\mathbf{z}\Vert^{4}]^{1/4})(1+K+M(1+\frac{1}{\sqrt{2\alpha}}))$ is the discretization error with $\delta=\frac{\max{\gamma_{k}}^{2}}{\min{\gamma_{k}}}$; $E_{est}(\bm{\phi},\bm{\theta})=2TM^{2}$ is the score estimation error.
	\end{theorem}
	\endgroup
	\begin{remark}
		Although the proof is based on the standard form of the Ornstein-Uhlenbeck process, the direct extension of the theorem holds for generic VPSDE if there exists $\bar{\beta}>0$ such that $\frac{1}{\bar{\beta}}\le\beta(t)\le\bar{\beta}$. See \citet{de2022convergence}.
	\end{remark}
	\begin{lemma}[Lemma S11 of \citet{de2021diffusion}]\label{lemma:S11}
		Let $(\mathsf{E},\mathcal{E})$ and $(\mathsf{F},\mathcal{F})$ be two measurable spaces and $K:\mathsf{E}\times\mathcal{F}\rightarrow [0,1]$ be a Markov kernel. Then for any $\mu_{0},\mu_{1}\in\mathcal{P}(\mathsf{E})$ we have
		\begin{align*}
		\Vert\mu_{0}K-\mu_{1}K\Vert_{TV}\le\Vert\mu_{0}-\mu_{1}\Vert_{TV}.
		\end{align*}
		In addition, for any $\varphi:\mathsf{E}\rightarrow\mathsf{F}$ measurable we get that
		\begin{align*}
		\Vert\varphi_{\#}\mu_{0}-\varphi_{\#}\mu_{1}\Vert_{TV}\le\Vert\mu_{0}-\mu_{1}\Vert_{TV},
		\end{align*}
		with equality if $\varphi$ is injective.
	\end{lemma}
	
	\begin{proof}[\textbf{Proof of Theorem \ref{thm:dsb}}]
		For any $k\in\{1,...,N\}$, denote $R_{k}$ the Markov kernel such that for any $\mathbf{z}\in\mathbb{R}^{d}$, $\mathsf{A}\in\mathcal{B}(\mathbb{R}^{d})$ and $k\in\{0,...,N-1\}$ we have
		\begin{align*}
		R_{k+1}^{\bm{\theta}}(\mathbf{z},\mathsf{A})=(4\pi\gamma_{k+1})^{-d/2}\int_{\mathsf{A}}\exp{\bigg[-\frac{\Vert\mathbf{\tilde{z}}-\mathcal{T}_{k+1}^{\bm{\theta}}(\mathbf{z})\Vert^{2}}{4\gamma_{k+1}}\bigg]}\diff\mathbf{\tilde{z}},
		\end{align*}
		where for any $\mathbf{z}\in\mathbb{R}^{d}$, $\mathcal{T}_{k+1}^{\bm{\theta}}(\mathbf{z})=\mathbf{z}+\gamma_{k+1}\{\mathbf{z}+2\mathbf{s}_{\bm{\theta}}(\mathbf{z},t_{k})\}$, where $t_{k}=\sum_{l}^{k-1}\gamma_{l}$. Define $Q_{N}^{\bm{\theta}}=\prod_{l=1}^{N}R_{l}^{\bm{\theta}}$. Analogously, let us define
		\begin{align*}
		R_{k+1}^{\bm{\phi}}(\mathbf{z},\mathsf{A})=(4\pi\gamma_{k+1})^{-d/2}\int_{\mathsf{A}}\exp{\bigg[-\frac{\Vert\mathbf{\tilde{z}}-\mathcal{T}_{k+1}^{\bm{\phi}}(\mathbf{z})\Vert^{2}}{4\gamma_{k+1}}\bigg]}\diff\mathbf{\tilde{z}},
		\end{align*}
		for $\mathcal{T}_{k+1}^{\bm{\phi}}(\mathbf{z})=\mathbf{z}+\gamma_{k+1}\{\mathbf{z}+2\nabla\log{p_{t}^{\bm{\phi}}(\mathbf{z},t_{k})}\}$ and $Q_{N}^{\bm{\phi}}=\prod_{l=1}^{N}R_{l}^{\bm{\phi}}$.
		
		Suppose $\mathbb{P}_{T\vert 0}$ is the transition kernel from time zero to $T$ and $\mathbb{P}^{R}$ is the reverse-time measure, i.e., for any $\mathsf{A}\in\mathcal{B}(\mathcal{C})$ we have $\mathbb{P}^{R}(\mathsf{A})=\mathbb{P}(\mathsf{A}^{R})$ with $\mathsf{A}^{R}=\{t\mapsto \omega(T-t):\omega\in \mathsf{A}\}$. Then,
		\begin{align}\label{appendix:127}
		p_{0}^{\bm{\phi}}\mathbb{P}_{T\vert 0}{\mathbb{P}^{R}}_{T\vert 0}(\mathsf{A})=\mathbb{P}_{T}{\mathbb{P}^{R}}_{T\vert 0}(\mathsf{A})={\mathbb{P}^{R}}_{0}{\mathbb{P}^{R}}_{T\vert 0}(\mathsf{A})={\mathbb{P}^{R}}_{T}(\mathsf{A})=p_{0}^{\bm{\phi}}(\mathsf{A}).
		\end{align}
		Combining Eq. \eqref{appendix:127} with Lemma \ref{lemma:S11}, we have
		\begin{align*}
		&\Vert p_{0}^{\bm{\phi}}-p_{0,N}^{\bm{\theta}}\Vert_{TV}=\Vert p_{0}^{\bm{\phi}}\mathbb{P}_{T\vert 0}{\mathbb{P}^{R}}_{T\vert 0}-\pi Q_{N}^{\bm{\theta}}\Vert_{TV}\\
		&\quad\quad\quad\quad\le \Vert p_{0}^{\bm{\phi}}\mathbb{P}_{T\vert 0}{\mathbb{P}^{R}}_{T\vert 0}-\pi {\mathbb{P}^{R}}_{T\vert 0}\Vert_{TV}+\Vert \pi {\mathbb{P}^{R}}_{T\vert 0}-\pi Q_{N}^{\bm{\phi}}\Vert_{TV}+\Vert \pi Q_{N}^{\bm{\phi}}-\pi Q_{N}^{\bm{\theta}}\Vert_{TV}\\
		&\quad\quad\quad\quad\le \underbrace{\Vert p_{0}^{\bm{\phi}}\mathbb{P}_{T\vert 0}-\pi\Vert_{TV}}_{E_{pri}}+\underbrace{\Vert\pi{\mathbb{P}^{R}}_{T\vert 0}-\pi Q_{N}^{\bm{\phi}}\Vert_{TV}}_{E_{dis}}+\underbrace{\Vert\pi Q_{N}^{\bm{\phi}}-\pi Q_{N}^{\bm{\theta}}\Vert_{TV}}_{E_{est}}.
		\end{align*}
		The first two terms, $E_{pri}(\bm{\phi})+E_{dis}(\bm{\phi})$, are those terms derived in Theorem 2 of \citet{guth2022wavelet}. By Lemma S13 of \citet{de2021diffusion}, the last term, $E_{est}(\bm{\phi},\bm{\theta})$, is bounded by
		\begin{align*}
		\Vert\pi Q_{N}^{\bm{\phi}}-\pi Q_{N}^{\bm{\theta}}\Vert_{TV}^{2}\le\frac{1}{2}\int_{0}^{T}\mathbb{E}\big[\Vert b_{\bm{\phi}}(\{\mathbf{z}_{t}\}_{t=0}^{T},t)-b_{\bm{\theta}}(\{\mathbf{z}_{t}\}_{t=0}^{T},t)\Vert^{2}\big]\diff t,
		\end{align*}
		where $b_{\bm{\phi}}(\{\mathbf{z}_{t}\}_{t=0}^{T},t)=\sum_{k=0}^{N-1}1_{[t_{k},t_{k+1})}(t)\{\mathbf{z}_{t_{k}}+2\log{p_{t}^{\bm{\phi}}(\mathbf{z}_{t_{k}})}\}$ and $b_{\bm{\theta}}(\{\mathbf{z}_{t}\}_{t=0}^{T},t)=\sum_{k=0}^{N-1}1_{[t_{k},t_{k+1})}(t)\{\mathbf{z}_{t_{k}}+2\mathbf{s}_{\bm{\theta}}(\mathbf{z}_{t_{k}},t_{k})\}$ are the drift terms of piecewise generative processes, given by
		\begin{align*}
		\diff\mathbf{z}_{t}=\Big[-\mathbf{z}_{t}-2\nabla\log{p_{t_{k}}^{\bm{\phi}}(\mathbf{z}_{t_{k}})}\Big]\diff\bar{t}+g(t)\diff\mathbf{\bar{w}}_{t}
		\end{align*}
		and
		\begin{align*}
		\diff\mathbf{z}_{t}=\Big[-\mathbf{z}_{t}-2\mathbf{s}_{\bm{\theta}}(\mathbf{z}_{t_{k}},t_{k})\Big]\diff\bar{t}+g(t)\diff\mathbf{\bar{w}}_{t}
		\end{align*}
		defined each of the interval $[t_{k},t_{k+1}]$ for $k=0,...,N-1$, respectively. Therefore, $E_{est}(\bm{\phi},\bm{\theta})$ is bounded by
		\begin{align*}
		E_{est}(\bm{\phi},\bm{\theta})&\le \frac{1}{2}\int_{0}^{T}\mathbb{E}\big[\Vert b_{\bm{\phi}}(\{\mathbf{z}_{t}\}_{t=0}^{T},t)-b_{\bm{\theta}}(\{\mathbf{z}_{t}\}_{t=0}^{T},t)\Vert^{2}\big]\diff t\\
		&= 2\sum_{k=0}^{N-1}\int_{t_{k}}^{t_{k+1}}\mathbb{E}\big[\Vert\nabla\log{p_{t_{k}}^{\bm{\phi}}(\mathbf{z}_{t_{k}})}-\mathbf{s}_{\bm{\theta}}(\mathbf{z}_{t_{k}},t_{k})\Vert^{2}\big]\diff t\\
		&\le 2TM^{2}.
		\end{align*}
		
		Now, from Lemma \ref{lemma:S11} and the invertibility of the flow transformation, we have
		\begin{align*}
		\Vert p_{r}-(\mathbf{h}_{\bm{\phi}}^{-1})_{\#}\circ p_{0,N}^{\bm{\theta}}\Vert_{TV}=\Vert (\mathbf{h}_{\bm{\phi}})_{\#}\circ p_{r}-p_{0,N}^{\bm{\theta}}\Vert_{TV}=\Vert p_{0}^{\bm{\phi}}-p_{0,N}^{\bm{\theta}}\Vert_{TV},
		\end{align*}
		which completes the proof.
	\end{proof}
	
	\begin{table*}[t]
		\caption{Performance comparison to linear/nonlinear diffusion models on CIFAR-10. We report both before/after correction of density estimation performances. We report the baseline performances of linear diffusions by training our PyTorch implementation based on \citet{song2020score, song2021maximum} with identical hyperparameters and networks on both linear/nonlinear diffusions in order to quantify the effect of nonlinearity in a fair setting. Boldface numbers represent the best performance in a column, and underlined numbers represent the second best.}
		\label{tab:performance_cifar10_full}
		\tiny
		\centering
		\begin{adjustbox}{max width=\textwidth}
			\begin{tabular}{c|lc@{\hskip 0.3cm}r|cccccccc}
				\toprule
				\multirow{3}{*}{SDE} & \multirow{3}{*}{Model} & \multirow{3}{*}{\shortstack{Nonlinear Data\\Diffusion}} & \multirow{3}{*}{$\#$ Params} & \multicolumn{2}{c|}{NLL ($\downarrow$)} & \multicolumn{2}{c|}{NELBO ($\downarrow$)} & \multicolumn{2}{c|}{Gap ($\downarrow$)} & \multicolumn{2}{c}{FID ($\downarrow$)} \\
				&&&& \multirow{2}{*}{\shortstack{after\\correction}} & \multicolumn{1}{c|}{\multirow{2}{*}{\shortstack{before\\correction}}} &\multicolumn{1}{c}{w/ residual} & \multicolumn{1}{c|}{w/o residual} &\multicolumn{2}{c|}{(=NELBO-NLL)}& \multirow{2}{*}{ODE} & \multirow{2}{*}{PC} \\
				&&&& & \multicolumn{1}{c|}{} & \multicolumn{1}{c}{(after)} & \multicolumn{1}{c|}{(before)} & after & \multicolumn{1}{c|}{before} & &\\\midrule
				\multirow{4}{*}[-2pt]{VE} & NCSN++ (FID) & \xmark & 63M & 4.86 & 3.66 & 4.89 & 4.45 & 0.03 & 0.79 & - & 2.38 \\
				& \cc{15}INDM (FID) & \cc{15}\cmark & \cc{15}76M & \cc{15}3.22 & \cc{15}3.13 & \cc{15}3.28 & \cc{15}3.24 & \cc{15}0.06 & \cc{15}0.11 & \cc{15}- & \cc{15}2.29 \\\cmidrule(lr){2-12}
				& NCSN++ (deep, FID) & \xmark & 108M & 4.85 & 3.45 & 4.86 & 4.43 & 0.01 & 0.98 & - & \textbf{2.20} \\
				& \cc{15}INDM (deep, FID) & \cc{15}\cmark & \cc{15}118M & \cc{15}3.13 & \cc{15}3.03 & \cc{15}3.14 & \cc{15}3.10 & \cc{15}0.01 & \cc{15}0.07 & \cc{15}- & \cc{15}\underline{2.28} \\\midrule
				\multirow{9}{*}[-7pt]{VP} & DDPM++ (FID) & \xmark & 62M & 3.21 & 3.16 & 3.34 & 3.32 & 0.13 & 0.16 & 3.90 & 2.89 \\
				& \cc{15}INDM (FID) & \cc{15}\cmark & \cc{15}75M & \cc{15}3.17 & \cc{15}3.11 & \cc{15}3.23 & \cc{15}3.18 & \cc{15}0.06 & \cc{15}0.07 & \cc{15}\textbf{3.61} & \cc{15}2.90 \\\cmidrule(lr){2-12}
				& DDPM++ (deep, FID) & \xmark & 108M & 3.19 & 3.13 & 3.32 & 3.29 & 0.13 & 0.16 & 3.69 & 2.64 \\
				& \cc{15}INDM (deep, FID) & \cc{15}\cmark & \cc{15}121M & \cc{15}3.09 & \cc{15}3.02 & \cc{15}3.13 & \cc{15}3.08 & \cc{15}0.04 & \cc{15}0.06 & \cc{15}\underline{3.67} & \cc{15}3.15 \\\cmidrule(lr){2-12}
				& DDPM++ (NLL) & \xmark & 62M & 3.03 & 2.97 & 3.13 & 3.11 & 0.10 & 0.14 & 6.70 & 5.17 \\
				& \cc{15}INDM (NLL) & \cc{15}\cmark & \cc{15}75M & \cc{15}\underline{2.98} & \cc{15}\underline{2.95} & \cc{15}\underline{2.98} & \cc{15}\underline{2.97} & \cc{15}\textbf{0.00} & \cc{15}\textbf{0.02} & \cc{15}6.01 & \cc{15}5.30 \\
				& \cc{15}INDM (NLL, ST) & \cc{15}\cmark & \cc{15}75M & \cc{15}3.01 & \cc{15}2.98 & \cc{15}3.02 & \cc{15}3.01 & \cc{15}0.01 & \cc{15}0.03 & \cc{15}3.88 & \cc{15}3.25 \\\cmidrule(lr){2-12}
				& DDPM++ (deep, NLL) & \xmark & 108M & 3.01 & 2.95 & 3.11 & 3.09 & 0.10 & 0.14 & 6.43 & 4.88 \\
				& \cc{15}INDM (deep, NLL) & \cc{15}\cmark & \cc{15}121M & \cc{15}\textbf{2.97} & \cc{15}\textbf{2.94} & \cc{15}\textbf{2.97} & \cc{15}\textbf{2.96} & \cc{15}\textbf{0.00} & \cc{15}\textbf{0.02} & \cc{15}5.71 & \cc{15}4.79 \\
				\bottomrule
			\end{tabular}
		\end{adjustbox}
	\end{table*}
	
	\begin{table*}[t]
		\caption{Performance comparison on CIFAR-10.}
		\label{tab:performance_cifar10_appendix}
		\tiny
		\centering
		\begin{adjustbox}{max width=\textwidth}
			\begin{tabular}{cccl|cccccccc}
				\toprule
				\multirow{3}{*}{Class} & \multirow{3}{*}{SDE} & \multirow{3}{*}{Type} & \multirow{3}{*}{Model} & \multicolumn{2}{c|}{NLL ($\downarrow$)} & \multicolumn{2}{c|}{NELBO ($\downarrow$)} & \multicolumn{2}{c|}{Gap ($\downarrow$)} & \multicolumn{2}{c}{FID ($\downarrow$)} \\
				&&&& \multirow{2}{*}{\shortstack{after\\correction}} & \multicolumn{1}{c|}{\multirow{2}{*}{\shortstack{before\\correction}}} &\multicolumn{1}{|c}{w/ residual} & \multicolumn{1}{c|}{w/o residual} &\multicolumn{2}{|c|}{(=NELBO-NLL)}& \multirow{2}{*}{ODE} & \multirow{2}{*}{PC} \\
				&&&& & & \multicolumn{1}{|c}{(after)} & \multicolumn{1}{c|}{(before)} & after & \multicolumn{1}{c|}{before} & &\\\midrule
				\multirow{4}{*}{GAN}  &  	&											& StyleGAN2 + ADA \cite{karras2020training}	 & - & - & - & - & - & - &  \multicolumn{2}{c}{2.92}  \\	 &  	&											& StyleFormer \cite{park2022styleformer}	 & - & - & - & - & - & - &  \multicolumn{2}{c}{2.82}  \\	
				&  	&											& SNGAN + DGflow \cite{ansari2020refining}	 & - & - & - & - & - & - &  \multicolumn{2}{c}{9.62} \\
				&  	&											& TransGAN \cite{jiang2021transgan}		 & - & - & - & - & - & - &  \multicolumn{2}{c}{9.26} \\ \midrule
				\multirow{3}{*}{Autoregressive}  &  	&									& PixcelCNN \cite{van2016pixel}	& 3.14 & - & - & - & - & - &  \multicolumn{2}{c}{65.9}  \\
				&  	&									& PixcelRNN \cite{van2016pixel} & 3.00 & - & - & - & - & - &  \multicolumn{2}{c}{-}  \\
				&  	&									& Sparse Transformer \cite{child2019generating}	& 2.80 & - & - & - & - & - &  \multicolumn{2}{c}{-}  \\ \midrule
				\multirow{6}{*}{Flow}  &  	&											& Glow \cite{kingma2018glow}	& 3.35 &  - & - & - & - & - &  \multicolumn{2}{c}{48.9}  \\
				&  	&											& Residual Flow \cite{chen2019residual}	& 3.28 & - & - & - & - & - &  \multicolumn{2}{c}{46.4}  \\
				&  	&											& Flow++ \cite{ho2019flow++} & 3.28 & - & - & - & - & - &  \multicolumn{2}{c}{46.4}  \\
				&  	&											& Wolf \cite{ma2020decoupling}	 & 3.27 & - & - & - & - & - &  \multicolumn{2}{c}{37.5}  \\
				&  	&											& VFlow \cite{chen2020vflow}	 & 2.98 & - & - & - & - & - &  \multicolumn{2}{c}{-}  \\
				&  	&											& DenseFlow-74-10 \cite{grcic2021densely}	 & 2.98 & - & - & - & - & - &  \multicolumn{2}{c}{34.9}  \\	 \midrule
				\multirow{5}{*}{VAE}  &  	&											& NVAE \cite{vahdat2020nvae}	 & - & - & 2.91  & - & - & - &  \multicolumn{2}{c}{23.5}  \\
				&  	&											& Very Deep VAE \cite{child2020very}	 & - & - & 2.87 & - & - & - &  \multicolumn{2}{c}{-}  \\
				&  	&											& $\delta$-VAE \cite{razavi2018preventing}	 & - & - & 2.83 & - & - & - &  \multicolumn{2}{c}{-}  \\
				&  	&											& DCVAE \cite{parmar2021dual}	 & - & - & - & - & - & - &  \multicolumn{2}{c}{17.9}  \\
				&  	&											& CR-NVAE \cite{sinha2021consistency} 	 & - & - & - & - & - & - &  \multicolumn{2}{c}{2.51}  \\	 \midrule
				\multirow{19}{*}[-1.5em]{Diffusion}  & \multirow{9}{*}{Linear} 	&						& DDPM \cite{ho2020denoising}	 & - & - & 3.75 & - & - & - &\multicolumn{2}{c}{3.17} \\	
				&						&											& NCSNv2 \cite{song2020improved}	& - & - & - & - & - & - &\multicolumn{2}{c}{10.87} \\	
				&						&											& DDIM \cite{song2020denoising}	 & - & - & - & - & - & - &\multicolumn{2}{c}{4.04} \\
				&						&											& IDDPM \cite{nichol2021improved}	 & 3.37 & - & - & - & - & - &\multicolumn{2}{c}{2.90} \\	
				&						&											& VDM \cite{kingma2021variational}	 & \textbf{2.65} & - & - & - & - & - &\multicolumn{2}{c}{7.41} \\	
				&						&											& NCSN++ (FID) \cite{song2020score}	 & 4.85 & 3.45 & 4.86 & 4.43 & 0.01 & 0.98 & - 		& \textbf{2.20} \\		
				& 						&											& DDPM++ (FID) \cite{song2020score}	 & 3.19 & 3.13 & 3.32 & 3.29 & 0.13 & 0.16 & 3.69 	& 2.64		 \\
				& 						&											& DDPM++ (NLL) \cite{song2021maximum}	 & 3.01 & 2.95 & 3.11 & 3.09 & 0.10 & 0.14 & 6.43 	& 4.88 		\\
				& 						&											& CLD-SGM \cite{dockhorn2021score}	 & - & - 	& - 	& 3.31 		& - & - & 2.25 	& - 		\\ \cmidrule(lr){2-12}
				& \multirow{10}{*}[-1.3em]{Nonlinear}	& SBP										& SB-FBSDE \cite{chen2021likelihood}	& - & 2.98 	& - 	& - 		& - & - & - 	& 3.18 		\\ \cmidrule(lr){3-12}
				& 						&\multirow{5}{*}{\shortstack[c]{VAE\\-based}}	& LSGM (FID) \cite{vahdat2021score} & - 	& - 	& 3.45 & 3.43	& - & - & \textbf{2.10} 	& -	\\
				& 						&											& LSGM (NLL)-269M 					 & - 	& - 	& - & 2.97	& - & - & 6.15 	&  -		\\
				&						&											& LSGM (NLL) 						 & - 	& - 	&\textbf{ 2.87} & \textbf{2.87}	& - & - & 6.89 	&  -		\\
				&						&											& LSGM (balanced)-109M					 & - 	& - 	& - & 2.96	& - & - & 4.60 	&  -		\\
				&						&											& LSGM (balanced)					 & - 	& - 	& 2.98 & 2.95	& - & - & 2.17 	&  -		\\ \cmidrule(lr){3-12}
				&						&\multirow{4}{*}[-0.45em]{\shortstack[c]{Flow\\-based}}	& DiffFlow (FID) \cite{zhang2021diffusion} & - 	& - 	& 3.04 & -	& - & - & - 	& 14.14 		\\ \cmidrule(lr){4-12}
				&						&											& \cc{15}INDM (FID) 					 &  \cc{15}3.13 & \cc{15}3.03 & \cc{15}3.14 & \cc{15}3.10 & \cc{15}0.01 & \cc{15}0.07 & \cc{15}- & \cc{15}2.28 		\\ 
				&						&											& \cc{15}INDM (NLL) 						 & \cc{15}2.97 & \cc{15}\textbf{2.94} & \cc{15}2.97 & \cc{15}2.96 & \cc{15}\textbf{0.00} & \cc{15}\textbf{0.02} & \cc{15}5.71 & \cc{15}4.79 \\ 
				&						&											& \cc{15}INDM (ST)					  & \cc{15}3.01 & \cc{15}2.98 & \cc{15}3.02 & \cc{15}3.01 & \cc{15}0.01 & \cc{15}0.03 & \cc{15}3.88 & \cc{15}3.25 \\
				\bottomrule
			\end{tabular}
		\end{adjustbox}
	\end{table*}

	\begin{table*}[t]
		\caption{Performance comparison on CelebA $64\times 64$.}
		\label{tab:performance_celeba_appendix}
		\scriptsize
		\centering
		\begin{tabular}{lc@{\hskip 0.3cm}c@{\hskip 0.3cm}c@{\hskip 0.2cm}c@{\hskip 0.3cm}c@{\hskip 0.3cm}c@{\hskip 0.3cm}c@{\hskip 0.3cm}c@{\hskip 0.3cm}c}
			\toprule
			\multirow{2}{*}{Model} & \multicolumn{2}{c}{NLL ($\downarrow$)} & \multicolumn{2}{c@{\hskip 0.5cm}}{NELBO ($\downarrow$)} & \multicolumn{2}{c@{\hskip 0.5cm}}{Gap ($\downarrow$)} & \multicolumn{2}{c@{\hskip 0.4cm}}{FID ($\downarrow$)} \\
			& after & before & w/ res- & w/o res- & after & before & ODE & PC \\\midrule
			
			UNCSN++ \cite{kim2022soft} & - & \textbf{1.93} & - & - & - & - & - & \textbf{1.92} \\
			DDGM \cite{nachmani2021non} & - & - & - & - & - & - & - & 2.92 \\
			Efficient-VDVAE \cite{hazami2022efficient} & - & \textbf{1.83} & - & - \\
			CR-NVAE \cite{sinha2021consistency} & - & - & 1.86 & - & - & - & - & - \\
			DenseFlow-74-10 \cite{grcic2021densely} & 1.99 & - & - & - & - & - & - & - \\
			StyleFormer \cite{park2022styleformer}  & - & - & - & - & - & - & \multicolumn{2}{c}{3.66}\\ \midrule			
			NCSN++ (VE) & 3.41 & 2.37 & 3.42 & 3.96 & 0.01 & 1.59 & - & 3.95 \\
			\cc{15}INDM (VE, FID) & \cc{15}2.31 & \cc{15}1.95 & \cc{15}2.33 & \cc{15}2.17 & \cc{15}0.02 & \cc{15}0.22 & \cc{15}- & \cc{15}2.54\\\midrule
			DDPM++ (VP, FID) & 2.14 & 2.07 & 2.21 & 2.22 & 0.06 & 0.14 & 2.32 & 3.03 \\
			\cc{15}INDM (VP, FID) & \cc{15}2.27 & \cc{15}2.13 & \cc{15}2.31 & \cc{15}2.20 & \cc{15}0.04 & \cc{15}0.07 & \cc{15}\textbf{1.75} & \cc{15}2.32 \\\cmidrule(lr){2-9}
			DDPM++ (VP, NLL) & 2.00 & 1.93 & 2.09 & 2.09 & 0.09 & 0.16 & 3.95 & 5.31 \\
			\cc{15}INDM (VP, NLL) & \cc{15}2.05 & \cc{15}1.97 & \cc{15}2.05 & \cc{15}2.00 & \cc{15}0.00 & \cc{15}0.03 & \cc{15}3.06 & \cc{15}5.14 \\
			\bottomrule
		\end{tabular}
	\end{table*}
	
	\begin{figure}
		\centering
		\includegraphics[width=\linewidth]{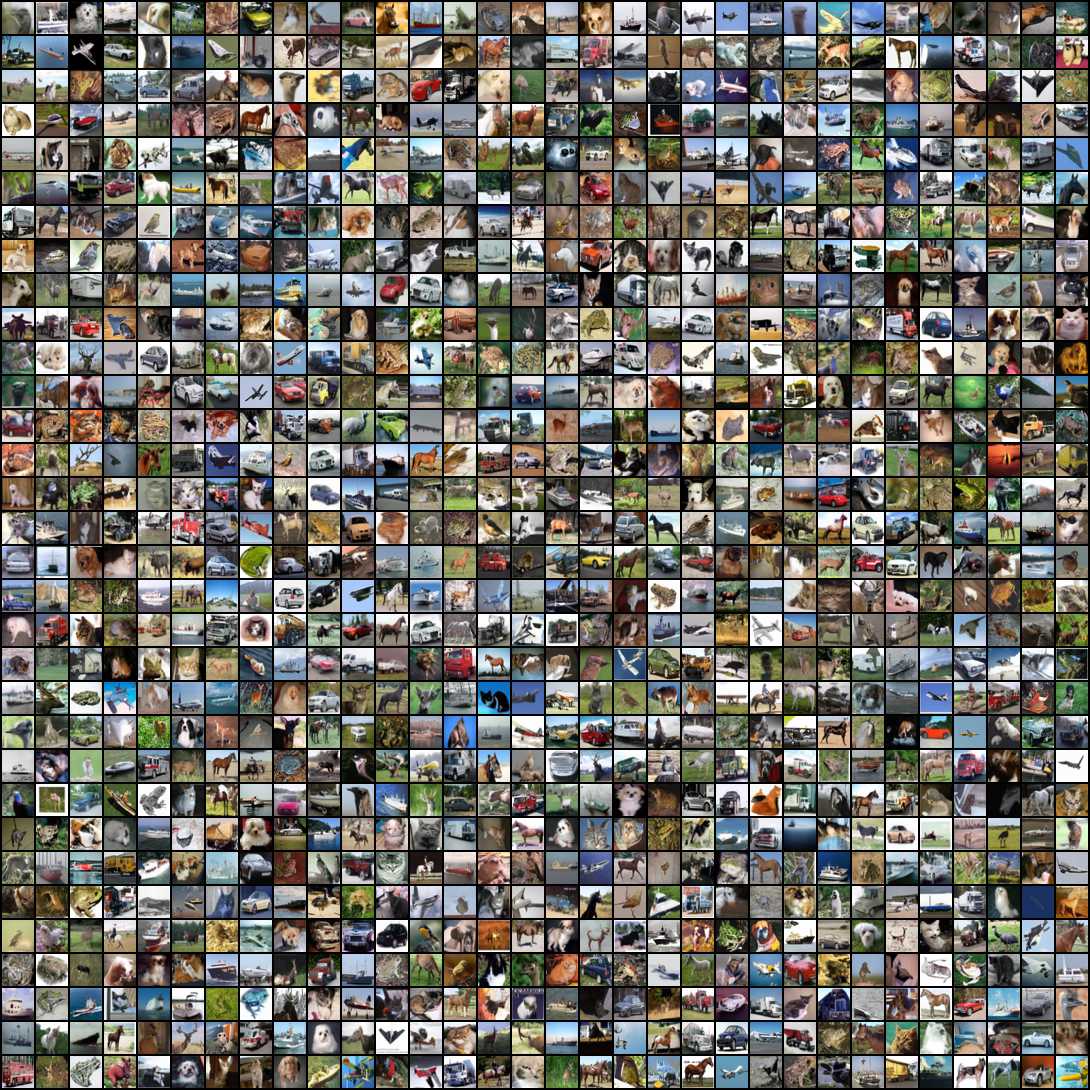}
		\caption{Non cherry-picked random samples from CIFAR-10 trained on INDM (VE, deep, FID).}
		\label{fig:CIFAR10_VE_FID_samples_tau_1.05}
	\end{figure}
	
	\begin{figure}
		\centering
		\includegraphics[width=\linewidth]{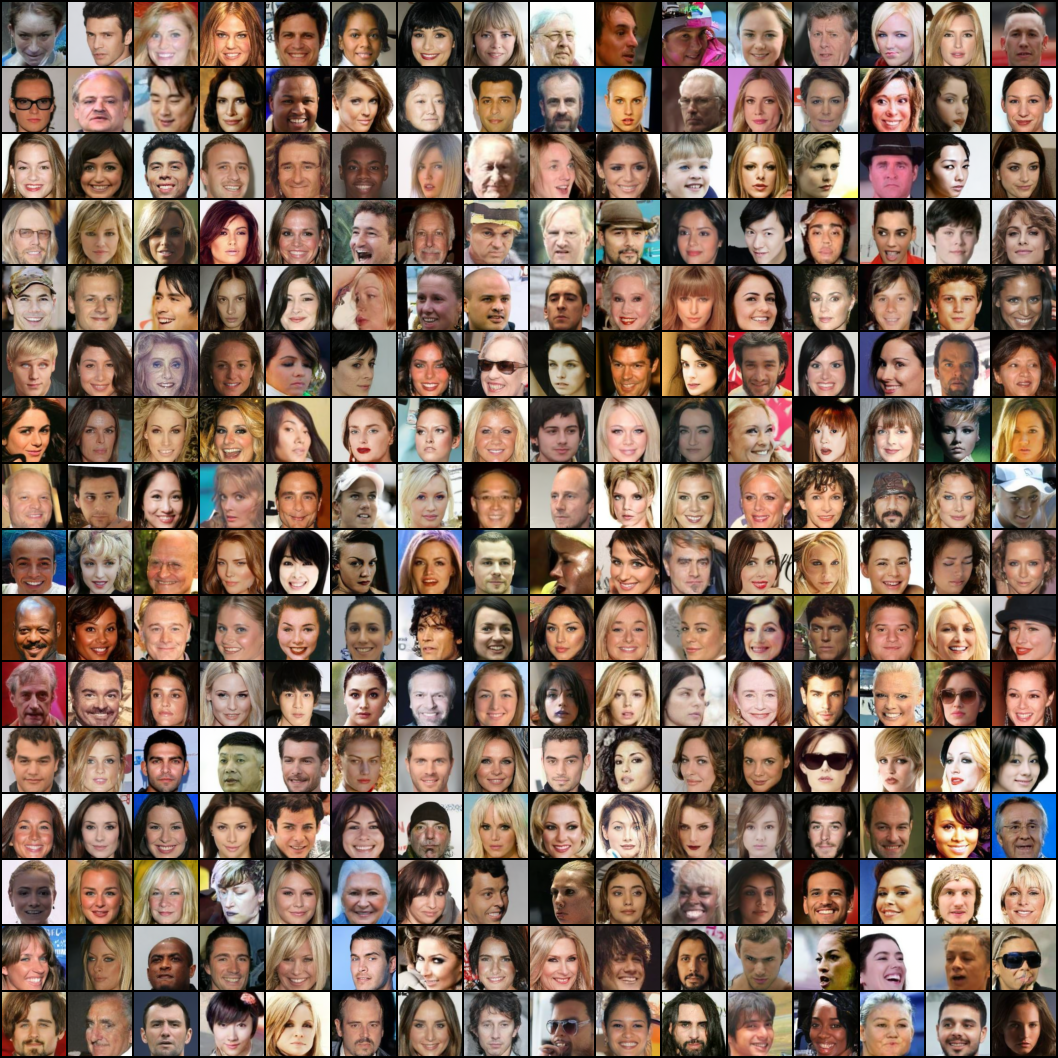}
		\caption{Non cherry-picked random samples fr om CelebA trained on INDM (VP, FID).}
		\label{fig:CelebA_VP_FID_samples_tau_1.11}
	\end{figure}
	

\end{document}